\documentclass[12pt, twoside, a4]{report}
\usepackage[x11names]{xcolor}
\usepackage{setspace} 
\usepackage[a4paper,width=150mm,top=25mm,bottom=25mm,bindingoffset=6mm]{geometry}
\usepackage{fancyhdr}
\usepackage[subfigure]{tocloft}
\usepackage{amsmath,amssymb,amsfonts,amsthm,stmaryrd}
\usepackage{algorithm, algorithmicx, algpseudocode}
\usepackage{adjustbox}
\usepackage{url}
\usepackage{multirow}
\usepackage{diagbox}
\usepackage[caption=false]{subfig}
\usepackage{graphicx}
\usepackage{tikz}
\usepackage[style=ieee]{biblatex}
\usepackage{enumitem}
\usepackage{booktabs}
\usepackage{rotating}
\usepackage[toc, nonumberlist, abbreviations, symbols]{glossaries-extra}
\usepackage[nottoc]{tocbibind}
\usepackage{caption}
\usepackage{hyperref}
\hypersetup{
    colorlinks,
    citecolor=blue,
    filecolor=blue,
    linkcolor=Blue4,
    urlcolor=blue
}
\usepackage[T1]{fontenc}

\makeatletter
\begingroup
  \let\newcounter\@gobble
  \let\setcounter\@gobbletwo
  \globaldefs\@ne
  \let\c@loadepth\@ne
  \newlistof{algorithms}{loa}{\listalgorithmname}
\endgroup
\let\l@algorithm\l@algorithms
\makeatother
\cftsetindents{algorithms}{1.5em}{4em}
\renewcommand{\listofalgorithms}{\begingroup
  \tocfile{\listalgorithmname}{loa}
\endgroup}

\newcounter{thmcounter}
\newtheorem{theorem}{Theorem}
\newtheorem{lemma}{Lemma}[thmcounter]

\newcommand{\textc}[1]{\mathtt{#1}}

\newcommand{\mnew}{\textrm{new}}

\newcommand{\mmin}{\mathrm{min}}
\newcommand{\mprog}{\mathrm{prog}}
\newcommand{\mnext}{\mathrm{next}}
\newcommand{\mprev}{\mathrm{prev}}
\newcommand{\mcol}{\mathrm{col}}
\newcommand{\mcrn}{\mathrm{crn}}
\newcommand{\mstart}{\mathrm{start}}
\newcommand{\mgoal}{\mathrm{goal}}
\newcommand{\mmnr}{\mathrm{mnr}}
\newcommand{\mmjr}{\mathrm{maj}}
\newcommand{\mthird}{\mathrm{thd}}
\newcommand{\mitx}{\mathrm{itx}}

\newcommand{\medge}{\varepsilon}
\newcommand{\mbest}{b}

\newcommand{\mtrace}{d}

\newcommand{\mlinks}{\mathbb{E}}
\newcommand{\mDelta}{\mathbf{\Delta}}

\newcommand{\mreal}{\mathbb{R}}
\newcommand{\mthres}{\mathbf{\varepsilon}}
\newcommand{\mnode}{n}

\newcommand{\mnodes}{\mathbb{N}}

\newcommand{\mlink}{l}
\newcommand{\mltype}{y_\mlink}

\newcommand{\mntype}{\eta}
\newcommand{\mtdir}{\kappa}

\newcommand{\mfalse}{\textc{False}}
\newcommand{\mtrue}{\textc{True}}
\newcommand{\mback}{\textc{Back}}
\newcommand{\mfront}{\textc{Front}}
\newcommand{\mnvy}{\textc{Vy}}
\newcommand{\mnvu}{\textc{Vu}}
\newcommand{\mnsy}{\textc{Sy}}
\newcommand{\mnsu}{\textc{Su}}
\newcommand{\mnty}{\textc{Ty}}
\newcommand{\mntu}{\textc{Tu}}
\newcommand{\mney}{\textc{Ey}}
\newcommand{\mneu}{\textc{Eu}}
\newcommand{\mnun}{\textc{Un}}
\newcommand{\mntm}{\textc{Tm}}
\newcommand{\mnoc}{\textc{Oc}}
\newcommand{\mnph}{\textc{Ph}}
\newcommand{\mray}{\lambda}
\newcommand{\mrtype}{\rho}

\newcommand{\mx}{\mathbf{x}}
\newcommand{\mxr}{\mx_\mtdir}

\newcommand{\mxs}{\mx_S}
\newcommand{\mxss}{\mx_{SS}}
\newcommand{\mxt}{\mx_T}
\newcommand{\mxtt}{\mx_{TT}}
\newcommand{\mxcol}{\mx_\mcol}

\newcommand{\mxstart}{\mx_\mathrm{start}}
\newcommand{\mxgoal}{\mx_\mathrm{goal}}

\newcommand{\mv}{\mathbf{v}}
\newcommand{\mvr}{\mv_\mtdir}
\newcommand{\mvrr}{\mv_{\mtdir\mtdir}}
\newcommand{\mvs}{\mv_S}
\newcommand{\mvss}{\mv_{SS}}
\newcommand{\mvt}{\mv_T}
\newcommand{\mvtt}{\mv_{TT}}
\newcommand{\mvray}{\mv_\mathrm{ray}}

\newcommand{\mvprog}{\mv_\mprog}
\newcommand{\mvprogs}{\mv_{\mprog,S}}
\newcommand{\mvprogt}{\mv_{\mprog,T}}

\newcommand{\mvnext}{\mv_\medge}
\newcommand{\mvprev}{\mv_\medge'}
\newcommand{\mvcrn}{\mv_\mcrn}
\newcommand{\mside}{\sigma}

\newcommand{\msidetrace}{\mside_\mtrace}

\newcommand{\mqtype}{y_\mquery}
\newcommand{\mqcast}{\textc{Cast}}
\newcommand{\mqtrace}{\textc{Trace}}

\newcommand{\mdot}{\boldsymbol{.}}

\DeclareMathOperator{\fsgn}{sgn}

\DeclareMathOperator{\ffloor}{floor}
\DeclareMathOperator{\fceil}{ceil}
\DeclareMathOperator{\fatantwo}{atan2}

\newcommand{\mora}[1]{\protect\overrightarrow{#1}}
\renewcommand{\llceil}{\left\lceil\kern-3.5pt\left\lceil} % redefine brackets in stmaryd
\renewcommand{\rrceil}{\right\rceil\kern-3.5pt\right\rceil}
\renewcommand{\llfloor}{\left\lfloor\kern-3.5pt\left\lfloor} % redefine brackets in stmaryd
\renewcommand{\rrfloor}{\right\rfloor\kern-3.5pt\right\rfloor}
\newcommand{\mafloor}[1]{\llfloor #1 \rrfloor}
\newcommand{\maceil}[1]{\llceil #1 \rrceil}

\newabbreviation{stree}{$S$-tree}{source-tree}
\newabbreviation{ttree}{$T$-tree}{target-tree}
\newabbreviation{los}{LOS}{line-of-sight}
\newabbreviation{ocgrid}{oc-grid}{occupancy grid}
\newabbreviation{dda}{DDA}{Digital Differential Analyzer}
\newabbreviation{ltrace}{$L$-trace}{left trace}
\newabbreviation{rtrace}{$R$-trace}{right trace}

\glsxtrnewsymbol[description={%
    Operator to access member property
    }]{mdot}{\ensuremath{\mdot}}
\glsxtrnewsymbol[description={%
    Constrains angle to $[-\pi, \pi)$ radians
    }]{mafloor}{\ensuremath{\mafloor{\cdot}}}
\glsxtrnewsymbol[description={%
    Constrains angle to $(-\pi, \pi]$ radians
    }]{maceil}{\ensuremath{\maceil{\cdot}}}
\glsxtrnewsymbol[description={%
    Coordinates
    }]{mx}{\ensuremath{\mx}}
\glsxtrnewsymbol[description={%
    Vector
    }]{mv}{\ensuremath{\mv}}

\newcommand{\rs}{RayScan}
\newcommand{\rsp}{RayScan+}
\newcommand{\rtwo}{R2}
\newcommand{\rtwop}{R2+}

\makeatletter
\renewcommand{\ALG@beginalgorithmic}{\footnotesize}
\makeatother
\algnewcommand{\IfThenElse}[3]{% \IfThenElse{<if>}{<then>}{<else>}
  \State \algorithmicif\ #1\ \algorithmicthen\ #2\ \algorithmicelse\ #3}
\algnewcommand{\IfThen}[2]{% \IfThen{<if>}{<then>}
  \State \algorithmicif\ #1\ \algorithmicthen\ #2}
\algnewcommand{\Break}{\textbf{break\ }}
\algnewcommand{\Continue}{\textbf{continue\ }}
\algnewcommand{\An}{\textbf{and\ }}
\algnewcommand{\Or}{\textbf{or\ }}
\algnewcommand{\Not}{\textbf{not\ }}
\algnewcommand\algorithmicforeach{\textbf{for each}}
\algdef{S}[FOR]{ForEach}[1]{\algorithmicforeach\ #1\ \algorithmicdo}
\algdef{SE}[DOWHILE]{DoWhile}{EndDoWhile}{\algorithmicdo}[1]{\algorithmicwhile\ #1}%
\renewcommand*\Call[2]{\textproc{#1}(#2)} %allow nested Call

\usetikzlibrary {shapes.geometric, shapes.misc, shapes.symbols, shapes.multipart}
\usetikzlibrary {arrows.meta, decorations.markings, decorations.pathreplacing}
\usetikzlibrary{calc, bending, backgrounds, fadings, angles, quotes, intersections}

\def\u{0.25cm}

\def\ul{0.5cm}
\def\uss{0.1cm}
\def\um{0.15cm}
\colorlet{swatch_blue} {yellow!10!cyan!80!blue}
\colorlet{swatch_obs} {white!30!lightgray}
\colorlet{swatch_bluegray}{blue!30!gray}
\colorlet{swatch_stree} {white!10!magenta!50!red}
\colorlet{swatch_ttree} {black!60!green}
\tikzset {
    blue pt/.style n args={2}{minimum size=2mm, inner sep=0, outer sep=0, circle, fill=swatch_blue, label={[swatch_blue, #1] #2}},
    blue pt/.default={}{},
    blue circ/.style n args={2}{minimum size=2mm, inner sep=0, outer sep=0, circle, draw=swatch_blue, fill=white, thick, label={[swatch_blue, #1] #2}},
    blue circ/.default={}{},
    black pt/.style n args={2}{minimum size=2mm, inner sep=0, outer sep=0, circle, fill=black, label={[black, #1] #2}},
    black pt/.default={}{},
    cross pt/.style n args={2}{minimum size=2mm, inner sep=0, outer sep=0, cross out, draw, label={[black, #1] #2}},
    cross pt/.default={}{},
    adhoc pt/.style n args={2}{minimum size=3mm, inner sep=0, circle, draw=violet, ultra thick, label={[violet, #1] #2}},
    adhoc pt/.default={}{},
    trace/.style n args={0}{line width=1mm, orange, {Triangle Cap[]}-{Triangle Cap[]}},
    trace2/.style n args={0}{line width=0.8mm, swatch_blue, -{Triangle Cap[]}},
    prune/.style n args={0}{cross out, draw, magenta!80!black, thick, minimum size=3mm, inner sep=0mm},
    bgframe/.style n args={0}{background rectangle/.style={draw=gray, dotted}, framed, tight background},
    vy pt/.style n args={3}{minimum size=2mm, inner sep=0, outer sep=0, circle, draw=black!60!#3, fill=#3, label={[#3, #1] #2}},
    vy pt/.default={}{}{swatch_blue},
    vu pt/.style n args={3}{minimum size=2mm, inner sep=0, outer sep=0, circle, fill=white, draw=#3, thick, label={[#3, #1] #2}},
    vu pt/.default={}{}{swatch_blue},
    ey pt/.style n args={3}{minimum size=2.5mm, inner sep=0, outer sep=0, diamond, draw=black!60!#3, fill=#3, label={[#3, #1] #2}},
    ey pt/.default={}{}{swatch_blue},
    eu pt/.style n args={3}{minimum size=2.5mm, inner sep=0, outer sep=0, diamond, draw=#3, fill=white, thick, label={[#3, #1] #2}},
    eu pt/.default={}{}{swatch_blue},
    tm pt/.style n args={3}{minimum size=2mm, forbidden sign, inner sep=0, outer sep=0, draw=#3, fill=white, thick, label={[#3, #1] #2}},%dash=on 2pt off 1pt phase 0pt
    tm pt/.default={}{}{swatch_ttree},
    trtm pt/.style n args={3}{tm pt={#1}{#2}{#3}},
    trtm pt/.default={}{}{},
    oc pt/.style n args={3}{minimum size=1.8mm, inner sep=0, rectangle, outer sep=0, draw=#3, fill=white, thick, label={[#3, #1] #2}},
    oc pt/.default={}{}{swatch_ttree},
    un cross/.style={path picture={ 
      \draw[#1] (path picture bounding box.south east) -- (path picture bounding box.north west) (path picture bounding box.south west) -- (path picture bounding box.north east);
    }},
    un pt/.style n args={3}{minimum size=2mm, inner sep=0, outer sep=0, circle, un cross={#3}, draw=#3, fill=white, thick, label={[#3, #1] #2}},
    un pt/.default={}{}{swatch_ttree},
    any pt/.style n args={3}{circle, inner sep=0, outer sep=0, minimum size=3mm, draw=#3, dash=on 1pt off 1pt phase 0, label={[#3, #1] #2}},
    any pt/.default={}{}{},
    svy pt/.style n args={2}{vy pt={#1}{#2}{swatch_stree}},
    svy pt/.default={}{},
    svu pt/.style n args={2}{vu pt={#1}{#2}{swatch_stree}},
    svu pt/.default={}{},
    sey pt/.style n args={2}{ey pt={#1}{#2}{swatch_stree}},
    sey pt/.default={}{},
    seu pt/.style n args={2}{eu pt={#1}{#2}{swatch_stree}},
    seu pt/.default={}{},
    tvy pt/.style n args={2}{vy pt={#1}{#2}{swatch_ttree}},
    tvy pt/.default={}{},
    tvu pt/.style n args={2}{vu pt={#1}{#2}{swatch_ttree}},
    tvu pt/.default={}{},
    tey pt/.style n args={2}{ey pt={#1}{#2}{swatch_ttree}},
    tey pt/.default={}{},
    ttm pt/.style n args={2}{tm pt={#1}{#2}{swatch_ttree}},
    ttm pt/.default={}{},
    toc pt/.style n args={2}{oc pt={#1}{#2}{swatch_ttree}},
    toc pt/.default={}{},
    tun pt/.style n args={2}{un pt={#1}{#2}{swatch_ttree}},
    tun pt/.default={}{},
    separate/.style n args={2}{rounded rectangle, draw=black, fill=white, minimum size=3.5mm, inner xsep=3.5mm, inner ysep=0mm, outer sep=1pt, label={[#1] #2}}, 
    separate/.default={}{},
    link/.style n args={1}{draw=#1, -{Latex[#1, length=2mm, width=1.5mm]}, dash=on 2pt off 1pt phase 0pt, thick}, %-{Arc Barb[arc=60, reversed, scale=2]
    link/.default={},
    qlink/.style n args={1}{link={#1}, ultra thick, solid},
    qlink/.default={},
    xlink/.style n args={1}{link={#1}, {}-{}},
    xlink/.default={},
    sxlink/.style={xlink={swatch_stree}},
    txlink/.style={xlink={swatch_ttree}},
    slink/.style={link={swatch_stree}},
    tlink/.style={link={swatch_ttree}},
    sqlink/.style={qlink={swatch_stree}},
    tqlink/.style={qlink={swatch_ttree}},
    rayl/.style={draw=black, solid, -{Triangle[angle'=90, open, left] Triangle[angle'=90, open, left]}},
    rayr/.style={draw=black, solid, -{Triangle[angle'=90, open, right] Triangle[angle'=90, open, right]}},
    rayprog/.style={draw=black, solid, -{Stealth[angle'=60] Stealth[angle'=60]}},
    merge/.style={double distance=1pt, thick}, 
    obs/.style={swatch_obs, line cap=rect, line join=miter, line width=\ul},
}
\pgfdeclarelayer{background}
\pgfdeclarelayer{foreground}
\pgfsetlayers{background,main,foreground}
\tikzset{
    pics/merge grp/.style n args={6}{ code={ 
        \node [separate={#1}{#2}] at (0, 0) {};
        \draw [merge] (-\um, 0) -- (\um, 0);
        \node (#3) [#4, xshift=-\um] at (0, 0) {};
        \node (#5) [#6, xshift=\um] at (0, 0) {};
    }}
}
\tikzset{
    pics/trace grp/.style n args={4}{ code={ 
        \node [separate={#1}{#2}, inner xsep=3mm] at (0,0) {};
        \node (#3) [trtm pt, xshift=-\uss] at (0,0) {};
        \node (#4) [trtm pt, xshift=\uss] at (0,0) {};
    }}
}
\begin{document}
    %%%%%%%%%%%%%%%%%%%%%%%%%%%%% Title Page %%%%%%%%%%%%%%%%%%%%%%%%%%%%%%%%%%%%%
\pagenumbering{gobble}
\begin{titlepage}
\begin{center}
    \vspace*{1cm}

    \singlespacing
    \Large
    \textbf{\textsc{Rapid Vector-based Any-angle Path Planning with Non-convex Obstacles}}\\
    
    \large
    \vspace*{3cm}   

    \textbf{\textsc{Lai Yan Kai}}\\
    \textit{(B.Eng, Electrical Engineering, 2019)}
    \doublespacing

    \vspace*{1cm}   
    \textbf{\textsc{A thesis submitted for the degree of\\
    Doctor of Philosophy}}\\
    \textbf{\textsc{Department of Electrical and Computer Engineering\\
    National University of Singapore}}
    
    \vspace*{3cm}   
    2024

    \singlespacing
    Supervisor:\\
    Associate Professor Prahlad Vadakkepat\\
    \vspace* {1cm}
    Examiners: \\
    Assistant Professor Zhao Lin \\
    Associate Professor Chew Chee Meng

\end{center}
\end{titlepage}

\singlespacing
\normalsize

%%%%%%%%%%%%%%%%%%%%%%%%%%%%% Abstract %%%%%%%%%%%%%%%%%%%%%%%%%%%%%%%%%%%%%
\chapter*{Abstract}
Vector-based algorithms belong to a nascent class of optimal any-angle path planners that prioritizes searches along the straight line between two queried points, and moving around any obstacles that lie along the straight line.
By searching obstacle contours, much free-space can be bypassed, and vector-based algorithms can find paths faster than conventional planners that search the free-space, like A* and Theta*.
Current vector-based planners are unable to navigate non-convex obstacles efficiently. Planners such as \rsp{} can conduct many undesirable line-of-sight checks from jagged contours, and planners that delay line-of-sight checks can severely underestimate path costs and branch exponentially.
The thesis aims to resolve the problems by introducing novel methods and concepts.
By using an angular counter, the target-pledge method allows searches to leave the contour of obstacles, and the source-pledge method places turning points at the perimeter of obstacles’ convex hulls.
The source progression method improves upon the source-pledge method by monitoring the maximum angular deviation and avoiding angular measurements.
The target progression method extends the source progression method for nodes leading to the goal point, and places phantom points, which are imaginary turning points, at non-convex corners.
The progression methods are combined to form the best hull, which is the smallest, inferable convex hull of a searched obstacle. The best hull enables path cost estimates to increase monotonically even as line-of-sight checks are delayed.
The progression methods are adapted to the optimal vector-based planners \rtwo{} and \rtwop{} that delay line-of-sight checks. The algorithms further rely on the sector and overlap rules, which discard undesirable searches based on geometrical reasoning. \rtwop{} improves upon \rtwo{} by simplifying and discarding more searches. \rtwo{} and \rtwop{} are much faster than state-of-the-art when paths are expected to have few turning points, regardless of path length.
A novel versatile multi-dimensional ray tracer is described, along with novel ideas for future work, such as a three-dimensional angular sector.

%%%%%%%%%%%%%%%%%%%%%%%%%%%%% Declaration %%%%%%%%%%%%%%%%%%%%%%%%%%%%%%%%%%%%%
\chapter*{Declaration}
I hereby declare that this thesis is my original work and it has been
written by me in its entirety. I have duly acknowledged all the sources
of information which have been used in the thesis.
This thesis has also not been submitted for any degree in any university
previously.
\vspace*{\fill}

\noindent\underline{\hspace{5cm}} \\
Lai Yan Kai \\

%%%%%%%%%%%%%%%%%%%%%%%%%%%%% Dedication %%%%%%%%%%%%%%%%%%%%%%%%%%%%%%%%%%%%%
\chapter*{Dedication}
\textit{To my grandparents},\\
\textit{parents}, \\
\textit{Mr. Pang, Prof. Prahlad, Prof. Lee}, and \textit{Prof. Xiang}.

%%%%%%%%%%%%%%%%%%%%%%%%%%%%% Acknowledgement %%%%%%%%%%%%%%%%%%%%%%%%%%%%%%%%%%%%%
\chapter*{Acknowledgements}
I want to thank my grandparents and parents for their unconditional love and care. 
Without them, I would not be able to last more than a few months in my studies, endeavors, and the challenging doctoral program.

Mr. Pang Hai Chet, for providing the environment to make mistakes to succeed, for propelling me to greater heights in my endeavors in robotics, and for guiding me to become a more caring technical leader.

Prof. Lee Tong Heng, for guiding in my senior years in undergraduate, sparing no effort to defend my application to the doctoral program, and in trusting in my potential.

Prof. Prahlad Vadakkepat, for supervising my work, 
providing me with opportunities to cultivate my potential, 
mentoring me in my character, and for fixing some of my worst mistakes with me.

Prof. Xiang Cheng, who along with Prof. Prahlad, for guiding me in paper writing, and providing me with opportunities.

The people listed here are key enablers to my works.
It is a privilege to have met these people, and in gratitude, I pledge to give back to the wider community with my work.

%%%%%%%%%%%%%%%%%%%%%%%%%%%%% TOC, AC, FIGURES, TABLES %%%%%%%%%%%%%%%%%%%%%%%%%%%%%%%%
\newpage
\tableofcontents

\clearpage
\pagenumbering{roman}
\setcounter{page}{1}

% \printunsrtglossary[type=abbreviations]
% \printunsrtglossary[type=symbols]
\newpage
\listoffigures
\newpage
\listoftables
\newpage
\listofalgorithms

\doublespacing
    \setcounter{secnumdepth}{4}% Number up to paragraphs
    \fancyhead[RE]{\chaptername~\thechapter}
    \chapter{Introduction}
\pagenumbering{arabic}
\setcounter{page}{1}

Path planning is a mature field, and a wide variety of solutions exists to find paths in maps with any number of dimensions.
Two-dimensional planners, such as Anya \cite{bib:anya} and \rsp{} \cite{bib:rayscanp} are able to return the shortest Euclidean paths, unconstrained by the discrete map representations that the algorithms rely on.
Multi-dimensional planners, such as RRT* \cite{bib:rrtstar}, are able to return feasible, sub-optimal paths quickly in high dimensions.
While the algorithms can find paths in reasonable time, the algorithms tend to rely on searching the free space to yield solutions.
As paths turn around obstacles, and the number of obstacle edges and corners tend to be much smaller than the amount of free-space in maps, searches can be prioritised to search obstacle contours instead of free space to accelerate searches.
By prioritizing searches along contours, an algorithm that searches the contours can potentially be much faster than existing methods.

Two classes of algorithms prioritize searches along contours to find paths.
The algorithms will try to move toward the destination in a straight line, and turn around any obstructing obstacles.
One class of algorithms are bug algorithms \cite{bib:bug, bib:tangentbug}. 
Bug algorithms are early local planners that guide robots around nearby obstacles, and are unable to find optimal paths around non-convex obstacles.
The other class of algorithms are \textit{vector-based algorithms} that attempt to find paths. Vector-based algorithms are any-angle path planners that return Euclidean shortest paths unconstrained to the geometry of the underlying grid (any-angle), and relies on contour searching to find paths. 
Vector-based algorithms are recent, with the earliest, Ray Path Finder \cite{bib:rpf}, published in 2017.

As of writing, only four vector-based algorithms, not introduced by this thesis, exists.
The algorithms are Ray Path Finder \cite{bib:rpf}, \rs{} \cite{bib:rayscan}, \rsp{} \cite{bib:rayscanp}, and Dual Pathfinding Search \cite{bib:dps}.
Ray Path Finder delays \gls{los} checks to prioritize searches along the straight line between two queried points (the start and goal points), has exponential time complexity in the worst case, and may be interminable. 
Ray Path Finder is fast on maps with convex obstacles, and can be slow if there are many obstacles.
\rs{} and \rsp{} find paths by recursively conducting \gls{los} checks whenever a potential turning point is found.
\rs{} and \rsp{} are prone to conducting undesirable \gls{los} checks along jagged contours, and has polynomial time complexity.
The algorithms are fast in dense maps with many obstacles, and may be slow in large maps with much free space and obstacles with jagged contours.
Dual Pathfinding Search attempts \gls{los} checks in two directions, one from each  queried point, and delays \gls{los} checks. \gls{los} checks can be conducted based on different edge selection policies, depending on the map types.
The algorithm cannot be implemented for maps with non-convex obstacles.
The vector-based algorithms outperform state-of-the-art free space planners such as Anya \cite{bib:anya} and Polyanya \cite{bib:polyanya}, and are promising research directions that aim to improve the speed of path planning.

The aforementioned vector-based algorithms can only find two-dimensional paths.
While the research focus can be shifted to extending vector-based algorithms to multiple dimensions, there are still areas of improvement for the two-dimensional case.
Delaying \gls{los} checks can help to eliminate unnecessary checks in \rs{} and \rsp{} and bypass contours that are unlikely to yield solutions.
However, delaying \gls{los} checks would mean that searches cannot be immediately discarded, and searches would multiply exponentially.
In addition, to ensure admissibility before \gls{los} checks can be conducted, a path's cost has to be estimated by assuming \gls{los} between nodes, even if the path passes through an obstacle. 
A combination of node pruning and admissible cost estimation can cause a path to be severely underestimated, as is the case for Ray Path Finder \cite{bib:me}.

To resolve the challenges of vector-based algorithms in two-dimensions, several novel methods and algorithms are introduced in this thesis.
Novel methods to navigate non-convex obstacles for vector-based planners that delay \gls{los} checks are introduced, and strategies to hasten computation are described.
The methods ensure that searches can navigate non-convex contours, while ensuring that path costs can increase monotonically as the searches progress.
Two novel algorithms, \rtwo{} and \rtwop{}, are introduced that leverages the novel methods and incorporates several concepts from other vector-based algorithms.
\rtwo{} and \rtwop{} are vector-based planners that delay \gls{los} checks, eliminating undesirable checks within the convex hulls of obstacles. 
The algorithms ameliorate the problems of interminability and severely underestimated costs, ensuring that non-convex obstacles can be navigated and optimal paths can be found.

\section{Synopsis of the Thesis}
In Chapter \ref{chap:litrev}, a literature review of path planning concepts and path planning algorithms are presented.
In Chapter \ref{chap:line}, a novel and versatile multi-dimensional ray tracer, which can be extended to any number of dimensions is introduced.
In addition, concepts involving collisions in the occupancy grid are introduced.
In Chapter \ref{chap:ncv}, several novel methods to navigate non-convex obstacles for vector-based algorithms that delay \gls{los} checks are described, and proven.
\textit{Phantom points}, which are imaginary future turning points, and the best-hull, which is the smallest known convex hull of an obstacle, are introduced in the chapter.
In Chapter \ref{chap:r2}, the algorithm \rtwo{} is introduced, which combines the methods in Chapters \ref{chap:line} and \ref{chap:ncv}.
Concepts from other algorithms are combined into \rtwo{}, and the proofs and results for \rtwo{} are described.
In Chapter \ref{chap:r2p}, the algorithm \rtwop{} is introduced. \rtwop{} is an evolved version of \rtwo{}, and includes proofs and results for the algorithm.
Chapter \ref{chap:conc} describes future works, and contains the conclusion.

Appendix \ref{chap:suppterm} provides brief descriptions on commonly used terms by the thesis.
Appendix \ref{chap:suppr2} and \ref{chap:suppr2p} describe the \rtwo{} and \rtwop{} algorithms in detail respectively.

\section{Contributions of the Thesis}
The thesis contributes to the field of vector-based, any-angle path planning.
As delaying line-of-sight checks in path planning can accelerate searches, the thesis introduces novel strategies to delay line-of-sight checks while searching in non-convex obstacles.
As of writing, the only vector-based planner to incorporate delayed line-of-sight checks is Dual Path Finding Search \cite{bib:dps}, but the algorithm can only work on maps with convex obstacles and limited non-convex obstacles. \rs{} and \rsp{} can work with non-convex obstacles, but are susceptible to undesirable searches along jagged contours.

The novel strategies include the \textit{phantom points} and the \textit{best-hull}, and the \textit{source-pledge} algorithm. 
The best-hull supersedes the pledge algorithms due to simplicity in implementation.
Phantom points are imaginary turning points placed on non-convex corners to obtain admissible convex hulls (\textit{best-hull}) while searching, leading to monotonically increasing cost estimates in algorithms with delayed line-of-sight checks.
The \textit{target-pledge algorithm} is described in \cite{bib:rpf} and the thesis provides proofs for the algorithm. 
The source-pledge algorithm is a novel concept that prevents turning points from being placed in convex hulls of obstacles.

Building upon the strategies, the algorithms \textit{\rtwo{}} and \textit{\rtwop{}} are introduced. \rtwo{} and \rtwop{} are the first in the field to delay line-of-sight checks and be able to return the shortest Euclidean paths. 
The algorithms rely on several novel concepts such as the \textit{progression rule}, \textit{sector rules}, and \textit{overlap rule} to discard repeated searches, especially as delaying line-of-sight checks can result in exponential search times.

By combining the novel concepts in \rtwo{} and \rtwop{} with the best-hull, \rtwo{} and \rtwop{} are superior to other any-angle algorithms if the shortest path solution has few turning points.
While having exponential time-complexity in the worst case with respect to the number of collided line-of-sight checks, the best case is linear in time complexity if the shortest path has at least one turning point.
If the shortest path is a straight line, there are no collisions, and the path is rapidly returned.

Other contributions of the thesis include a novel, versatile ray tracer for multiple dimensional occupancy grids, and a description of three-dimensional angular sector for extension to three-dimensional vector-based path planning. 
The ray tracer is based on symmetric ray tracers that eliminates the dependence on a driving axis while ray tracing and returns all intermediate cells unlike the Bresenham line algorithm \cite{bib:bresenham}. 
The ray tracer can additionally process lines that start and end at non-integer coordinates, and accounts for ambiguous situations when the line travels on or passes through cell boundaries.

\section{Applications of the Thesis}
The novel strategies for navigating non-convex obstacles with delayed line-of-sight checks can be applied to mobile robotics, especially with point-to-point global planning where a path is to be found between the robot and a destination that is on the other side of the map. 

As vector-based path planners that delay line-of-sight checks, the novel algorithms \rtwo{} and \rtwop{} introduced in the thesis significantly improves the search time for path planning, particularly if the optimal path has few turning points and the path is long.
In practical use cases, the free operating space is large and sparse to account for fine robot motion and sufficient obstacle representations, as evident in occupancy grid maps of indoor locations such as offices and shopping malls, or outdoor locations such as farms or urban areas.
As the free operating space is large, the number of turning points on a shortest path solution is significantly smaller than the amount of free space. 
In discrete maps such as occupancy grids, the amount of free space corresponds to the number of free cells, and in randomly sampled space, the amount of free space corresponds to the number of random samples. 
Commonly used planners in the literature, such as RRT* \cite{bib:rrtstar}, A* \cite{bib:astar}, and Theta* \cite{bib:thetastar}, search the free space extensively to find a solution, even if there is line-of-sight between the robot and the destination, or if the path turns around a few obstacles.
In such cases, \rtwo{} and \rtwop{} will outperform the planners.

In view of the recent advancements of artificial intelligence methods in path planning and motion planning, a non-expert may assume that conventional global planning can be superseded by methods such as deep learning and deep reinforcement learning.
Global planning is necessary especially in environments with \textit{non-convex obstacles}, even in works involving the methods.
Deep learning methods \cite{bib:dl1, bib:next} and deep reinforcement methods \cite{bib:drl5} that mimic global planning require learning over a predefined map, requiring re-learning every time the map has changed.
The time to re-learn is significantly longer than the time to replan the path on a new map by conventional planners such as A* \cite{bib:astar}, Theta* \cite{bib:thetastar}, and vector-based algorithms introduced in the thesis.
Moreover, deep reinforcement learning methods generally consider the local window for obstacle avoidance \cite{bib:drl3} and is incapable of global planning \cite{bib:drl1, bib:drl4}.
A widely cited work \cite{bib:drl7} in deep reinforcement learning claims that motion planning in an unknown environment can be done with reinforcement learning and without a map, as the model considers the sensory inputs directly to generate motion. 
Without a map, no global planning is done, and a non-expert may arrive at the conclusion that deep reinforcement learning can completely replace global planning.
The map-free claim is misleading as the work fails to consider instances where the environment contains highly non-convex obstacles.
By utilizing a reward function which rewards actions that lead the robot closer to the goal, the robot can get stuck in the convex hull of a non-convex obstacle, such as a G-shaped wall, by repeatedly following the walls in its local surroundings \cite{bib:rpf}.
As the model does not recall past obstacle information (map-less), the robot would not know if it is located within a non-convex obstacle, and would be unable to traverse out of it.
To the best understanding of the thesis' author, no known reinforcement learning methods exist that allow robots to navigate non-convex obstacles. As such, the concepts developed in the thesis may be able to aid the development of such methods.

% Line-of-sight checks are computationally expensive and can occur frequently in vector-based algorithms, especially if the map contains many rasterized, slanted obstacles. 
% Delaying line-of-sight checks can improve search time, but can result in exponential time-complexity.
% As strategies that delay line-of-sight checks are unable to verify path costs, the number of possible paths becomes exponential, and \rtwo{} and \rtwop{} have exponential time-complexity in the worst case with respect to the number of collided line-of-sight checks. 
% However, \rtwo{} and \rtwop{} return paths much more rapidly than other any-angle algorithms especially if the shortest path is expected to have few turning points. 
    \chapter{Literature Review} 
\label{chap:litrev}
A broad overview of path planning will be presented in this chapter.
In section The described literature includes concepts in path planning, map representations, and different types of path planners such as grid-based algorithms, topological algorithms, sampling-based algorithms, artificial potential fields, deep reinforcement learning algorithms, and a novel class of vector-based algorithms.

\section{Path Planning Concepts}
\subsection{Optimality and Completeness}
Path planners must be \textit{complete} -- a path is returned if it exists, otherwise none is returned.
A \textit{resolution complete} planner finds a path if it exists when the formulation of the world is fine enough.
For an occupancy grid, this means the cell size is small enough \cite{bib:rescomprrt},
or for a topological planner, sufficient nodes and edges are generated.
a \textit{probabilistically complete} planner finds a path given enough enough samples
\cite{bib:rrtstar}. This applies only to sampling based methods.
If the planner is not complete, the algorithm may not terminate in finite time \cite{bib:lavalle,bib:karman&frazzoli}.

An \textit{optimal} planner finds a path that is shortest given the representation. 
An \textit{asymptotically optimal} planner finds an optimal path given infinite samples -- 
i.e. the probability that an optimal path is found converges to unity with infinite samples \cite{bib:rrtstar,bib:bitstar}.
This concept is applicable only to \textit{anytime} algorithms which are also sampling based methods \cite{bib:lavalle,bib:karman&frazzoli}.

\subsection{World Representations}
Path planning use artificial world representations to find paths, which influence their search strategies heavily \cite{bib:lavalle}.
Grid-based planners use \textbf{occupancy grids} to find paths. These grids subdivide the world into discrete cells, which are usually
hypercubic (i.e. square for 2D, cubes for 3D). Each cell has a cost of traversal, depending on the difficulty of accessing the real region represented,
but is typically \textbf{free} (can be traversed) or \textbf{occupied} (obstructed and cannot be traversed). 
Grids with only these two costs are called \textbf{binary occupancy grids} \cite{bib:probrob}.
Hierarchical planners may rely on a grid with multiple layers of resolutions to find paths \cite{bib:hpastar}.
Grids are very easy to implement and are used widely in low-dimensional settings \cite{bib:lavalle,bib:karman&frazzoli}.

While grid planners subdivide the world evenly, \textbf{topological} planners represent the world as sparser graphs.
A common approach is to use polygons to represent obstacles for the 2D case \cite{bib:vg,bib:lavalle}.
Since paths must turn around convex corners of these polygons, they form nodes on the graph.
Pairs of nodes that can reach each other unobstructed have \textbf{line-of-sight} (LOS), and an edge can be drawn between them on the graph.
However, path planning becomes very complicated \cite{bib:vg3d0} and even intractable for higher dimensions. 
Another common approach draws nodes and edges that are far away from the obstacles \cite{bib:voronoi0,bib:lavalle}. See Section \ref{sec:litrevtopo}.

In high dimensional settings, it is computationally intractable to use occupancy grids and inefficient to 
get topological representations.
Thus, collision detection modules are used to detect collisions between the agent and obstacles \cite{bib:lavalle,bib:karman&frazzoli}.
These are commonly known as \textbf{continuous} approaches.

\subsection{Curse of Dimensionality}
The curse of dimensionality refers to the quickly intractable problem of grid-based approaches in high dimensional path planning \cite{bib:lavalle,bib:karman&frazzoli,bib:petrovic}.
Specifically, any hypercubic occupancy grid has at most $3^D - 1$ adjacent cells, 
where $D$ is the number of dimensions.
The number of adjacent cells in a hypercubic occupancy grid can be proven inductively. For the two dimensional case ($D=2$), a cell with coordinate $\mathbf{x} = [x_0,x_1]^\top$ can have an adjacent cell that is at $\mathbf{x}_a = [x_0 + \Delta_0, x_1 + \Delta_1]^\top$ where $\Delta_d \in \{-1,0,1\}$ for $d \in \{0, 1\}$, and $\mathbf{x} \ne \mathbf{x}_a$. 
As there are three sets of values for each axis, and the combinations cannot result in the current cell, the total number of adjacent cells has to be $3^D - 1$.
For the two-dimensional case, there is at most 8 neighbours, and 26 for the three-dimensional case. 
The number of adjacent cells blows up quickly for a small change in $D$.

Take for example, a grid implementing the configuration space of a typical 6 degree-of-freedom manipulator. Each cell will have 728 neighbouring cells. 
The cell size has to be small to accommodate smoother, realistic paths. 
Suppose the grid is subdivided into a coarse resolution of 1 degree,
to a total of 360 degrees for each degree-of-freedom. 
$360^6\approx2^{15}$ cells are needed for the entire configuration space,
which easily exceeds the memory capabilities of any current computing device.

In a topological 2D space with polygonal obstacles, the exact, optimal planning problem 
was found to be at least PSPACE-complete \cite{bib:2dpolyphard}.
In a similar 3D space, the same problem is at least NP-hard \cite{bib:3dpolynphard,bib:3dpolynphard2}.
While the lower bounds of space and time complexities are discouragingly high,
it is possible to circumvent these bounds by designing algorithms in new ways that are probabilistically complete and resolution complete \cite{bib:lavalle}. 

Instead of relying on occupancy grids, high dimensional planners often use collision modules 
to find obstacles on demand. \cite{bib:lavalle,bib:karman&frazzoli}.

\section{Types of Path Planners}
\subsection{Grid-based Planners} \label{sec:litrev:grid}
Occupancy grids are a popular choice of world representation that discretizes the world into grids \cite{bib:probrob}.
Each cell on the grid can be implemented with a cost to indicate movement penalties into them or between adjacent cells.
Popular algorithms like \textit{A*} \cite{bib:astar} and \textit{Dijkstra} \cite{bib:dijkstra}
find optimal paths by considering costs along the grids.
Earlier methods like \textit{Breadth First Search} \cite{bib:bfs}, and \textit{Depth First Search} \cite{bib:dfs} do not use these costs.

The grid's resolution is a balance between computational resources and path quality -- 
while smaller sizes cell sizes may result in smoother paths, both time and memory requirements increase.
Some algorithms use quad-tree implementations to reduce the cell size near obstacle boundaries 
and increase them around sparse, equal-cost regions \cite{bib:multires0,bib:multires1} 
to speed up searches while improving path quality.
The grid may also be broken into multiple hierarchies of different resolutions to speed up searches, 
albeit at the cost of optimality. Examples are \textit{HPA*} \cite{bib:hpastar} and \textit{Block A*} \cite{bib:blockastar}.

Early algorithms using occupancy grids are constrained by the direction of adjacent cells as paths are found by incremental searches along adjacent cells.
For $D$ dimensions, every cell has $3^D-1$ adjacent cells, causing a path to be constrained to at most $3^D-1$ directions. 
Post processing techniques are usually applied to smooth paths and form practical trajectories. %TODO: CITE POST PROC
Even with post-processing, the resulting path is not likely to be optimal when measured with the Euclidean metric, and post processing is an extra step that slows path acquisition \cite{bib:anya,bib:fielddstar}.
\textit{Field D*} is an early attempt to overcome the constrained angular problem by interpolating costs and allowing paths to traverse over grid vertices and grid cells \cite{bib:fielddstar}. 
Field D* was used for navigation for the Mars rovers Spirit and Opportunity \cite{bib:fielddstarmars}.

In robotics, path planning regularly deals with regions that are have two states, accessible or obstructed.
A \textbf{free cell} and an \textbf{occupied cell} in a binary occupancy grid  corresponds respectively to a traversable area and a non-traverable area in the mapped environment.
For mobile ground robots, a small area of ground that is free of obstructions can be represented by a free cell;
or for robotic manipulators, reachable regions in configuration-space that are free of singularities.
% Two-dimensional binary occupancy grids often feature jagged polygons of obstacles.
A realistically optimal path in a binary occupancy grid can wrap around some obstacles, 
and convex corners on the obstacles form turning points of a path \cite{bib:anya}. 
In a binary cost grid, an optimal path would have straight path segments instead of a curved segments like a multiple-cost grid.
Such a path is an \textbf{any-angle} path, with straight path segments that can point in any direction, and turning points located at convex corners that are at grid vertices.

An \textbf{any-angle path planner} finds the shortest any-angle paths on binary occupancy grids without post-processing. 
These planners, like Field D*, finds paths along edges and vertices instead of cell centres. 
Early and well-known examples of any-angle path planners are \textit{Theta*} \cite{bib:thetastar} and \textit{Lazy-Theta*} \cite{bib:lazythetastar}, and can be easily extended to binary occupancy grids.
However, the algorithms do not always find optimal paths \cite{bib:thetastar,bib:lazythetastar}.
Anya finds optimal paths by considering the underlying computational structure of the map by using \textit{interval} and \textit{cone} nodes, and is formulated only for two dimensions.

The curse of dimensionality is a well-known problem for occupancy grids -- as the number of dimensions increase,
the number of cells increase exponentially. As such, grid-based planners are often discarded in higher dimensional situations
in favor of continuous world representations that uses collision-detection modules (See Sec. \ref{sec:litrev:continuous}).

\subsection{Topological Planners} 
\label{sec:litrevtopo}
Topological planners simplify the world representation to graphs where each node represents a path to take or points to turn.

From a binary-cost, simple-polygonal representation of the world, a \textit{visibility graphs} is a graph of convex corners with \gls{los} \cite{bib:vg, bib:ohleong}.
A node in a visibility graph is a convex corner, and two nodes are connected if their corners have \gls{los}.
An algorithm like A* and Dijkstra is then used over the graph to find the shortest path between two points.

Unlike two-dimensional visibility graphs, three-dimensional visibility graphs are nodes along one-dimensional edges which may be partially occluded from other edges. 
As there can be uncountably infinite number of points along an edge, its is difficult to design a path planner that can account for all positions.
While it may be possible to find the shortest paths with a convex optimizer, the problem becomes intractable in environments with non-convex obstacles \cite{bib:vg3d1}.
as such, very few works on 3D visibility graph currently exists.
A three-dimensional visibility graph may be implemented approximately as cross sectional two-dimensional planes \cite{bib:vg3d1} or subdivided points along one-dimensional edges \cite{bib:vg3d2}.

As of writing, no works on higher dimensional visibility graphs exist. 
For a 2D VG, the exact path planning problem is expected to be at least polynomial-space hard \cite{bib:2dpolyphard}, 
while for a 3D VG, it is NP-hard \cite{bib:vg3d0,bib:3dpolynphard,bib:3dpolynphard2}.

\textit{Sub-goal graphs} are hierarchical algorithms that adapt the visibility graphs to two-dimensional binary occupancy grids \cite{bib:uras,bib:subgoal}. It first pre-processes
the map to find the nodes on a visibility graph counterpart, called \textit{subgoals}. When a query between two points occur, both points are connected to subgoals with \gls{los}, and a \textit{high-level} graph search begins along the subgoals. 
Next, the low level search finds the shortest paths between the identified subgoals using grid planners, concatenating the paths together to return a solution. 
As a hierarchical algorithm, the path may not be any-angle optimal \cite{bib:uras}.

% \textit{Voronoi diagrams} creates line segments that avoid obstacles as much as possible, and a graph planner is run over it. 
% However, such approaches often results in sub-optimal paths. 
% Works on Voronoi diagrams are extensive and typically focus on improving the optimality of the paths \cite{bib:voronoi0}. 
% Currently, no work exists for high dimensional problems.

% \textit{Cell decomposition} methods can be seen as a cross between visibility graphs and occupancy grids.
% For every node on the visibility graph, a line is extended in one consistent direction until the edge of the representation or another obstacle.
% This lines subdivide the free space into polygonal cells, where a path would then be found.
% Since the found path intersects the center of these cells and the edges bordering them, it avoids obstacles in a 
% similar manner to Voronoi diagrams \cite{bib:lavalle}.

\subsection{Sampling-based Planners} \label{sec:litrev:continuous}
Due to the large number of adjacent cells in high-dimensional occupancy grids, standard planners like A* and Dijkstra becomes intractable slow.
Any-angle algorithms that exploit the geometry of high dimensional spaces do not exist,
since it is inefficient to calculate the shapes of all obstacles.
Rather than expanding adjacent cells, sampling based algorithms find paths by sampling the free space and expanding a search tree towards the sampled points.
Sampling based algorithms rely on on-demand collision detection modules to efficiently locate obstructions \cite{bib:lavalle,bib:karman&frazzoli,bib:rrt}. 
The algorithms find approximate solutions quickly, 
and are \textbf{any-time} -- a path is first rapidly found, and becomes more optimal with more samples and iterations \cite{bib:rrtconnect,bib:karman&frazzoli,bib:rrtstar}.
Sampling based algorithms are \textbf{probabilistically complete} and \textbf{asymptotically optimal}, meaning that a path will be found and the path will be optimal by the time an infinite number of samples are considered \cite{bib:karman&frazzoli,bib:bitstar}.

\textit{Probabilistic Roadmaps} (PRM) \cite{bib:prm}
find paths by first sampling the free space for points, and subsequently attempts to connect nearby points that have \gls{los}. The resulting graph is called a  \textbf{roadmap}. 
A graph planner like A* and Dijkstra is then run on the roadmap to find a path quickly.
\textit{PRM*} improves upon PRM by scaling the radius to identify points with respect to the number of already connected points. 
The scaling prevents clustering of points in local regions and increases the number of
connections between distant points \cite{bib:karman&frazzoli}. 

\textit{Rapidly-Expanding Random Trees} (RRT) find paths by sampling points in free space, and growing the tree towards the sampled point incrementally \cite{bib:rrt}.
\textit{RRT*} evolves RRT by considering the cost of the nodes, and reconnecting new nodes on the tree to find straighter paths\cite{bib:rrtstar}.

\textit{Batch Informed Trees} (BIT*) is an algorithm that introduces heuristic costs  used in A*, the cost-to-come and cost-to-go, to random sampling.
The free space is first sampled to form a batch of sampled points, which includes the start and goal points.
Points that have the least heuristic costs are prioritized for connection and \gls{los} checks.
As an any-time algorithm, the algorithms stops once a path is found, or continues to find a smaller cost path.
When the algorithm continues, 
a new, denser batch is resampled, and points that will result in a longer path than the prior path will not be added to the newer batch.
The process repeats, and the path that is found will be shorter or the same length as the prior path.
Tests with the algorithm show that BIT* is faster and produces shorter paths than the aforementioned sampling methods \cite{bib:bitstar}.

\subsection{Artificial Potential Fields}
Artificial Potential Fields (APF) are local, reactive planners created for real-time collision avoidance \cite{bib:apf}.
The goal point forms a basin of attraction while obstacles generates repelling potential fields.
The robot then travels along minimum potential valleys that are free of collisions.
This idea is further extended into optimal path planning by incorporating heuristic costs \cite{bib:apfplan}
, gradient descent \cite{bib:apfplan1} or others. 
However, APF methods generally suffer from local minima issues that affects the optimality and completeness of algorithms,
and limits practical use in high-dimensional settings \cite{bib:karman&frazzoli}.

\subsection{Deep Reinforcement Learning}
Path planning algorithms hinging on learning methods are popular research topics
at the time of writing. 

\textit{Deep Reinforcement learning} (DRL) are the most popular algorithms and they focus primarily
on the overall integration of raw environmental input to the found path instead of replacing global planners.
They are generally local path planning strategies to avoid collisions, 
satisfy constraints, or adapt to dynamic situations
\cite{bib:drl0,bib:drl1,bib:drl2,bib:drl3,bib:drl4}. 
These approaches still require the aforementioned algorithms to plot global paths.

Some approaches focus on replacing path planning altogether.
\cite{bib:oneshot} uses a trained \textit{convolutional neural network} to find paths in small 2D and 3D maps.
\cite{bib:drl5} finds preliminary results on global path planning using DRL.
\textit{NEXT} uses DRL and remembers past expansions to find paths \cite{bib:next}.
Recognising the slow training of DRL in high dimensional spaces, 
\cite{bib:drl6} alleviates this by using a \textit{soft-actor critic} to explore high-dimensional spaces better 
and \textit{hindsight experience replay} to avoid sparse rewards in these spaces \cite{bib:drl6}. 

While highly adaptable to different input and problems, such methods generally suffer from the need to train
when different scenarios are shown. In addition, 
problem sets used in the high dimensional settings tend to be very sparsely populated with obstacles,
and few non-convex obstacles exist.

\subsection{Vector-based Algorithms} \label{sec:litrevvector}
Vector-based algorithms are \textit{any-angle path planners} (see Sec. \ref{sec:litrev:grid}) that find the shortest any-angle paths using vector-based searches. 
A vector-based algorithm tries to move toward a desired point \textbf{cast} as much as possible, and around any obstruction by moving along the obstacle's contour \textbf{trace}.
Such search strategies are well-known and used in reactive local planners such as Bug1 and Bug2 \cite{bib:bug}, and Tangent Bug \cite{bib:tangentbug}.
As local planners, Bug1, Bug2, and Tangent Bug do not find optimal paths.
By incorporating heuristic costs to vector-based strategies, vector-based algorithms can find paths significantly more quickly than non-vector-based approaches, especially if the shortest path has few turning points. 
Vector-based algorithms are a nascent class of path planners and four such algorithms currently exist: \textit{Ray Path Finder} \cite{bib:rpf} and \textit{\rs{}} \cite{bib:rayscan}, \textit{\rsp{}} \cite{bib:rayscanp}, and Dual Pathfinding Search \cite{bib:dps}, not including the author's algorithms \textit{\rtwo{}} \cite{bib:r2} and \textit{\rtwop{}} \cite{bib:r2p}.

Ray Path Finder is the first vector-based any-angle algorithm, and delays \gls{los} checks to find paths quickly.
By moving towards the goal point and ignoring points that lie far from the straight line between the start and goal points, Ray Path Finder's search complexity is largely invariant to the distance between points.
As such, Ray Path Finder can find paths rapidly if the start and goal points are far apart, provided that few convex obstacles lie between the points.
While fast for such cases, Ray Path Finder's search complexity is largely dependent on the number of collided casts, which is exponential in the worst case. 
Ray Path Finder may severely underestimate path costs, and may not be terminable \cite{bib:me}.
Nevertheless, Ray Path Finder has introduced an approach (the target-pledge algorithm, see Sec. \ref{ncv:sec:tpldg}) to navigate non-convex obstacles for vector-based algorithms that delay \gls{los} checks, and is the first to introduce several concepts that influence the works of \rtwo{} and \rtwop{} in this thesis.

\rs{} is an early vector-based algorithm that finds successive turning points by recursively casts to suitable convex corners.
\rs{} is optimal, and is fast in maps with highly non-convex obstacles that overlap each other.
Due to the recursive casts, \rs{} is likely to be slow on obstacles with many convex corners, such as slanted obstacles that will contain jagged contours when rasterized to binary occupancy grids.
As projections of \gls{los} checks are required for completeness, \rs{} is likely to be slow on large sparse maps, even if there are few obstacles between the start and goal points.
To improve the performance of \rs{}, jagged contours that represent $45^\circ$ diagonal lines before rasterization are conducted to a single $45^\circ$ edge during the search process.
While the number of corners to recursively cast to are reduced, the algorithm is still susceptible to larger jagged structures, or slanted edges that are not oriented to multiples of $45^\circ$ before rasterization.

\rsp{} is an evolved prototype of \rs{} that improves the speed of \rs{} by introducing extra rules. 
For example, the convex extension to ignore convex corners where paths would lead into the convex hull of obstacles; and further extensions to reduce the number of recursive casts on jagged contours.

An important concept from \rs{} and \rsp{} that influenced the works in this thesis is the angular sector.
Angular sectors are conical areas that originate from a turning point.
A subsequent turning point that lies in the angular sector of a turning point will result in a taut path.
A secondary function of an angular sector is to constrain searches from an expanded turning point to within the point's angular sector.
By constraining the searches, repeated searches that find more expensive paths can be discarded.
Angular sectors are bounded by rays that represent recursive casts. 
The concept is adapted to the works in the thesis to discard repeated searches.

Dual Pathfinding Search is a recent algorithm that attempts to find paths by searching from the start and goal points simultaneously.
Like Ray Path Finder, Dual Pathfinding Search does not test for \gls{los} immediately once a potential turning point is discovered.
Instead, a pair of consecutive turning points on a potential path is subsequently tested for \gls{los} based on a set of edge selection policies.
% while recursive casts like in \rs{} are implemented, Dual Pathfinding Search does not attempt to find all successive turning points at once, and instead expands turning points that deviate the least from a straight line.
As such, like Ray Path Finder, it generally unable to discard paths that overlap each other during the search process, and has exponential complexity in the worst case with  respect to the number of collided casts.

% More details on their limitations and strengths are indicated in Chapter \ref{chap:proto}.

\section{Summary}
In this chapter, concepts of path planning are introduced, along with world representations and path planner types.
The concepts include \textit{completeness}, \textit{optimality}, and the \textit{curse of dimensionality}.

Different world representations include occupancy grids that uniformly subdivide the world, and can have different costs associated to each grid cell depending on the difficulty of traversing an area.
A topological map represents a world as a graph, and costs are assigned to the edges between the different graph nodes.
In multi-dimensional settings, topological and grid maps are not computationally efficient, and collision detection modules are implemented instead.

The different types of path planners described in the chapter are as follows.
Grid-based planners are widely used algorithms that operate on occupancy grids. Any-angle planners are grid-based planners that find any-angle paths.
Topological algorithms are algorithms that search along a graph-based representation of the world to find paths.
Sampling based algorithms find paths by sampling points in free space, and joining the sampled points to find paths. 
Sampling based algorithms rely on collision modules to detect collisions, and are suited for high dimensional settings where grid-based and topological planners are inefficient or intractable.
Artificial Potential Fields are local planners that rely on forming attractive or repulsive fields to find feasibly short paths between two points, and are susceptible to getting trapped in local extrema.
Deep reinforcement learning algorithms rely on learning an optimal path policy on a static map to find optimal paths or avoid collisions, and may not be suited to global path planning situations on a rapidly changing map.
Vector-based algorithms are recent, any-angle algorithms that find paths by moving toward a destination in a straight line and around any obstruction.
Vector-based algorithms are potentially faster than free-space planners in finding paths, as the search space are the contours of obstacles and the line of cells between corners of obstacles.

The contributions of the thesis are on non-convex vector-based path planning with delayed line-of-sight checks, a versatile multi-dimensional ray-tracer, and the three dimensional angular-sector concept. The concepts can be applied to global planning and even guide deep learning and deep reinforcement learning motion planning methods in non-convex obstacles.
    \setcounter{algorithm}{0}
\renewcommand{\thealgorithm}{\thechapter.\arabic{algorithm}}

\chapter{Ray Tracers for Binary Occupancy Grid}
\label{chap:line}

A ray tracer is a line algorithm that finds collisions in free space.
Ray tracers are used extensively in graphic rendering, mapping, and any-angle path planners, and many designs exist.
A ray tracer can traverse a discrete or continuous space, with costs that can be discrete or continuous, and may also have to account for hierarchical subdivisions such as octrees \cite{bib:octree2017}.
As the focus of this thesis in on non-probabilistic any-angle path planners, ray tracers that can find collisions on a uniform, binary occupancy grid are discussed.

\section{Bresenham Line Algorithm and Digital Differential Analyzer}
For a binary occupancy grid, the most common ray tracers are the Bresenham algorithm \cite{bib:bresenham} and the \gls{dda} \cite{bib:rayama, bib:3ddda}.
The ray tracers find cells that lie are intersected by a ray.
If one of the cells is occupied, a collision is detected.

Both algorithms are similar in that iterations depend on a driving axis, which is the axis that has the longest projection of the drawn ray. 
For every unit length along the driving axis, the cell coordinate along the axis is incremented.
At each interval, a diagonal line can lie in between two adjacent cells along a shorter axis.
The Bresenham algorithm chooses the cell that is closer, while \gls{dda} chooses the cell by rounding the short coordinate.
As both algorithms choose only one cell, and that the diagonal cell can cross both cells between two intervals, a cell that is crossed by the ray can be skipped.

By using the Bresenham algorithm and \gls{dda}, a planner may find a path that passes through an obstacle.
To ensure that the paths are correct, any-angle planners that use the algorithms implicitly modify the algorithms to account for the missing cells.
In \rs{} \cite{bib:rayscan}, \rsp{} \cite{bib:rayscanp}, and Theta* \cite{bib:thetastar,bib:nashthesis} and its variants \cite{bib:lazythetastar, bib:phistar}.
A work from this thesis, \rtwo{} \cite{bib:r2}, uses the modifications as well.

\section{Symmetric Ray Tracing}
Basing the calculations off a preferred axis will require conditional statements that depend on the axis. 
By determining the intersections of the ray with the boundaries of a cell, the conditional statements can be removed, and a ray tracing algorithm can be simplified.
In such an algorithm, the incremental calculations along each axis does not depend on another axis, and are thus \textit{symmetric} for all axes \cite{bib:rayslater}.

A widely cited symmetric algorithm is described by Amanatides and Woo in \cite{bib:rayama}, and independently again by Cleary and Wyvill in \cite{bib:raycleary}. 
In general, the algorithm finds a generalized distance $k\in(0, 1)$ where a ray intersects a boundary, such that $k=0$ is at the start of the ray, and $k=1$ is at the destination.
The algorithm is defined for up to 3 dimensions in both works, but can be extended to more dimensions.
In general, consider a \textbf{$d$-axis hyperplane} that has a normal parallel to the $d$-axis, where $d=\{0,1,\cdots,D-1\}$ in a $D$-dimensional space.
The hyperplane is a cell boundary, and is located at integer units along the axis.
During the initialization, the intersection of the ray with the next hyperplane is identified for every axis, and the corresponding generalized distance $k$ is calculated.
The algorithm enters the main loop, and for every iteration, the smallest $k$ is picked, and the next hyperplane along the corresponding $d$-axis is identified.
The algorithm is simple, as $k$ can be incrementally determined by adding $1/\Delta_d$ where $\Delta_d$ is the difference between the start and goal points along the $d$-axis.

\def\clipper{\clip (-0.5*\ul, -0.5*\ul) rectangle ++(11*\ul, 5*\ul)}
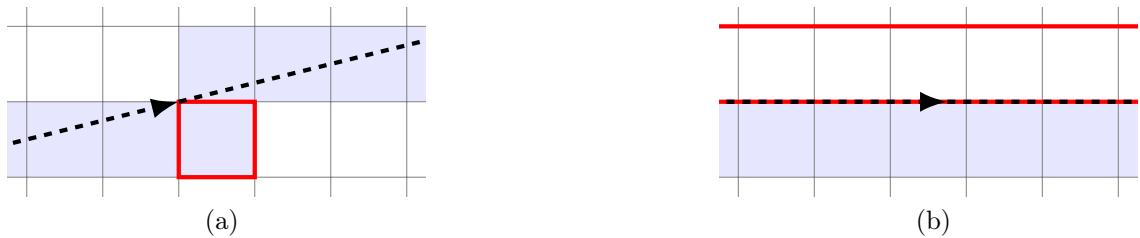
\begin{figure}[!ht]
\centering
\subfloat[\label{line:fig:symray:a} ] {%
    \centering
    \begin{tikzpicture}[]
        \clipper;
        \path
            (-2*\ul, 0.5*\ul) coordinate (xs) 
            ++(6*\ul, 1.5*\ul) coordinate (xt)
            ;

        \draw [white!90!blue, line cap=rect, line join=miter, line width=2*\ul]
            (-1*\ul, 1*\ul) -- ++(0:6*\ul)
            (5*\ul, 3*\ul) -- ++(0:8*\ul)
            ;

        \draw [gray, ultra thin, step=2*\ul] (-2*\ul, -2*\ul) grid ++(16*\ul, 10*\ul);
        \draw [red, ultra thick] (4*\ul, 0) rectangle ++(2*\ul, 2*\ul);

        \draw [ultra thick, dashed, -Latex] (xs) -- (xt);
        \draw [ultra thick, dashed, -Latex] (xt) -- ($(xs)!4!(xt)$);
        % \node [black pt]  at (xt) {};
        
    \end{tikzpicture}
} \hfill
\subfloat[\label{line:fig:symray:b}] {%
    \centering
    \begin{tikzpicture}[]
        \clipper;
        \path
            (-2*\ul, 2*\ul) coordinate (xs) 
            ++(7.5*\ul, 0*\ul) coordinate (xt)
            ;

        \draw [white!90!blue, line cap=rect, line join=miter, line width=2*\ul]
            (-1*\ul, 1*\ul) -- ++(0:14*\ul)
            ;
      
        \draw [gray, ultra thin, step=2*\ul] (-2*\ul, -2*\ul) grid ++(16*\ul, 10*\ul);
        \draw [red, ultra thick] (-2*\ul, 2*\ul) rectangle ++(14*\ul, 2*\ul);

        \draw [ultra thick, dashed, -Latex] (xs) to (xt);
        \draw [ultra thick, dashed, -Latex] (xt) to ($(xs)!2!(xt)$);
    \end{tikzpicture}
}
\caption[Problems with symmetric ray tracing algorithm.]{
   Problems with symmetric ray tracing algorithm. 
   (a) a line that passes through a corner may cause an extra cell to be identified (red bordered blue cell).
   (b) a line that travels along a grid line may miss out cells on one side of the line (red bordered cells) and only return the other (blue cells).
}
\label{line:fig:symray}
\end{figure}

While the algorithm is able to identify all cells traversed by a ray in most cases, it is not clear for the cases where $k$ from multiple axes are the same, or if a ray is travelling in between the cells.
In both cases, the algorithm may misidentify cells. 
When $k$  from multiple axes are the same, the path would have crossed the intersection between multiple hyperplanes, and caused a simultaneous increment along all affected axes.
In the two-dimensional case, the ray would have crossed a corner of a cell; and for the three-dimensional case, an edge or corner of a cell.
The algorithm increments each axis independently instead, and can some extra cells to be identified (see Fig. \ref{line:fig:symray:a}).
If a ray travels along the boundaries of the cell, the algorithm may only identify cells along one side of the boundary, and a ray may be prematurely determined to be collided
(see Fig. \ref{line:fig:symray:b}).
In the context of path planning, both cases can cause a collision to be identified even when none occurs.
To remedy the problem, an algorithm that accounts for both cases is designed in the next section.

\section{Extending Symmetric Ray Tracing}
The ray tracer described in the section enhances symmetric ray tracer by (i) extending the algorithm described in \cite{bib:rayama} and \cite{bib:raycleary} to a $D$-dimensional binary occupancy grid; (ii) accounting for the special cases where a ray travels along and intersects multiple cell boundaries; and (iii) accepting arbitrary coordinates which are not integers.

The pseudocode of the algorithm is in Alg. \ref{alg:raynd}, and supporting functions are in Alg. \ref{alg:raynd:init}, \ref{alg:raynd:getfront}, and \ref{alg:raynd:getcell}.
The algorithm makes use of a priority queue $\mathbb{Q}$ to sort the values of $k$. 
The queue is at most $D$ long, and a complex sorting mechanism should be avoided if $D$ is small.
The root vertex $\mx_\mathrm{root}$ is the coordinate incrementally adjusted by the algorithm, and is located at a vertex.
$\mathbb{F}$ stores the directional vectors of adjacent cells in front of the root vertex, and is constructed in Alg. \ref{alg:raynd:getfront}.
In cases where the ray does not travel along a boundary, $\mathbb{F}$ contains only one directional vector.
The number of directional vectors in $\mathbb{F}$ is $2^n$, where $n$ is the number of boundaries that the ray travels on.
For every interval, the cells in front of the root vertex are determined by the directional vectors in Alg. \ref{alg:raynd:getcell}, and if all cells are occupied, the ray would have collided.

\begin{algorithm}[!ht]
\begin{algorithmic}[1]
\caption{Extended Symmetric Ray Tracer}
\label{alg:raynd}
\Function{RayND}{$\mx_\mathrm{from}, \mx_\mathrm{to}$}
    \Comment{$\mx_\mathrm{from} \in \mreal^D, \mx_\mathrm{to} \in \mreal^D$}
    \State $(\mDelta, \mDelta_\mathrm{sgn}, \mx_\mathrm{root}, \mathbb{Q}) \gets $ \Call{Init}{$\mx_\mathrm{from}, \mx_\mathrm{to}$}
    \State $\mathbb{F} \gets $ \Call{GetFront}{$0, \mDelta_\mathrm{sgn}, \mx_\mathrm{from} - \mx_\mathrm{root}$} 
    
    \While {$\mathbb{Q} \ne \{\}$}
        \Comment{\underline{For each smallest distance $k$ in queue $\mathbb{Q}$,}}
        \State $\mathbb{D}_\mathrm{sameK} \gets \{\}$
        \Comment{\underline{find all $d$ where, at $k$, line crosses an integer $d$-axis hyperplane.}}
        \State $(k_{\min}, d) \gets \mathbb{Q}[0]$ 
        \DoWhile
            \State Push $d$ into $\mathbb{D}_\mathrm{sameK}$.
            \State Remove $\mathbb{Q}[0]$ from $\mathbb{Q}$.
            \If {$\mathbb{Q} = \{\}$}
                \State \Break
            \EndIf
            \State $(k, d) \gets \mathbb{Q}[0]$ 
        \EndDoWhile{$k_{\min} = k$} 

        \For {$d \in \mathbb{D}_\mathrm{sameK}$}
            \Comment{\underline{Increment root vertex along all affected $d$-axes.}}
            \State $\mx_\mathrm{root}[d] \gets \mx_\mathrm{root}[d] + \mDelta_\mathrm{sgn}[d]$
        \EndFor

        \For {$d \in \mathbb{D}_\mathrm{sameK}$}
            \Comment {\underline{Queue next $k$ for all affected $d$-axes, if next $k < 1$.}}
            \State $k \gets (\mx_\mathrm{root}[d] + \mDelta_\mathrm{sgn}[d] - \mx_\mathrm{from}[d]) / \mDelta[d]$
            % \Comment {Less float pt. errors than equivalent $k + \mDelta_\mathrm{sgn}[d] / \mDelta[d]$.}
            \If{$k \ge 1$} 
                \State \Continue
            \EndIf
            % \Comment {Skip if next $k \gtrsim 1$.}
            \State Queue $(k, d)$ to $\mathbb{Q}$, with smallest $k$ at front of $\mathbb{Q}$.
        \EndFor

        \If {\Call{GetCell}{$\mx_\mathrm{root}, \mathbf{f}$} is occupied for all $\mathbf{f} \in \mathbb{F}$}
            \State \Return $\mtrue$
            \Comment{\underline{Collision detected when all front cells occupied.}}
        \EndIf
        \Comment Collision coordinate is $\mx_\mathrm{from} + k \mDelta$. 
        %     \Continue;
        % \Comment{If all cells in front are not free, collision occurred.}
        % \State $collided \gets \mtrue$
        % \For {$\mathbf{f} \in \mathbb{F}$}
        %     \If {\Call{GetCell}{} is free}
        %         \State $collided \gets \mfalse$
        %         \State \Break
        %     \EndIf
        % \EndFor
        % \If {$collided = \mtrue$}
        %     \State \Return $\mtrue$
        % \EndIf
    \EndWhile
    \State \Return $\mfalse$ \Comment{Reached $\mx_\mathrm{to}$}
\EndFunction
\end{algorithmic}
\end{algorithm}

\begin{algorithm}[!ht]
\begin{algorithmic}[1]
\caption{Initialize variables.}
\label{alg:raynd:init}
\Function{Init}{$\mx_\mathrm{from}, \mx_\mathrm{to}$}
    \State $\mDelta \gets \mx_\mathrm{to} - \mx_\mathrm{from}$
    \State $\mDelta_\mathrm{sgn} \gets \fsgn(\mDelta)$
    \Comment{Values close to zero are rounded to zero.}
    \State $\mx_\mathrm{root} \gets \mathbf{0}_D$
    \For {$d \in \{0,1,\cdots,D-1\}$}
        \If {$\mDelta[d] \ge 0$}
            \State $\mx_\mathrm{root}[d] \gets \ffloor(\mx_\mathrm{from}[d])$
        \Else
            \State $\mx_\mathrm{root}[d] \gets \fceil(\mx_\mathrm{from}[d])$
        \EndIf
    \EndFor
    \State $\mathbb{Q} \gets \{\}$ \Comment{Note: $\mathbb{Q}$ has at most $D$ elements.}
    \For {$d = \{0, 1, \cdots, D-1\}$}
        \If {$\mDelta_\mathrm{sgn} = 0$} 
            \State \Continue
        \EndIf
        \State Push $(0, d)$ into $\mathbb{Q}$
        \State $\mx_\mathrm{root}[d] \gets \mx_\mathrm{root}[d]  - \mDelta_\mathrm{sgn}[d]$
    \EndFor

    \State \Return $(\mDelta, \mDelta_\mathrm{sgn}, \mx_\mathrm{root}, \mathbb{Q})$
\EndFunction
\end{algorithmic}
\end{algorithm}

\begin{algorithm}[!ht]
\begin{algorithmic}[1]
\caption{Gets directional vectors pointing to front, adjacent cells of root vertex.}
\label{alg:raynd:getfront}
\Function{GetFront}{$d, \mathbf{f}, \mx_\mathrm{err}$}
    \If {$d \ge D$}     
        \State \Return $\{\mathbf{f}\}$
    \EndIf
    \State $\mathbb{F} \gets \{\}$
    \State $\mathbb{I} \gets \{\mathbf{f}[d]\}$
    \If {$\lvert \mx_\mathrm{err}[d]\rvert < \mthres$ \An $\mathbf{f}[d] = 0$}
        \Comment{Will travel along cell boundaries}
        \State $\mathbb{I} \gets \{-1, 1\}$
    \EndIf
    \For {$i \in \mathbb{I}$}
        \State $\mathbf{f}_\mnew \gets \mathbf{f}$
        \State $\mathbf{f}_\mnew[d] \gets i$
        \State $\mathbb{F}_\mnew \gets $ \Call{GetFront}{$d+1, \mathbf{f}_\mnew, \mx_\mathrm{err}$}
        \State Append $\mathbb{F}_\mnew$ to back of $\mathbb{F}$.
    \EndFor
    \State \Return $\mathbb{F}$
\EndFunction
\end{algorithmic}
\end{algorithm}

\begin{algorithm}[!ht]
\begin{algorithmic}[1]
\caption{Returns the cell in direction $\mathbf{f}$ of the root vertex.}
\label{alg:raynd:getcell}
\Function{GetCell}{$\mx_\mathrm{root}, \mathbf{f}$}
    \State $\mx_\mathrm{cell} \gets \mx_\mathrm{root} + \min(\mathbf{f}, \mathbf{0}_D)$
    \Comment{element-wise minimum.}
    \State \Return cell at $\mx_\mathrm{cell}$
\EndFunction
\end{algorithmic}
\end{algorithm}

\section{Occupancy Grid Collisions}
Only two-dimensional collisions are considered in this thesis. A future work can extend the collision to n-dimensions.
The concepts in this section are used in the planners \rtwo{} and \rtwop{}, and for the rest of the thesis.

\subsection{The Contour Assumption}
\label{line:sec:contour}
A trace travels along the grid lines, and along an obstacle contour.
While sharing the trace walks the same coordinates as the traced contour, the contour can be assumed to lie an infinitesimal distance away from the grid (see Fig. \ref{line:fig:collision}).
This will be termed as the \textbf{contour assumption}.

A path planner takes in two coordinates, the \textbf{start point}, and \textbf{goal point}, and finds a path between the points.
If the start or goal point lies on an obstacle contour, it can be likewise assumed to lie an infinitesimal distance away from the contour.
The assumption ensures that comparisons between directions are not ambiguous in a vector-based planner, especially when traces are coincident with sector rays and progression rays in \rtwo{} and \rtwop{}, and when the start point lies at a corner.

\subsection{Line-of-sight Collisions}
\def\clipper{\clip (-1*\ul, -1*\ul) rectangle ++(8*\ul, 8*\ul)}
\def\difl{3*\ul}
\def\dif{\difl/sqrt(2)}
\def\difs{4mm}

\tikzset{
    pics/linecollision grid/.style n args={0}{ code={ 
        \clipper;
        \draw [white!50!gray, step=\difl] (-\difl, -\difl) grid ++(4*\difl, 4*\difl);
    }},
}
\tikzset{
    pics/linecollision col/.style args={(#1) #2}{ code={ 
        \clipper;
        \path (#1) coordinate (xcol);
        \path (xcol) -- ++(#2:2*\difl) coordinate (xs);
        \draw [gray] 
            (xcol) -- (xs) ;
        \draw [dotted]
            (xcol) -- ($(xs)!2!(xcol)$);
        \node [cross pt] at (xcol) {};
    }},
}
\tikzset{
    pics/linecollision ncv/.style n args={0}{ code={ 
        \clipper;
        \path (\difl, \difl) coordinate (xcol);
        
        \draw [swatch_obs, line cap=rect, line join=miter, line width=3*\ul, shift={(-\difs,\difs)}]
            (2.5*\difl, 1.5*\difl) coordinate (xa0)
            -- ++(180:{2*\difl}) coordinate (xa1)
            -- ++(-90:{2*\difl}) coordinate (xa2)
            ;
        \pic at (0,0) {linecollision grid};
        \draw [->] (xcol) -- ++(-\difl, \difl);
        
    }},
}
\tikzset{
    pics/linecollision cv/.style n args={0}{ code={ 
        \clipper;
        \path (\difl, \difl) coordinate (xcol);
        
        \draw [swatch_obs, line cap=rect, line join=miter, line width=3*\ul, shift={(-\difs, \difs)}]
            (0.5*\difl, 2.5*\difl) coordinate (xa0)
            -- ++(-90:{\difl}) coordinate (xa1)
            -- ++(180:{\difl}) coordinate (xa2)
            ;
        \pic at (0,0) {linecollision grid};
        \draw [->] (xcol) -- ++(-\difl, \difl);
    }},
}
\tikzset{
    pics/linecollision wall/.style n args={0}{ code={ 
        \clipper;
        
        \draw [swatch_obs, line cap=rect, line join=miter, line width=3*\ul, shift={(-\difs, \difs)}]
            (0.5*\difl, 2.5*\difl) coordinate (xa0)
            -- ++(-90:{3*\difl}) coordinate (xa1)
            ;
        \pic at (0,0) {linecollision grid};
    }},
}
\tikzset{
    pics/linecollision left/.style n args={1}{ code={ 
        \node (n) [circle, draw=swatch_blue, minimum size=2mm, line width=0.5mm, inner sep=0, outer sep=0] at (0, 0) {};
        \draw [swatch_blue, ->, line width=0.5mm] (n) -- ++(#1:\ul);
    }},
}
\tikzset{
    pics/linecollision right/.style n args={1}{ code={ 
        \node (n) [circle, draw=orange, minimum size=3mm, line width=0.5mm, inner sep=0, outer sep=0] at (0, 0) {};
        \draw [orange, ->, line width=0.5mm] (n) -- ++(#1:\ul);
    }},
}

\begin{figure}[!ht]
\centering
\subfloat[\label{ncv:fig:collision:a} ] {%
    \centering
    \begin{tikzpicture}[]
        \clipper;
        \pic at (0,0) {linecollision ncv};
        \pic at (xcol) {linecollision left={180}};
        \pic at ($(xcol) + (0:\difl)$) {linecollision right={0}};
        \pic at (0,0) {linecollision col={(xcol) -60}};

        \node at ($(xcol) + ({-\difs*tan(30)}, {\difs})$) [cross pt, scale=0.7] {};
        
    \end{tikzpicture}
}\hfill
\subfloat[\label{ncv:fig:collision:b} ] {%
    \centering
    \begin{tikzpicture}[]
        \clipper;        
        \pic at (0,0) {linecollision ncv};
        \pic at ($(xcol) + (-90:\difl)$) {linecollision left={-90}};
        \pic at (xcol) {linecollision right={90}};
        \pic at (0,0) {linecollision col={(xcol) -30}};

        \node at ($(xcol) + ({-\difs}, {\difs*tan(30)})$) [cross pt, scale=0.7] {};
    
    \end{tikzpicture}
}\hfill
\subfloat[\label{ncv:fig:collision:c} ] {%
    \centering
    \begin{tikzpicture}[]
        \clipper;
        \pic at (0,0) {linecollision ncv};
        \pic at (xcol) {linecollision left={-90}};
        \pic at (xcol) {linecollision right={0}};
        \pic at (0,0) {linecollision col={(xcol) -45}};

        \node at ($(xcol) + ({-\difs}, {\difs})$) [cross pt, scale=0.7] {};
    \end{tikzpicture}
} \\
\subfloat[\label{ncv:fig:collision:d} ] {%
    \centering
    \begin{tikzpicture}[]
        \clipper;
        \pic at (0,0) {linecollision cv};
        \pic at (xcol) {linecollision left={-90}};
        \pic at ($(xcol) + (90:\difl)$) {linecollision right={90}};
        \pic at (0,0) {linecollision col={(xcol) -60}};

        \node at ($(xcol) + (-\difs, {\difs/tan(30)})$) [cross pt, scale=0.7] {};
        
    \end{tikzpicture}
}\hfill
\subfloat[\label{ncv:fig:collision:e} ] {%
    \centering
    \begin{tikzpicture}[]
        \clipper;        
        \pic at (0,0) {linecollision cv};
        \pic at ($(xcol) + (180:\difl)$) {linecollision left={180}};
        \pic at (xcol) {linecollision right={0}};
        \pic at (0,0) {linecollision col={(xcol) -30}};

        \node at ($(xcol) + ({-\difs/tan(30)}, {\difs})$) [cross pt, scale=0.7] {};
    
    \end{tikzpicture}
}\hfill
\subfloat[\label{ncv:fig:collision:f} ] {%
    \centering
    \begin{tikzpicture}[]
        \clipper;
        \pic at (0,0) {linecollision cv};
        \pic at (xcol) {linecollision left={180}};
        \pic at (xcol) {linecollision right={90}};
        \pic at (0,0) {linecollision col={(xcol) -45}};

        \node at ($(xcol) + ({-\difs}, {\difs})$) [cross pt, scale=0.7] {};
    \end{tikzpicture}
} \\
\subfloat[\label{ncv:fig:collision:g} ] {%
    \centering
    \begin{tikzpicture}[]
        \clipper;

        \pic at (0,0) {linecollision wall};
        
        \path (\difl, 1.5*\difl) coordinate (xcol);
        \pic at ($(xcol) + (-90:0.5*\difl)$) {linecollision left={-90}};
        \pic at ($(xcol) + (90:0.5*\difl)$) {linecollision right={90}};
        \pic at (0,0) {linecollision col={(xcol) -45}};

        \node at ($(xcol) + ({-\difs}, {\difs})$) [cross pt, scale=0.7] {};
        
    \end{tikzpicture}
}\hfill
\subfloat[\label{ncv:fig:collision:h} ] {%
    \centering
    \begin{tikzpicture}[]
        \clipper;        

        \pic at (0,0) {linecollision wall};
        
        \path (\difl, \difl) coordinate (xcol);
        \pic at ($(xcol)$) {linecollision left={-90}};
        \pic at ($(xcol) + (90:\difl)$) {linecollision right={90}};
        \pic at (0,0) {linecollision col={(xcol) -45}};

        \node at ($(xcol) + ({-\difs}, {\difs})$) [cross pt, scale=0.7] {};
    \end{tikzpicture}
}\hfill
\subfloat[\label{ncv:fig:collision:i} ] {%
    \centering
    \begin{tikzpicture}[]
        \clipper;        

        \pic at (0,0) {linecollision wall};
        
        \path (\difl, \difl) coordinate (xcol);
        \pic at ($(xcol)$) {linecollision left={-90}};
        \pic at ($(xcol)$) {linecollision right={90}};
        \pic at (0,0) {linecollision col={(xcol) 0}};

        \node at ($(xcol) + ({-\difs}, {0})$) [cross pt, scale=0.7] {};
    \end{tikzpicture}
}
\caption[Contour assumption and line-of-sight collisions.]{
The contour assumption reduces ambiguity when a ray collides with an obstacle.
A blue circle and arrow represents the left trace's starting position and direction respectively, and an orange circle and arrow represents the right trace.
The black arrow is the bisecting vector $\mvcrn$.
Under the contour assumption, the obstacle lies an infinitesimal distance away from the grid lines (exaggerated in illustration). 
A cast that collides with a corner, and that points slightly to the right of $\mvcrn$ (a, d), or to the left (b, e), will be treated to have collided at the right or left edge respectively.
The collided corner will be the first corner for a left trace in (a, d), or for a right trace in (b, e).
The collided corner will be the first corner for (c, f), and the trace directions depend on the implementation, (\rtwop{} is shown).
As there is no corner ambiguity for (g, h, i), the adjacent vertices can lie on any adjacent vertex.
By eliminating the ambiguity, the sector rule for \rtwo{} and \rtwop{} can infer that a ray has been crossed. A ray is crossed when the trace is located at the first corner found after the ray collides.
}
\label{line:fig:collision}
\end{figure}
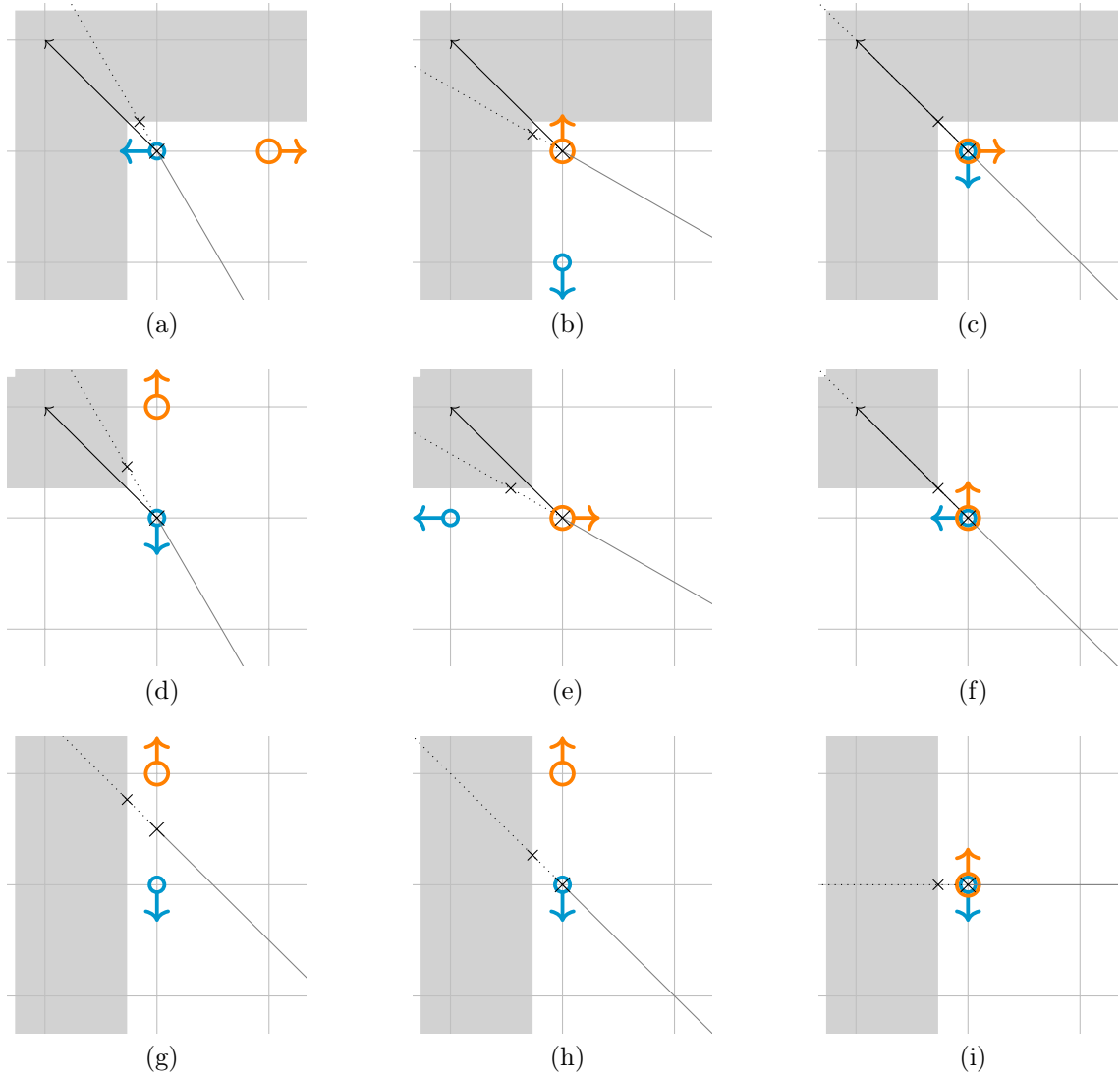
The collision point of a \gls{los} check, also called a \textbf{cast}, rarely occurs at a grid vertex.
In an optimal planner, it may be more accurate to use integer forms instead of floats to avoid ambiguity in directional comparisons.
To do so, it is possible to augment vectors and coordinates to store the fractional form of numbers. One extra value is required to store the common denominator between the original values, but doing so will incur additional calculations, and store very large numbers in the numerator.

An alternative is to resolve the coordinates of the vertices adjacent to the collision and along the contour.
With the contour assumption, the adjacent vertices can be found by comparing the direction of the ray with the \textbf{bisecting vector} of a corner.
The bisecting vector is a directional vector that bisects the interior angle of a corner.
Let $\mvcrn$ represent the smallest integer bisecting vector.
For a binary occupancy grid, $\mvcrn=[a,b]^\top$ points in the ordinal directions (north-west, north-east, etc.), where $a=\{-1,1\}$ and $b=\{-1,1\}$.

To transition to a trace, the trace direction has to be known after finding the adjacent vertices.
The adjacent vertices and trace direction for all cases are listed in Fig. \ref{line:fig:collision}.

After finding the adjacent vertices, a trace is performed to find the first corner from the collision point.
By associating the first corners on each side to a ray that represents the cast, an algorithm can compare against the positions of the corners to determine if a ray has been crossed, especially if a trace's current position lies exactly on a ray.
while the bisecting vector can be used to break ties, comparing against the bisecting vector may require slightly more expensive calculations.
Comparing against the bisecting vector requires the two-dimensional cross product, which is slightly slower than comparing against the pair of coordinates defining the first corners.

% \subsection{Ray Comparisons}
% A ray is a directional vector that begins from a point.
% In a vector-based planner, a ray is used to compare directions, and is first described extensively in the context of a vector-based planner in the algorithm \rs{} \cite{bib:rayscan}.
% A ray in \rs{} discards a trace if the trace crosses the ray, keeping the number of searches small.

% When comparing directions between a vector and a ray, ambiguities may arise if the vector and ray are coincident. In such cases, it is 

% From the adjacent corners and trace direction, the first corners adjacent to the collision point are found.
% By storing the first corners into a ray, there is no ambiguity when a ray check if the ray has collided with the traced edge, 

\section{Conclusion}
Fundamental to an any-angle planner is the ray tracer, or line algorithm, which detects collision along a line between two points. 
To ensure that a correct and optimal path can be found in an occupancy grid, a ray tracer has to identify all cells intersected by a line.

The Bresenham algorithm and Digital Differential Analyzer are widely used ray tracers that rely on a driving axis to detect cells.
However, the algorithms are unable to detect all intersected cells.
Symmetric ray tracers, described by Amanatides and Woo, and Cleary and Wyvill,  rely on the distance of a line to detect cells.
A symmetric ray tracer can identify all cells, and 
by parameterizing calculations based on the distance instead of a driving axis, a symmetric ray tracer is simpler to implement than the Bresenham algorithm and Digital Differential Analyzer.

A symmetric ray tracer may not account for instances where a line lies along the grid and between the cells, resulting in missed collisions. 
A novel multi dimensional symmetric ray tracer is introduced to identify all cells adjacent to such a line. 
To allow for versatility, the ray tracer is able to accept any real number coordinate from within an occupancy grid, and not just the integer coordinates.

    \chapter{Navigating Non-convex Obstacles for Vector-based Planners}
\label{chap:ncv}

A na\"ive vector-based algorithm moves to the goal point greedily, and is prone to getting trapped in non-convex obstacles \cite{bib:rpf,bib:bug}.
An example is a `G'-shaped non-convex obstacle, where a point from within the obstacle attempts to reach a point outside the obstacle (see Fig. \ref{ncv:fig:gshape}).
The chapter lists several methods that can allow vector-based algorithms to move around non-convex obstacles.

\input{chap_ncv/fig_gshape}

\section{Pledge Algorithm}
The Pledge algorithm is a well-known method of exiting a non-convex maze in one direction \cite{bib:turtle}.
While tracing an obstacle's contour, the algorithm monitors the angular displacement from the desired direction, leaving the contour only when the total displacement becomes zero again.
In a highly non-convex obstacle, like a spiral, the angular displacement can be winded to more than $360^\circ$.
If the interior angles of all obstacle corners are multiples of each other, the angles can be discretized, and a discrete counter can be used to monitor the displacement.

While the Pledge algorithm is able to navigate a maze of non-convex obstacles, the algorithm is not suitable for optimal path planning.
The algorithm can only leave a maze in a desired direction, and cannot reach a desired point within the maze.
Nevertheless, the angular counter serves as a good starting point for designing methods to navigate non-convex obstacles.

\section{Placement and Pruning Methods in Vector Based Algorithms}
\label{ncv:sec:placeandprune}
A vector-based algorithm that does not verify line-of-sight may find turning points that appear to form part of a taut path when first found.
As the algorithm progresses, the turning points may stop being part of a taut path, and have to be pruned.
The pruning method is first described in \cite{bib:rpf}.
\input{chap_ncv/fig_prune}

Before pruning can be described, the placement of points have to be explained.
A $\mside$-sided trace, where $\mside\in\{L,R\}$, and $L=-1$ for a left trace and $R=1$ for a right trace, can only place $\mside$-sided nodes (turning points).
The path on the $\mside$-side of the turning point leads to the goal point, while the path on the $\mside$-side leads to the start point.

Consider the simplest case where a turning point is found immediately after the initial edge is traced, at $\mx_S$, in Fig. \ref{ncv:fig:prune:a}.
At $\mx$, the path segment $(\mx_{SS}, \mx_S, \mx)$ stops being taut around $\mx_S$. Since the trace is a left trace, the source point at $\mx_S$ is left-sided, and a necessary condition to prune can be inferred to be
\begin{equation}
    {isTautSrc}:= \mside_S (\mv_S \times \mv_{SS}) < 0,
    \label{ncv:eq:prune:istautsrc}
\end{equation}
where a prune occurs if ${isTautSrc}$ evaluates to $\mfalse$.
$\mside_S$ is the side of a node at $\mx_S$, and $\mside_S = \mside$ in this example. The $\times$ operator is the two-dimensional cross product, and
$\mv_S = \mx - \mx_{S}$, and $\mv_{SS} = \mx_S - \mx_{SS}$.
Eq. (\ref{ncv:eq:prune:istautsrc}) can be extended to turning points found on other obstacles, and is used by \rtwo{} and \rtwop{}.

An algorithm that implements the pruning method has to work around the angular constraints of Eq. (\ref{ncv:eq:prune:istautsrc}).
Eq. (\ref{ncv:eq:prune:istautsrc}) assumes that $\mv_S$ and $\mv_{SS}$ has not rotated by more than $\pi$ radians with respect to each other. 
If so, Eq. (\ref{ncv:eq:prune:istautsrc}) would break down as the cross-product is only valid for angles that are between $-\pi$ and $\pi$ radians.
If the prune does not occur immediately after a path segment stops being taut, a trace that walks around a non-convex obstacle may cause $\mv_S$ to be rotated by more than half a round with respect to $\mv_{SS}$ (see Fig. \ref{ncv:fig:prune:c}).

An algorithm that delays \gls{los} checks may begin to verify \gls{los} by examining the parts of a path that are closer to the start point.
As such, a turning point can lie in the target direction of the current position and be pruned.
The pruning method is extended to a target turning point at $\mx_T$ with
\begin{equation}
    {isTautTgt}:= \mside_T (\mv_{TT} \times \mv_T) < 0,
    \label{ncv:eq:prune:istauttgt}
\end{equation}
where $\mside_T$ is the side of the target point at $\mx_T$, and $\mv_T = \mx - \mx_T$ and $\mv_{TT} = \mx_T - \mx_{TT}$ (see Fig. \ref{ncv:fig:prune:b}).
The point in the target direction of the target point at $\mx_{T}$ lies at $\mx_{TT}$.
Like Eq. (\ref{ncv:eq:prune:istautsrc}), Eq. (\ref{ncv:eq:prune:istauttgt}) is constrained to $-\pi$ and $\pi$ radians, and is used by \rtwo{} and \rtwop{}.

\section{Target-pledge method}
\label{ncv:sec:tpldg}
The target-pledge method is first described and used by Ray Path Finder \cite{bib:rpf}.
The algorithm is adapted from the Pledge algorithm, and instead measuring the angular displacement with respect to a static direction, the displacement is measured with respect to a direction that always points to the destination (\textbf{target}).
In \cite{bib:rpf}, the target-pledge method is discrete because an occupancy grid is used by Ray Path Finder.
The algorithm can be generalized to a continuous range of angles for it to be extended to a map of polygonal obstacles \cite{bib:me}.

While the pledge algorithm can be thought of as a spring powered toy with one knob winding a spring, the target-pledge method can be thought of as another toy that has one knob for each side of the spring.
At every corner visited by a trace, both knobs are adjusted.
For both algorithms, the interior angle of the corner winds or unwinds one knob.
For the target-pledge method, the change in direction of the target point adjusts the second knob.

\subsection{Target-pledge Update Equations}
Let $\theta_T$ be the angle the vector $\mora{\mx\mx_T}$  makes with the positive $x$-axis, where $\mx$ is the current traced position and $\mx_T$ is the target's position.
Let $\theta_\medge$ be the angle the next trace direction makes with the positive $x$-axis at $\mx$, and $\mvnext$ is the vector pointing from $\mx$ to the subsequent corner.
The angles are
\begin{align}
    \hat{\theta}_t &= \fatantwo(\mx_T - \mx) \label{ncv:eq:tpldg:t} \\
    \hat{\theta}_\medge &= \fatantwo(\mvnext).
    \label{ncv:eq:tpldg:e}
\end{align}
Let the change in angles be
\begin{align}
    \vartheta_t &= \llfloor \hat{\theta}_t - \hat{\theta}_t' \rrfloor \label{ncv:eq:tpldg:tchange} \\ 
    \vartheta_\medge &= \llfloor \hat{\theta}_\medge - \hat{\theta}_\medge' \rrfloor \label{ncv:eq:tpldg:echange}
\end{align}
where $\hat{\theta}_t'$ and $\hat{\theta}_\medge'$ are the angles defined at the previous trace position for Eq. (\ref{ncv:eq:tpldg:t}) and (\ref{ncv:eq:tpldg:e}) respectively.
The operator $\llfloor\cdot\rrfloor$ constrains its angular operand to $[-\pi,\pi)$ radians.
The \textbf{target-pledge} at $\mx$ is the angle
\begin{equation}
    \theta_T = \theta_T' + \vartheta_t - \vartheta_\medge
    \label{ncv:eq:tpldg}
\end{equation}
where $\theta_T'$ is the target-pledge at the previous trace position.
The trace begins at $\mx_0$ where a cast collides with an obstacle, and the initial target-pledge is
\begin{equation}
    \theta_{T,0} = \llfloor \fatantwo(\mx_T - \mx_0) - \hat{\theta}_{\medge,0} \rrfloor
    \label{ncv:eq:tpldg:init}
\end{equation}
where $\hat{\theta}_{\medge,0}$ is the angle of the initial trace direction.

Eq. (\ref{ncv:eq:tpldg:t}--\ref{ncv:eq:tpldg:init}) describes the target pledge for a single obstacle.
Let the side of a trace be $\mside\in\{L,R\}$ where $L=-1$ and $R=1$.
The trace can leave the contour at $\mx$ if
\begin{equation}
    \mside\theta_T \le 0
\end{equation}

\subsection{Corners in a Target-pledge Method}
By examining how the target pledge evolves as a trace walks along a contour, four types of corners can be identified.
The corners are identified based on their convexity, and whether passing through the corner will cause the angular direction of the trace to reverse.
The angular direction is examined from the target point, and a trace can be counter-clockwise or clockwise when viewed from the point.

The corners can be identified by considering the angular half the initial target-pledge $\theta_{T,0}$ lies in, and deriving the possible cases thereafter.
To simultaneously derive for an $L$-sided and $R$-sided trace, the generalized angle $\mside\theta$ is used, where $\mside=\{L,R\}$ and $L=-1$ and $R=1$.
By considering Eq. \ref{ncv:eq:tpldg:init}, $\mside\theta_{T,0} \in [0, \pi)$.
From the initial pledge, the trace walks to the first corner.
Let the first corner be at $\mx$, and the previous target-pledge be the initial pledge, such that $\mside\theta_T' = \mside\theta_{T,0}$. Suppose that the angular range of $\mside\theta_T'$ can be generalized into the first angular half (\textbf{first-half}), such that
\begin{equation}
    \pi k_T < \mside\theta_T' \le \pi(k_T+1), \qquad k_T \in \{\cdots,-2,0,2,4,\cdots\}.
    \label{ncv:eq:th1}
\end{equation}
From Eq. (\ref{ncv:eq:th1}), the angular range of $\mside\theta_T$ is derived.
As the initial edge faces away from the target point, the angular range of $\mside\vartheta_t$ for the initial edge can be derived as
\begin{equation}
    0 \le \mside\vartheta_t <  \pi.
    \label{ncv:eq:th1:tchange}
\end{equation}
After considering all cases,
Eq. (\ref{ncv:eq:th1:tchange}) is correct for all edges when $\mside\theta_T'$ lies in the first-half.

To determine the upper limit of $\mside\vartheta_t$, the initial edge can be extended to infinity from the collision point.
As such, $\mside\vartheta_t$ cannot be larger than $\pi - (\mside\theta_T' - \pi k_T)$ radians, and 
\begin{equation}
    \mside(\theta_T' + \vartheta_t) < \pi(k_T+1)
    \label{ncv:eq:th1:infedge}
\end{equation}
Adding the constraints in Eq. (\ref{ncv:eq:th1}) and (\ref{ncv:eq:th1:tchange}), and intersecting with Eq. (\ref{ncv:eq:th1:infedge}), the \textbf{first-half prior constraint} can be obtained where
\begin{equation}
    \pi k_T < \mside(\theta_T' + \vartheta_t) < \pi(k_T+1)
    \label{ncv:eq:th1:prior}
\end{equation}

Eq. (\ref{ncv:eq:th1} - \ref{ncv:eq:th1:prior}) are shown in Table \ref{ncv:tab:th1}.
The final range of $\mside\theta_T$ can be determined by considering the four types of corners.
In Table \ref{ncv:tab:th1},
the \textbf{convexity} constraint is the range of angles that are allowed for $\mside\theta_\medge$ depending on the convexity of the corner at $\mx$.
The \textbf{angular reversal} constraint is the range of angles allowed for $\mside\theta_\medge$ depending on whether the subsequent edge from $\mx$ causes the trace to reverse its angular direction when viewed from a target point.
By adding the convexity constraint to the first-half prior constraint, and intersecting the resulting range with the angular reversal constraint, the final constraint of $\mside\theta_T$ can be determined for each of the four cases.

\def\clipper{\clip (-\u, 0) rectangle ++(9*\ul, 8*\ul)}
\tikzset{
    pics/th1/.style n args={4}{ code={ 
        \clipper;
        \path (0*\ul, 4*\ul) coordinate (xprev);
        \path (4.5*\ul, 4*\ul) coordinate (x);
        \path (5*\ul, 7.5*\ul) coordinate (xt);

        \pgfmathsetmacro\angdiff{180-#1}
        \pgfmathsetmacro\anghinge{\angdiff/2 + #1}
        \path ($(x) + (\anghinge:{\ul/sin(\anghinge)})$) coordinate (xhinge);
        \fill [swatch_obs] (xprev) -- (x) -- ++(#1:8*\ul) coordinate (xnext) -- ($(xnext) + ({#1+90}:\ul)$) -- (xhinge) -- ($(xprev) + (135:{\ul/sin(135)})$) -- ++(225:2*\ul) -- ++(315:\ul) coordinate (xpprev);

        \draw [trace] (xpprev) -- (xprev) -- (x) -- ++(#1:\u);

        \draw [dotted] (xt) -- ($(xt)!3!(x)$) coordinate (tf);
        \draw [dotted, -{To[gray,reversed] To[gray,reversed]}] (xt) -- ($(xt)!0.5!(xprev)$) coordinate (tprevarr);
        \draw [dotted] (tprevarr) -- ($(xt)!3!(xprev)$) coordinate (tprevf);

        \path ($(x) - (xprev)$) coordinate (transt);
        \path ($(xt) + (transt)$) coordinate (tprev2i);
        \path ($(tprevf) + (transt)$) coordinate (tprev2f);
        \path ($(tprev2i)!0.5!(tprev2f)$) coordinate (tprev2arr);
        \draw [dotted, -{To[gray,reversed] To[gray,reversed]}] (tprev2i) -- (tprev2arr);
        \draw [dotted] (tprev2arr) -- (tprev2f);

        \draw [dotted] (x) -- ++(0:6*\ul) coordinate (ef);

        \node [black pt={shift={(#2)}}{center:$\mx_T$}] at (xt) {};
        \node [black pt={shift={(#3)}}{center:$\mx$}] at (x) {};
        \node [black pt={shift={(#4)}}{center:$\mx'$}] at (xprev) {};
    }},
}

\begin{table}[!ht]
\centering
\caption{Four corner cases when $\mside\theta_T'$ lies in first-half.}
\setlength{\tabcolsep}{3pt}
\renewcommand{\arraystretch}{1.5}
\begin{tabular}{ c c r p{6.5cm} }
\hline
\multicolumn{1}{r}{First-half:} & \multicolumn{3}{l}{$\pi k_T < \mside \theta_T' \le \pi (k_T+1) \text{ s.t. } k_T \in \{\cdots, -2,0,2,4,\cdots\}$} \\
\hline
\multicolumn{1}{r}{$\mside\vartheta_t$ Range:} & \multicolumn{3}{l}{$0 \le \mside\vartheta_t < \pi$} \\
\hline
\multicolumn{1}{r}{Edge Constraint:} & \multicolumn{3}{l}{$\pi - (\mside\theta_T' - \pi k_T) > \mside\vartheta_t$} \\
\hline
\multicolumn{1}{r}{First-half Prior:} & \multicolumn{3}{l}{$\pi k_T < \mside (\theta_T' + \vartheta_t) < \pi (k_T+1)$} \\
\hline
\hline

\multicolumn{2}{l}{
    \multirow{3}{*}{
        \begin{tikzpicture}[]
            \clipper;
            \pic at (0,0) {th1={60}{-4mm, 0}{-2mm, -2mm}{2mm, -2mm}};
            \pic [draw, ->, angle eccentricity=2, angle radius=2*\ul] {angle=ef--x--xnext};
            \node at ($(x) + (33:1.2*\ul)$) {\footnotesize $\mside\vartheta_\medge$};
            
            \pic [draw, ->, angle eccentricity=2, angle radius=2.5*\ul] { angle=ef--x--tprev2i};     
            \node [rotate=-70] at ($(x) + (5:3.3*\ul)$) {\footnotesize $\mside\theta_T'- \pi k_T$};
            
            \pic [draw, ->, angle eccentricity=2, angle radius=3*\ul] { angle=tprev2i--x--xt};
            \node at ($(x) + (60:3.5*\ul)$) {\footnotesize $\mside\vartheta_t$};
        \end{tikzpicture} 
    }} & \multicolumn{2}{c}{\textbf{Corner C1}} \\
\cline{3-4}
\multicolumn{2}{l}{} & Convex: &  $0 < \mside\vartheta_\medge < \pi$ \\
\cline{3-4}
\multicolumn{2}{l}{} & No Reversal: & 
    \marginbox{0 3pt 0 3pt}{$\begin{aligned}
        \mside(\hat{\theta}_\medge \times & \hat{\theta}_t) > 0 \\
        \implies & \mside\theta_T' - \pi k_T + \mside\vartheta_t > \mside\vartheta_\medge \\
        \implies & \pi k_T < \mside \theta_T
    \end{aligned}$} \\
\cline{3-4}
\multicolumn{2}{l}{} & Final: & $\pi k_T < \mside \theta_T < \pi(k_T+1)$ \\
\hline
\hline

\multicolumn{2}{l}{
    \multirow{3}{*}{
        \begin{tikzpicture}[]
            \clipper;
            \pic at (0,0) {th1={135}{-4mm, 0}{-2mm, -2mm}{2mm, -2mm}};
            \pic [draw, ->, angle eccentricity=2, angle radius=2*\ul] {angle=ef--x--xnext};
            \node at ($(x) + (60:1*\ul)$) {\footnotesize $\mside\vartheta_\medge$};
            
            \pic [draw, ->, angle eccentricity=2, angle radius=2.5*\ul] { angle=ef--x--tprev2i};     
            \node [rotate=-70] at ($(x) + (5:3.3*\ul)$) {\footnotesize $\mside\theta_T'- \pi k_T$};
            
            \pic [draw, ->, angle eccentricity=2, angle radius=3*\ul] { angle=tprev2i--x--xt};
            \node at ($(x) + (60:3.5*\ul)$) {\footnotesize $\mside\vartheta_t$};
        \end{tikzpicture} 
    }} & \multicolumn{2}{c}{\textbf{Corner C2}} \\
\cline{3-4}
\multicolumn{2}{l}{} & Convex: &  $0 < \mside\vartheta_\medge < \pi$ \\
\cline{3-4}
\multicolumn{2}{l}{} & Reversal: & 
    \marginbox{0 3pt 0 3pt}{$\begin{aligned}
        \mside(\hat{\theta}_\medge \times & \hat{\theta}_t) \le 0 \\
        \implies & \mside\theta_T' - \pi k_T + \mside\vartheta_t \le \mside\vartheta_\medge \\
        \implies & \mside \theta_T \le \pi k_T
    \end{aligned}$} \\
\cline{3-4}
\multicolumn{2}{l}{} & Final: & $\pi (k_T-1) < \mside \theta_T \le \pi k_T$ \\
\hline
\hline

\multicolumn{2}{l}{
    \multirow{3}{*}{
        \begin{tikzpicture}[]
            \clipper;
            \pic at (0,0) {th1={-45}{-4mm, 0}{-2mm, 2mm}{0mm, 3mm}};
            \pic [draw, <-, angle eccentricity=2, angle radius=1.5*\ul] {angle=xnext--x--ef};
            \node at ($(x) + (-21:2.2*\ul)$) {\footnotesize $\mside\vartheta_\medge$};
            
            \pic [draw, ->, angle eccentricity=2, angle radius=1.5*\ul] { angle=xprev--x--tprev2f};     
            \node [rotate=0] at ($(x) + (192:3.1*\ul)$) {\footnotesize $\mside\theta_T'- \pi k_T$};
            
            \pic [draw, ->, angle eccentricity=2, angle radius=2*\ul] { angle=tprev2f--x--tf};
            \node at ($(x) + (235:2.5*\ul)$) {\footnotesize $\mside\vartheta_t$};
        \end{tikzpicture} 
    }} & \multicolumn{2}{c}{\textbf{Corner C3}} \\
\cline{3-4}
\multicolumn{2}{l}{} & Non-convex: &  $-\pi < \mside\vartheta_\medge < 0$ \\
\cline{3-4}
\multicolumn{2}{l}{} & No Reversal: & 
    \marginbox{0 3pt 0 3pt}{$\begin{aligned}
        \mside(\hat{\theta}_\medge \times & \hat{\theta}_t) \ge 0 \\
        \implies & \pi - (\mside\theta_T' - \pi k_T + \mside\vartheta_t) \ge -\mside\vartheta_\medge \\
        \implies & \mside \theta_T \le \pi (k_T+1)
    \end{aligned}$} \\
\cline{3-4}
\multicolumn{2}{l}{} & Final: & $\pi k_T < \mside \theta_T \le \pi(k_T+1)$ \\
\hline
\hline

\multicolumn{2}{l}{
    \multirow{3}{*}{
        \begin{tikzpicture}[]
            \clipper;
            \pic at (0,0) {th1={240}{-4mm, 0}{-2mm, 2mm}{0mm, 3mm}};
            \pic [draw, <-, angle eccentricity=2, angle radius=1.5*\ul] {angle=xnext--x--ef};
            \node at ($(x) + (-21:2.2*\ul)$) {\footnotesize $\mside\vartheta_\medge$};
            
            \pic [draw, ->, angle eccentricity=2, angle radius=1.5*\ul] { angle=xprev--x--tprev2f};     
            \node [rotate=0] at ($(x) + (192:3.1*\ul)$) {\footnotesize $\mside\theta_T'- \pi k_T$};
            
            \pic [draw, ->, angle eccentricity=2, angle radius=2*\ul] { angle=tprev2f--x--tf};
            \node at ($(x) + (235:2.5*\ul)$) {\footnotesize $\mside\vartheta_t$};
        \end{tikzpicture} 
    }} & \multicolumn{2}{c}{\textbf{Corner C4}} \\
\cline{3-4}
\multicolumn{2}{l}{} & Non-convex: &  $-\pi < \mside\vartheta_\medge < 0$ \\
\cline{3-4}
\multicolumn{2}{l}{} & Reversal: & 
    \marginbox{0 3pt 0 3pt}{$\begin{aligned}
        \mside(\hat{\theta}_\medge \times & \hat{\theta}_t) < 0 \\
        \implies & \pi - (\mside\theta_T' - \pi k_T + \mside\vartheta_t) < -\mside\vartheta_\medge \\
        \implies & \pi (k_T+1) < \mside\theta_T
    \end{aligned}$} \\
\cline{3-4}
\multicolumn{2}{l}{} & Final: & $\pi (k_T+1) < \mside \theta_T < \pi(k_T+2)$ \\
\hline
% \multicolumn{3}{p{\textwidth}}{
% \footnotesize
% For each of the four cases, the final constraint of $\mside\theta_T$ is obtained by subtracting the convexity constraint from the first-half prior, and intersecting the result with the progression constraint.
% }
\end{tabular}
\label{ncv:tab:th1}
\end{table}
\def\clipper{\clip (-\u, 0) rectangle ++(9*\ul, 8*\ul)}
\tikzset{
    pics/th2/.style n args={4}{ code={ 
        \clipper;
        \path (8*\ul, 4*\ul) coordinate (xprev);
        \path (4*\ul, 4*\ul) coordinate (x);
        \path (3.5*\ul, 7.5*\ul) coordinate (xt);

        \pgfmathsetmacro\angdiff{360-#1} % accepts 1 0 to 360
        \pgfmathsetmacro\anghinge{\angdiff/2 + #1}
        \path ($(x) + (\anghinge:{\ul/sin(\anghinge-180)})$) coordinate (xhinge);
        \fill [swatch_obs] (xprev) -- (x) -- ++(#1:8*\ul) coordinate (xnext) -- ($(xnext) + ({#1+90}:\ul)$) -- (xhinge) -- ($(xprev) + (-45:{\ul/sin(45)})$) -- ++(45:2*\ul) -- ++(135:\ul) coordinate (xpprev);

        \draw [trace] (xpprev) -- (xprev) -- (x) -- ++(#1:\u);

        \draw [dotted] (xt) -- ($(xt)!3!(x)$) coordinate (tf);
        \draw [dotted, -{To[gray, reversed] To[gray, reversed]}] (xt) -- ($(xt)!0.8!(xprev)$) coordinate (tprevarr);
        \draw [dotted] (tprevarr) -- ($(xt)!3!(xprev)$) coordinate (tprevf);

        \path ($(x) - (xprev)$) coordinate (transt);
        \path ($(xt) + (transt)$) coordinate (tprev2i);
        \path ($(tprevf) + (transt)$) coordinate (tprev2f);
        \path ($(tprev2i)!0.6!(tprev2f)$) coordinate (tprev2arr);
        \draw [dotted, -{To[gray, reversed] To[gray, reversed]}] (tprev2i) -- (tprev2arr);
        \draw [dotted] (tprev2arr) -- (tprev2f);

        \draw [dotted] (x) -- ++(180:6*\ul) coordinate (ef);

        \node [black pt={shift={(#2)}}{center:$\mx_T$}] at (xt) {};
        \node [black pt={shift={(#3)}}{center:$\mx$}] at (x) {};
        \node [black pt={shift={(#4)}}{center:$\mx'$}] at (xprev) {};
    }},
}

\begin{table}[!ht]
\centering
\caption{Four corner cases when $\mside\theta_T'$ lies in second-half.}
\setlength{\tabcolsep}{3pt}
\renewcommand{\arraystretch}{1.5}
\begin{tabular}{ c c r p{6.5cm} }
\hline
\multicolumn{1}{r}{Second-half:} & \multicolumn{3}{l}{$\pi k_T < \mside \theta_T' \le \pi (k_T+1) \text{ s.t. } k_T \in \{\cdots, -3,-1,1,3,\cdots\}$} \\
\hline
\multicolumn{1}{r}{$\mside\vartheta_t$ Range:} & \multicolumn{3}{l}{$-\pi \le \mside\vartheta_t \le 0$} \\
\hline
\multicolumn{1}{r}{Edge Constraint:} & \multicolumn{3}{l}{$\mside\theta_T' - \pi k_T > -\mside\vartheta_t$} \\
\hline
\multicolumn{1}{r}{Second-half Prior:} & \multicolumn{3}{l}{$\pi k_T < \mside (\theta_T' + \vartheta_t) \le \pi (k_T+1)$} \\
\hline
\hline

\multicolumn{2}{l}{
    \multirow{3}{*}{
        \begin{tikzpicture}[]
            \clipper;
            \pic at (0,0) {th2={225}{-4mm, 0}{2mm, 2mm}{0mm, 3mm}};
            \pic [draw, ->, angle eccentricity=2, angle radius=2*\ul] {angle=ef--x--xnext};
            \node [rotate=0] at ($(x) + (210:1.2*\ul)$) {\footnotesize $\mside\vartheta_\medge$};
            
            \pic [draw, ->, angle eccentricity=2, angle radius=2.5*\ul] {angle=ef--x--tprev2f};     
            \node [rotate=-45] at ($(x) + (-135:3*\ul)$) {\footnotesize $\mside\theta_T'- \pi k_T$};
            
            \pic [draw, <-, angle eccentricity=2, angle radius=3*\ul] { angle=tf--x--tprev2f};
            \node  at ($(x) + (-60:3.5*\ul)$) {\footnotesize $\mside\vartheta_t$};
        \end{tikzpicture} 
    }} & \multicolumn{2}{c}{\textbf{Corner C1}} \\
\cline{3-4}
\multicolumn{2}{l}{} & Convex: &  $0 < \mside\vartheta_\medge < \pi$ \\
\cline{3-4}
\multicolumn{2}{l}{} & No Reversal: & 
    \marginbox{0 3pt 0 3pt}{$\begin{aligned}
        \mside(\hat{\theta}_\medge \times & \hat{\theta}_t) < 0 \\
        \implies & \mside\theta_T' - \pi k_T + \mside\vartheta_t > \mside\vartheta_\medge \\
        \implies & \pi k_T < \mside \theta_T
    \end{aligned}$} \\
\cline{3-4}
\multicolumn{2}{l}{} & Final: & $\pi k_T < \mside \theta_T < \pi(k_T+1)$ \\
\hline
\hline

\multicolumn{2}{l}{
    \multirow{3}{*}{
        \begin{tikzpicture}[]
            \clipper;
            \pic at (0,0) {th2={300}{-4mm, 0}{2mm, 2mm}{0mm, 3mm}};
            \pic [draw, ->, angle eccentricity=2, angle radius=2*\ul] {angle=ef--x--xnext};
            \node [rotate=0] at ($(x) + (-120:1.2*\ul)$) {\footnotesize $\mside\vartheta_\medge$};
            
            \pic [draw, ->, angle eccentricity=2, angle radius=2.5*\ul] {angle=ef--x--tprev2f};     
            \node [rotate=-45] at ($(x) + (-135:3*\ul)$) {\footnotesize $\mside\theta_T'- \pi k_T$};
            
            \pic [draw, <-, angle eccentricity=2, angle radius=3*\ul] { angle=tf--x--tprev2f};
            \node  at ($(x) + (-60:3.5*\ul)$) {\footnotesize $\mside\vartheta_t$};
        \end{tikzpicture} 
    }} & \multicolumn{2}{c}{\textbf{Corner C2}} \\
\cline{3-4}
\multicolumn{2}{l}{} & Convex: &  $0 < \mside\vartheta_\medge < \pi$ \\
\cline{3-4}
\multicolumn{2}{l}{} & Reversal: & 
    \marginbox{0 3pt 0 3pt}{$\begin{aligned}
        \mside(\hat{\theta}_\medge \times & \hat{\theta}_t) \ge 0 \\
        \implies & \mside\theta_T' - \pi k_T + \mside\vartheta_t \le \mside\vartheta_\medge \\
        \implies & \mside \theta_T \le \pi k_T
    \end{aligned}$} \\
\cline{3-4}
\multicolumn{2}{l}{} & Final: & $\pi (k_T-1) < \mside \theta_T \le \pi k_T$ \\
\hline
\hline

\multicolumn{2}{l}{
    \multirow{3}{*}{
        \begin{tikzpicture}[]
            \clipper;
            \pic at (0,0) {th2={120}{4mm, 0}{2mm, -2mm}{0mm, -3mm}};
            \pic [draw, <-, angle eccentricity=2, angle radius=2*\ul] {angle=xnext--x--ef};
            \node [rotate=0] at ($(x) + (150:1.2*\ul)$) {\footnotesize $\mside\vartheta_\medge$};
            
            \pic [draw, ->, angle eccentricity=2, angle radius=2.5*\ul] {angle=xprev--x--tprev2i};     
            \node [rotate=-30] at ($(x) + (40:3.2*\ul)$) {\footnotesize $\mside\theta_T'- \pi k_T$};
            
            \pic [draw, <-, angle eccentricity=2, angle radius=3*\ul] { angle=xt--x--tprev2i};
            \node  at ($(x) + (120:3.7*\ul)$) {\footnotesize $\mside\vartheta_t$};
        \end{tikzpicture} 
    }} & \multicolumn{2}{c}{\textbf{Corner C3}} \\
\cline{3-4}
\multicolumn{2}{l}{} & Non-convex: &  $-\pi < \mside\vartheta_\medge < 0$ \\
\cline{3-4}
\multicolumn{2}{l}{} & No Reversal: & 
    \marginbox{0 3pt 0 3pt}{$\begin{aligned}
        \mside(\hat{\theta}_\medge \times & \hat{\theta}_t) \le 0 \\
        \implies & \pi - (\mside\theta_T' - \pi k_T + \mside\vartheta_t) \ge -\mside\vartheta_\medge \\
        \implies & \mside \theta_T \le \pi (k_T+1)
    \end{aligned}$} \\
\cline{3-4}
\multicolumn{2}{l}{} & Final: & $\pi k_T < \mside \theta_T \le \pi(k_T+1)$ \\
\hline
\hline

\multicolumn{2}{l}{
    \multirow{3}{*}{
        \begin{tikzpicture}[]
            \clipper;
            \pic at (0,0) {th2={45}{4mm, 0}{2mm, -2mm}{0mm, -3mm}};
            \pic [draw, <-, angle eccentricity=2, angle radius=2*\ul] {angle=xnext--x--ef};
            \node [rotate=0] at ($(x) + (120:1.2*\ul)$) {\footnotesize $\mside\vartheta_\medge$};
            
            \pic [draw, ->, angle eccentricity=2, angle radius=2.5*\ul] {angle=xprev--x--tprev2i};     
            \node [rotate=-30] at ($(x) + (40:3.2*\ul)$) {\footnotesize $\mside\theta_T'- \pi k_T$};
            
            \pic [draw, <-, angle eccentricity=2, angle radius=3*\ul] { angle=xt--x--tprev2i};
            \node  at ($(x) + (120:3.7*\ul)$) {\footnotesize $\mside\vartheta_t$};
        \end{tikzpicture} 
    }} & \multicolumn{2}{c}{\textbf{Corner C4}} \\
\cline{3-4}
\multicolumn{2}{l}{} & Non-convex: &  $-\pi < \mside\vartheta_\medge < 0$ \\
\cline{3-4}
\multicolumn{2}{l}{} & Reversal: & 
    \marginbox{0 3pt 0 3pt}{$\begin{aligned}
        \mside(\hat{\theta}_\medge \times & \hat{\theta}_t) > 0 \\
        \implies & \pi - (\mside\theta_T' - \pi k_T + \mside\vartheta_t) < -\mside\vartheta_\medge \\
        \implies & \pi (k_T+1) < \mside\theta_T
    \end{aligned}$} \\
\cline{3-4}
\multicolumn{2}{l}{} & Final: & $\pi (k_T+1) < \mside \theta_T < \pi(k_T+2)$ \\
\hline
% \multicolumn{3}{p{\textwidth}}{
% \footnotesize
% For each of the four cases, the final constraint of $\mside\theta_T$ is obtained by subtracting the convexity constraint from the first-half prior, and intersecting the result with the progression constraint.
% }
\end{tabular}
\label{ncv:tab:th2}
\end{table}

From Table \ref{ncv:tab:th1}, 
if a subsequent edge causes a reversal in the trace's angular direction, the target-pledge moves to another half angular range.
Suppose that the trace has moved to the next corner for such a case, and $\mx$ is now the next corner. $\mside\theta_T$ becomes the new $\mside\theta_T'$, which resides in the second angular half (\textbf{second-half}) where
\begin{equation}
    \pi k_T < \mside\theta_T' \le \pi (k_T+1), \qquad k_T \in \{\cdots,-3,-1,1,3,\cdots\}.
    \label{ncv:eq:th2}
\end{equation}
As the angular direction has reversed over the previous corner, the range of $\vartheta_t$ has to be 
\begin{equation}
    -\pi \le \mside\vartheta_t < 0.
    \label{ncv:eq:th2:tchange}
\end{equation}
By considering the limit when the subsequent edge extends to infinity, the second-half prior constraint can be derived as
\begin{equation}
    \pi k_T < \mside(\theta_T' + \vartheta_t) \le \pi (k_T+1)
    \label{ncv:eq:th2:prior}
\end{equation}

By considering the four types of corners and deriving in the same way as the first-half, the final angular ranges of $\mside\theta_T$ can be found.
The ranges are listed in Table \ref{ncv:tab:th2}.

From Tables \ref{ncv:tab:th1} and \ref{ncv:tab:th2}, the cases can be found to lead to each other.
Let $k_T$ be the target-pledge \textbf{winding counter}, which monitors the angular half the target-pledge lies in.
When $k_T$ is even, the target-pledge lies in the first-half. If $k_T$ is odd, the target-pledge lies in the second-half.
Corners C1 and C3 does not change $k_T$, while convex corner C2 causes $k_T$ to unwind, and non-convex corner C4 causes $k_T$ to wind.

\subsection{Casting From a Trace}
The target-pledge method generates two traces when a cast to the target point collides.
When a trace begins, the target-pledge winding counter $k_T$ is zero, such that $0 < \mside\theta_T < \pi$.
Consider the simplest obstacle with only two C2 corners (see Fig. \ref{ncv:fig:tproof:a}), where a trace encounters one of the C2 corners.
For the algorithm to be complete and reach the target point, a cast has to occur at the corner.
At this corner, the condition
\begin{equation}
    \mside\theta_T \le 0
    \label{ncv:eq:tsuccess}
\end{equation}
is satisfied, and $k_T$ unwinds from $0$ to $-1$.
Since the obstacle is the simplest obstacle, the condition to cast when Eq. (\ref{ncv:eq:tsuccess}) is met is \textit{necessary} for the target-pledge method to be complete.

\subsection{Proof of Completeness}
We now show that Eq. \ref{ncv:eq:tsuccess}  is sufficient for the target-pledge method to be complete.
\begin{theorem}
    In an unbounded map, or bounded map with a convex boundary, the target-pledge method can find a path to the target point if a path exists, provided that a cast occurs from a trace at the first corner that satisfies  $\mside\theta_T \le 0$. All traces and casts have to be simultaneously examined.
    \label{ncv:thm:tpldg}
\end{theorem}
\input{chap_ncv/fig_tproof}
\begin{proof}
The target-pledge method  generates a trace on each side of a collided cast.
Suppose that the map contains only the simplest obstacle in Fig. \ref{ncv:fig:tproof}.
Two traces would occur on the obstacle, one on each side.
A cast occurs from a trace when the condition in Eq. \ref{ncv:eq:tsuccess} is met at $\mx_a$ in Fig. \ref{ncv:fig:tproof:a}.
In \textbf{Case 1.1}, a straight-line obstacle is extended to contain at most two C2 corners and no C4 corners. If the obstacle does not enclose the starting or target points, a C2 corner (at $\mx_a$ in Fig. \ref{ncv:fig:tproof}) can be found and the algorithm is complete.

In \textbf{Case 1.2} (see Fig. \ref{ncv:fig:tproof:b}), consider the trace that leads to $\mx_a$, and suppose that the obstacle is extruded between the initial point and $\mx_a$. The extrusion causes C1 and C3 corners to be introduced.
The extrusion may cause $k_T$ to be wound to $k_T \ge 1$ along a non-convex corner. 
To cast, it is a necessary to  unwound $k_T$ to $-1$, or else it will not be able to escape a `G' shaped extrusion , or any highly non-convex extrusion, before being able to cast from $\mx_a$.

In \textbf{Case 1.3} (see Fig. \ref{ncv:fig:tproof:c} and \ref{ncv:fig:tproof:d}),
suppose that $k_T$ is unwound to $-1$ along the non-convex extrusion, causing a cast to occur at $\mx_b$ and before the trace reaches the initial non-extruded part of the obstacle.
When the cast collides with the same obstacle, the trace would proceed as if a part or all of the extrusion never existed, and the trace continues to cast from $\mx_a$. 
This is due to $k_T$ being zero when a new trace occurs from the newly collided edge, and $k_T$ being zero if a cast never occurs and the original trace continues to the same edge.
$k_T$ would be zero for the continued trace regardless of how the contour is extruded after $\mx_b$ and before the edge, provided that any extrusion does not intersect the new cast.
As such, the trace would reach $\mx_a$, and the algorithm is complete.

In \textbf{Case 2} (see Fig. \ref{ncv:fig:tproof:e} and \ref{ncv:fig:tproof:f}), suppose that a cast from Cases 1.1 to 1.3 collides with a contour that does not belong to the same obstacle.
The new traces along the new obstacle can be treated with Cases 1.1 to 1.3, and there must exist a trace that can eventually cast to the target point.
The only way where a trace will not cast is if the trace reaches the interior boundary of an obstacle enclosing the target point and the casts (Fig. \ref{ncv:fig:tproof:g}), there is no path (Fig. \ref{ncv:fig:tproof:h}), or if the trace stops at the map boundary. 
Since the target point is reachable, the enclosing obstacle cannot separate the casts and the target point, and there cannot be an obstacle where both traces would arrive at a convex map boundary.
As such, at least one trace will be able to cast to and reach the target point.

A trace that reaches the interior boundary of an obstacle enclosing the target point and casts will not terminate. As such, it is necessary for all casts and traces to be simultaneously examined by the algorithm in order for the algorithm to terminate.
This can be done by using a first-in-first-out queue that queues a trace at every corner.
The algorithm maybe interminable if no path can be found.
\end{proof}

% \input{chap_ncv/fig_notopt}
% Note that, Theorem \ref{ncv:thm:tpldg} does not imply completeness in the optimal sense.
% While a path would be found by the target-pledge method, none of the cases considered by the theorem may lead to the true shortest path. 
% An example is given in Fig. \ref{ncv:fig:notopt}.
% The proof of completeness in the optimal sense depends on the algorithm incorporating the target-pledge method.

\subsection{Pledge Update After Pruning}
In a vector-based algorithm that delays \gls{los} checks, a search may have to re-examine the part of a path that is closer to the start point. 
As such, the target point may not be the goal point. For example, the target point may be part of a path $(\cdots, \mx, \mx_T, \mx_{TT}, \cdots, \mx_G)$,
where $\mx_G$ is the goal point.

To keep the path taut and admissible, $\mx_T$ may be pruned during a trace, at the first traced corner $\mx$ where the path stops being taut around $\mx_T$.
When a prune occurs, the path becomes $(\cdots, \mx, \mx_{TT}, \cdots, \mx_G)$, and the new target point becomes $\mx_{TT}$.
As the target-pledge is defined with respect to $\mx_T$, it has to be re-defined with respect to $\mx_{TT}$.
The new target-pledge at $\mx$ is
\begin{equation}
    \theta_{TT} = \theta_T + \hat{\theta}_{tt} - \hat{\theta}_t,
    \label{ncv:eq:tpldg:prune}
\end{equation}
where $\theta_T$ is the target-pledge as if the target point at $\mx_T$ is not pruned, and $\hat{\theta}_{tt} = \fatantwo(\mx_{TT} - \mx)$. 
$\theta_{TT}$ would be reused as $\theta_T'$ at the next corner traced.

\input{chap_ncv/fig_tprune}
\begin{lemma}
    Suppose a trace reaches a corner at $\mx$ with an expanded path 
    \linebreak
    $(\cdots,\mx,\mx_T,\mx_{TT},\cdots)$.
    The path was taut for all previous corners walked by the trace, and stops being taut around the segment $(\mx,\mx_T,\mx_{TT})$ when it reaches $\mx$.
    The target-pledge can be correctly updated with Eq. (\ref{ncv:eq:tpldg:prune}).
    \label{ncv:lem:tprune}
\end{lemma}
\begin{proof}
To prove Eq. (\ref{ncv:eq:tpldg:prune}), pruning is first shown to occur only when the winding counter $k_T=0$.
As a trace casts when $k_T=-1$, $k_T \ge 0$ in a trace.
At the collision point, $k_T=0$.
Suppose that the angular deviation is zero at the collision point when viewed from the target point.
For every subsequent, consecutive corner where $k_T=0$, the angular deviation of the corner increases.
When the trace first arrives at a non-convex corner where $k_T$ winds to one ($\mx_{\max}$ in Fig. \ref{ncv:fig:tprune}), the trace will be at a local maximum deviation. 
For non-intersecting obstacles, every subsequent corner where $k_T > 0$ has to lie at a smaller angular deviation than the local maximum.
When $k_T > 0$, the trace has entered a non-convex extrusion like Case 1.2 in Theorem \ref{ncv:thm:tpldg}.
For the angular deviation to increase again, $k_T$ would have to unwind back to zero.

A prune occurs only when the line colinear to $\mx_T$ and $\mx_{TT}$ is crossed by a trace.
Assuming that the path is taut at the collision point, a trace has to deviate far enough from the collision point to cross the line.
Since the angular deviation can only increase when $k_T=0$, a prune can only occur when $k_T=0$.

Consider the edge immediately before reaching $\mx$. Let $\mx_{\mitx}$ be the intersection of the edge with the line colinear to $\mx_T$ and $\mx_{TT}$.
As $k_T=0$ at a collision point, a new cast that tries to reach $\mx_{TT}$ can be assumed to have collided at $\mx_{\mitx}$.
Since $\mx_{\mitx}$, $\mx_T$, and $\mx_{TT}$ are colinear, the target-pledge at $\mx_{\mitx}$ for the new trace has to be the same as the target-pledge for the old trace.

Let $\theta_{TT}$ be the target-pledge at $\mx$ for the new trace, and $\theta_T$ be the target-pledge for the old trace. 
Since both target-pledges are affected by the same $\vartheta_\medge$ at $\mx$, and that $k_T=0$ for both target-pledges at $\mx_{\mitx}$, the difference only lies in the directions of their target points.
As $\mx$ has a higher angular deviation than $\mx$, and that $\mx_{\mitx}$, $\mx_T$, and $\mx_{TT}$ are colinear, the angular difference between $\hat{\theta}_{tt} - \hat{\theta}_t$ cannot exceed $\pi$ radians.
Therefore, Eq. (\ref{ncv:eq:tpldg:prune}) is correct.
\end{proof}

\subsection{Target-pledge Angular Discretization} \label{ncv:sec:tpldg:discrete}
In an occupancy grid, the cardinal directions south, east, north, west correspond to the headings $-\pi$, $-\pi/2$, $0$, and $\pi/2$ radians, respectively.
As an obstacle's edge is parallel to the cardinal directions, the ordinal directions can be defined as the \textit{angular ranges} between the cardinal directions.
As such, the angles in the target-pledge can be discretized with 
\begin{equation}
    \mathrm{z}(\theta) = \begin{cases}
    \frac{4}{\pi} \llfloor \theta \rrfloor + 4 & \text{if } \llfloor \theta \rrfloor \in \{-\frac{\pi}{2}, 0, \frac{\pi}{2}, \pi\} \\
        2\left\lfloor \frac{2}{\pi}\llfloor \theta \rrfloor \right\rfloor + 5 & \text{otherwise}
    \end{cases}.
    \label{ncv:eq:tpldg:discrete}
\end{equation}
\input{chap_ncv/fig_tdis}

In Eq. (\ref{ncv:eq:tpldg:discrete}), south, east, north, and west, are discretized to 0, 2, 4, and 6, respectively (see Fig \ref{ncv:fig:tdis}).
The angular ranges between and not including south and east, east and north, north and west, west and south, are discretized to 1, 3, 5, and 7, respectively.
The discretizer has to uniquely assign values to the cardinal directions, and the values must be monotonically increasing or decreasing as the angular parameter rotates one round from a cardinal direction.

Discretizing $\hat{\theta}_t$ in Eq. (\ref{ncv:eq:tpldg:t}) and $\hat{\theta}_\medge$ in Eq. (\ref{ncv:eq:tpldg:e})  removes the need to calculate the computationally expensive $\fatantwo$ function.
To avoid double counting angles within an angular range,
Eq. (\ref{ncv:eq:tpldg:tchange}) and Eq. (\ref{ncv:eq:tpldg:echange}) has to be discretized to
\begin{align}
    \mathrm{z}(\vartheta_t) &= \mathrm{z}(\hat{\theta}_t) - \mathrm{z}(\hat{\theta}_t') \\
    \mathrm{z}(\vartheta_\medge) &= \mathrm{z}(\hat{\theta}_\medge) - \mathrm{z}(\hat{\theta}_\medge').
\end{align}
The initial target-pledge at the collision point $\mx_0$ is
\begin{equation}
    \mathrm{z}(\theta_{T,0}) = \mathrm{z}(\mx_T - \mx_0) - \mathrm{z}(\hat{\theta}_{\medge,0}),
\end{equation}
where $\mathrm{z}(\mx_T - \mx_0)$ is the discrete heading of $\mx_T$ from $\mx_0$.
The update equation from Eq. (\ref{ncv:eq:tpldg}) is adjusted to
\begin{equation}
    \mathrm{z}(\theta_T) = \mathrm{z}(\theta_T') + \mathrm{z}(\vartheta_t) - \mathrm{z}(\vartheta_\medge),
\end{equation}
and an $\mside$-sided trace can leave the contour and cast to the target point at $\mx_T$ if
\begin{equation}
    \mside\mathrm{z}(\theta_T) < 0.
\end{equation}
When the target point at $\mx_T$ is pruned, Eq. (\ref{ncv:eq:tpldg:prune}) can be adjusted to
\begin{equation}
    \mathrm{z}(\theta_{TT}) = \mathrm{z}(\theta_T) + \mathrm{z}(\hat{\theta}_{tt}) - \mathrm{z}(\hat{\theta}_t).
\end{equation}

\clearpage
\section{Source-pledge Method}
The source-pledge method is a repurposed target-pledge method that is used for placing turning points, by examining the heading with respect to a source point.
A source point leads to the start point of a query along a path.

The algorithm prevents turning points from being placed within the convex hull of an obstacle.
Points that are placed within an obstacle's convex hull will not lead to the shortest path, unless there is another obstacle that lies partially within the convex hull.
Finding the other obstacle is not the concern of the target-pledge method, but the planner that utilizes the algorithm.
As such, it is sufficient for the algorithm to consider only the traced obstacle when placing a turning point.

\subsection{Source-pledge Update Equations}
Let $\hat{\theta}_s$ be the angle the vector $\mora{\mx_S\mx}$ makes with the positive $x$-axis, where $\mx$ is the current traced position and $\mx_S$ is the source point's position.
The source point is a turning point that leads to the start point, and can be the start point.
$\hat{\theta}_\medge$ is the angle that the next traced direction makes with the positive $x$-axis.
The angles are
\begin{align}
    \hat{\theta}_s &= \fatantwo(\mx - \mx_S) \label{ncv:eq:spldg:s} \\
    \hat{\theta}_\medge &= \fatantwo(\mvnext).
    \label{ncv:eq:spldg:e}
\end{align}
The change in angles are
\begin{align}
    \vartheta_s &= \llfloor \hat{\theta}_s - \hat{\theta}_s' \rrfloor 
    \label{ncv:eq:spldg:schange} \\
    \vartheta_\medge &= \llfloor \hat{\theta}_\medge - \hat{\theta}_\medge' \rrfloor
    \label{ncv:eq:spldg:echange}
\end{align}
where $\hat{\theta}_s'$ and $\hat{\theta}_\medge'$ are the angles defined at the previous trace position for Eq. (\ref{ncv:eq:spldg:s}) and (\ref{ncv:eq:spldg:e}) respectively.
The \textbf{source-pledge} at $\mx$ is
\begin{equation}
    \theta_S = \theta_S' + \vartheta_s - \vartheta_\medge,
    \label{ncv:eq:spldg}
\end{equation}
where $\theta_S'$ is the source-pledge at the previous traced position.
The trace begins at $\mx_0$ where a cast collides with an obstacle, and the initial source-pledge is
\begin{equation}
    \theta_{S,0} = \llfloor \fatantwo(\mx_0 - \mx_S) - \hat{\theta}_{\medge,0} \rrfloor
    \label{ncv:eq:spldg:init}
\end{equation}
where $\hat{\theta}_{\medge,0}$ is the direction of the trace from the collision point. A point can be placed at $\mx$  for an $\mside$-sided trace if
\begin{equation}
    \mside\theta_S < 0,
    \label{ncv:eq:spldg:place}
\end{equation}
where $\mside\in\{L,R\}$ is the side of the trace, and $L=-1$ and $R=1$.

\subsection{Corners in a Source-pledge Method} \label{ncv:sec:spldg:corners}
As a trace walks along a contour, the source-pledge will fall into two angular-half ranges, like the target-pledge.
The \textbf{source-pledge winding counter} $k_S$ determines the angular-half which the source-pledge lies in.
Like the target-pledge, four types of corners can be derived.

Consider the generalized source-pledge $\mside\theta_S$ for an $\mside$-sided trace, where $\mside\in\{L,R\}$ and $L=-1$ and $R=1$.
Let the first corner be at $\mx$.
As a collision can only occur on an edge facing the source point, $\mside\theta_S'$ has to lie in the angular range $[0, \pi)$ radians, or in the \textbf{first-half} angular range. 
In general, if $\theta_S'$ lies in the first-half,
\begin{equation}
    \pi k_S \le \mside\theta_S' < \pi(k_S + 1),\qquad k_S \in \{\cdots,-2,0,2,4,\cdots\}.
    \label{ncv:eq:firsthalf}
\end{equation}
The range of $\mside\vartheta_s$ for the initial edge is $-\pi \le \mside\vartheta_s < 0$.
By considering subsequent edges where the source point can intersect the edge, the range can be generalized to
\begin{equation}
    -\pi \le \mside\vartheta_s \le 0.
    \label{ncv:eq:sh1:schange}
\end{equation}
Since the initial edge is a straight line, the lower bounds of $\mside\vartheta_s$ can be determined by extending the edge to infinity. As such,
\begin{equation}
    \mside(\theta_S' + \vartheta_s) > \pi k_S.
    \label{ncv:eq:sh1:smax}
\end{equation}
Adding Eq. (\ref{ncv:eq:firsthalf}) is added to (\ref{ncv:eq:sh1:schange}). 
The resulting range is intersected with Eq. (\ref{ncv:eq:sh1:smax}) to obtain the \textbf{first-half prior constraint} where,
\begin{equation}
    \pi k_S < \mside(\theta_S' + \vartheta_s) < \pi(k_S+1).
    \label{ncv:eq:sh1:prior}
\end{equation}

Four cases can occur, depending on the convexity of the corner at $\mx$, and whether the next edge causes a change in angular direction when viewed from the source point. 
The cases are described in Table \ref{ncv:tab:sh1}, and the final constraints of $\mside\theta_S$ are shown.

\input{chap_ncv/tab_sh1}

From Table \ref{ncv:tab:sh1}, 
if the subsequent edge does not cause a reversal in angular direction, the source pledge remains in the first-half, and the subsequent edge faces the source point.
If the angular direction reverses, the source pledge moves into the second-half, and the subsequent edge faces away from the source point.

Suppose that the angular direction reverses, and the trace proceeds to the subsequent corner. $\mside\theta_S'$ will now lie in the \textbf{second-half} angular range such that
\begin{equation}
    \pi k_S \le \mside\theta_S' < \pi (k_S+1), \qquad k_S \in \{\cdots,-3,-1,1,3,\cdots\}.
    \label{ncv:eq:sh2}
\end{equation}
Deriving in the same way as the first-half prior constraint, the \textbf{second-half prior constraint} is
\begin{equation}
    \pi k_S \le \mside(\theta_S' + \vartheta_s) < \pi (k_S+1).
    \label{ncv:eq:sh2:prior}
\end{equation}
Listing the same four cases as Table \ref{ncv:tab:sh1}, the final constraints of $\mside\theta_S$ are shown in Table \ref{ncv:tab:sh2}.
\input{chap_ncv/tab_sh2}

The cases in Tables \ref{ncv:tab:sh1} and \ref{ncv:tab:sh2} lead to each other, and no other cases exist.
Like the target pledge, the source-pledge shifts from one angular-half to another if the subsequent edge from a corner causes a reversal in the trace's angular direction.

\subsection{Turning Point Placement}
A turning point can be placed if the source-pledge satisfies Eq. (\ref{ncv:eq:spldg:place}).
For the trace to continue, the source-pledge has to be recalculated with respect to the new point at the current traced position $\mx$.
Eq. (\ref{ncv:eq:spldg:place}) is satisfied only when $k_S$ unwinds from 0 to -1, indicating that the source-pledge winding is not winded more than half a round when the point is placed.
As there is no additional winding, any position along the next edge can be treated like a point of collision from a cast that originates at $\mx$, where the source pledge is $0$ from Eq. (\ref{ncv:eq:spldg:init}).
As such, when a new turning point is placed, the new source pledge is
\begin{equation}
    \theta_{S,\mnew} = 0,
    \label{ncv:eq:spldg:placeadjust}
\end{equation}
where $\theta_{S,\mnew}$ becomes $\theta_S'$ at the subsequent traced edge (Fig. \ref{ncv:fig:splace}).
\input{chap_ncv/fig_splace}

\subsection{Source-pledge Update After Pruning}
The source point at $\mx_S$ can be part of a longer path $(\cdots,\mx_{SS},\mx_S,\mx,\cdots)$, where $\mx$ is the current corner traced.
The source point can be pruned if the path segment $(\mx_{SS}, \mx_S, \mx)$ is not taut, exposing $\mx_{SS}$ as the new source point.
The new source pledge with respect to the new source point at $\mx_{SS}$ is
\begin{equation}
    \theta_{SS} = \theta_S + \hat{\theta}_{ss} - \hat{\theta}_s,
    \label{ncv:eq:spldg:prune}
\end{equation}
where $\theta_S$ is the source-pledge calculated at $\mx$ as if the source point at $\mx_S$ is not pruned. 
\input{chap_ncv/fig_sprune}

From Eq. (\ref{ncv:eq:spldg:place}) and (\ref{ncv:eq:spldg:placeadjust}),
if a turning point can be placed with respect to the new source point such that $\mside\theta_{SS} < 0$, $\theta_{SS}$ is changed to $0$.
At the subsequent corner,
$\theta_{SS}$  becomes $\theta_{S}'$ in Eq. (\ref{ncv:eq:spldg}).
As a prune can only occur when $k_S=0$ and outside of a non-convex extrusion,
the proof from Lemma \ref{ncv:lem:tprune} is applicable to the prune (see Fig. \ref{ncv:fig:sprune}).

\subsection{Source-pledge Angular Discretization}
In an occupancy grid, the source-pledge can be discretized in a similar manner as described in Sec. \ref{ncv:sec:tpldg:discrete} and Eq. (\ref{ncv:eq:tpldg:discrete}).
Discretization eliminates the computationally expensive $\fatantwo$ function from calculations.
Eq. (\ref{ncv:eq:spldg:schange}) and Eq. (\ref{ncv:eq:spldg:echange}) are respectively discretized to
\begin{align}
    \mathrm{z}(\vartheta_s) &= \mathrm{z}(\hat{\theta}_s) - \mathrm{z}(\hat{\theta}_s') \\
    \mathrm{z}(\vartheta_\medge) &= \mathrm{z}(\hat{\theta}_\medge) - \mathrm{z}(\hat{\theta}_\medge').
\end{align}
The initial source-pledge at the collision point $\mx_0$ is
\begin{equation}
    \mathrm{z}(\theta_{S,0}) = \mathrm{z}(\mx_0 - \mx_S) - \mathrm{z}(\hat{\theta}_{\medge,0}),
\end{equation}
where $\mathrm{z}(\mx_0 - \mx_S)$ is the discrete heading of $\mx_0$ from $\mx_S$.
The update equation from Eq. (\ref{ncv:eq:spldg}) is adjusted to
\begin{equation}
    \mathrm{z}(\theta_S) = \mathrm{z}(\theta_S') + \mathrm{z}(\vartheta_s) - \mathrm{z}(\vartheta_\medge),
    \label{ncv:eq:sdis}
\end{equation}
When a prune occurs, Eq. (\ref{ncv:eq:spldg:prune}) is adjusted to
\begin{equation}
    \mathrm{z}(\theta_{SS}) = \mathrm{z}(\theta_S) + \mathrm{z}(\hat{\theta}_{ss}) - \mathrm{z}(\hat{\theta}_s).
    \label{ncv:eq:sdis:prune}
\end{equation}
For a $\mside$-sided trace, a turning point can be placed at $\mx$ if
\begin{equation}
    \mside\mathrm{z}(\theta_S) < 0,
    \label{ncv:eq:sdis:place}
\end{equation} 
and $\mathrm{z}(\theta_S)$ is assigned a zero value.

Once a trace reaches $\mx$, prune checks have to be conducted before a turning point can be placed.
If a prune has occurred, $\mathrm{z}(\theta_{SS})$ from Eq. (\ref{ncv:eq:sdis:prune}) becomes $\mathrm{z}(\theta_S)$ in Eq. (\ref{ncv:eq:sdis:place}).
Otherwise, $\mathrm{z}(\theta_{SS})$ becomes $\mathrm{z}(\theta_S')$ at the subsequent corner for Eq. (\ref{ncv:eq:sdis}).

\clearpage
\section{Source Progression}
\label{ncv:sec:sprog}

The \textbf{source angular deviation}, or \textbf{source deviation}, is the angular deviation of a trace's position from its initial position, when viewed from a source point.
A source point leads to the start point along a path.
A trace will have \textbf{source angular progression}, or \textbf{source progression}, if the source deviation is at the maximum so far.
An algorithm that utilizes the source progression method places turning points at convex corners where there is source progression, and only at the perimeter of the convex hull known so far by the trace of the traced obstacle.

\input{chap_ncv/fig_sprog}

The source progression method is superior to the target-pledge method, as
(i) the method eliminates angular measurements by comparing only directional vectors, and (ii) is less likely to place turning points within the true convex hull of the traced obstacle.
Consider $\mx$ in Fig. \ref{ncv:fig:sprog}. A point can be placed at $\mx$ by the target-pledge method as the source-pledge winding counter $k_S$ is unwinded from $0$ to $-1$, which is within the convex hull of the traced obstacle.
By avoiding any placements when the trace has a smaller angular deviation than the maximum, a placement at $\mx$ can be avoided by the source progression method.
As such, unlike the target-pledge method, the method is guaranteed to avoid placing turning points within the (i) convex hull of the trace, and (ii) the convex hull known so far of the traced obstacle.

\subsection{Source Progression Update Equations}
The source progression method relies on a \textbf{source progression ray} to record the maximum angular deviation. The ray points from the source point, and can be quantified with a vector $\mvprogs$.
When a cast from a source point at $\mx_S$ collides, the source progression ray is initialized to
\begin{equation}
    \mv_{\mprog,S,0} = \mx_0 - \mx_S,
    \label{ncv:eq:sprog:init}
\end{equation}
where $\mx_0$ is the collision point.
Consider an edge traced after a collision and the source deviation has been increasing.
A $\mside$-sided trace lies ahead of the previous progression ray $\mvprogs'$ at its current position $\mx$ if
\begin{equation}
    {isFwdSrc} := \mside(\mvprogs' \times \mv_S) \le 0
    \label{ncv:eq:sprog:isfwdsrc}
\end{equation}
is $\mtrue$.
$\mv_S = \mx - \mx_S$, and the $\times$ operator is the two-dimensional cross product.
If ${isFwdSrc}$ is $\mtrue$, the source deviation at $\mx$ stays the same or is increasing.
If ${isFwdSrc}$ is $\mfalse$, the source deviation at $\mx$ would have decreased. 
Comparing vectors using the cross-produce eliminates angular measurements, especially the computationally expensive $\fatantwo$ function.

Due to the cross-product, ${isFwdSrc}$ breaks down if both vectors' true rotation with respect to each other is more than half a round.
In a highly non-convex obstacle, the source deviation can increase by more than half a round, and subsequently decrease by at most the same amount.
When the source deviation decreases, $\mv_S$ can be rotated by more than half a round with respect to $\mvprogs'$.

To ensure correct comparisons, the \textbf{source progression winding counter}, $w_S$, is introduced.
$w_S$ is initialized to zero.
$w_S$ is changed only if ${isFwdSrc}=\mfalse$. 
When $w_S$ is changed, $\mvprogs'$ is reversed, and $w_S$ is incremented (winded) or decremented (unwinded).
The winding depends on the intersection of the source progression ray with the edge leading to $\mx$.
Let the scalar $i$ indicate the direction of the intersection along the ray. The intersection can be found by solving the vector equation
\begin{equation}
    \mx_S + i \mvprogs' = \mx + i_p(\mvprev)
    \label{ncv:eq:sprog:itxfrom}
\end{equation}
If $i > 0$, the intersection lies in the direction of the ray from $\mx_S$, and if $i < 0$, the intersection lies in the opposite direction.
Since only the sign of $i$, $\fsgn(i)$, is interesting, Eq. (\ref{ncv:eq:sprog:itxfrom}) can be solved to find
\begin{align}
    {windSrc} & := \fsgn(i) > 0 \\
        & := \fsgn(\mv_S \times \mvprev) \fsgn(\mvprev \times \mvprogs') > 0.
    \label{ncv:eq:sprog:windsrc}
\end{align}
$\mvprev$ is the vector pointing from the previous trace position to $\mx$.
If the intersection lies in the direction of the ray, ${windSrc} = \mtrue$ and $w_S$ is incremented.
If the intersection lies in the opposition direction, ${windSrc} = \mfalse$ and $w_S$ is decremented.

There may be cases where $i=0$ for some traces. For example, when the source progression ray begins from the start point, and the start point lies at a corner or on the obstacle edge that is being traced. 
Using the contour assumption in Sec. \ref{line:sec:contour} when $i=0$, $\fsgn(i)$ can be re-evaluated by re-considering the position of the current or previous traced corner.
For example, if the start point lies at the previous traced corner, a coordinate $\mx'$ can be found that adds the previous corner's coordinate to its bisecting vector $\mvcrn$.
As the Chebyshev distance of the bisecting vector is one, and the width of an obstacle in an occupancy grid is non-zero, the previous traced direction can be reconsidered as $\mx-\mx'$. 
By reconsidering the trace direction, $i$ in Eq. (\ref{ncv:eq:sprog:windsrc}) will no longer evaluate to zero.

To summarize, the source progression method first determines if the winding counter needs to be changed, such that
\begin{equation}
    w_S = w_S' + \begin{cases}
         0 & \text{if } {isFwdSrc} \\
         1 & \text{if } \neg {isFwdSrc} \land {windSrc} \\
         -1 & \text{if } \neg {isFwdSrc} \land \neg {windSrc} 
    \end{cases},
    \label{ncv:eq:sprog:wind}
\end{equation}
where $w_S'$ is the value of the winding counter at the previous traced position.
The source progression ray is flipped when $w_S$ changes, or updated to point to $\mx$ from the source point when $w_S$ remains the same, such that
\begin{equation}
    \mvprogs = \begin{cases}
        -\mvprogs' & \text{if } w_S \ne 0 \land w_S \ne w_S' \\
        \mvprogs' & \text{if } w_S \ne 0 \land w_S = w_S' \\
        \mv_S & \text{if } w_S = 0
    \end{cases},
    \label{ncv:eq:sprog:flip}
\end{equation}
The trace has source progression at $\mx$ if $w_S=0$, such that
\begin{equation} 
    {isProgSrc} := (w_S = 0)
    \label{ncv:eq:sprog:isprogsrc}
\end{equation}
evaluates to $\mtrue$. An example is provided in Fig. \ref{ncv:fig:sprog:formulation}
\def\clipper{\clip (0*\ul, 0*\ul) rectangle ++(9*\ul, 9*\ul)}
\tikzset{
    pics/sprogformulation/.style n args={1}{ code={ 
        \clipper;
        \path 
            (5*\ul, 4*\ul) coordinate (xs)
            (2*\ul, 8.5*\ul) coordinate (xt)
            ($(xs)!6/9!(xt)$) coordinate (xcol)
            ;
            
        \draw [obs]
            (3.5*\ul, 9.5*\ul) coordinate (xa0)
            -- ++(-90:2*\ul) coordinate (xa1)
            -- ++(180:3*\ul) coordinate (xa2)
            -- ++(-90:7*\ul) coordinate (xa3)
            -- ++(0:8*\ul) coordinate (xa4)
            -- ++(90:7.5*\ul) coordinate (xa5)
            -- ++(180:2*\ul) coordinate (xa6)
            -- ++(-90:5.5*\ul) coordinate (xa7)
            -- ++(180:4*\ul) coordinate (xa8)
            -- ++(90:3*\ul) coordinate (xa9)
            ;

        \draw [gray] (xs) -- (xcol);
        \draw [dotted] (xcol) -- (xt);
        \path
            ($(xa2)+(-45:{sqrt(2)*\u})$) coordinate (xt1)
            ($(xa3)+(45:{sqrt(2)*\u})$) coordinate (xt2)
            ($(xa4)+(135:{sqrt(2)*\u})$) coordinate (xt3)
            ($(xa5)+(-135:{sqrt(2)*\u})$) coordinate (xt4)
            ($(xa6)+(-45:{sqrt(2)*\u})$) coordinate (xt5)
            ($(xa7)+(-45:{sqrt(2)*\u})$) coordinate (xt6)
            ($(xa8)+(-135:{sqrt(2)*\u})$) coordinate (xt7)
            ($(xa9)+(135:{sqrt(2)*\u})$) coordinate (xt8)
            ($(xa9)+(45:{sqrt(2)*\u})$) coordinate (xt9)
            ($(xa8)+(45:{sqrt(2)*\u})$) coordinate (xt10)
            ($(xa7)+(135:{sqrt(2)*\u})$) coordinate (xt11)
            ($(xa6)+(135:{sqrt(2)*\u})$) coordinate (xt12)
            ;
        \draw [trace]
            (xcol) -- (xt1) -- (xt2) -- (xt3) -- (xt4) -- (xt5) 
            #1
            ;

        \node [black pt={shift={(-4mm, 0mm)}}{center:$\mx_T$}] at (xt) {};
        \node [black pt={shift={(-4mm, 0mm)}}{center:$\mx_S$}] at (xs) {};
        \node [cross pt] at (xcol) {};
    }},
}

\begin{figure}[!ht]
\centering
\subfloat[\label{ncv:fig:sprog:formulation:a} ] {%
    \centering
    \begin{tikzpicture}[]
        \clipper;
        \pic at (0,0) {sprogformulation={
            -- ++(-90:\u)
        }};

        \draw [dotted] (xt5) -- ($(xs)!2!(xt5)$);
        \node (nt5) [black pt={shift={(0mm, 3mm)}}{center:$\mx_{\max}$}] at (xt5) {};
        \draw [rayprog] (xs) -- (nt5)
            node [pos=0.5, sloped, above] {\scriptsize $w_S=0$};
        
    \end{tikzpicture}
} \hfill
\subfloat[\label{ncv:fig:sprog:formulation:b}] {%
    \centering
    \begin{tikzpicture}[]
        \clipper;
        \pic at (0,0) {sprogformulation={
            -- (xt6) -- ++(180:\u)
        }};

        \draw [dotted] (xs) -- ($(xs)!2!(xt5)$);
        \draw [dotted] ($(xt5)!2!(xs)$) -- ($(xt5)!3!(xs)$);
        \node [black pt={shift={(0mm, 3mm)}}{center:$\mx_{\max}$}] at (xt5) {};
        \node [black pt={shift={(-2.5mm, 2.5mm)}}{center:$\mx$}] at (xt6) {};
        \draw [rayprog] (xs) -- ($(xt5)!2!(xs)$)
            node [pos=0.4, sloped, below] {\scriptsize $w_S=1$};

    \end{tikzpicture}
} \hfill
\subfloat[\label{ncv:fig:sprog:formulation:c}] {%
    \centering
    \begin{tikzpicture}[]
        \clipper;
        \pic at (0,0) {sprogformulation={
            -- (xt6) -- (xt7) -- ++(90:\u)
        }};

        \draw [dotted] (xt5) -- ($(xs)!2!(xt5)$);
        \draw [dotted] (xs) -- ($(xt5)!3!(xs)$);
        \node (nt5) [black pt={shift={(0mm, 3mm)}}{center:$\mx_{\max}$}] at (xt5) {};
        \node [black pt={shift={(2.5mm, 2.5mm)}}{center:$\mx$}] at (xt7) {};
        \draw [rayprog] (xs) -- (nt5)
            node [pos=0.5, sloped, above] {\scriptsize $w_S=2$};

    \end{tikzpicture}
} \\
\subfloat[\label{ncv:fig:sprog:formulation:d}] {%
    \centering
    \begin{tikzpicture}[]
        \clipper;
        \pic at (0,0) {sprogformulation={
            -- (xt6) -- (xt7) -- (xt8) -- (xt9) -- (xt10) -- (xt11) -- ++(90:\u)
        }};

        \draw [dotted] (xs) -- ($(xs)!2!(xt5)$);
        \draw [dotted] ($(xt5)!2!(xs)$) -- ($(xt5)!3!(xs)$);
        \node [black pt={shift={(0mm, 3mm)}}{center:$\mx_{\max}$}] at (xt5) {};
        \node [black pt={shift={(2.5mm, -2.5mm)}}{center:$\mx$}] at (xt11) {};
        \draw [rayprog] (xs) -- ($(xt5)!2!(xs)$)
            node [pos=0.5, sloped, below] {\scriptsize $w_S=1$};

    \end{tikzpicture}
} 
\hspace{1cm}
\subfloat[\label{ncv:fig:sprog:formulation:e}] {%
    \centering
    \begin{tikzpicture}[]
        \clipper;
        \pic at (0,0) {sprogformulation={
            -- (xt6) -- (xt7) -- (xt8) -- (xt9) -- (xt10) -- (xt11) -- (xt12) -- ++(0:\u)
        }};

        \draw [dotted] (xt12) -- ($(xs)!2!(xt12)$);
        \draw [dotted] (xs) -- ($(xs)!2!(xt5)$);
        \node (nt12) [black pt={shift={(2.5mm, -2.5mm)}}{center:$\mx$}] at (xt12) {};
        \draw [rayprog] (xs) -- (nt12)
            node [pos=0.6, sloped, above] {\scriptsize $w_S=0$};

    \end{tikzpicture}
}
\caption[Illustration of source progression ray and winding counter.]{
An example illustrating how the source progression ray (double tipped arrow) changes with the winding counter.
(a) A trace reaches the maximum source deviation at $\mx_{\max}$.
(b) At the subsequent corner, the ray flips and the source progression winding counter $w_S$ winds to $1$ from $0$.
(c) Like (b), the trace crosses the ray in the direction of the ray from $\mx_S$, causing $w_S$ to wind to $2$, and the ray to be flipped.
(d) The trace crosses the ray in the opposite direction of the ray, causing $w_S$ to be unwinded to $1$, and the ray to be flipped.
(e) Like (d), $w_S$ is winded to $0$, and the ray to be flipped. As the source deviation has increased, the ray is updated to $\mv_S = \mx - \mx_S$.
}
\label{ncv:fig:sprog:formulation}
\end{figure}
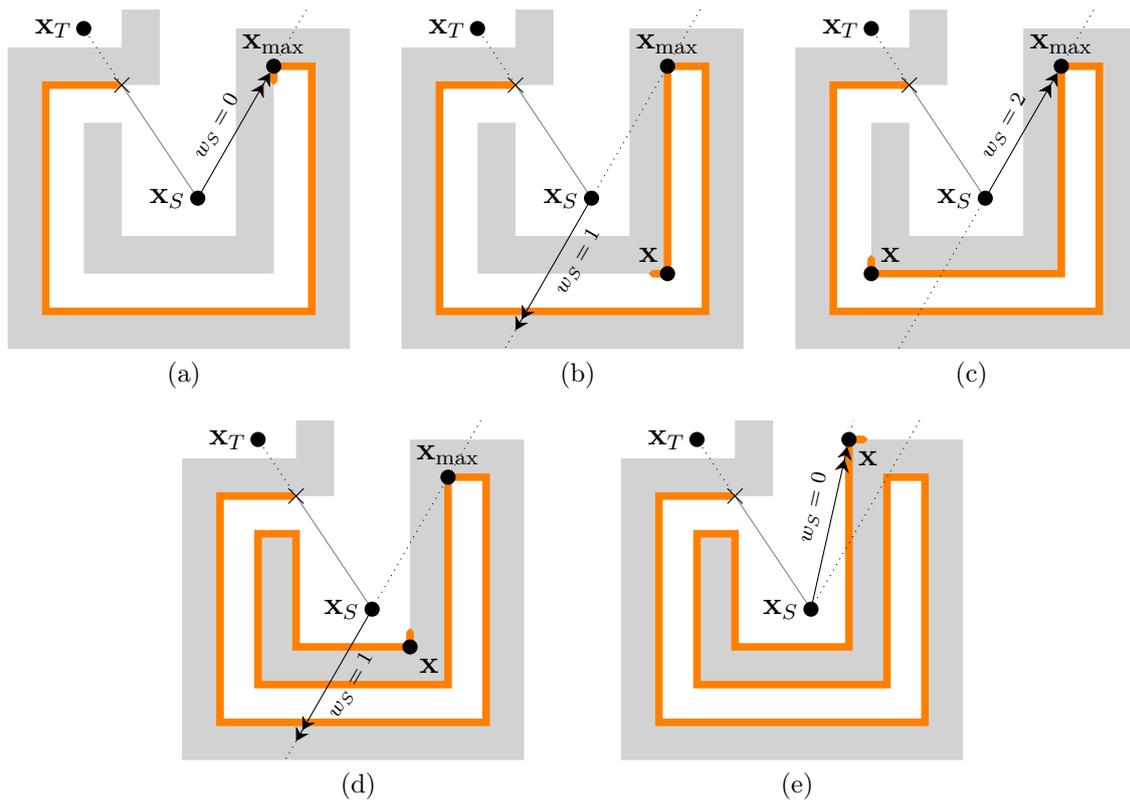

By reversing the source progression ray and changing the source progression winding counter $w_S$ when ${isFwdSrc}=\mtrue$, the cross-product comparison remains valid for any obstacle contour.
Unlike the target-pledge method, the source progression method relies on vector comparisons to bypass expensive angular measurements.

\subsection{Source Progression Update After Pruning}
The source point at $\mx_S$ may be part of a longer path $(\cdots, \mx_{SS}, \mx_S, \mx, \cdots)$, and the source point may be pruned when the path segment $(\mx_{SS}, \mx_S, \mx)$ is not taut. When a prune occurs, $\mx_{SS}$ becomes the new source point.
\input{chap_ncv/fig_sprog_prune}

Reusing ${isTautSrc}$ from Eq. (\ref{ncv:eq:prune:istautsrc}) to check for tautness,
and ${isProgSrc}$ from Eq. (\ref{ncv:eq:sprog:isprogsrc}) to check for source progression,
a source point at $\mx_S$ is prunable if
\begin{equation}
    isPrunableSrc := isProgSrc \land \neg isTautSrc
    \label{ncv:eq:sprog:isprunablesrc}
\end{equation}
evaluates to $\mtrue$.
The source progression check is necessary to avoid undesirable prunes in highly non-convex obstacles, when $\mv_S$ can rotate more than half a round around $\mv_{SS}$ in Eq. (\ref{ncv:eq:prune:istautsrc}), and cause ${isTautSrc}$ to evaluate to $\mfalse$.
${isPrunableSrc}$ can be used multiple times at $\mx$ until the new path segment around the new source point is taut.

When a prune occurs, the source progression ray has to be re-adjusted from $\mx_S$ to $\mx_{SS}$ (see Fig. \ref{ncv:fig:sprog:prune}). 
The source progression ray is updated to
\begin{equation}
    \mvprogs = \mx - \mx_{SS}.
    \label{ncv:eq:sprog:prune}
\end{equation}

\begin{lemma}
    After adjusting for the progression ray at $\mx$,
    suppose that the path segment $(\mx_{SS}, \mx_S, \mx)$ first stops being taut when a trace reaches $\mx$, and there is source progression.
    The source point at $\mx_S$ can be pruned, and
    the source progression ray can be updated with Eq. (\ref{ncv:eq:sprog:prune}).
    \label{ncv:lem:istautsrc}
\end{lemma}
\begin{proof}
Consider a prune that occurs at the initial edge.
From Lemma \ref{ncv:lem:tprune}, the path segment will first stop being taut along an edge where there is increasing angular deviation. 
Suppose that the path was not pruned before in the trace, $\mx$ has to be the first corner where the path segment stops being taut.
In such a case,
${isProgSrc}$ from Eq. (\ref{ncv:eq:sprog:isprogsrc}) will evaluate to $\mtrue$, and ${isTautSrc}$ from Eq. (\ref{ncv:eq:prune:istautsrc}) will evaluate to $\mfalse$.

Let $\mx_\mitx$ be the point of intersection between the edge leading to $\mx$, and the line colinear to $\mx_{SS}$ and $\mx_S$. 
As the edge is straight, the angular deviation has to be increasing at $\mx_\mitx$, implying that $w_S=0$ at $\mx_\mitx$.
This would have the same effect as casts from $\mx_S$ and $\mx_{SS}$ colliding at $\mx_\mitx$. The source progression rays from $\mx_S$ and $\mx_{SS}$ for both casts would be coincident.
Consider the trace with $\mx_{SS}$ as the source point. When the trace reaches $\mx$, the source progression ray would have been updated to Eq. (\ref{ncv:eq:sprog:prune}) by Eq. (\ref{ncv:eq:sprog:flip}). 
As $\mx_{SS}$ is the source point, 
Eq. (\ref{ncv:eq:sprog:prune}) is treated as $\mx - \mx_{SS}$ for this trace.
\end{proof}

\subsection{Turning Point Placement}
The source progression method places a turning point at a convex corner where there is source progression, and where the subsequent traced edge would cause the source deviation to decrease.
A turning point can only be placed after all source prunes at $\mx$ are processed. 
Compared to the target-pledge method, the source progression method places turning points at C2 corners where the angular deviation is at the maximum.
Provided that $w_S=0$, the subsequent edge from $\mx$ reverses source progression for a $\mside$-sided trace if
\begin{equation}
    {isRevSrc} := \mside(\mvnext \times \mv_S) < 0
    \label{ncv:eq:sprog:isrevsrc}
\end{equation}
evaluates to $\mtrue$.
$\mvnext$ is the vector pointing from $\mx$ to the next corner.
To get the convexity of the corner at $\mx$, let
\begin{equation}
    {isConvex} := \begin{cases}
        \mtrue & \text{if corner at $\mx$ is convex} \\
        \mfalse & \text{otherwise}
    \end{cases}.
    \label{ncv:eq:sprog:isconvex}
\end{equation}
A turning point can be placed at $\mx$ if
\begin{equation}
    {isPlaceableSrc} := {isProgSrc} \land {isConvex} \land {isRevSrc}
    \label {ncv:eq:sprog:isplaceablesrc}
\end{equation}
evaluates to $\mtrue$.
\input{chap_ncv/fig_sprog_place}

The turning point becomes the new source point of the trace, and $\mvprogs$ is updated to point to the next corner (see Fig. \ref{ncv:fig:sprog:place}), such that
\begin{equation}
    \mvprogs = \mvnext.
    \label{ncv:eq:sprog:place}
\end{equation}
Eq. (\ref{ncv:eq:sprog:place}) is correct as a subsequent turning point that lies at the perimeter of a true convex hull of the traced obstacle can be found.

\section{Target Progression and Phantom Points}
\label{ncv:sec:tprog}
The \textbf{target angular deviation}, or \textbf{target deviation}, is the angular deviation of a trace's position from its initial position, when viewed from a target point.
A target point leads to the goal point along a path.
A trace will have \textbf{target angular progression}, or \textbf{target progression}, if the target deviation is at the maximum so far (see Fig. \ref{ncv:fig:tprog}).
An algorithm that utilizes the target progression method places phantom points at non-convex corners where the trace has target progression, and only at the perimeter of the convex hull known so far by the trace of the traced obstacle.
A \textbf{phantom point} is an imaginary, future turning point that becomes the new target point of the trace when placed.
\input{chap_ncv/fig_tprog}

Like the source progression method, the target progression method is superior to its pledge algorithm counterpart as angular measurements are made by comparing directional vectors.
Additionally, casts are less likely to occur within the true convex hulls of obstacles, and the casts are more likely to be guided out of the convex hulls due to the phantom points.

Unlike the target-pledge method, the target progression method places a phantom point instead of casting to the target point.
The cast is managed by an external method instead.

\subsection{Phantom Points as Imaginary Future Turning Points}
A phantom point is an imaginary, future turning point that guides searches around the convex hull of a non-convex obstacle.
The smallest convex hull of an obstacle can be inferred by a trace by placing phantom points and turning points at the perimeter of an obstacle.
By assuming that the traced contour is part of a zero-width obstacle, a non-convex corner encountered by a trace is a convex corner on the other side of the contour.
The non-convex corner would be a vertex of the smallest possible convex hull known so far of the obstacle (the \textbf{best-hull}), and a phantom point is placed at the non-convex corner to mark the largest extent of the best-hull (see Fig. \ref{ncv:fig:tprog:phantom}).
As a phantom point lies in the obstacle, a phantom point is pruned before it can be reached by the trace that placed it. 
\input{chap_ncv/fig_tprog_phantom}

The phantom point is placed only if a path has to pass through the corner to reach the target point under the zero-width assumption.
When viewed from a source point or target point along a path, a trace's angular direction across any turning point would result reverse.
As such, a phantom point is placed only if the target deviation is at a maximum, and if the target deviation would decrease over the subsequent edge.

When an algorithm uses both the source progression and target progression method,
the best-hull of a trace is formed by the phantom points and turning points that are placed by a trace.
The best-hull provides monotonically increasing path cost estimates as a trace progresses along a traced contour, and prevents severe underestimates of path costs in vector-based planners with delayed line-of-sight checks (see Sec. \ref{ncv:sec:besthull}).

\subsection{Target Progression Update Equations}
The target progression method is adapted from the source progression method, and relies on a \textbf{target progression ray} to record the maximum angular deviation with respect to the target point at $\mx_T$. 
Let the previous target progression ray point from the target point with the directional vector $\mvprogt'$.
When a cast from a source point to the target point collides at $\mx_0$, the target progression ray is initialized to
\begin{equation}
    \mv_{\mprog,T,0} = \mx_0 - \mx_T.
    \label{ncv:eq:tprog:init}
\end{equation}
Let
\begin{equation}
    {isFwdTgt} := \mside(\mv_T \times \mvprogt') \le 0,
    \label{ncv:eq:tprog:isfwdtgt}
\end{equation}
where 
$\mv_T = \mx - \mx_T$, and the $\times$ operator is the two-dimensional cross product.
Let $w_T$ be the \textbf{target progression winding counter} to ensure that the cross product remains valid when the angular deviation decreases by more than half a round, and
$w_T$ is initialized to zero.
Let
\begin{align}
    {windTgt} & := \fsgn(\mv_T \times \mvprev) \fsgn(\mvprev \times \mvprogt') > 0.
    \label{ncv:eq:tprog:windtgt}
\end{align},
where $\mvprev$ is the vector pointing from the previous trace position to $\mx$. 
By considering the contour assumption like the source progression method, the intersection of the ray with the previous traced edge will lie away from the target point, and the $\fsgn$ functions in Eq. (\ref{ncv:eq:tprog:windtgt}) will not evaluate to zero.

The target progression method first determines if the winding counter needs to be changed, such that
\begin{equation}
    w_T = w_T' + \begin{cases}
         0 & \text{if } {isFwdTgt} \\
         1 & \text{if } \neg {isFwdTgt} \land {windTgt} \\
         -1 & \text{if } \neg {isFwdTgt} \land \neg {windTgt} 
    \end{cases},
    \label{ncv:eq:tprog:wind}
\end{equation}
where $w_T'$ is the value of the winding counter at the previous traced position.
The target progression ray is updated according to any change in $w_T$, such that
\begin{equation}
    \mvprogt = \begin{cases}
        -\mvprogt' & \text{if } w_T \ne 0 \land w_T \ne w_T' \\
        \mvprogt' & \text{if } w_T \ne 0 \land w_T = w_T' \\
        \mv_T & \text{if } w_T = 0
    \end{cases},
    \label{ncv:eq:tprog:flip}
\end{equation}
The trace has target progression at $\mx$ if $w_T=0$, such that
\begin{equation} 
    {isProgTgt} := (w_T = 0)
    \label{ncv:eq:tprog:isprogtgt}
\end{equation}
evaluates to $\mtrue$.

\subsection{Target Progression Update After Pruning}
The target progression method prunes target points in a similar way as the source progression method.
The target point at $\mx_T$ may be part of a longer path $(\cdots, \mx_{TT}, \mx_T, \mx, \cdots)$, and the target point may be pruned when the path segment $(\mx_{TT}, \mx_T, \mx)$ is not taut.
As a phantom point mimics a turning point, the pruned target point can be a phantom point or a turning point.
\input{chap_ncv/fig_tprog_prune}

Reusing ${isTautTgt}$ from Eq. (\ref{ncv:eq:prune:istauttgt}) to check for tautness,
and ${isProgTgt}$ from Eq. (\ref{ncv:eq:tprog:isprogtgt}) to check for target progression,
a target point at $\mx_T$ is prunable if
\begin{equation}
    isPrunableTgt := isProgTgt \land \neg isTautTgt
    \label{ncv:eq:sprog:isprunabletgt}
\end{equation}
evaluates to $\mtrue$.
After the prune has occurred, the target progression ray changes direction to 
\begin{equation}
    \mvprogt = \mx - \mx_{TT},
    \label{ncv:eq:tprog:prune}
\end{equation}
and points from $\mx_{TT}$ (see Fig. \ref{ncv:fig:tprog:prune}).

\subsection{Phantom Point Placement}
The target progression method places a phantom point at a non-convex corner where there is target progression, and where the subsequent traced edge would cause the target deviation to decrease.
A phantom point can only be placed after all target prunes at $\mx$ are processed, and a $\mside$-sided trace places a $\mside$-sided phantom point, which becomes the new target point of the trace.
Provided that $w_T=0$, the subsequent edge from $\mx$ reverses target progression for a $\mside$-sided trace if
\begin{equation}
    {isRevTgt} := \mside(\mv_T \times \mvnext) < 0
    \label{ncv:eq:tprog:isrevtgt}
\end{equation}
evaluates to $\mtrue$, and where $\mvnext$ is the vector pointing from $\mx$ to the next corner.
Using the definition of ${isConvex}$ from Eq. (\ref{ncv:eq:sprog:isconvex}), a turning point can be placed at $\mx$ if
\begin{equation}
    {isPlaceableTgt} := {isProgTgt} \land \neg{isConvex} \land {isRevTgt}
    \label {ncv:eq:tprog:isplaceabletgt}
\end{equation}
evaluates to $\mtrue$.
\def\clipper{\clip (0*\ul, -0.5*\ul) rectangle ++(10*\ul, 7*\ul)}
\tikzset{
    pics/tprogplace/.style n args={1}{ code={ 
        \clipper;
        \path
            (7*\ul, 6*\ul) coordinate (xt)
            (8*\ul, 0*\ul) coordinate (xs)
            ($(xs)!4/6!(xt)$) coordinate (xcol)
            ;
        
        \draw [obs]
            (10.5*\ul, 4.5*\ul) coordinate (xa0)
            -- ++(180:9*\ul) coordinate (xa1)
            -- ++(-90:4*\ul) coordinate (xa2)
            -- ++(0:4*\ul) coordinate (xa3)
            ;

        \path
            ($(xa1) + (-45:{sqrt(2)*\u})$) coordinate (xt1)
            ($(xa2) + (45:{sqrt(2)*\u})$) coordinate (xt2)
            ;

        \draw [trace]
            (xcol)  -- (xt1) #1;

        \node [cross pt] at (xcol) {};
        \node (ns) [black pt={shift={(4mm, 0)}}{center:$\mx_{S}$}] at (xs) {};
        \node (nt) [black pt={shift={(4mm, 0)}}{center:$\mx_{T}$}] at (xt) {};
    
    }},
}
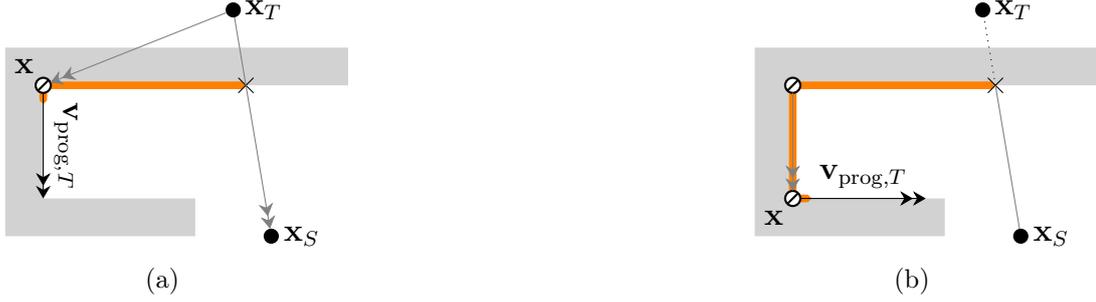
\begin{figure}[!ht]
\centering
\subfloat[\label{ncv:fig:tprog:place:a}] {%
    \begin{tikzpicture}[]
        \clipper;
        \pic at (0,0) {tprogplace={
            -- ++(-90:\u)
        }};
    
        \node (nt1) [tm pt={shift={(-2.5mm, 2.5mm)}}{center:$\mx$}{}] at (xt1) {};
        
        \draw [rayprog, gray] (nt) -- (nt1);
        \draw [rayprog, gray] (nt) -- (ns);
        \draw [rayprog] (nt1) -- ++(-90:3*\ul)
            node [pos=0.5, sloped, above] {$\mvprogt$};
    \end{tikzpicture}

} \hfill
\subfloat[\label{ncv:fig:tprog:place:b}] {%
    \begin{tikzpicture}[]
        \clipper;
        \pic at (0,0) {tprogplace={
            -- (xt2) -- ++(0:\u)
        }};

        \draw [dotted] (xcol) -- (nt);
        \draw [gray] (ns) -- (xcol);
        \node (nt1) [tm pt={}{}{}] at (xt1) {};
        \node (nt2) [tm pt={shift={(-2.5mm, -2.5mm)}}{center:$\mx$}{}] at (xt2) {};
    
        \draw [rayprog, gray] (nt1) -- (nt2);
        \draw [rayprog] (nt2) -- ++(0:3.5*\ul)
            node [pos=0.5, sloped, above] {$\mvprogt$};
    \end{tikzpicture}
}
\caption[Placing a phantom point in the best hull.]{
Before a corner is checked for phantom point placement, the target progression is checked, and the target point is checked for pruning.
If a phantom point is placeable, a $\mside$-sided phantom point will be placed by a $\mside$-sided trace. The phantom point becomes the new target point, and 
the target progression ray is updated to point in the next trace direction from the new point.
}
\label{ncv:fig:tprog:place}
\end{figure}

The phantom point becomes the new target point of the trace, and $\mvprogt$ is updated to point to the next corner (see Fig. \ref{ncv:fig:tprog:place}), such that
\begin{equation}
    \mvprogt = \mvnext.
    \label{ncv:eq:tprog:place}
\end{equation}

\subsection{Casting from a Trace}
The trace leaves the contour and casts to a target point at $\mx_T$ when the target point becomes castable.
As a taut path has to go around a convex corner, the trace can only leave the contour at a convex corner. 
The target point is potentially visible at the convex corner at $\mx$ if it does not point into the obstacle at the convex corner.
Let 
\begin{equation}
    {isVis} := \mside(\mv_T \times \mvnext)
    \label{ncv:eq:tprog:isvis}
\end{equation}
find the potential visibility of a point by considering the subsequent edge of a convex corner. $\mvnext$ is the directional vector of the $\mside$-sided trace along the subsequent edge, and $\mv_T = \mx - \mx_T$.

To ensure that a cast does not point into the best-hull, the trace at the convex corner has to have target progression. The necessary condition for casting is thus
\begin{equation}
    {isCastable} := {isConvex} \land {isProgTgt} \land {isVis},
    \label{ncv:eq:tprog:iscastable}
\end{equation}
where ${isConvex}$ is from Eq. (\ref{ncv:eq:sprog:isconvex}), and ${isProgTgt}$ is from Eq. (\ref{ncv:eq:tprog:isprogtgt}).

Due to the readjustment of the target progression ray when phantom points are placed, 
if a cast occurs for the first time Eq. (\ref{ncv:eq:tprog:iscastable}) is satisfied for all traces, a trace will always have target progression.
As such, ${isProgTgt}$ is no longer required in Eq. (\ref{ncv:eq:tprog:iscastable}). 
However, if the reader chooses to design a vector-based algorithm that avoids placing phantom points, or continue tracing once Eq. (\ref{ncv:eq:tprog:iscastable}) is satisfied, ${isProgTgt}$ becomes necessary. For example, the reader may choose to avoid a cast once a phantom point placed by the same trace becomes castable.

By continuing to trace from a castable convex corner if the target point is a phantom point placed by the same trace, the number of collided casts and subsequent searches can be reduced.
While it may seem beneficial, the subsequent interactions with the path planning algorithm has to be considered.
A trace that continues instead of casting may have to navigate an extremely long contour of a highly non-convex obstacle, and phantom points that lie on a different best-hull as the casting trace has to be identified, which can complicate and slow the algorithm.
A phantom point that lies on a different best-hull can appear as a target point of a trace if a prior trace is interrupted, such as in \rtwo{} and \rtwop{}.
An interruption is necessary to avoid lengthy traces around highly non-convex obstacles, and to generate recursive traces to ensure calculations involving the two-dimensional cross product are valid.

\section{Best-Hulls and Monotonically Increasing Costs}
\label{ncv:sec:besthull}
Combining the source progression method, target progression method, and pruning, the smallest convex hull known of a traced obstacle can be inferred by a trace.
The smallest convex hull is termed as the \textbf{best-hull}.
The best-hull expands in size as a trace progresses, and is formed by turning points and phantom points placed by the trace (see Fig. \ref{ncv:fig:besthull}).
\def\clipper{\clip (-0.25*\ul, -0.5*\ul) rectangle ++(5.5*\ul, 7*\ul)}
\tikzset{
    pics/besthull/.style n args={1}{ code={ 
        \clipper;
        \path
            (0, 3*\ul) coordinate (xs)
            (5*\ul, 0) coordinate (xt)
            ($(xs)!1/3!(xt)$) coordinate (xcol)
            ;
        \draw [obs]
            (1.5*\ul, -0.5*\ul) coordinate (xa0)
                -- ++(90:2*\ul) coordinate (xa1)
                -- ++(0:2.5*\ul) coordinate (xa2)
                -- ++(90:3*\ul) coordinate (xa3)
                -- ++(180:2.5*\ul) coordinate (xa4)
                -- ++(-90:1*\ul) coordinate (xa5)
                (xa4)
                -- ++(90:1*\ul) coordinate (xa6)
            ;
        \path
            ($(xa2) + (135:{sqrt(2)*\u})$) coordinate (xt1)
            ($(xa3) + (-135:{sqrt(2)*\u})$) coordinate (xt2)
            ($(xa4) + (-45:{sqrt(2)*\u})$) coordinate (xt3)
            ($(xa5) + (-45:{sqrt(2)*\u})$) coordinate (xt4)
            ($(xa5) + (-135:{sqrt(2)*\u})$) coordinate (xt5)
            ($(xa6) + (135:{sqrt(2)*\u})$) coordinate (xt6)
            ($(xa6) + (45:{sqrt(2)*\u})$) coordinate (xt7)
            ;

        \draw[trace]
            (xcol) -- (xt1)
            #1
            ; 

        \draw [gray] (xs) -- (xcol);
        \draw [dotted] (xcol) -- (xt);
        \node [cross pt] at (xcol) {};
        
        \node (ns) [black pt={shift={(1mm, -3.5mm)}}{center:$\mx_S$}] at (xs) {};
        \node (nt) [black pt={shift={(-3.5mm, -1mm)}}{center:$\mx_T$}] at (xt) {};
    }},
}
\begin{figure}[!ht]
\centering
\subfloat[\label{ncv:fig:besthull:a} ]
{%
    \centering
    \begin{tikzpicture}[]
    \clipper;
    \pic at (0, 0) {besthull={
         -- ++(90:\u)
    }};
    
    \end{tikzpicture}
}\hfill
\subfloat[\label{ncv:fig:besthull:b}]
{%
    \centering
    \begin{tikzpicture}[]
    \clipper;
    
    \pic at (0, 0) {besthull={
         -- (xt2) -- ++(180:\u)
    }};

    \node (nt2) [tm pt={}{}{}] at (xt2) {};
    \draw [dashed] (ns) -- (nt2) -- (nt);
    
    \end{tikzpicture}
}\hfill
\subfloat[\label{ncv:fig:besthull:c} ]
{%
    \centering
    \begin{tikzpicture}[]
    \clipper;
    \pic at (0, 0) {besthull={
         -- (xt2) -- (xt3) -- ++(-90:\u)
    }};

    \node (nt2) [tm pt={}{}{}] at (xt2) {};
    \node (nt3) [tm pt={}{}{}] at (xt3) {};
    \draw [dashed] (ns) -- (nt3) -- (nt2) -- (nt);
    
    \end{tikzpicture}
    
}\hfill
\subfloat[\label{ncv:fig:besthull:d} ]
{%
    \centering
    \begin{tikzpicture}[]
    \clipper;
    \pic at (0, 0) {besthull={
         -- (xt2) -- (xt3) -- (xt4) -- (xt5) -- (xt6) -- ++(0:\u)
    }};

    \node (nt2) [tm pt={}{}{}] at (xt2) {};
    \node (nt6) [black pt] at (xt6) {};
    \draw [dashed] (ns) -- (nt6) -- (nt2) -- (nt);
    
    \end{tikzpicture}
}\hfill
\subfloat[\label{ncv:fig:besthull:e} ]
{%
    \centering
    \begin{tikzpicture}[]
    \clipper;
    \pic at (0, 0) {besthull={
         -- (xt2) -- (xt3) -- (xt4) -- (xt5) -- (xt6) -- (xt7) -- ++(-90:\u)
    }};
    
    \node (nt2) [tm pt={}{}{}] at (xt2) {};
    \node (nt6) [black pt] at (xt6) {};
    \node (nt7) [black pt] at (xt7) {};
    \draw [dashed] (ns) -- (nt6) -- (nt7) -- (nt2) -- (nt);
    
    \end{tikzpicture}
}

\caption[Best hulls allow path cost estimates to increase monotonically.]{
Turning points and phantom points form the smallest convex hull (best hull) that a trace knows so far. 
The dashed line represents the path, which has a cost estimate that increases monotonically as the trace progresses.
}
\label{ncv:fig:besthull}
\end{figure}
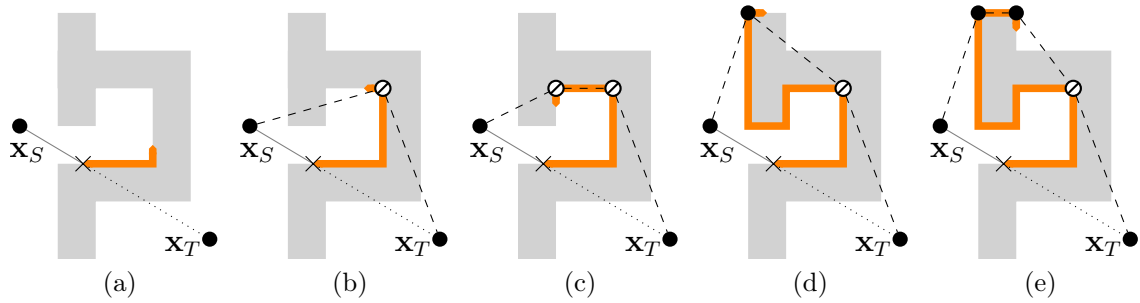

\def\clipper{\clip (0*\ul, 0*\ul) rectangle ++(6*\ul, 8*\ul)}
\tikzset{
    pics/besthull2/.style n args={1}{ code={ 
        \clipper;
        \path
            (2*\ul, 2*\ul) coordinate (xs)
            (2*\ul, 4.5*\ul) coordinate (xt)
            ($(xs)!2/5!(xt)$) coordinate (xcol)
            ;
        \draw [obs]
            (0.5*\ul, 5.5*\ul) coordinate (xa0)
                -- ++(0:3*\ul) coordinate (xa1)
                -- ++(-90:2*\ul) coordinate (xa2)
                -- ++(180:3*\ul) coordinate (xa3)
                -- ++(-90:3*\ul) coordinate (xa4)
                -- ++(0:5*\ul) coordinate (xa5)
                -- ++(90:7*\ul) coordinate (xa6)
                -- ++(180:5*\ul) coordinate (xa7)
            ;
        \path
            ($(xa3) + (-45:{sqrt(2)*\u})$) coordinate (xt1)
            ($(xa4) + (45:{sqrt(2)*\u})$) coordinate (xt2)
            ($(xa5) + (135:{sqrt(2)*\u})$) coordinate (xt3)
            ($(xa6) + (-135:{sqrt(2)*\u})$) coordinate (xt4)
            ;

        \draw[trace]
            (xcol) -- (xt1)
            #1
            ; 

        \draw [gray] (xs) -- (xcol);
        \draw [dotted] (xcol) -- (xt);
        \node [cross pt] at (xcol) {};
        
        \node (ns) [black pt={shift={(2mm, -2.7mm)}}{center:$\mx_S$}] at (xs) {};
        \node (nt) [black pt={shift={(-4mm, 0mm)}}{center:$\mx_T$}] at (xt) {};
    }},
}
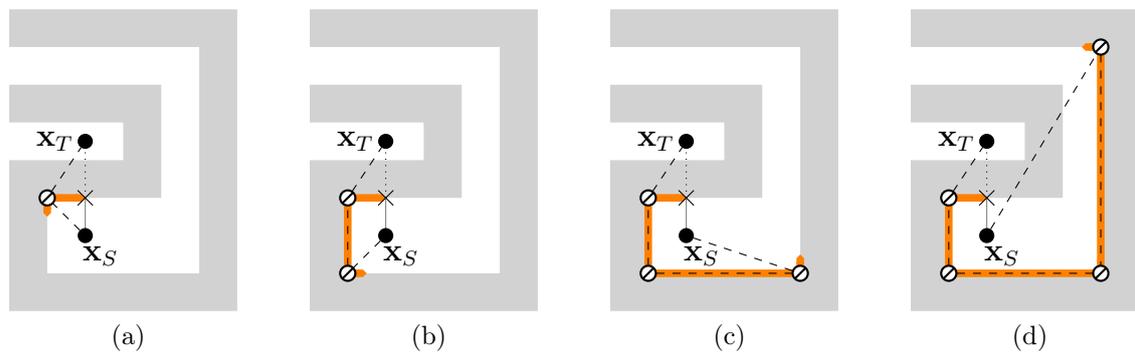
\begin{figure}[!ht]
\centering
\subfloat[\label{ncv:fig:besthull2:a} ]
{%
    \centering
    \begin{tikzpicture}[]
    \clipper;
    \pic at (0, 0) {besthull2={
         -- ++(-90:\u)
    }};

    \node (nt1) [tm pt={}{}{}] at (xt1) {};
    \draw [dashed] (ns) -- (nt1) -- (nt);
    
    \end{tikzpicture}
}\hfill
\subfloat[\label{ncv:fig:besthull2:b}]
{%
    \centering
    \begin{tikzpicture}[]
    \clipper;
    
    \pic at (0, 0) {besthull2={
         -- (xt2) -- ++(0:\u)
    }};

    \node (nt1) [tm pt={}{}{}] at (xt1) {};
    \node (nt2) [tm pt={}{}{}] at (xt2) {};
    \draw [dashed] (ns) -- (nt2) -- (nt1) -- (nt);
    
    \end{tikzpicture}
}\hfill
\subfloat[\label{ncv:fig:besthull2:c} ]
{%
    \centering
    \begin{tikzpicture}[]
    \clipper;
    \pic at (0, 0) {besthull2={
         -- (xt2) -- (xt3) -- ++(90:\u)
    }};

    \node (nt1) [tm pt={}{}{}] at (xt1) {};
    \node (nt2) [tm pt={}{}{}] at (xt2) {};
    \node (nt3) [tm pt={}{}{}] at (xt3) {};
    \draw [dashed] (ns) -- (nt3) -- (nt2) -- (nt1) -- (nt);
    
    \end{tikzpicture}
    
}\hfill
\subfloat[\label{ncv:fig:besthull2:d} ]
{%
    \centering
    \begin{tikzpicture}[]
    \clipper;
    \pic at (0, 0) {besthull2={
         -- (xt2) -- (xt3) -- (xt4) -- ++(180:\u)
    }};

    \node (nt1) [tm pt={}{}{}] at (xt1) {};
    \node (nt2) [tm pt={}{}{}] at (xt2) {};
    \node (nt3) [tm pt={}{}{}] at (xt3) {};
    \node (nt4) [tm pt={}{}{}] at (xt4) {};
    \draw [dashed] (ns) -- (nt4) -- (nt3) -- (nt2) -- (nt1) -- (nt);
    
    \end{tikzpicture}
}

\caption[Best hulls prevent path costs from being severely underestimated.]{
In non-convex contours, traces may not be able to place turning points.
If the cost-to-go is estimated like the A* algorithm, which is the distance between the target point at $\mx_T$ and the current trace position, the total path cost will be severely underestimated.
To improve cost-to-go and total path cost estimates, phantom points are placed at non-convex corners. 
The best hull enlarges as a result, allowing the total path cost to increase monotonically as the trace progresses.
}
\label{ncv:fig:besthull2}
\end{figure}
For a vector-based algorithm that delays \gls{los} checks, cost-to-come can only be estimated admissibly by assuming line-of-sight between the placed turning points.
It is not possible to place turning points on some traced contours, and the cost-to-come has to be estimated from the straight line between the current position to a distant source point, severely underestimating the cost-to-come (see Fig. \ref{ncv:fig:besthull2}).

By placing phantom points and forming the best-hull, total cost estimates are improved by enabling more reliable cost-to-go estimates.
Phantom points are imaginary future turning points that guide traces and casts around an obstacle.
Like how placing a turning point improves cost-to-come estimates by deviating the path around a traced obstacle, a phantom point improves cost-to-go estimates.

The best-hull enlarges as a trace progresses around an obstacle, allowing the total path cost estimate to increase monotonically for a vector-based algorithm that delays line-of-sight checks.

\begin{theorem}
    Let $\mathbb{X} = (\mx_1, \mx_2, \cdots, \mx_m)$ represent a sequence of corners reached by the trace where there is source progression and target progression.
    Let the path be $(\mxstart, \cdots, \mx_i, \cdots, \mxgoal)$ at each $\mx_i$ where $1 \le i \le m$. 
    The path includes all the taut nodes (turning points and phantom points) placed by the source progression method and target progression method before reaching $\mx_i$, and does not include all nodes that were pruned by the methods.
    The total cost $f_i$ of the path increases monotonically such that $f_{i-1} \le f_i$ for all  $2 \le i \le m$.
    \label{ncv:thm:fcost}
\end{theorem}
\begin{proof}
    Consider the subsequent edge traced at corner $\mx_i$. 
    If the trace at  subsequent corner progresses for both the source and target nodes, it is a forward-forward (f-f) edge. 
    If only the source node progresses, it is forward-reverse (f-r). 
    If only the target node progresses, it is a reverse-forward (r-f), and if none progresses, it is reverse-reverse (r-r).

    Fig. \ref{ncv:fig:fcost1} and \ref{ncv:fig:fcost2} illustrate the cases for this theorem.
    \textbf{Case 1.1} examines a sequence of consecutive f-f edges. 
    Assuming no pruning occurs, the corner following each f-f edge will result in a larger path cost than the path cost at the previous corner.
    
\input{chap_ncv/fig_fcost1}

    \textbf{Case 1.2} and \textbf{Case 1.3} respectively examines an f-r and r-f edge that follows an f-f edge.
    A non-convex corner can occur at $\mx_i$ if the subsequent edge is not r-f, and a convex corner can occur at $\mx_i$ if the subsequent edge is not f-r. 
    Phantom points are placed at $\mx_i$ if the subsequent edge is an f-r, and turning points are placed at $\mx_i$ if the subsequent edge is an r-f.
    When a node is placed, the progression ray points in the direction of the trace along the subsequent edge, causing the subsequent f-r and r-f edge to become f-f. 
    Since the edge before $\mx_i$ is f-f, from Case 1.1, the  path at $\mx_i$ has a higher cost than the path at $\mx_{i-1}$.
    As the subsequent edge is an f-f edge, the total path cost increases at the subsequent corner, and subsequent next corner is evaluated as Cases 1.1, 1.2, 1.3 and 1.4.

    \textbf{Case 1.4} examines an r-r edge following an f-f edge.
    If the subsequent edge is r-r and $\mx_i$ is non-convex, a phantom point is placed, converting the subsequent edge to r-f. 
    If $\mx_i$ is convex, a turning point is placed, converting the edge to an f-r. The convex case can be ignored as the target is castable and the trace stops.
    For the non-convex case, the source progression ray stops updating at $\mx_i$.

    From Lemma \ref{ncv:lem:tprune}, the trace has began tracing a non-convex extrusion.
    By examining the sequence of edges, the only way the trace crosses the source progression ray is when the ray reaches $\mx_{i+1}$ and the previous edge is f-f.
    The source progression ray does not determine the placement of phantom points, and phantom points can be generated on the non-convex extrusion, within the best-hull.
    All phantom points created after $\mx_i$ on the non-convex extrusion are pruned before the trace reaches $\mx_{i+1}$.
    The source progression ray points to the phantom point, and let the intersection of the ray with the f-f edge be at $\mx_j$. 
    As the phantom point at $\mx_i$ is the target point, $f_j = f_i$.
    Since f-f edges increase the total path cost, $f_{i+1} > f_{j} \implies f_{i+1} > f_i$.

    Cases where the previous edge is not f-f occur within the best-hull and on a non-convex extrusion. The cases can be ignored as the source progression ray does not change and no cost calculations occur.

\input{chap_ncv/fig_fcost2}
    Cases 2.1, 2.2 and 2.3  show that the path cost estimate increases when nodes are pruned.
    In \textbf{Case 2.1}, the source point at $\mx_S$ is pruned when the trace reaches $\mx_i$, exposing a new source point at $\mx_{SS}$.
    As the trace has progressed at $\mx_i$, the previous edge is f-r or f-f. 
    When pruning occurs, the trace crosses $\mv_{SS}=\mx_S - \mx_{SS}$ at the previous edge.
    Let the intersection of the previous edge and $\mv_{SS}$ be $\mx_{j}$, and the source node is pruned because $\mx_j$, $\mx_S$ and $\mx_{SS}$ are colinear. 
    The edge between $\mx_j$ and $\mx_i$ is f-f, and $f_i > f_j$.
    
    When multiple ancestor nodes are pruned while reaching $\mx_i$, the path cost estimate increases. 
    As the source node is placed at a point that is progressed with respect to its ancestor, the trace at $\mx_i$ will  progress with respect to the ancestor. 
    The edge before $\mx_i$ is an f-f edge for the ancestor, and the path cost increases. 
    
    The analysis in Case 2.1 can be applied to target nodes in \textbf{Case 2.2}. 
    The difference lies in the previous edge being r-f or f-f. 
    r-f edges occur when the trace is tracing a non-convex extrusion.
    Since the trace does not progress at a non-convex extrusion, no corners are added to $\mathbb{X}$, and cases involving r-f edges can be ignored.

    Consider case \textbf{Case 2.3} where, by reaching $\mx_i$, both source and target nodes are pruned. The same analyses from Case 2.1 and Case 2.2 can be applied to show that the path cost increases.

    By considering all possible cases, the path-cost is shown to increase monotonically in an algorithm that uses the source and  target progression method.
\end{proof}

\section{Conclusion}
To enhance the speed of path finding, vector-based searches that delay line-of-sight checks do not verify \gls{los} between turning points immediately, 
and can become trapped in non-convex obstacles without appropriate search strategies.
The section introduces several novel methods and concepts for such planners to navigate non-convex obstacles. 
Novel concepts include the phantom point, which is an imaginary future turning point, and the best-hull, which is the smallest convex hull that can be inferred of a traced obstacle.

The novel methods include the target-pledge and source-pledge methods, and the source progression and target progression methods.
The \textbf{target-pledge} method is first described in Ray Path Finder \cite{bib:rpf}, and developed in this thesis to include a proof of completeness and update equations when pruning. 
The \textbf{source-pledge} method is a novel method to place turning points using an angular counter to reduce the number of points placed in the convex hulls of obstacles.
As the source-pledge and target-pledge methods rely on expensive angular measurements for any polygonal obstacle, and occupancy grids contain only rectangular obstacles, angles are discretized for the algorithms in occupancy grids to improve the speed of calculations.
The \textbf{source progression} method compares against a ray that records the maximum angular deviation of a trace with respect to a source point, allowing an algorithm to be more effective than the source-pledge algorithm at placing turning points away from the convex hull of obstacles.
The \textbf{target progression} method records the maximum angular deviation of a trace with respect to a target point. 
The method places phantom points, which are imaginary future turning points, at non-convex corners to guide searches around non-convex obstacles.
By combining the source progression and target progression methods, the best-hull of a traced obstacle can be obtained, and path cost estimates can increase monotonically despite delayed \gls{los} checks.

    \setcounter{algorithm}{0}
\renewcommand{\thealgorithm}{\thechapter.\arabic{algorithm}}

\chapter{\rtwo{}: a Novel Vector-Based Any-angle Algorithm with Delayed Line-of-sight Checks}
\label{chap:r2}

`R' in two-dimensions (\textbf{\rtwo{}}), is a novel vector-based path planner that delays \gls{los} checks to expand the most promising turning points. 
The promising turning points are those that deviate the least from the straight line between the start and goal points.
\rtwo{} builds upon the best-hull from Sec. \ref{ncv:sec:besthull}, which combines the source progression method from Sec. \ref{ncv:sec:sprog} and target progression method from Sec. \ref{ncv:sec:tprog}.
The methods will be combined and expanded upon in the subsequent sections.
\rtwo{} borrows the concept of \textit{angular sector} from \rsp{} and \rs{} to prevent repeated searches.

Like A*, the `R' in \rtwo{} is simply an alphabet and is inspired the word `Ray', as it relies heavily on rays in its calculations. 
As the vector-based concept can be extended to three or more dimensions in future works, `R' is appended with the number `2' to reflect the two-dimensional aspect of the current algorithm.

\section{Overview of \rtwo{}} 
\label{r2:sec:overview}
\input{chap_r2/fig_overview}

\rtwo{} has two query phases -- \textbf{casting} and \textbf{tracing}. The casting phase attempts to test line-of-sight between two nodes and the tracing phase searches along the obstacle contours to find nodes. A \textbf{query} is an intermediate search that is in either phase, which is polled from or queued into the open-list. 

A node is in the \textbf{source} direction if it leads to the start node, and is in the \textbf{target} direction if it leads to the goal node. A node can be a turning point that is placed at a convex corner or a phantom point that is placed at a non-convex corner. 
A query’s source node stores cumulative cost and visibility information to the start node, and to the goal node if it is a target node. 
A node has \textbf{cumulative visibility} to another node if all pairs of nodes lying between both nodes have line-of-sight. 

The node tree branches depending on the cumulative visibility and node direction, behaving like a sparse and optimistic visibility graph where the edges may not have line-of-sight. 
A query's source node has one source node and multiple target nodes, while a target node can have multiple source nodes and multiple target nodes.
A target node that has cumulative visibility to the goal node has one target node. 

During the casting phase, a ray is cast from a source node to a target node, checking the line-of-sight between both nodes. If there is line-of-sight, \rtwo{} queues casting queries depending on the cumulative visibility of the target node to the goal node. A casting query between the source node and its source is queued if the the target node has cumulative visibility, otherwise a casting query is queued from the target node to its target.

If the ray collides and there is no line-of-sight between the source and target nodes, a tracing query towards the left ($L$-trace) and another to the right ($R$-trace) of the collision point are generated. An additional \textbf{third} trace from the source node will be generated if the target node is the goal node.

Six rules are observed during tracing – the progression, pruning, placement, overlap, angular-sector and occupied-sector rules. The rules allow a query to identify and update the best-hull, and infer an admissible path cost estimate from the hull.
As the tracing query proceeds along an obstacle’s contour, the path deviates from the collided ray and the best-hull increases in size, allowing the estimated path cost to increase monotonically. 

The \textbf{progression rule} monitors the angle the path has deviated from a node after collision. The \textbf{pruning rule} prunes nodes that are not taut. The \textbf{placement rule} places nodes at suitable corners. If the query overlaps with another query and multiple nodes are placed at the same location, the \textbf{overlap rule} interrupts the trace and checks line-of-sight to verify cost-to-come. The \textbf{angular-sector rule} discards repeated traces and generates a recursive trace to allow \rtwo{} to be complete. The \textbf{occupied-sector rule} generates a recursive trace from a source node if the current trace can only be reached by the recursive trace. Detailed explanations of the rules are given in the subsequent subsections.

When a placed turning point is potentially visible to a target node, a casting query is queued. If a number of turning points are created and the target node is not potentially visible, the trace is interrupted and queued.
A simple run of \rtwo{} is shown in Fig. \ref{r2:fig:overview}.

\subsection{Progression Rule and Winding}
\label{r2:sec:prog}
The progression rule of \rtwo{} combines the source progression method (see Sec. \ref{ncv:sec:sprog}) and the target progression method (see Sec. \ref{ncv:sec:tprog}).
The rule ensures that trace operations occur only when there is source progression or target progression, barring placement and pruning when the trace has no progression.

\input{chap_r2/fig_trace}

Let $\msidetrace \in \{L,R\}$ be the side of the trace, where $R=1$ (right trace) and $L=-1$ (left trace). 
Let $\mx$ be the current corner expanded by the trace, $\mx_\mtdir$ be the location of a source or target node, and $\mv_\mtdir = \mx - \mx_\mtdir$. 
Let $\mtdir \in \{S,T\}$ where $S=-1$ (source direction) and $T=1$ (target direction).
Fig. \ref{r2:fig:trace} illustrates the contour information.
Let $w_\mtdir$ be the winding counter is used to monitor the number of progression ray flips.
Let $\mtrace$ encapsulate all the information described above.

The progression ray $\mv_{\mprog,\mtdir}$ at the collision point is initialized to $\mx - \mtdir$,
and the winding counter $w_\mtdir$ is initialized to zero.
Let
\begin{equation}
    {isFwd}(\mtrace, \mtdir) := \msidetrace \mtdir (\mv_\mtdir \times \mv_{\mprog,\mtdir}') \le 0,
\end{equation}
and
\begin{equation}
    {wind}(\mtrace, \mtdir) := \fsgn(\mv_\mtdir \times \mvprev) \fsgn(\mvprev \times \mv_{\mprog,\mtdir}') > 0,
\end{equation}
where $\mvprev$ is a directional vector indicating the direction of the trace immediately before reaching $\mx$.

The progression rule first determines if the winding counter needs to be changed, such that
\begin{equation}
    w_\mtdir = w_\mtdir' + \begin{cases}
         0 & \text{if } {isFwd}(\mtrace, \mtdir) \\
         1 & \text{if } \neg {isFwd}(\mtrace, \mtdir) \land {wind}(\mtrace, \mtdir) \\
         -1 & \text{if } \neg {isFwd}(\mtrace, \mtdir) \land \neg {wind}(\mtrace, \mtdir) 
    \end{cases},
    \label{r2:eq:wind}
\end{equation}
where $w_\mtdir'$ is the value of the winding counter at the previous traced position.
The progression ray is flipped when $w_\mtdir$ changes, or is updated to point to $\mx$ from the source or target node at $\mx_\mtdir$ when $w_\mtdir$ remains the same:
\begin{equation}
    \mv_{\mprog,\mtdir} = \begin{cases}
        -\mv_{\mprog,\mtdir}' & \text{if } w_\mtdir \ne 0 \land w_\mtdir \ne w_\mtdir' \\
        \mv_{\mprog,\mtdir}' & \text{if } w_\mtdir \ne 0 \land w_\mtdir = w_\mtdir' \\
        \mv_{\mprog,\mtdir} & \text{if } w_\mtdir = 0
    \end{cases},
    \label{r2:eq:flip}
\end{equation}
and $\mv_{\mprog,\mtdir}'$ is the progression ray at the previous traced position.
The trace has progression at $\mx$ if $w_\mtdir=0$, and
\begin{equation} 
    {isProg}(\mtrace, \mtdir) := (w_\mtdir = 0)
    \label{r2:eq:isprog}
\end{equation}

When a prune occurs (see Sec. \ref{r2:sec:prune}), the node at $\mx_\mtdir$ is pruned, exposing a node at $\mx_{\mtdir\mtdir}$.
The progression ray has to be re-adjusted after pruning to 
\begin{equation}
    \mv_{\mprog,\mtdir} = \mx - \mx_{\mtdir\mtdir}.
\end{equation}

When a point is placed (see Sec. \ref{r2:sec:place}), the progression ray is updated to 
\begin{equation}
    \mv_{\mprog,\mtdir} = \mvnext,
\end{equation}
where $\mvnext$ is the directional vector of the subsequent trace direction.

\subsection{Pruning Rule}
\label{r2:sec:prune}
\input{chap_r2/fig_prune}
The pruning rule ensures that paths formed by nodes are taut and the estimated costs are admissible.
The rule is adapted from Sec. (\ref{ncv:sec:placeandprune}).
The rule checks the path segment $(\mx,\mx_\mtdir,\mx_{\mtdir\mtdir})$ for tautness, where $\mx_{\mtdir\mtdir}$ is the position of the source node's source node, or the position of the target node's target node. 
Let $\mv_{\mtdir\mtdir} = \mx_\mtdir - \mx_{\mtdir\mtdir}$ and $\mside_\mtdir$ be the side of the source or target node. 
$\mside_\mtdir$ is identical to the side of the trace $\msidetrace$ that placed the node. 

The path segment $(\mx, \mx_\mtdir, \mx_{\mtdir\mtdir})$ is taut if
\begin{equation}
    {isTaut}(\mtrace, \mtdir) := 
    \begin{cases}
        \mv_\mtdir \cdot \mv_{\mtdir\mtdir} \ge 0  & \text{if } \mv_{\mtdir\mtdir} \times \mv_\mtdir = 0 \\
        \mside_\mtdir \mtdir (\mv_{\mtdir\mtdir} \times \mv_\mtdir) < 0  & \text{otherwise} 
    \end{cases},
\end{equation}
where $\cdot$ denotes the dot product. 
The source or target node at $\mx_\mtdir$ can be pruned if there is progression:
\begin{equation}
    {isPrunable}(\mtrace, \mtdir) := isProg(\mtrace, \mtdir) \land \neg isTaut(\mtrace, \mtdir).
\end{equation}
The dot product prevents pruning and incorrect cost reductions if $\mv_\mtdir$ and $\mv_{\mtdir\mtdir}$ point in opposite directions.
If $\mv_\mtdir$ and $\mv_{\mtdir\mtdir}$ point in the same direction, pruning occurs as the cost estimate is unchanged.

\subsection{Placement Rule} 
\label{r2:sec:place}
\input{chap_r2/fig_placement} 
The rule places phantom points or turning points on corners where the angular progression of the trace reverses. 
If the angular progression reverses at a convex corner when viewed from a source node, a turning point is placed. 
If the angular progression reverses at a non-convex corner when viewed from a target node, a phantom point is placed.

Let $\mvnext$ be the direction of the trace along the subsequent edge from $\mx$. The angular progression reverses on the subsequent corner if
\begin{equation}
    {isRev}(\mtrace, \mtdir) := \msidetrace \mtdir (\mv_\mtdir \times \mvnext) < 0.
\end{equation}
Let ${isCrn}$ represent the convexity requirement for placing a turning point or phantom point,
\begin{equation}
    {isCrn}(\mtrace, \mtdir) := (\mtdir=1 \land \neg {isConvex}) \lor (\mtdir=-1 \land {isConvex}), 
\end{equation}
where ${isConvex}$ checks the convexity of the corner at $\mx$ (see Eq. (\ref{ncv:eq:sprog:isconvex})).

The placement rule is
\begin{equation}
    {isPlaceable}(\mtrace, \mtdir) := {isProg}(\mtrace, \mtdir) \land {isCrn}(\mtrace, \mtdir) \land {isRev}(\mtrace, \mtdir)
    \label{r2:eq:isplaceable}.
\end{equation}

\subsection{Casting from a Trace}
A trace  leaves the contour at $\mx$ as a cast when the target node at $\mx_T$ is potentially visible, and when a turning point is placed at $\mx$.
A target node is potentially visible when
\begin{equation}
    {isVis}(\mtrace) := \msidetrace(\mv_T \times \mvnext)
\end{equation}
evaluates to $\mtrue$.
The condition for casting is
\begin{equation}
    {isCastable}(\mtrace) := {isPlaceable}(\mtrace, S) \land {isProg}(\mtrace, T) \land isVis(\mtrace).
    \label{r2:eq:iscastable}
\end{equation}
When ${isCastable}$ evaluates to $\mtrue$, a casting query is queued between the current node at $\mx$ and the target node.
Note that casting requires that source progression and target progression at $\mx$ with respect to the previous source node at $\mx_S$ and current target node at $\mx_T$.

As multiple target nodes may be examined during a tracing query, 
the castable node is queued and discarded from the current trace, and the trace continues for the other non-castable target nodes.

\subsection{Occupied-sector rule} \label{r2:sec:edge}

\input{chap_r2/fig_sectors}
\input{chap_r2/fig_occsecrule}

Let the \textbf{occupied sector} of a corner at $\mxr$ be the sector bounded by the obstacle edges adjacent to the corner, but not including the edges. A point $\mx$ lies in the obstacle sector if $\mx-\mxr$ points into the obstacle. Fig. \ref{r2:fig:sectors} illustrates the occupied and angular sectors.

A recursive \textbf{occupied-sector trace} occurs if $\mvs$ points into the occupied sector at the source node. 
Occupied-sector traces allow the pruning rule for source nodes to remain valid by ensuring that $\mvs$ is not greater than a $180^\circ$ with respect to $\mvss$.
The occupied-sector trace begins from the source node, continuing in the same direction as the trace that found the source node. 
The trace stops when it can cast to the current position of the calling trace (see Fig. \ref{r2:fig:occsecrule}).

The occupied-sector rule cannot be implemented for target nodes due to phantom points. 
As phantom points do not form part of any path, occupied-sector traces from phantom points can generate wrong turning points, causing the algorithm to be incomplete.
While the rule can be implemented for target nodes that are convex turning points, it cannot be generalized to all target nodes and the pruning rule would continue to be invalid for phantom points.
To address the problem, ad hoc points are introduced.

\subsection{Ad hoc Points as Temporary Target Turning Points}
\input{chap_r2/fig_adhocbc}
\input{chap_r2/fig_adhoca}

ad hoc points $\mnode_{\mathrm{ad},b}$ and $\mnode_{\mathrm{ad},c}$ allow the pruning rule to remain valid for target nodes by re-pointing $\mvtt$ and $\mvt$ so that the angle between them is less than $180^\circ$.
The ad hoc points are placed once the trace begins to travel more than a half-circle around a target node, where $\mvt$ becomes larger than $180^\circ$ from $\mvtt$. 
If the path has to detour at $\mx$ to reach the target node's target  at $\mxtt$ from the target node at $\mxt$, $\mnode_{\mathrm{ad},b}$ is placed at $\mx$ (e.g. Fig. \ref{r2:fig:adhocb}).
If the path has to detour $\mxtt$ to reach $\mxt$, $\mnode_{\mathrm{ad},c}$ is placed at $\mxtt$. $\mnode_{\mathrm{ad},c}$ becomes the new target node, while the old target node becomes the target of $\mnode_{\mathrm{ad},c}$ (e.g. Fig. \ref{r2:fig:adhocc}).
As $\mvtt$ is reoriented for both ad hoc points to no more than $180^\circ$, the pruning rule remains valid for target nodes.
A query that reaches $\mnode_{\mathrm{ad},b}$ or $\mnode_{\mathrm{ad},c}$ can be rejected as its path may have intersected or looped with itself.

A secondary function of ad hoc points is to ensure angular progression in a third-trace. The third-trace does not begin from the collision point, and will not progress with respect to the target node. Placing an ad hoc point $\mnode_{\mathrm{ad},a}$ allows a loop around the source node's obstacle to be identified, and the angular progression to progress with respect to the target nodes (see Fig. \ref{r2:fig:adhoca}).

While ad hoc points solve two problems, two additional problems arise.
The first problem involves \textbf{chases} where trace queries loop around obstacles trying to reach $\mnode_{\mathrm{ad}, b}$ and $\mnode_{\mathrm{ad},c}$. The chases occur when ad hoc points are placed on the same contour as the traces but on opposite sides of the obstacle.
Due to the chases, \rtwo{} terminates only if traces can be interrupted and if a path exist between the start and goal points. If no path exists, \rtwo{} can become interminable.

The second problem occurs when a trace finds a path that is not taut after entering a target node's occupied sector and exiting from the other side.
A $-\mside_T$ trace may enter the occupied sector of a target node of side $\mside_T$ while tracing on the same contour as the node. 
When the trace exits the obstacle sector, the target pruning rule is satisfied but the path may not be taut.
To reject the non-taut paths, a tautness check is implemented during casts. The searches can be rejected as a $\mside_T$ trace would have been generated that finds a taut path to the target.
For this problem, an occupied-sector trace is not a viable solution as  chases can be generated.

Regardless of the problems, \rtwo{} is correct and optimal if a path exists.
The problems with the ad hoc points are addressed in \rtwop{} (see Chapter \ref{r2p:sec:tgtocsec}).

\subsection{Angular-sector Rule} \label{r2:sec:sector}
\input{chap_r2/fig_sector}

\textbf{Angular sectors} allow \rtwo{} to terminate and run faster by rejecting repeated searches. 
The angular sectors are adapted from \rs{} and \rsp{} and are bounded on at least one side by \textbf{sector-rays}.
When a trace exits the angular sector of the source node, the \textbf{sector rule} determines the actions taken by the trace.

A sector-ray represents a ray that was cast from a source node at $\mx_s$ to a target node at $\mx_t$. The sector-ray  $\mray$ is
\begin{equation}
    \mray = ( \mrtype,\mx_s,\mx_t,\mxcol )
\end{equation}
where, $\mrtype$ indicates the visibility between the nodes.
If the ray can reach $\mxt$ from $\mxs$, the ray can be \textbf{projected} from $\mxt$ in the direction $\mvray = \mxt - \mxs$ \cite{bib:rayscanp}. 
$\mxcol$ is the collision point when the ray collides with an obstacle.

A turning point node may contain an angular sector bounded by a left sector-ray  $\mray_L$, a right sector-ray $\mray_R$, or both.
If the ray on one side does not exist, the sector is unbounded on this side. 

Rays are assigned to a node $\mnode$ after every cast with the function \textsc{MergeRay}$(\mside, \mnode, \mray)$. 
If replacing the $\mside$-side sector-ray of the node with $\mray$ causes the angular sector to shrink, \textsc{MergeRay} replaces the $\mside$-side sector-ray with $\mray$. By shrinking the angular sector, repeated searches can be terminated.

The sector-ray assignment depends on the line-of-sight between the source and target node and the cumulative visibility of the source node to the start node. 
If the cast reaches the target node and the source node has cumulative visibility to the start node, the ray that is cast becomes a sector-ray for both nodes. For an $\mside$-sided target node, the cast ray is merged to the target node's $\mside$-side sector-ray, and to the source node's $-\mside$-side sector-ray.
If the cast collides, the source node is duplicated for the $L$ and $R$-traces, and the cast ray becomes the $R$-side sector-ray for the $L$-trace's source node, and $L$-side sector-ray for the $R$-trace's source node.

During an $\mside$-trace, the $\mside$-side sector-ray of the source node is examined.
When a trace crosses a sector-ray, the sector rule determines whether the trace can continue or a \textbf{recursive angular-sector trace} is called. 
If the sector-ray is able to reach the trace, no recursive call is made (e.g. Fig. \ref{r2:fig:sectorc} and \ref{r2:fig:sectord}). 
If the sector-ray is unable to reach the trace (e.g. Fig. \ref{r2:fig:sectora} and \ref{r2:fig:sectorb}), a recursive angular-sector trace is called from the sector-ray's collision point. 
The angular-sector trace traces in the opposite direction (with side $-\msidetrace$) to the calling trace (has side $\msidetrace$) from the ray's collision point, and attempts to reach the calling trace.

If the trace crosses a sector-ray that ends at the source node, the source node is pruned and the trace continues (Fig. \ref{r2:fig:sectorb} and \ref{r2:fig:sectord}). 
The trace becomes a separate trace from any recursive angular-sector trace (Fig. \ref{r2:fig:sectorb}) as it can be visible from an earlier source node.
If the trace crosses a sector-ray that begins from the source node, the calling trace is terminated as the trace is repeated (Fig. \ref{r2:fig:sectora} and \ref{r2:fig:sectorc}). Any recursive angular-sector trace (Fig. \ref{r2:fig:sectora}) will attempt to reach nodes on the terminated trace and continue it.

The angular sector for the start node is a full circle \cite{bib:rayscan}.
As the cross-product is used to compare against the rays and is valid up to $180^\circ$, the angular sector for the start node is split into two $180^\circ$ angular sectors in \rtwo{}.

\subsection{Overlap Rule and Discarding Expensive Nodes} \label{r2:sec:enc}

Delaying line-of-sight checks enables \rtwo{} to return queries rapidly if the shortest path has few turning points. As the cumulative visibility of a source node to the start node cannot be determined immediately, queries that discover the same turning points cannot be discarded, causing \rtwo{} to be exponential with respect to the number of casts.

To improve average search times, the overlap rule verifies the cumulative visibility of the source nodes once a tracing query places a turning point at a corner where turning points from other queries already exist. 
The tracing query is interrupted by the rule, and for all turning points at the corner with no cumulative visibility to the start node, the rule searches along their respective paths toward the start node. For each path, a source node $\mnode_{S,m}$ that has cumulative visibility to the start node is identified. 
Before reaching $\mnode_{S,m}$, a search will have to reach its target node $\mnode_{S,m-1}$ first. 
As the node tree may branch to multiple target nodes from $\mnode_{S,m-1}$ and any of its target nodes, queued queries examining the target nodes in these branches are discarded to avoid data races. A casting query is finally queued from $\mnode_{S,m}$ and $\mnode_{S,m-1}$ to verify cumulative visibility.

The verification is extended beyond the overlap rule to casting queries.
When the casting query reaches a target node and the source node has cumulative visibility to the start node, the cost-to-come is tested at the target node. If the target node has a more expensive cost-to-come than the minimum recorded at the corner so far, it is marked as an \textbf{expensive} node. 
By ensuring cumulative visibility to the start node, the cost-to-come can be verified, and  queries can be discarded to improve average search times.

Casting queries are discarded if a ray from an expensive source node reaches the target node.
If the target node has a side $-\mside$ that is opposite to the expensive source node's side $\mside$, the query can be discarded as the source node can no longer be pruned by future queries. 
If the sides are the same, casting query(s) are queued normally from the target node.

Tracing queries are discarded if a ray from an expensive source node collides. For a $\mside$-sided expensive source node, only the $\mside$-sided trace will be generated, as the source node can never be pruned from a future query resulting from a $-\mside$-sided trace or a third trace. 
Turning points placed by the $\mside$-sided trace will be marked as expensive.
The trace continues until the target node is castable, and instead of queuing a casting query to the target node, \rtwo{} finds the earliest expensive node and queues a casting query from the expensive node to verify line-of-sight.

Expensive nodes are discarded if reaching the target node results in the cheapest cost-to-come at the target node's corner. 
All other nodes $\mnode_\mathrm{ex} \in \mathbb{N}_\mathrm{ex}$ at the target node's corner, which has cumulative visibility to the start node, are identified. 
Every $\mnode_\mathrm{ex}$ is guaranteed to have a larger cost-to-come than the target node, and the node tree of every $\mnode_\mathrm{ex}$ are subsequently searched in the target direction from $\mnode_\mathrm{ex}$. 
If there is a source-target node pair that has cumulative visibility to the start node and the nodes have different sides, the nodes and corresponding queries are discarded, as all nodes between the pair and $\mnode_\mathrm{ex}$ are expensive.

\section{\rtwo{} Algorithm and Proofs}
\begin{algorithm}
\caption{R2's main algorithm.}
\label{r2:alg:r2a}
\begin{algorithmic}[1]
\Function{Run}{$\mnode_\mstart$, $\mnode_\mgoal$} \label{r2:alg:run}
    \State \Call{Caster}{$\mnode_\mstart$, $\mnode_\mgoal$} \Comment{From start $\mnode_\mstart$ to goal $\mnode_\mgoal$}
    \While {open-list $\ne \varnothing$ \An path $\,= \varnothing$}
        \State Poll query $(\mnode_S, \mnode_T)$ from open-list.
        \If {query is interrupted trace}
            \State $\msidetrace \gets $ side of $\mnode_T$. 
            \State $\mnodes_T \gets $ target nodes of $\mnode_T$.
            \State $\mx \gets $ corner at $\mnode_T$.
            \State \Call{Tracer}{$\msidetrace$, $\mx$, $\{\mnode_S\}$, $\mnodes_T$}
        \Else
            \IfThen{\Call{Caster}{$\mnode_S$, $\mnode_T$} returns path}{\Break}
        \EndIf
    \EndWhile
    \State \Return path.
\EndFunction
\end{algorithmic}
\end{algorithm}

\begin{algorithm}
\caption{R2's Caster for casting queries: ray casting and collision handling.}
\label{r2:alg:r2b}
\begin{algorithmic}[1]
\Function{Caster}{$\mnode_S$, $\mnode_T$}
    \If {$\mnode_S$ reached $\mnode_T$} \Comment{Cast from $\mnode_S$ reached $\mnode_T$}
        \If {${CV}(\mnode_T, \mnode_\mgoal)$ \An ${CV}(\mnode_S, \mnode_\mstart)$} 
            \State \Return path. 
        \ElsIf{${CV}(\mnode_T, \mnode_\mgoal)$ \An $-{CV}(\mnode_S, \mnode_\mstart)$}
            \State Queue cast query $(\mnode_{SS},\mnode_S)$ and return $\varnothing$.
        \ElsIf{$-{CV}(\mnode_T, \mnode_\mgoal)$ \An ${CV}(\mnode_S, \mnode_\mstart)$}
            \If {$\mnode_T$ and $\mnode_S$ are expensive}
                \State Return $\varnothing$ if $\mnode_S$'s side $\ne$ and $\mnode_T$'s side.
                \State Merge rays to $\mnode_S$ and $\mnode_T$.
            \ElsIf {$\mnode_T$ is cheapest}
                \State Get $\mathbb{N}_\mathrm{ex}$ and discard expensive target nodes and queries.
                \State Update min. cost-to-come at $\mnode_T$. 
            \EndIf
            \State Merge rays to $\mnode_S$ and $\mnode_T$.
        \EndIf
        \State Queue cast query $(\mnode_T,\mnode_{TT})$ for each target node $\mnode_{TT}$ of $\mnode_T$.
        % \If{} \Comment{$\mnode_T$ has no cumulative visibility to $\mnode_\mgoal$}
        %     \If {$\mnode_S$ has cumulative visibility to $\mnode_\mgoal$}
        %         \If {$\mnode_T$ is expensive}
        %             \State Discard search if $\mnode_S$ is expensive and $\mnode_T$ has opposite side as $\mnode_S$.
        %         \ElsIf {$\mnode_T$ is cheapest} 
        %             \State Obtain $\mathbb{N}_\mathrm{ex}$ and discard targets where applicable.
        %             \State update min. cost-to-come at $\mnode_T$'s corner.
        %         \EndIf
        %         \State 
        %     \EndIf
        % \EndIf
        % \If {$\mnode_S$ is $\mnode_\mstart$}
        %     \State \Return path
        % \Else   \Comment{$\mnode_{SS}$ is source of $\mnode_S$}
        %     \State Queue $(\mnode_{SS}, \mnode_S)$
        % \EndIf
    \Else       \Comment{Cast Collided}
        \State Duplicate $\mnode_S$ to nodes $\mnode_S{i}$ and $\mnode_S{j}$.
        \State Merge ray of cast to the new nodes.
        % \State \Comment{Trace from collision point $\mxcol$}
        \State \Call{Tracer}{$-\mside_s$, $\mxcol$, $\mnode_S{i}$, $\mnode_T$}
        \State \Call{Tracer}{$\mside_s$, $\mxcol$, $\mnode_S{j}$, $\mnode_T$}
        \If {$\mnode_S\ne\mnode_\mstart$  \An $\mnode_T=\mnode_\mgoal$} \Comment{Third-trace} \label{r2:alg:thirdtrace}
            \State Create new node $\mnode_S{k}$ at $\mxs$ with side $\mside_s$.
            \State Merge ray of cast to $\mnode_S{k}$
            \State \Call{Tracer}{$\mside_s$, $\mxs$, $\mnode_S{k}$, $\mnode_T$}
        \EndIf
    \EndIf
\EndFunction
\end{algorithmic}
\end{algorithm}

\begin{algorithm}
\caption{R2's Tracer for tracing queries: tracing around an obstacle's contour.}
\label{r2:alg:r2c}
\begin{algorithmic}[1]
\Function{Tracer}{$\msidetrace$, $\mx$, $\mnodes_S$, $\mnodes_T$}
    \While {$\mx$ is in map} \Comment {For all $\mx$, $|\mnodes_S| = 1$}
        \ForEach{ $\mnodes \in \{\mnodes_S, \mnodes_T\} $ }
            \ForEach{$\mnode \in \mnodes$} \Comment{$\mnodes$ is modified by the rules.}
                \IfThen{traced to $\mnode$} {\Return}
                \If{not progressed for $\mnode$} 
                    \State \Continue
                \EndIf
                \If{$\mnode \in \mnodes_S$}
                    \State Do \underline{Angular-sector rule}.
                    \State Do \underline{Occupied-sector rule}.
                \Else
                    \State Try placing ad-hoc point $\mnode_{\mathrm{ad},b}$ or $\mnode_{\mathrm{ad},c}$ if $\mnode \in \mnodes_T$.
                \EndIf
                \State Do \underline{Pruning rule}.
            \EndFor
        \EndFor
        \State Do \underline {Placement rule}, creating $\mnode_\mathrm{new}$ if new node placed.
        \State Do \underline {Overlap rule} if new turning point $\mnode_\mathrm{tp}=\mnode_\mathrm{new}$ is placed, or if $\mnode_\mathrm{tp}$ is expensive and any $\mnode_T \in \mnodes_T$ is castable. 
        \State Queue casting query $(\mnode_\mathrm{tp}, \mnode_T)$ for all castable $\mnode_T$.
        \State Queue tracing query $(\mnode_S, \mnode_\mathrm{new})$ if $>m$ nodes placed.
        % \IfThen {traced to $\mnode_S$ or $\mnode_T$} {\Return} \Comment{Must be reached by casts}
        % \State Check progression for $\mnode_S$ and $\mnode_T$ with \textit{progression rule}.
        % \State Do free sector check with \textit{free-sector rule}.
        % \State Do occupied-sector check with \textit{occupied-sector rule}.
        % \State Prune and replace $\mnode_S$ and $\mnode_T$ with \textit{pruning rule} if needed.
        % \State Place \textit{ad-hoc pseudo targets} $\mnode_{ps,b}$ or $\mnode_{ps,a}$ if needed.
        % \State Place new turn. pt. or pseudo tgt. with sides $\msidetrace$ with \textit{placement rule} if needed.
        % \IfThen {new turn. pt. placed \An is \textit{castable} to $\mnode_T$} {queue $(\mnode_S, \mnode_T)$ and \Return}
        % \IfThen{ $>m$ nodes placed} {\textit{interrupt} and queue trace, and \Return}
        \State $\mx \gets$ subsequent $\msidetrace$ corner
    \EndWhile
\EndFunction
\end{algorithmic}
\end{algorithm}

The pseudocode for \rtwo{} is shown in Algs. \ref{r2:alg:r2a}, \ref{r2:alg:r2b} and \ref{r2:alg:r2c}. In the pseudocode, nodes $\mnode_a$ and $\mnode_b$ have cumulative visibility if ${CV}(\mnode_a, \mnode_b)$ returns true.
More comprehensive pseudocodes, that delve into the implementation, are shown in Appendix \ref{chap:suppr2}.

Theorem \ref{r2:thm:complete} shows that \rtwo{} is complete, and Theorem \ref{r2:thm:optimal} shows that \rtwo{} is optimal.

\input{chap_r2/fig_complete}

\begin{theorem}
    \rtwo{} is complete.
    \label{r2:thm:complete}
\end{theorem}
\begin{proof}
    Without loss of generality, consider all possible topologies for an obstacle $\mathcal{O}_{st}$, where a cast from  the source node at $\mxs$ to the target node at $\mxt$ collides. 
    The topologies can be derived as the progression rule ignores any intermittent reverses of angular progression.
    
    From the topologies, \textbf{end-point convex corners} are identified.
    The end-point corners lie on edges facing the target, and traces that reach the end-points will stop and cast to the target node at $\mxt$.
    Taut paths from the source node to the target node have to pass through  the end-points.
    The shortest path can be shown by contradiction to pass through the end-points.
    If the shortest path does not pass through the end-points, it is not taut. A non-taut path has to be longer than a taut path around an obstacle, and cannot be the shortest path \cite{bib:rayscan}.
    By showing that \rtwo{} finds paths to the end-points, the proof can be applied inductively to all collided casts to show that \rtwo{} is complete.

    Fig. \ref{r2:fig:complete} contains examples of the cases described below.
    In \textbf{Case 1.1}, the collided obstacle results in two end-points,  $\mx_a$ and $\mx_b$, lying on the side facing the target.
    In \textbf{Case 2.1}, the obstacle extends beyond $\mora{\mxt\mx_a}$ and behind $\mx_a$, resulting in a new end-point  $\mx_c$.
    In \textbf{Case 3.1}, the obstacle crosses $\mora{\mxt\mx_a}$ in front of $\mx_a$, resulting in $\mx_a$ being obscured and a new end-point $\mx_d$.
    We consider Cases 1.2, 2.2 and 3.2 respectively from Cases 1.1, 2.1 and 3.1, with the obstacle wrapping around the source but not intersecting the cast. 
    For the collision and cast to occur, the wrapping cannot intersect the cast.
    Obstacles that wrap around the target are either Case 1.1 or Case 1.2.

    End-points are shown to be reachable from the source node at $\mxs$. 
    The collided cast between $\mx_s$ and $\mx_t$ generates two traces, each arriving at $\mx_a$ and $\mx_b$. 
    To arrive at $\mx_c$, the trace arriving at $\mx_a$ has to continue from $\mx_a$. The trace that continues from $\mx_a$ is similar to a \textit{third-trace}, and will be called a \textbf{continued-trace}.
    % The continued-trace is generated from $\mx_a$ when the cast from $\mx_a$ to $\mxt$ collides. 
    If the cast from $\mx_a$ reaches $\mxt$, the path via $\mx_c$ is more expensive, and the continued-trace can be ignored.
    
    To reach $\mx_d$, a cast occurs from $\mx_s$ to $\mx_d$, generating a trace that finds $\mx_{a'}$, which is $\mx_a$ or a subsequent corner. 
    The trace stops at $\mx_{a'}$, and a cast $\mathcal{C}_{a'd}$ occurs from $\mx_{a'}$ to $\mx_d$. $\mathcal{C}_{a'd}$ can be reconsidered as Cases 1.1, 1.2, 2.1 and 2.2. If $\mathcal{C}_{a'd}$ collides, let the collided obstacle be $\mathcal{O}_{a'd}$, with the end points $\mx_{a,a'd}$, $\mx_{b,a'd}$ and $\mx_{c,a'd}$. 
    $\mathcal{O}_{a'd}$ cannot intersect $\mathcal{O}_{st}$. The path via $\mx_{c,a'd}$ has to be longer than the path via $\mx_{a,a'd}$, and the continued-trace from $\mx_{a,a'd}$ can be ignored.

    When reaching end-points, turning points are generated. The turning points will not be shown to be reachable from $\mx_s$.
    In \textbf{Case 4.1}, the cast between $\mx_s$ and the first turning point is examined. 
    In \textbf{Case 4.2}, the casts between the turning points are examined.
    In \textbf{Case 4.3}, the cast from an end-point to $\mx_t$ is examined.
    
    In Case 4.1, a cast between $\mx_s$ and the first turning point may collide at an obstacle $\mathcal{O}_{s1}$. As $\mathcal{O}_{s1}$ cannot intersect the cast $\mathcal{C}_{st}$ between $\mx_s$ and $\mx_t$, $\mathcal{O}_{s1}$ can belong to Case 1.1, 1.2, 2.1 or 2.2.
    $\mathcal{O}_{s1}$ has three end-points $\mx_{a,s1}$, $\mx_{b,s1}$ and $\mx_{c,s1}$.
    As $\mathcal{O}_{s1}$ does not intersect $\mathcal{C}_{st}$, $\mx_{a,s1}$ and $\mx_{c,s1}$ lie on opposite sides of $\mathcal{C}_{st}$. 
    As the first turning point lies on the side closer to $\mx_{a,s1}$, $\mx_{a,s1}$ results in a shorter path than $\mx_{c,s1}$. 
    Since $\mx_{c,s1}$ is reached by a continued-trace, the continued-trace can be ignored.
    
    In Case 4.2, a cast between a turning point at $\mx_i$ and a subsequent turning point at $\mx_j$ may collide at an obstacle $\mathcal{O}_{ij}$. 
    $\mathcal{O}_{ij}$ cannot intersect $\mathcal{O}_{st}$ between the turning points. By applying the same analysis as Case 4.1, $\mx_{c,ij}$ and the continued-trace finding $\mx_{c,ij}$ can be ignored. 
    Case 4.3 can be reconsidered as Cases 1.1, 1.2, 2.1, 2.2, 3.1 and 3.2 where the end points are considered as new source nodes. 

    From Cases 1.1, 1.2, 2.1, 2.2, 3.1 and 3.2, turning points and the goal node can be reached from casts, traces generated from collisions, and continued-traces.
    From Cases 4.1, 4.2 and 4.3, turning points can be reached from casts and traces generated from collisions.
    As all turning points can be reached without continuing from $\mx_a$, the continued-traces can be condensed to third-traces, where the trace continues from the source point only if the target is a goal node and a cast from the source point collides.
    \rtwo{} is complete as the turning points and goal node can be reached, and the shortest path has to pass through the turning points.
\end{proof}

\begin{theorem}
    \rtwo{} is optimal.
    \label{r2:thm:optimal}
\end{theorem}
\begin{proof}
    By casting a ray, \rtwo{} tries to draw a straight line between two points first, before splitting into two tracing queries around a collided obstacle.
    The path formed by each tracing query follows the smallest convex hull, the best-hull, known by the query, which increasingly deviates from the straight line path as the query proceeds along the obstacle's contour.
    
    The best-hull allows for admissible estimates of the path cost without overestimating them. The best-hull is inferred only from the contour that is traced, which future queries must go around. By resizing the best-hull based on only the traced contour, and by maintaining the hull's convexity with the tracing rules, the path cost is estimated admissibly. 
    
    From Theorem \ref{ncv:thm:fcost}, the best-hull increases in size, and the path cost increases monotonically when there is angular progression to all nodes. 
    By queuing queries only when the angular progression of the trace has increased with respect to all nodes, increasingly costlier queries are queued into the open-list.

    From Theorem \ref{r2:thm:complete}, all paths around obstacles can be found. By ensuring that the shortest possible (straight-line) solution is searched first, and by ensuring the path cost increases monotonically and admissibly between queues, \rtwo{} is able to find the optimal path between two points.
\end{proof}

\section{Methodology of Comparing Algorithms}
\rtwo{}, \rsp{}, Anya and Theta* are compared across benchmarks \cite{bib:bench}. The implementation of \rsp{} is obtained from \cite{bib:rayscanp} and Anya is obtained from \cite{bib:anya}.
Each \textit{scenario} in the benchmark is a pair of start and goal points where a path exists between them.
Comparisons are done by assuming that the map is unknown, and no cached information except for the occupancy grid exists before each scenario is run. 
As such, Polyanya \cite{bib:polyanya}, Sub-goal graphs \cite{bib:subgoal}, Sparse Visibility Graphs \cite{bib:ohleong} and Visibility graphs \cite{bib:vg} are not compared.

\rsp{} is run with the skip, bypass and block extensions, which is the fastest configuration for an unknown binary occupancy grid. 
The current implementation of \rsp{} scales the map twice and moves the start and goal points by one unit in both dimensions to avoid starting and ending on obstacle contours. The scaling and translation are required to prevent the start and goal points from occurring within obstacles due to on-the-fly smoothing of rasterized diagonal contours.
All algorithms are run on the same scaled maps and benchmark scenarios as \rsp{}. 

\rtwo{} does not smooth the contours, and is able to handle all scenarios except for scenarios where the start point is located at a checkerboard corner. A \textbf{checkerboard corner} is a non-convex corner lying in the center of a pair of diagonally opposite free cells and a pair of diagonally opposite occupied cells. A checkerboard corner occurs as a pair and its counterpart lies at the same location and faces the opposite direction.
A starting point lying on a pair of checkerboard corners is ambiguous, as it can lie on either corner, one of which may not lead to a solution. 
As such cases are unlikely to occur and are complicated to handle, scenarios beginning from checkerboard corners will not be solved by \rtwo{}.

The \rsp{} extension chosen for comparison relies on the occupancy grid and a hierarchy of obstacle polygons that are generated on-the-fly. 
\rtwo{} is similar, caching corners and rays within each scenario but the cached information is deleted after the scenario. 
For ease of implementation, \rtwo{} and Theta* relies on simple insertion sorts for the open-list. 
Anya relies on the Fibonacci heap, and \rsp{} relies on the pairing heap.
Future works can examine the effects of open-list sorting on the performance of \rtwo{} under various conditions.

All scenarios are run \textit{ten} times and the run-times averaged, except for Theta*.
Theta* is too slow for the comparisons with the other algorithms to be significant and runs longer than 5s are terminated. 

The scenarios are run on  Ubuntu 20.04 in Windows Subsystem for Linux 2 (WSL2) and on a single core of an Intel i9-11900H (2.5GHz), Turbo-boost disabled.
\rtwo{} is available at \cite{bib:my2dcode}.

\section{Results}
Specific results for selected maps are shown in Fig. \ref{r2:fig:results}.
The left column of Fig. \ref{r2:fig:results} show the thumbnails of the map, with black pixels indicating obstacles.

\begin{sidewaystable}[!ht]
\centering
\caption{Benchmark characteristics and average search time.}
\setlength{\tabcolsep}{3pt}
\begin{tabular}{ c c c c c c c c c }
\toprule
Map & $P$ & $G$ & $r$ & $\rho$ (\%) & R2 ($\mu$s) & Theta* ($\mu$s) \textsuperscript{\textdagger} & ANYA ($\mu$s) & RS+($\mu$s) \\
\hline
dao/arena & 5 & 100.5 & 0.433 & 1.315 & 5.235 & 161.135 & 104.238 & 2.787 \\
\hline
bg512/AR0709SR & 13 & 953.3 & 0.614 & 0.144 & 34.713 & 218115.116 & 220.163 & 55.426 \\
\hline
bg512/AR0504SR & 22 & 1019.0 & 0.793 & 0.570 & 153.183 & 292573.232 & 573.309 & 214.672 \\
\hline
bg512/AR0014SR & 21 & 969.4 & 0.852 & 0.611 & 135.715 & 228492.361 & 390.753 & 174.695 \\
\hline
bg512/AR0304SR & 16 & 1010.7 & 0.793 & 0.272 & 62.004 & 210728.144 & 235.891 & 73.662 \\
\hline
bg512/AR0702SR & 17 & 984.4 & 0.797 & 0.251 & 60.401 & 187307.605 & 277.864 & 80.442 \\
\hline
bg512/AR0205SR & 33 & 1441.5 & 0.923 & 0.697 & 519.853 & 371804.631 & 1237.524 & 393.564 \\
\hline
bg512/AR0602SR & 46 & 1880.4 & 0.952 & 2.072 & 1434.856 & 236938.560 & 1848.220 & 642.356 \\
\hline
bg512/AR0603SR & 42 & 2228.5 & 0.963 & 1.299 & 838.284 & 275366.995 & 1043.081 & 357.629 \\
\hline
street/Denver\_2\_1024 & 16 & 2835.8 & 0.770 & 0.028 & 87.307 & 2358568.305 & 937.458 & 424.931 \\
\hline
street/NewYork\_0\_1024 & 22 & 2834.8 & 0.687 & 0.310 & 288.919 & 2479575.838 & 1320.881 & 555.569 \\
\hline
street/Shanghai\_2\_1024 & 26 & 2885.7 & 0.622 & 0.404 & 471.062 & 2162738.816 & 1554.052 & 817.105 \\
\hline
street/Shanghai\_0\_1024 & 22 & 2816.5 & 0.513 & 0.258 & 248.099 & 2191778.202 & 996.741 & 283.793 \\
\hline
street/Sydney\_1\_1024 & 24 & 2844.5 & 0.696 & 0.128 & 148.724 & 2069735.549 & 984.123 & 380.839 \\
\hline
da2/ht\_mansion2b & 30 & 776.2 & 0.912 & 1.748 & 320.723 & 31907.571 & 495.497 & 156.186 \\
\hline
da2/ht\_0\_hightown & 18 & 1061.9 & 0.906 & 0.876 & 285.493 & 106234.498 & 1043.069 & 282.288 \\
\hline
dao/hrt201n & 31 & 905.8 & 0.942 & 2.751 & 442.453 & 46090.908 & 637.365 & 198.965 \\
\hline
random/random512-10-1 & 69 & 1373.0 & 0.954 & 8.667 & 28126.967 & 462856.631 & 9403.338 & 20139.733 \\
\hline
room/32room\_000 & 41 & 1579.2 & 0.989 & 0.272 & 1005.513 & 1890675.108 & 1645.235 & 579.428 \\
\hline
room/16room\_000 & 69 & 1477.7 & 0.992 & 1.065 & 5260.004 & 1694568.685 & 3757.778 & 1755.947 \\
\bottomrule
\multicolumn{9}{p{18cm}}{
\footnotesize
All scenarios for each map are run, and every scenario is solved for the shortest any-angle path.
$r$ is the correlation coefficient between the number of turning points and the path cost for all scenarios in the map. 
$\rho$ is the map density, which is the ratio between the number of corners to the number of free cells on the map. 
$P$ is the largest number of turning points and $G$ is the largest path cost among all scenarios. 
RS+ refers to \rsp.
All scenarios can solve for the shortest paths, except for Theta* which is sub-optimal. \newline % for p column only
\textsuperscript{\textdagger} For Theta*, scenarios taking longer than 5s are not solved but counted into the average running-time. Each map is run only one time as it is significantly slower than other algorithms.
}
\end{tabular}
\label{r2:tab:average}
\end{sidewaystable}

\begin{sidewaystable}[!ht]
\centering
\caption{Average speed-ups for 3, 10, 20, and 30 turning points.}
\setlength{\tabcolsep}{3pt}
\begin{tabular}{ c c c c | c c c  | c c c | c c c }
\hline
\multirow{2}{*}{Map} & \multicolumn{3}{c | }{3 Turning Pts.} & \multicolumn{3}{c | }{10 Turning Pts.} & \multicolumn{3}{c | }{20 Turning Pts.} & \multicolumn{3}{c}{30 Turning Pts.} \\
\cline{2-13}
& $g_{3}$ & R/A & R/P & $g_{10}$ & R/A & R/P & $g_{20}$ & R/A & R/P & $g_{30}$ & R/A & R/P \\
\hline
dao/arena & 69.5 & 10.1 & 0.605 & -- & -- & -- & -- & -- & -- & -- & -- & -- \\
\hline
bg512/AR0709SR & 418.2 & 9.72 & 2.68 & 663.2 & 4.13 & 1.26 & -- & -- & -- & -- & -- & -- \\
\hline
bg512/AR0504SR & 242.3 & 7.95 & 3.66 & 747.6 & 3.87 & 1.56 & 774.9 & 2.28 & 0.881 & -- & -- & -- \\
\hline
bg512/AR0014SR & 231.0 & 5.42 & 3.02 & 584.1 & 2.6 & 1.41 & -- & -- & -- & -- & -- & -- \\
\hline
bg512/AR0304SR & 283.5 & 6.22 & 2.34 & 743.9 & 3.33 & 1.11 & -- & -- & -- & -- & -- & -- \\
\hline
bg512/AR0702SR & 220.5 & 6.63 & 2.29 & 727.3 & 4.24 & 1.31 & -- & -- & -- & -- & -- & -- \\
\hline
bg512/AR0205SR & 155.1 & 5.72 & 2.47 & 515.2 & 2.93 & 1.45 & 1081.6 & 2.74 & 0.899 & 1247.5 & 1.91 & 0.543 \\
\hline
bg512/AR0602SR & 162.5 & 5.96 & 2.3 & 456.7 & 2.27 & 1.05 & 951.0 & 1.74 & 0.642 & 1396.0 & 1.15 & 0.388 \\
\hline
bg512/AR0603SR & 212.1 & 4.68 & 2.43 & 585.4 & 2.02 & 0.963 & 1218.6 & 1.77 & 0.613 & 1650.8 & 1.36 & 0.432 \\
\hline
street/Denver\_2\_1024 & 774.3 & 17.4 & 6.59 & 2329.0 & 9.79 & 4.77 & -- & -- & -- & -- & -- & -- \\
\hline
street/NewYork\_0\_1024 & 865.3 & 7.19 & 5.34 & 1839.4 & 4.62 & 2.06 & 2556.4 & 3.48 & 1.43 & -- & -- & -- \\
\hline
street/Shanghai\_2\_1024 & 1025.3 & 7.29 & 4.51 & 1889.8 & 3.53 & 1.93 & 1790.3 & 2.38 & 1.41 & -- & -- & -- \\
\hline
street/Shanghai\_0\_1024 & 1371.7 & 7.03 & 2.6 & 1529.5 & 3.28 & 1.08 & 2036.5 & 3.65 & 0.954 & -- & -- & -- \\
\hline
street/Sydney\_1\_1024 & 878.0 & 9.46 & 4.36 & 1996.8 & 5.96 & 2.5 & 2315.0 & 5.75 & 1.71 & -- & -- & -- \\
\hline
da2/ht\_mansion2b & 59.4 & 7.31 & 1.47 & 246.5 & 2.4 & 0.901 & 585.1 & 1.52 & 0.515 & 763.0 & 2.07 & 0.615 \\
\hline
da2/ht\_0\_hightown & 134.4 & 6.83 & 2.53 & 576.4 & 4.37 & 1.34 & -- & -- & -- & -- & -- & -- \\
\hline
dao/hrt201n & 81.7 & 6.64 & 1.63 & 288.2 & 1.58 & 0.742 & 645.1 & 1.45 & 0.478 & 848.7 & 1.84 & 0.47 \\
\hline
random/random512-10-1 & 41.3 & 8.96 & 20.3 & 215.0 & 2.7 & 7.73 & 522.8 & 1.26 & 3.72 & 734.2 & 0.852 & 2.15 \\
\hline
room/32room\_000 & 79.0 & 7.83 & 1.75 & 344.8 & 2.47 & 1.2 & 732.2 & 1.9 & 0.792 & 1137.8 & 1.73 & 0.644 \\
\hline
room/16room\_000 & 42.9 & 6.41 & 1.78 & 186.0 & 1.76 & 1.08 & 396.4 & 1.3 & 0.837 & 614.5 & 1.1 & 0.677 \\
\hline
\multicolumn{13}{p{18cm}}{
\footnotesize
$g_i$ refers to the average path cost for the shortest paths with $i$ turning points.
For the respective turning points, R/A is the ratio of ANYA's average run-time to R2's average run-time, and R/P is the ratio of \rsp's average run-time to R2's. The higher the ratio, the higher the speed-ups.
}
\end{tabular}
\label{r2:tab:speedups1}
\end{sidewaystable}

\begin{figure*}[!ht]
\centering
\includegraphics[width=\linewidth]{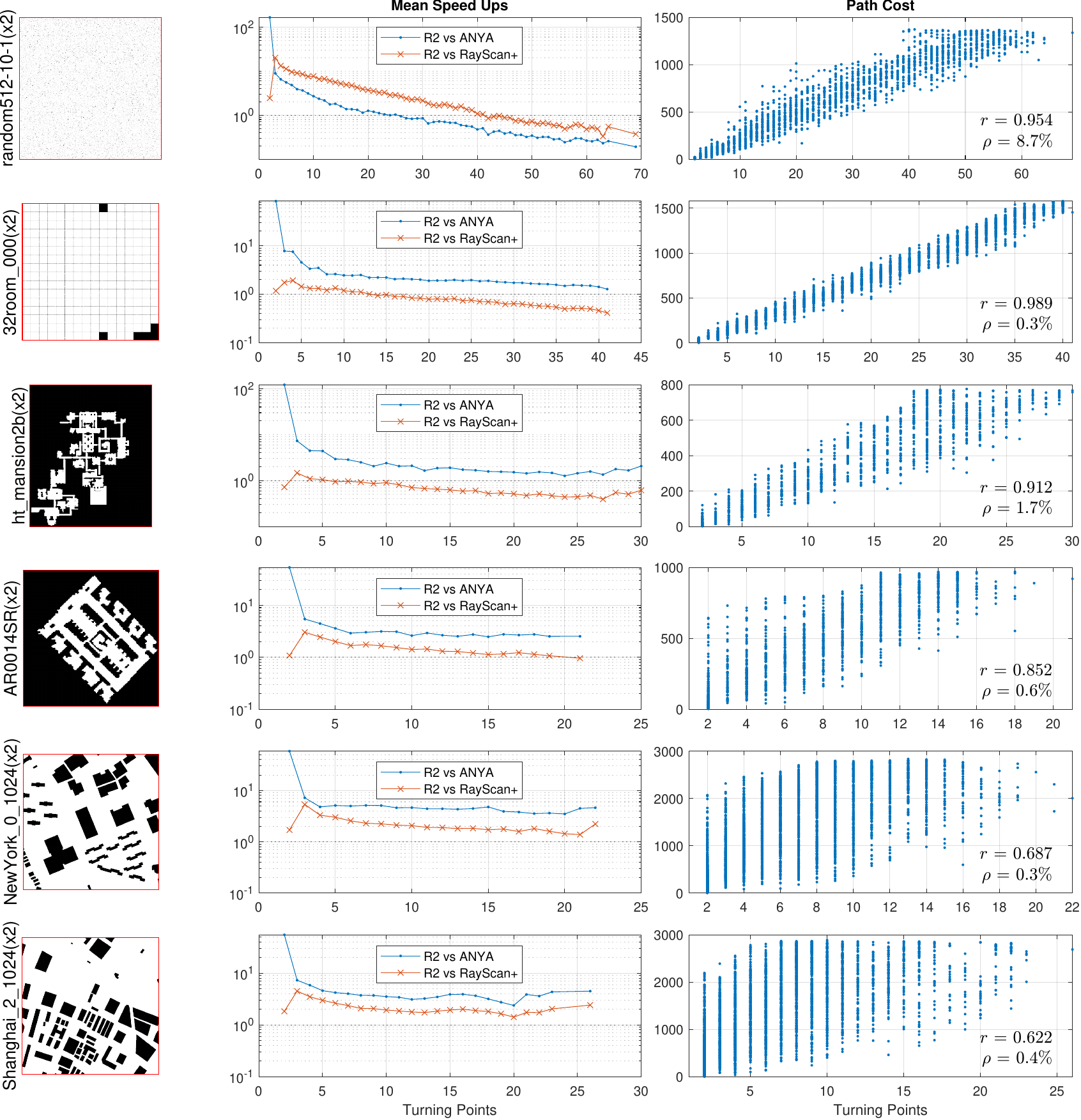}
\caption[Selected results for \rtwo{}.]
{
The thumbnails of the binary occupancy grid maps are shown on the left, with black pixels indicating obstacles.
The plots in the middle compare the speed-ups of \rtwo{} against the number of turning points in the shortest path. \rtwo{} runs faster if the speed-ups are larger than 1.
The plots on the right compare the path cost with the number of turning points in the shortest path for all scenarios. 
If the number of turning points and path cost correlates strongly, the map is likely to be dense and highly non-convex as points are less likely to have line-of-sight to points that are far away.
The start and goal points are considered turning points.
}
\label{r2:fig:results}
\end{figure*}
\clearpage

Plots in the middle column of Fig. \ref{r2:fig:results} indicate the performance of \rtwo{} with respect to \rsp{} and Anya for shortest path solutions with the same number of turning points.
The vertical axes indicate the \textbf{speed-up}, which is the number of times \rtwo{} is faster than \rsp{} or Anya, or the ratio of average run-time of \rsp{} or Anya to the average run time of \rtwo{}.
The speed-ups are averaged across scenarios where the shortest paths have the same number of turning points, regardless of path cost.
The start and goal points are considered turning points. 

Plots in the right column show the variation of shortest paths' costs with respect to the number of turning points the paths have, indicating how likely a shortest path solution will turn around obstacles in the map as the solution increases in length.
The ratio of corners to free cells ($\rho$) is indicated in the plots.

\rtwo{} is considerably faster than the other algorithms in sparse maps with few disjoint and non-convex obstacles.
The path costs and number of turning points correlate ($r$) strongly if points have line-of-sight to only its local neighborhood, as the path has to turn around more obstacles that block line-of-sight to farther points.
$r$ provides an indicator on how likely corners have line-of-sight to other corners, with a smaller $r$ indicating corners having line-of-sight to a larger number of other corners.
When corners have line-of-sight to other corners, the map contains fewer obstacles and obstacles are more likely to conform to their convex hull.
Collisions during line-of-sight checks are less likely to occur, decreasing the number of casts required in \rtwo{} before finding the shortest path, and causing the shortest path to contain fewer turning points.
As \rtwo{}'s run-time is exponential with respect to the number of casts, maps with fewer obstacles and non-convex obstacles makes \rtwo{} run faster than \rsp{}, especially since \rtwo{} delays line-of-sight checks to expand only the successors with the least angular deviation. 

The speed-up of \rtwo{} on sparse maps with few disjoint and non-convex obstacles is evident in Table \ref{r2:tab:average} and Table \ref{r2:tab:speedups1}.
Table \ref{r2:tab:average} shows the average run times in microseconds between different algorithms and the benchmark characteristics for selected maps, such as $r$, $\rho$.
Table \ref{r2:tab:speedups1} show the average speed-ups for scenarios with the same number of turning points (3, 10, 20 and 30). While the average run-time for all scenarios in the benchmark may be slower, \rtwo{} performs considerable faster than \rsp{} and Anya when the shortest path is less likely to turn around obstacles. Noteworthy is that the shortest path cost has little impact on the speed-up.

Vector-based planners \rtwo{} and \rsp{} outperform free-space planners like Anya as maps tends to have much fewer corners than free-space. 
The ratio is indicated by $\rho$ in Table \ref{r2:tab:average} and in Fig. \ref{r2:fig:results}.
By prioritising the shortest possible solution (a straight line), the speed-ups can be close to a hundred times if the start and goal points have line-of-sight, as indicated by the middle-column plots in Fig. \ref{r2:fig:results}.
$\rho$ is not linearly proportional to the speed-ups between Anya and the vector-based planners due to repeated searches along contours by the vector-based planners. 

\section{Conclusion}
A novel, any-angle and vector-based path planner \rtwo{} is introduced. 
\rtwo{} is optimal, complete and can work on maps with non-convex obstacles.
\rtwo{} delays line-of-sight checks to expand points that the least from the straight line between the start and goal points, and is much faster than the state-of-the-art algorithms Anya and \rsp{} when the optimal path is expected to have few turning points.
Such paths are more likely to occur on sparse maps with few disjoint and non-convex obstacles.

The best-hull is a novel mechanism to ensure that path costs increases monotonically and admissibly regardless of line-of-sight is introduced. Delayed line-of-sight checks can cause path costs to be severely underestimated due to pruning. To prevent the severe underestimate, \rtwo{} infers the smallest known convex hull (best-hull) from a traced contour. The best-hull informs queries about the past contour that was traced and the future contour to go around, allowing the path cost to be estimated admissibly without being too small.
As the best-hull increases in size as the trace progresses, the estimate increases monotonically regardless of line-of-sight.

The best-hull is constructed from turning points placed at convex corners, and phantom points placed at non-convex corners. Phantom points are imaginary turning points placed on a traced contour to guide future queries around the traced obstacle. Phantom points will not appear in the shortest path as they cannot be reached, and are always pruned after guiding the queries.

The best-hull increases in size and is kept convex by the progression, placement and pruning rules. The progression rule monitors the angular deviation of a path around a traced obstacle without measuring angles, activating the other rules only when the deviation increases. The placement rule places the turning and phantom points, and the pruning rule prunes points that lie within the best-hull.

While considerably faster than state-of-the-art planners for the aforementioned cases, \rtwo{} has limitations that future works can address.
Due to the delayed line-of-sight checks, \rtwo{} has exponential search times in the worst case with respect to the number of collided casts, and can be much slower than \rsp{} or Anya on maps with many non-convex or disjoint obstacles. 
To ensure that the pruning rule remain valid for target nodes, ad hoc points are introduced, but the points may produce interminable chases when a path does not exist.
    \setcounter{algorithm}{0}
\renewcommand{\thealgorithm}{\thechapter.\arabic{algorithm}}

\chapter{\rtwop{}: Simplifying and Speeding Up \rtwo{} in Dense Maps with Disjoint Obstacles.}

\label{chap:r2p}
\rtwo{} is a novel vector-based algorithm that delays \gls{los} checks to find paths.
\rtwo{}'s search complexity is largely dependent on the number of collided casts, and less dependent on the distance between the two search points.
If a path is expected to have few turning points, the path can be found very rapidly, regardless of the distance between the points.
Such paths are likely to occur in maps with few disjoint obstacles, and in maps with few highly non-convex obstacles.

% Due to delayed \gls{los} checks, \rtwo{} is exponential in the worst case with respect to the number of collided casts.
% If the \gls{los} between nodes are unknown, costs cannot be verified. \rtwp{} would be unable to discard overlapping paths, causing the number of searches to grow exponetnially.
% As such, \rtwo{} performs poorly in maps containing many disjoint obstacles and overlapping non-convex obstacles, such as a labyrinth with narrow corridors, or a map generated by a uniform random number generator.
% While such maps are unlikely to occur in practice, the exponential worst case complexity can cause \rtwo{} to perform slower than other online any-angle algorithms that uses binary occupancy grids, such as \rsp{} and ANYA.

As \rtwo{} has exponential search time in the worst case, \rtwop{} introduces new conditions to the overlap rule to reduce the number of expensive searches and improve the averages search time. Additionally, \rtwop{} simplifies the algorithm by (i) guaranteeing target and source progression at the start of the trace, which simplifies the progression rule and removes a complicated tracing phase in \rtwo{}; (ii) replacing the ad-hoc points  in \rtwo{} with a simple rule to limit recursive traces from target nodes; and (iii) replacing the fundamental search unit from the \textit{node} (placed at one point) to a \textit{link} (connects two points) to provide more clarity in the search process.

% To improve average search time, the overlap rule is introduced in \rtwo{}. 
% The overlap rule discards paths that are guaranteed to remain costly as the search progresses.
% By discarding costlier paths, \rtwo{} is asymptotically polynomial in search time with respect to the number of collided casts.

% Another problem with \rtwo{} are the complicated conditions required to add ad hoc points to alleviate the interminability of \rtwo{} when no path can be found.

% \rtwop{} introduces solutions to address the problems in \rtwo{}. 
% Additional geometrical checks to the overlap rule, resulting in better performance on maps with many disjoint obstacles.
% The ad hoc points $\mnode_{ad,b}$ and $\mnode_{ad,c}$ were superseded by limited recursive occupied-sector traces from target nodes.
% In addition, \rtwop{} simplifies \rtwo{} by organizing the search into two search trees, providing more clarity to the search structures.

\section{Concepts in \rtwop{}}
\rtwop{} relies on casts and traces to find the shortest path. 
Like \rtwo{}, \rtwop{} delays line-of-sight checks to expand turning points with the least deviation from the straight line between the start and goal points.
\rtwop{} is an evolved algorithm of \rtwo{}, primarily focusing on improving search time in maps  many disjoint obstacles.
This section describes the nomenclature and structures used in \rtwop{}.

The \textbf{tree-direction} determines the direction of an object along a path from the start point to goal point.
Consider two objects $a$ and $b$. If $a$ lies in the \textbf{source} direction of $b$, $a$ leads to the start point from $b$. Conversely, $b$ lies in the \textbf{target} direction of $a$, as $b$ leads to the goal from $a$.

\rtwop{} relies on two search trees connected at their leaf nodes.
The \textbf{source-tree} ($S$-tree) is rooted at the start point, and the \textbf{target-tree} ($T$-tree) is rooted at the goal point. 
An edge connecting two points in the trees is called a \textbf{link}.
Links enables data like sector-rays, progression rays to be organized more neatly than points, and prevent unnecessary line-of-sight checks.

\input{chap_r2p/fig_tree.tex}

While connected to two points, each link is \textbf{anchored} to only one point.
The anchored point of a link in the $S$-tree is the target point, or the source point if the link is in the $T$-tree. 
The anchored point is the \textbf{leaf point} of the link, and the other point is the \textbf{root point}.
A \textbf{root link} for an $S$-tree link or $T$-tree link refers to the connected link in the source or target direction respectively, and a \textbf{leaf link} refers to the connected link in the opposite direction.

If a link has \textbf{cumulative visibility}, there is an unobstructed path from the anchored point to the start point or goal point.
If the link lies in the $S$-tree, there is an unobstructed path to the start point; if the link lies in the $T$-tree, there is an unobstructed path to the goal point.
To describe the state, the \textbf{link type} is used. The link types are explained in Table \ref{r2p:tab:nodetypes}.

\begin{table}[!ht]
\centering
\caption{Link types in \rtwop{}.}
\label{r2p:tab:nodetypes}
\setlength{\tabcolsep}{3pt}
\begin{tabular}{ c  c  p{5.5cm}}
\textbf{Type} & \textbf{Sm.} & \textbf{Description of Anchored Point.} \\
\hline
$\mnvy$ & \raisebox{0pt}{
    \tikz{
        \clip (-1.1mm, -1.1mm) rectangle ++(2.2mm, 2.2mm);
        \node [vy pt={}{}{}] at (0,0) {};
    }
} & Turning point with cumulative visibility. Path via link has cheapest verified cost known so far.  \\
\hline
$\mnvu$ & \raisebox{0pt}{
    \tikz{
        \clip (-1.1mm, -1.1mm) rectangle ++(2.2mm, 2.2mm);
        \node [vu pt={}{}{}] at (0,0) {};
    }
}  & Turning point with unknown cumulative visibility. \\
\hline
$\mney$ & \raisebox{0pt}{
    \tikz{
        \clip (-1.4mm, -1.4mm) rectangle ++(2.8mm, 2.8mm);
        \node [ey pt={}{}{}] at (0,0) {};
    }
}  & Turning point with cumulative visibility. Path via link is costlier than the cheapest known so far. \\
\hline
$\mneu$* & \raisebox{0pt}{
    \tikz{
        \clip (-1.4mm, -1.4mm) rectangle ++(2.8mm, 2.8mm);
        \node [eu pt={}{}{}] at (0,0) {};
    }
} & Turning point with unknown cumulative visibility that has an ancestor $\mney$ root link. \\
\hline
% $\mnph$* & \input{chap_r2p/sym_ph.tex} & A phantom point, or temporary node placed when a trace is interrupted. \\
% \hline
$\mntm$\textsuperscript{\textdagger} & \raisebox{0pt}{
    \tikz{
        \clip (-1.1mm, -1.1mm) rectangle ++(2.2mm, 2.2mm);
        \node [tm pt={}{}{}] at (0,0) {};
    }
} & A temporary point that is placed when a trace is interrupted. \\
\hline
$\mnun$\textsuperscript{\textdagger} & \raisebox{0pt}{
    \tikz{
        \clip (-1.1mm, -1.1mm) rectangle ++(2.2mm, 2.2mm);
        \node [un pt={}{}{}] at (0,0) {};
    }
}  & An unreachable phantom point that discards queries when reached. \\
\hline
$\mnoc$\textsuperscript{\textdagger} & \raisebox{0pt}{
    \tikz{
        \clip (-1.05mm, -1.05mm) rectangle ++(2.1mm, 2.1mm);
        \node [oc pt={}{}{}] at (0,0) {};
    }
} & A turning point or phantom point placed by a target recursive-angular sector trace. \\
\hline
\multicolumn{3}{p{7.5cm}}{
\footnotesize
The \textit{Sm.} column denotes the symbol used in figures.
*$\mneu$ links appear only in the $S$-tree.
\textsuperscript{\textdagger}$\mntm$, $\mnun$, and $\mnoc$ links appear only in the $T$-tree.
}
\end{tabular}
\end{table}

\begin{table}[!ht]
\centering
\caption{Legend of symbols used in figures.}
\label{r2p:tab:legend}
\setlength{\tabcolsep}{3pt}
\begin{tabular}{ c  p{5.5cm}}
\textbf{Sm.} & \textbf{Description} \\
\hline
\raisebox{-.5ex}{
    \tikz[]{
        \clip (-1.6mm, -1.6mm) rectangle ++(13.2mm, 3.2mm);
        \node (n1) [any pt] at (0,0) {\footnotesize 1};
        \node (n2) [any pt] at (1cm,0) {\footnotesize  2};
        \draw [link] (n1) -- (n2); 
    }
} & A link anchored at point 1, connected to links (not shown) anchored at point 2.  \\
\hline
\raisebox{-.5ex}{
    \tikz{
        \clip (-1.6mm, -1.6mm) rectangle ++(13.2mm, 3.2mm);
        \node (n1) [any pt] at (0,0) {\footnotesize 1};
        \node (n2) [any pt] at (1cm,0) {\footnotesize  2};
        \draw [qlink] (n1) -- (n2); 
    }
} & Same as above, and the link is associated with a queued query. \\
\hline
\raisebox{-5.5mm}{
    \tikz[]{
        % \clip (-4.2mm, -2.5mm) rectangle ++(8.4mm, 5mm);
        % \node [merge, inner xsep=3.9mm, inner ysep=2.2mm] at (0,0) {};
        % \node (n1) [any pt, xshift=-1.6mm] at (0,0) {\footnotesize 1};
        % \node (n2) [any pt, xshift=1.6mm] at (0,0) {\footnotesize  2};
        \clip (-4.3mm, -4.5mm) rectangle ++(8.6mm, 9mm);
        % \node [separate, circle, minimum size=0, inner sep=3mm] at (0,0) {};
        \node (n1) [any pt, xshift=-2.5mm] at (0,0) {\footnotesize 1};
        \node (n2) [any pt, xshift=1.2mm, yshift=1.9mm] at (0,0) {\footnotesize  2};
        \node (n3) [any pt, xshift=1mm, yshift=-2mm] at (0,0) {\footnotesize  3};
        \draw [merge] (n1) -- (n2);
        \draw [merge] (n1) -- (n3);
    }
} & Links at (1) are connected to links at (2) and (3), and links at (2) are not connected to links at (3). Links in (2) have a different type as links in (3). \\
\hline
\raisebox{-5.5mm}{
    \tikz[]{
        \clip (-4.3mm, -4.5mm) rectangle ++(8.6mm, 9mm);
        \node [separate, inner xsep=3.9mm, inner ysep=4.2mm] at (0,0) {};
        \node (n1) [any pt, xshift=-2mm] at (0,0) {\footnotesize 1};
        \node (n2) [any pt, xshift=1mm, yshift=1.8mm] at (0,0) {\footnotesize  2};
        \node (n3) [any pt, xshift=0.9mm, yshift=-1.9mm] at (0,0) {\footnotesize  3};
    }
} & Links are anchored at the same corner, and with the example above. \\
\hline
\raisebox{-3mm}{
    \tikz[]{
        % \clip (-3mm, -1.6mm) rectangle ++(6mm, 3.2mm);
        \pic at (0,0) {trace grp={}{}{}{}};
        % \node [separate] at (0,0) {};
        % \node [trtm pt={}{}{}, xshift=-\uss] at (0,0) {};
        % \node [trtm pt={}{}{}, xshift=\uss] at (0,0) {};
    }
} & A moving trace point that anchors disconnected $S$-tree and $T$-tree link. \\
\hline
\raisebox{-.5ex}{
    \tikz{
        \clip (-1.6mm, -1.6mm) rectangle ++(11.6mm, 3.2mm);
        \node (n1) [any pt] at (0,0) {\footnotesize 1};
        \draw [rayl] (n1) -- ++(0:1cm); 
    }
} & Left sector-ray of an angular-sector at point 1.\\
\hline
\raisebox{-.5ex}{
    \tikz{
        \clip (-1.6mm, -1.6mm) rectangle ++(11.6mm, 3.2mm);
        \node (n1) [any pt] at (0,0) {\footnotesize 1};
        \draw [rayr] (n1) -- ++(0:1cm); 
    }
} & Right sector-ray of an angular-sector at point 1.\\
\hline
\raisebox{-.5ex}{
    \tikz{
        \clip (-1.6mm, -1.6mm) rectangle ++(11.6mm, 3.2mm);
        \node (n1) [any pt] at (0,0) {\footnotesize 1};
        \draw [rayprog] (n1) -- ++(0:1cm); 
    }
} & Progression ray with respect to point 1.\\
\hline
\textcolor{swatch_stree}{$S$-tree} & $S$-tree objects are colored red. \\
\hline
\textcolor{swatch_ttree}{$T$-tree} & $T$-tree objects are colored green. \\
% \hline
% \multicolumn{2}{p{7.5cm}}{
% \footnotesize
% The \textit{Sm.} column denotes the symbol used in figures.
% }
\end{tabular}
\end{table}

A \textbf{query} in \rtwop{} refers to a cast or trace. A query can be found for every connected pair of $S$-tree and $T$-tree links, at the leaf points of the $S$-tree and $T$-tree.

Tables \ref{r2p:tab:nodetypes} and \ref{r2p:tab:legend} illustrate the symbols used in figures. 
Fig. \ref{r2p:fig:tree} describes the trees with respect to nodes, links and queries. % TODO show tracing query transfer.

\section{Evolving \rtwo{} to \rtwop{}}
% \rtwop{} is similar to \rtwo{}.
% The same tracing rules are used -- the progression, pruning, angular-sector, occupied-sector, and placement rules.
% Each trace infers the \textbf{best-hull}, which is smallest convex hull known of the traced obstacle, to allow  cost estimates to increase monotonically.
% Two traces in the opposite direction ($L$ and $R$ traces) are generated by collided casts from the collision point, and a \textbf{third-trace} from the source node may be generated if the destination is the goal point.
% The algorithms delay line-of-sight checks and cast greedily to target nodes, verifying line-of-sight only when paths overlap.
% The shortest path is found when a cast reaches a $\mnvy$ $T$-tree node from a $\mnvy$ $S$-tree node.

The following subsections describe the changes made to evolve \rtwo{} to \rtwop{}. 
% In \rtwop{}, phantom points that are on the same best-hull as a trace cannot be cast to by the trace (Sec. \ref{r2p:sec:phantom});
In \rtwop{}, short occupied-sector traces from target nodes in \rtwop{} supersedes the ad hoc points from \rtwo{} (Sec. \ref{r2p:sec:tgtocsec});
the complicated tracing phase before a recursive trace from the source point in \rtwo{} is replaced by simpler corrective steps in \rtwop{} (Sec. \ref{r2p:sec:tgtprog});
the interrupt rule counts corners in \rtwop{} instead of nodes placed (Sec. \ref{r2p:sec:interrupt});
and the overlap rule is modified to include additional conditions to discard expensive paths (Sec. \ref{r2p:sec:overlap}).

\subsection{Limited, Target Recursive Occupied-Sector Trace} \label{r2p:sec:tgtocsec}
% Recursive occupied-sector traces for target nodes (\textbf{target oc-sec trace}) are not implemented in \rtwo{} due to chases \cite{bib:r2}. 
% A chase occurs when two traces in the same direction try to cast to each other but are unable to do so as the traces are on the same contour.
% (i) a recursive trace from a phantom point can lead to the wrong path; and (ii) a recursive occupied sector trace from a target node can try to reach the calling trace if both are on the same contour, resulting in both traces trying to reach each other (a chase). 
% For \rtwo{} to be complete in the absence of target oc-sec traces, two ad hoc points $\mnode_{ad,b}$ and $\mnode_{ad,c}$ are introduced in \rtwo{}.
% While the ad hoc points can reduce the number of chases for \rtwo{} to find a path, the points do not eliminate chases completely, and \rtwo{} can be interminable if no path exists.

A limited recursive occupied-sector trace from target points is implemented in \rtwop{} in place of the ad hoc points in \rtwo{}. 
Ad hoc points are ad hoc solutions that attempt to address interminable traces that occur after a full recursive occupied-sector trace from target point takes place.
Ad hoc points require complicated conditions, while a limited recursive occupied-sector trace simply inserts a $\mnoc$ type link.
While both solutions do not fully address the interminability of \rtwop{} when no path exists, the limited recursive trace is much simpler to implement than ad hoc points.
Future works can address the interminability of \rtwop{} when no path can be found.

A limited recursive occupied-sector trace inserts a $T$-tree $\mnoc$ type link when a trace enters the occupied-sector of a target point. Fig. \ref{r2p:fig:tgtocsec} illustrates the $\mnoc$ link being inserted. 
The $\mnoc$ link prevents further recursive traces from occurring, and therefore prevents any interminable chases that occur when the recursive trace and the current trace try to reach each other.

\input{chap_r2p/fig_tgtocsec}

% To  simplify \rtwop{}, limited target oc-sec traces are implemented, and the ad hoc points $\mnode_{ad,b}$ and $\mnode_{ad,c}$ are removed.
% If the target node is $\mside$-sided, the limited trace ends at the first $(-\mside)$-side corner from the target node, where a new $\mnoc$ node is placed.
% The $\mnoc$ node prevents a subsequent oc-sec trace from occurring, especially if the calling tracing query is on a different best-hull. 
% If the query arrives at the same contour as the oc-sec trace, a chase can occur.
% Fig. \ref{r2p:fig:tgtocsec} illustrates a limited oc-sec trace.

% While a limited recursive trace is much simpler to implement than ad hoc points and eliminates chases, the current solution does not ensure terminability when a path cannot be found. A future work may seek to address this problem.

% A recursive occupied-sector trace will never occur from a phantom point that is placed by the same trace.
% A phantom point has the same side as the trace that placed it, and the occupied-sector rule will not examine nodes that has the same side as a trace.
% If the phantom point belongs to a different best-hull, it is placed by a different trace that was interrupted before the current trace. 
% A recursive trace is allowed in this case, but no special instruction is required because such a phantom point would have been converted to a $\mntm$ point in Sec. \ref{r2p:sec:phantom}.

% In \rtwop{}, the ad hoc point $\mnode_{ad,a}$ from \rtwo{} is superseded by an $\mnun$ node.

\subsection{Ensuring Target Progression} \label{r2p:sec:tgtprog}

The \textbf{angular progression} of a trace is the angle deviated from the collided cast that resulted the trace with respect to the source point or a target point of the trace.
The \textbf{target progression} refers to the angle deviated with respect to a target point, while \textbf{source} progression refers to the angle deviated with respect to the source point.

In \rtwo{} and \rtwop{}, angular progression is ensured at the start of a trace by ensuring progression with respect to the source and target points when a trace is interrupted.
Ensuring progression at the start of a trace is critical to ensure that the pruning, placement, and sector rules can function correctly.

In \rtwo{}, a complicated tracing phase is required to ensure target progression when a trace is interrupted for a recursive trace. 
\rtwop{} simplifies the problem by replacing the complicated phase with two solutions to avoid interruptions where there will not be target progression.

Additionally, both solutions ensure that the source and target progressions can never decrease by more than $180^\circ$. As such, the winding counter in the progression is no longer required and can be removed from implementation.

\input{chap_r2p/fig_tgtprog1}
\input{chap_r2p/fig_tgtprog2}

% In \rtwo{}, when recursive angular-sector (\textbf{ang-sec}) and oc-sec traces are called from source nodes (\textbf{source recursive traces)}, the calling trace may not have progressed with respect to the target nodes (no \textbf{target progression}).
% If the recursive traces proceed, the target nodes may be incorrectly pruned in a subsequent trace, and the algorithm would become incomplete.
% % For brevity, let the term \textit{target progression} represent the situation when a trace is progressed with respect to all of its target nodes, and \textit{source progression} represent the situation when a trace is progressed with respect to its source node.
% To ensure target progression, \rtwo{} relies on a special tracing phase once a recursive call is identified.
% As the phase requires special conditions to move to a corner where there is target progression, and to backtrack once target progression is achieved, the special phase complicates the algorithm.

The first solution places an unreachable $\mnun$ link at the start of a recursive angular-sector trace. The $\mnun$ link allows for target progression when there is no target progression at the initial edge of a recursive angular-sector trace. The solution is illustrated in Fig. \ref{r2p:fig:tgtprog1}.

The second solution queues a cast when the source progression has decreased by more than $180^\circ$ in non-convex obstacles. 
When this occurs, a cast is queued between the source point and the target point of the trace. 
There is only one target point, as it is a phantom point at a non-convex corner where the source progression stops increasing. The solution is illustrated in Fig. \ref{r2p:fig:tgtprog2}.

\subsection{Interrupt Rule} \label{r2p:sec:interrupt}
The interrupt rule interrupts traces for queuing, so as to avoid expanding long, non-convex contours that are unlikely to find the shortest path..

In \rtwo{} a trace is interrupted and queued after several points are placed, and the check occurs within the placement rule.
To simplify the algorithm, \rtwop{} interrupts and queues the trace after several corners are traced instead. 
The check occurs before the placement rule, and is called the \textbf{interrupt rule}.

A trace that calls a recursive angular sector trace or recursive occupied sector trace will have to be interrupted, but it is not interrupted by the interrupt rule. The trace is interrupted by the angular sector rule and occupied sector rule respectively.

\subsection{Overlap Rule} \label{r2p:sec:overlap}
The overlap rule is a broad set of instructions dictating how \rtwo{} and \rtwop{} behave when paths from different queries overlap, depending on the \textbf{overlap conditions} being triggered.
The overlap rule greedily verifies line-of-sight for the overlapping paths, which provides more confidence for the algorithm to estimate the cost of the paths and discard paths that will not lead to the optimal solution.
The overlap rule plays a critical role in reducing the number of queries, which blows up exponentially in the worst case due to delayed line-of-sight checks.

% As \rtwo{} delays line-of-sight checks, some visible turning points cannot be found immediately.
% Expensive queries cannot be discarded without more specific conditions, and the conditions are specified by the rule.
% The rule discards 
\subsubsection{\rtwo{}'s Overlap Rule}
In \rtwo{} and \rtwop{}, a path that satisfies overlap conditions \textit{O1}, \textit{O2}, or \textit{O3} of the overlap rule will be discarded.
Conditions \textit{O4}, \textit{O5}, \textit{O6}, and \textit{O7} are introduced in \rtwop{} to be more effective at discarding paths.
Conditions \textit{O6} and \textit{O7} are similar to a path pruning rule in \cite{bib:dps}.

\input{chap_r2p/fig_overlap1}
Condition \textbf{O1} is triggered when overlapping paths are detected. 
Upon detection, the overlap rule shifts the queries down the affected paths, toward the start point of the $S$-tree, to verify line-of-sight.
The purpose of moving the queries down the $S$-tree is to verify the minimum cost-to-come of the affected links, so that expensive paths can be safely discarded, which improves search time. 

Fig. \ref{r2p:fig:overlap1} illustrates the shifting of the queries when condition O1 is satisfied.
$S$-tree links in the overlapping paths are converted to $T$-tree links, until the most recent $S$-tree link with cumulative visibility ($\mney$ or $\mnvy$ type) is reached.
A cast is subsequently queued in the affected target link of the $S$-tree $\mney$ or $\mnvy$ link.

% In \textbf{Case O1} (see Fig. \ref{r2p:fig:overlap1}), a query passes through a corner that contains other paths,  the $S$-tree is shrunk, and for each overlapping path, a cast is queued on the earliest source link that has no verified cost.
% When a query passes through a corner that anchors links from other paths, an overlap is identified.
% The purpose of the rule is to verify cost-to-come and discard expensive paths. 
% As such, every link that is anchored at an $S$-tree $\mnvu$ or $S$-tree $\mneu$ node at the corner are identified.
% For each link, the algorithm moves down the $S$-tree along each path, until the first link that has a source $S$-tree $\mnvy$ or $S$-tree $\mney$ node is found.
% The algorithm subsequently moves up the tree, removing queries to avoid data races and re-anchoring the target links to $T$-tree $\mnvu$ nodes.
% Finally, a cast is queued at the first link to verify cost for the target links.

% A tracing query is interrupted and shifted closer to the start node if a turning point is placed at a corner containing turning points from other queries. 
% The query is shifted to the first node in the source direction with cumulative visibility to the start node (an $\mnvy$-node). 
% Other queries in the target direction of the shifted query are discarded to avoid data races.
% \label{r2p:enum:overlap1}

\tikzset{
    pics/overlap2/.style={ code={ 
        % \begin{pgfonlayer}{background}
        %     \fill [swatch_obs] % obs a
        %         (0, -\ul) -- ++(90:6*\ul) -- ++(0:#1) coordinate (xa1) -- ++(-90:1.5*\ul) coordinate (xa2) -- ++(180:{#1-\ul}) -- ++(-90:4.5*\ul);
        %     \fill [swatch_obs] % obs b
        %         (2*\ul, \ul) -- ++(90:\ul) -- ++(0:1.5*\ul) coordinate (xb1) -- ++(-90:\ul);
        % \end{pgfonlayer}
        % \path (\ul, 6*\ul) coordinate (xtgt);
        % \node (nsrc) [svy pt] at (4.5*\ul, 0) {};
        % \node (ntgt) [tvu pt] at (xtgt) {};
        % \draw [tlink] (ntgt) edge ++(180:2*\ul);
        % \draw [slink] (nsrc) edge ++(-90:2*\ul);
        
        \begin{pgfonlayer}{background}
            \fill [swatch_obs] % obs b
                (1*\ul, 3.5*\ul) coordinate (xb1) -- ++(90:\ul) -- ++(0:2.5*\ul) coordinate (xb3) -- ++(-90:\ul);
            \fill [swatch_obs] % obs a
                (2.5*\ul, 1*\ul) -- ++(90:\ul) -- ++(0:1.5*\ul) coordinate (xa3) -- ++(-90:\ul);
        \end{pgfonlayer}
        \path (4*\ul, 0) coordinate (xsrc);
        \path (0.5*\ul, 7*\ul) coordinate (xtgt);

        \node (nsrc) [svy pt] at (xsrc) {};
        \node (ntgt) [tvu pt] at (xtgt) {};

        \draw [slink] (nsrc) -- ++(-120:2*\ul);
        \draw [tlink] (ntgt) -- ++(180:2*\ul);
    }},
}
\tikzset{
    pics/overlap2i/.style={code={
        \begin{pgfonlayer}{background}
            \fill [swatch_obs] % obs c
                (0, 5.5*\ul) -- ++(90:\ul) -- ++(0:1.5*\ul) -- ++(90:1*\ul) coordinate (xc1) -- ++(0:1*\ul) coordinate (xc2) -- ++(-90:2*\ul) coordinate (xc3);
        \end{pgfonlayer}
    }},
}
\tikzset{
    pics/overlap2ii/.style={code={
        \begin{pgfonlayer}{background}
            \fill [swatch_obs] % obs c
                (0, 5.5*\ul) -- ++(90:\ul) -- ++(0:4*\ul) coordinate (xc2) -- ++(-90:\ul) coordinate (xc3);
        \end{pgfonlayer}
    }},
}

\begin{figure}[!ht]
\centering
\subfloat[\label{r2p:fig:overlap2a} ] {%
    \centering
    \begin{tikzpicture}[]
        \clip (0, -\u) rectangle ++(5*\ul, 8.5*\ul);
        \pic at (0, 0) {overlap2};
        \pic at (0, 0) {overlap2i};

        \pic at (xa3) {merge grp={shift={(-4mm, -2.5mm)}}{center:$\mx_2$}{na3 tvu}{tvu pt}{na3 sey}{sey pt}};
        
        \node (nb1) [tvu pt={shift={(-2mm, -2mm)}}{center:$\mx_3$}] at (xb1) {};
        \node (nb3) [tvu pt={shift={(2mm, 2mm)}}{center:$\mx_4$}] at (xb3) {};

        \draw [slink] (na3 sey) edge (nsrc);
        \draw [tqlink] (na3 tvu) edge (nb1);
        \draw [tlink] (nb1) -- ++(140:3*\ul);
        \draw [tqlink] (na3 tvu) edge (nb3);
        \draw [tlink](nb3) edge (ntgt);
    \end{tikzpicture}
} \hfill
\subfloat[\label{r2p:fig:overlap2b}] {%
    \centering
    \begin{tikzpicture}[]
        \clip (0, -\u) rectangle ++(5*\ul, 8.5*\ul);
        \pic at (0, 0) {overlap2};
        \pic at (0, 0) {overlap2i};

        \draw [dotted] (nb3) -- (ntgt)
            node (xcol) [coordinate, pos=2/5] {};
        \draw [trace] (xcol) -- (xc3) -- (xc2) -- (xc1) -- ++(-90:\u);
        \node [cross pt] at (xcol) {};
        
        \node (na3) [sey pt={shift={(3mm, -2mm)}}{center:$\mx_2$}] at (xa3) {};
        \node (nc2) [tvu pt={}{}] at (xc2) {};

        \pic at (xb3) {merge grp={shift={(6mm, 0mm)}}{center:$\mx_4$}{nb3 tvu}{tvu pt}{nb3 sey}{sey pt}};
        \node (nc1) [tvu pt] at (xc1) {};

        \draw [slink] 
            (nb3 sey) edge (na3)
            (na3) edge (nsrc);
        \draw [tqlink] (nb3 tvu) 
            edge[bend right=10] 
            % node [pos=0.5, label={[swatch_ttree, rotate=20, shift={(2.5mm,0)}] center:$\mlink_5$}] {} 
            (nc2.-45);
        \draw [tlink] 
            (nc2) edge[bend right=45] (nc1)
            (nc1) edge (ntgt);
            
        % \draw [tqlink] (na3 tvu) edge (nb1);
        % \draw [tlink] (nb1) -- ++(170:3*\ul);
        % \draw [tqlink] (na3 tvu) edge (nb3);
        % \draw [tlink](nb3) edge (ntgt);
    \end{tikzpicture}
} \hfill
\subfloat[\label{r2p:fig:overlap2c}] {%
    \centering
    \begin{tikzpicture}[]
        \clip (0, -\u) rectangle ++(5*\ul, 8.5*\ul);
        \pic at (0, 0) {overlap2};
        \pic at (0, 0) {overlap2ii};

        \draw [dotted] 
            (nsrc) -- (na3) -- (nb3) 
            (nb3) edge
                node (xcol) [coordinate, pos=2/5] {}
                (ntgt);
        \draw [trace] (xcol) -- (xc3) -- (xc2) -- ++(180:\u);
        \node [cross pt] at (xcol) {};
        
        % \pic at (xc2) {merge grp={shift={(0, 3.5mm)}}{center:$\mx_5$}{nc2 tvu}{tvu pt}{nc2 svu}{svu pt}};
        \pic at (xc2) {merge grp={}{}{nc2 tvu}{tvu pt}{nc2 svu}{svu pt}};
        
        \draw [tqlink] 
            (nc2 tvu) edge (ntgt);
        \draw [slink]
            (nc2 svu) edge[bend left=5] (nsrc);

    \end{tikzpicture}
}    
\caption[Case O2 of \rtwo{}'s and \rtwop{}'s overlap rule.]
{
    Case O2 of the overlap rule handles queries with expensive cost-to-come paths. 
    (a) After a successful cast to $\mx_2$, the path is found to have a larger cost-to-come than the minimum at $\mx_2$, and the target node is replaced by an $S$-tree $\mney$ node.
    (b) If a cast from an $\mney$ node is successful, the target node is replaced by an $S$-tree $\mney$ node ($\mx_3, \mx_4$). 
    If consecutive $\mney$ nodes have different sides, the path is discarded ($\mx_3$).
    An unsuccessful cast will generate a trace with the same side as the $\mney$ node ($\mx_4$) and call Case O1 when the trace becomes castable.
    (c) The trace resumes normal behavior after all $\mney$ source nodes ($\mx_2, \mx_4$) are pruned from the path.
}
\label{r2p:fig:overlap2}
\end{figure}
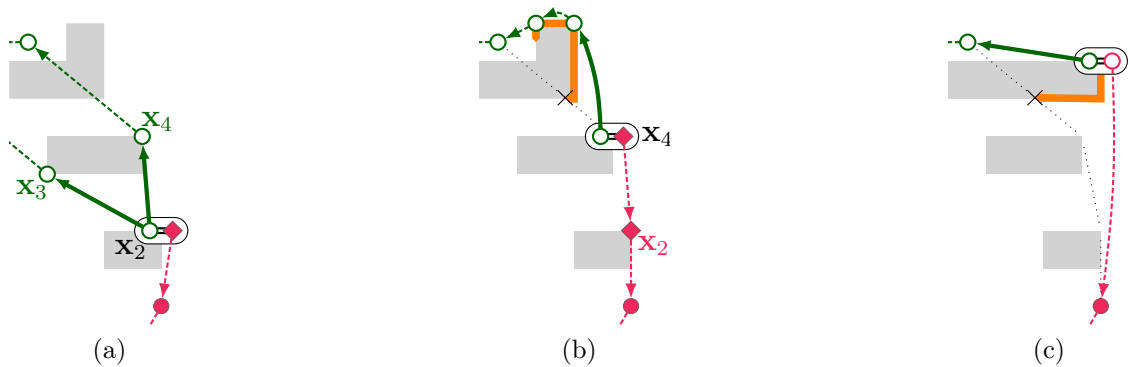
Condition \textbf{O2} is satisfied when a cast identifies a link as an $S$-tree link that is $\mney$ type. 
At the anchored point of the link, there is cumulative visibility, and the path described by the link has a more expensive cost-to-come than the cheapest known so far.
The overlap rule ensures that a trace with a different side from the point cannot be generated subsequently, as a path found by the trace will always be expensive.
In addition, a subsequent, same-sided trace can only place $\mneu$ turning points, and once the trace can cast again, condition O1 will be triggered to greedily verify line-of-sight.

Fig. \ref{r2p:fig:overlap2} illustrates the trace being generated when a cast over a target link of the $S$-tree $\mney$ link collides. 
The trace has the same side as the anchored point of the $S$-tree $\mney$ link, and the trace with a different side from the point is not generated.

The same sided-trace is necessary for the algorithm to find an optimal path.
A subsequent trace caused by the same-sided trace may prune the $\mney$ link, and the resulting path will no longer contain expensive links. 
By actively triggering condition O1 after a same-sided trace and verifying line-of-sight, the algorithm seeks to prune the $\mney$ link as soon as possible.

Conversely, a subsequent trace from the different-sided trace will never be able to prune the $\mney$ link due to the angular-sector rule and pruning rule, and should be discarded. 
Likewise, any subsequent query from a connected target link, which anchors a point with a different side from the anchor point of the $\mney$ link, will only result in an expensive path, and should be discarded.

% In \textbf{Case O2}, a successful cast finds an expensive cost-to-come path at the target node's corner,
% causing the target node to be replaced by an $S$-tree $\mney$ node.
% A subsequent cast from a $\mside$-side $S$-tree $\mney$ node that collides will generate only an $\mside$-side trace (see Fig. \ref{r2p:fig:overlap2}). 
% Once the trace is able to cast to a target node, Case O1 will be called on its path, and a cast is queued on the most recent link with a parent $\mney$ node.

% A subsequent, successful cast from a $S$-tree $\mney$ node will cause the target node of the cast to be replaced by an $S$-tree $\mney$ node regardless of the cost.
% If the replacement results in a pair of consecutive $\mney$ nodes with different sides, the path over the nodes will be discarded.

% A trace with a $\mney$ or $\mneu$ source node will place  $\mneu$ turning points instead of $\mnvu$ turning points.
% Source recursive traces cannot be called in such a trace -- the occupied-sector rule will not be triggered for a trace with the same side as the source node, and a query that follows a $(-\mside)$-sided recursive ang-sec trace will never be able to prune the $\mney$ source node.
% If the trace is able to prune all $S$-tree $\mney$ nodes in the source direction, the trace resumes normal behavior.

\input{chap_r2p/fig_overlap3}

Condition \textbf{O3} is similar to condition \textit{O2}, except that a $S$-tree $\mnvy$ link is identified instead of an $\mney$ link.
There is cumulative visibility at the anchored point of the $\mnvy$ link, and the path described by the link has the cheapest known cost-to-come at the anchored point. 
If there are other, more expensive links anchored at the anchor point, condition \textit{O2} will be triggered for the expensive links. 
Fig. \ref{r2p:fig:overlap3} illustrates condition \textit{O2} being triggered by condition \textit{O3}.

\subsubsection{\rtwop{}'s Overlap Rule}
\rtwop{} extends \rtwo{}'s overlap rule with four additional conditions. Conditions \textit{O4} and \textit{O5} are extensions of \textit{O2} and \textit{O3} respectively, while conditions \textit{O6} and \textit{O7} are new.

Condition \textbf{O4} extends condition \textit{O2} for $T$-tree links and cost-to-go, and is satisfied when a $T$-tree $\mney$ link is identified. 
Unlike condition \textit{O2}, condition \textit{O4} does not trigger condition \textit{O1}.

Condition \textbf{O5} extends condition \textit{O3}, and is satisfied when a $T$-tree $\mnvy$ link is identified. If there are other expensive cost-to-go $T$-tree links at the anchor point, condition \textit{O4} will be triggered for these links.

% \textbf{Case O4} extends Case O2 for cost-to-go.
% If the target node of a successful cast is a $T$-tree $\mney$ or $\mnvy$ typed, the source node's cost-to-go is examined.
% If the cost-to-go is larger than the minimum so far, the source node will be replaced by a $T$-tree $\mney$ node.
% % A subsequent, successful cast with a target $\mney$ node will cause the cast's source node to be replaced by a $\mney$ node regardless of the cost. 
% % If both nodes have different sides, the path over the nodes will be discarded.
% Unlike Case O2, Case O4 does not restrict traces, and Case O1 will not be called after a trace.

\tikzset{
    pics/overlap4/.style n args={0}{ code={ 
        \begin{pgfonlayer}{background}
            \fill [swatch_obs] %lowest
                (3*\ul, 1*\ul) coordinate (xa1) -- ++(90:\ul) -- ++(0:3*\ul) coordinate (xa3) -- ++(-90:\ul) coordinate (xa4);
            \fill [swatch_obs] % mid
                (1.5*\ul, 3.5*\ul) coordinate (xb1) -- ++(90:\ul) -- ++(0:4*\ul) coordinate (xb3) -- ++(-90:\ul);
            \fill [swatch_obs]
                (1*\ul, 5.5*\ul) coordinate (xc1) -- ++(90:\ul) -- ++(0:2*\ul) coordinate (xc3) -- ++(-90:\ul);
        \end{pgfonlayer}

        \path (1*\ul, 7.5*\ul) coordinate (xtgt);
        \path (7*\ul, 0) coordinate (xsrc);
        \node (nsrc) [svu pt] at (xsrc) {};
        \node (ntgt) [tvy pt] at (xtgt) {};
        \draw [tlink] (ntgt) -- ++(45:2*\ul);
        \draw [slink] (nsrc) -- ++(-90:2*\ul);
        
    }},
}

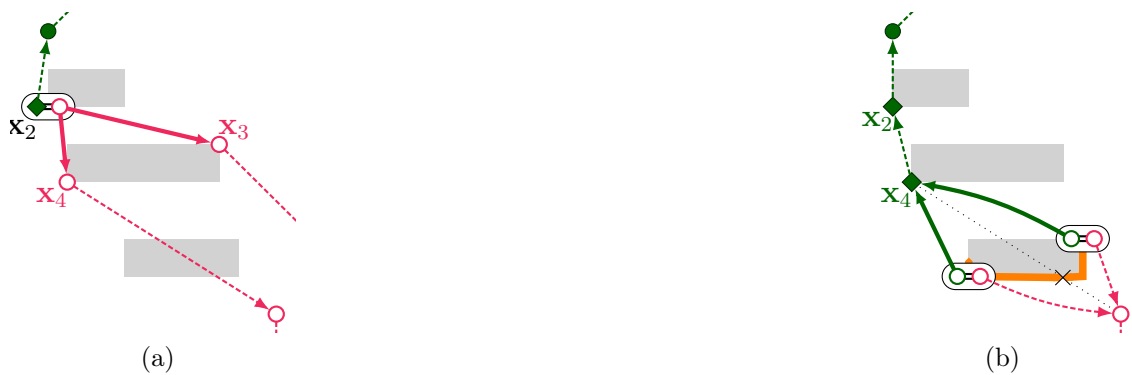
\begin{figure}[!ht]
\centering
\subfloat[\label{r2p:fig:overlap4a} ] {%
    \centering
    \begin{tikzpicture}[]
        \clip (0, -\u) rectangle ++(7.5*\ul, 8.5*\ul);
        \pic at (0, 0) {overlap4};

        \pic at (xc1) {merge grp={shift={(-3.5mm, -3mm)}}{center:$\mx_2$}{nc1 tey}{tey pt}{nc1 svu}{svu pt}};

        \node (nb1) at (xb1) [svu pt={shift={(-2mm, -2mm)}}{center:$\mx_4$}] {};
        \node (nb3) at (xb3) [svu pt={shift={(2mm, 2mm)}}{center:$\mx_3$}] {};

        \draw [sqlink] 
            (nc1 svu) edge (nb1)
            (nc1 svu) edge (nb3);
        \draw [slink]
            (nb1) edge (nsrc)
            (nb3) edge ++(-45:4*\ul);
        \draw [tlink]
            (nc1 tey) edge (ntgt);

    \end{tikzpicture}
} \hfill
\subfloat[\label{r2p:fig:overlap4b}] {%
    \centering
    \begin{tikzpicture}[]
        \clip (0, -\u) rectangle ++(7.5*\ul, 8.5*\ul);
        \pic at (0, 0) {overlap4};

        \begin{pgfonlayer}{background}
            \draw [dotted] (xsrc) -- (nb1)
                node (xcol) [coordinate, pos=2/7] {};
            \draw [trace] (xcol) -- (xa4) -- (xa3) -- ++(180:\u);
            \draw [trace] (xcol) -- (xa1) -- ++(90:\u);
            \node at (xcol) [cross pt] {};
        \end{pgfonlayer}
        
        \node (nc1) at (xc1) [tey pt={shift={(-2mm, -2mm)}}{center:$\mx_2$}] {};
        \node (nb1) at (xb1) [tey pt={shift={(-2mm, -2mm)}}{center:$\mx_4$}] {};

        \pic at (xa1) {merge grp={}{}{na1 tvu}{tvu pt}{na1 svu}{svu pt}};
        \pic at (xa3) {merge grp={}{}{na3 tvu}{tvu pt}{na3 svu}{svu pt}};

        \draw [slink]
            (na1 svu) edge[bend right=10] (nsrc)
            (na3 svu) edge (nsrc);
        \draw [tlink]
            (nc1) edge (ntgt)
            (nb1) edge (nc1);
        \draw [tqlink]
            (na1 tvu) edge (nb1)
            (na3 tvu) edge[bend right=10] (nb1);

    \end{tikzpicture}
}    
\caption[Case O4 of \rtwop{}'s overlap rule.]
{
    Case O4 extends Case O2 for cost-to-go.
    (a) In a successful cast, Case O4 is triggered when the cost-to-go is larger than the minimum at the source node, causing the source node to be replaced by a $T$-tree $\mney$ node.
    A successful cast to a $T$-tree $\mney$ node will cause the cast's source node to be replaced by an $\mney$ node regardless of the cost ($\mx_3, \mx_4$).
    A consecutive pair of $\mney$ nodes with different sides will cause the path passing through the nodes to be discarded ($\mx_3$).
    (b) Unlike Case O2, there are no restrictions to traces, and Case O1 will not be called.
}
\label{r2p:fig:overlap4}
\end{figure}
\tikzset{
    pics/overlap5/.style n args={1}{ code={ 
        \begin{pgfonlayer}{background}
            \fill [swatch_obs] %lowest
                (4*\ul, 0*\ul) coordinate (xa1) -- ++(90:\ul) -- ++(0:2.5*\ul) coordinate (xa3) -- ++(-90:\ul) coordinate (xa4);
            \fill [swatch_obs] % mid
                (2.5*\ul, 2.5*\ul) coordinate (xb1) -- ++(90:\ul) -- ++(0:3.5*\ul) coordinate (xb3) -- ++(-90:\ul);
            \fill [swatch_obs]
                (2*\ul, 5.5*\ul) coordinate (xc1) -- ++(90:\ul) -- ++(0:2*\ul) coordinate (xc3) -- ++(-90:\ul);
        \end{pgfonlayer}

        \node [separate={shift={(5mm, -3.5mm)}}{center:$\mx_2$}, inner xsep=5mm] at (xc1) {};
        \draw [merge] ($(xc1) + (2.5mm, 0)$) -- (xc1);
        \node (nc1 tvy) at (xc1) [tvy pt, shift={(0, 0)}] {};
        \node (nc1 svu) at (xc1) [svu pt, shift={(3mm, 0)}] {};
        \node (nc1 tey) at (xc1) [#1, shift={(-3mm, 0)}]  {};
        
        \path (2*\ul, 7.5*\ul) coordinate (xtgt);
        \node (ntgt) at (xtgt) [tvy pt] {};
        \node (ntgt2) at (1.2*\ul, 7*\ul) [tvy pt] {};
        \draw [tlink] 
            (ntgt) edge ++(90:2*\ul)
            (ntgt2) edge ++(90:2*\ul)
            (nc1 tvy) edge (ntgt);
        \draw [slink]
            (nc1 svu) -- ++(-10:8*\ul);

        \node (nb1) at (xb1) [#1={shift={(-2mm,-2mm)}}{center:$\mx_4$}] {};
        \node (na3) at (xa3) [tvu pt={shift={(2mm,2mm)}}{center:$\mx_5$}] {};
        \node (na1) at (xa1) [#1={shift={(-2mm,-2mm)}}{center:$\mx_6$}] {};
        \draw [tlink]
            (na3) +(-30:6*\ul) edge (na3)
            (na1) +(-20:6*\ul) edge (na1);
        \draw [tlink]
            (na1) edge[bend left=10] (nb1)
            (na3) edge (nb1)
            (nb1) edge[bend left=10] (nc1 tey)
            (nc1 tey) edge (ntgt2);
    }},
}

\begin{figure}[!ht]
\centering
\subfloat[\label{r2p:fig:overlap5a} ] {%
    \centering
    \begin{tikzpicture}[]
        \clip (0, -\ul) rectangle ++(7.5*\ul, 9*\ul);
        \pic at (0, 0) {overlap5={tvy pt}};

        \node (nb3) at (xb3) [tvy pt={shift={(2mm,2mm)}}{center:$\mx_3$}] {};
        \draw [tlink]
            (nb3) +(-30:6*\ul) edge (nb3);
        \draw [tlink]
            (nb3) edge[bend left=20] (nc1 tey);
    \end{tikzpicture}
} \hfill
\subfloat[\label{r2p:fig:overlap5b}] {%
    \centering
    \begin{tikzpicture}[]
        \clip (0, -\ul) rectangle ++(7.5*\ul, 9*\ul);
        \pic at (0, 0) {overlap5={tey pt}};
    \end{tikzpicture}
}    
\caption[Case O5 of \rtwop{}'s overlap rule.]
{
Case O5 extends Case O3 to cost-to-go.
(a) A successful cast finds the smallest cost-to-go at the source node's corner ($\mx_2$). More expensive cost-to-go paths at $\mx_2$ are scanned, and the relevant $T$-tree $\mnvy$ nodes along the path are converted to $T$-tree $\mney$ nodes.
(b) A path will be discarded if it passes through a consecutive pair of $\mney$ nodes with different sides ($\mx_3$). 
Unlike Case O3, Case O5 does not call Case O1.
}
\label{r2p:fig:overlap5}
\end{figure}
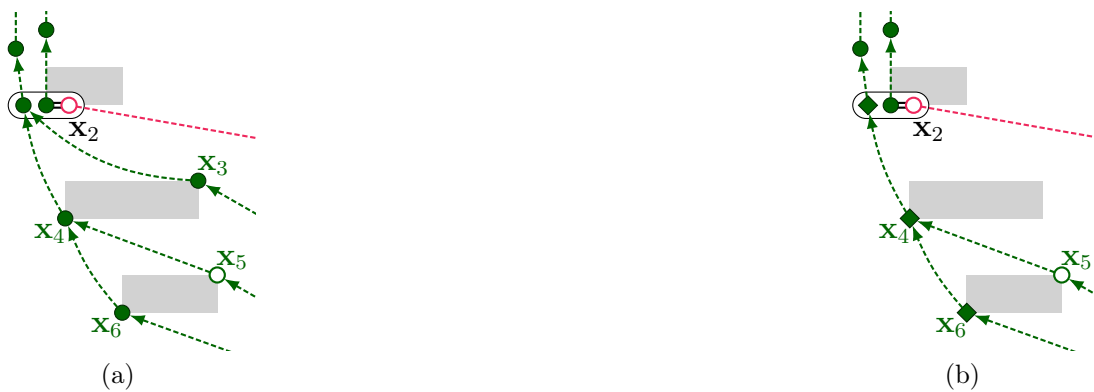
% \textbf{Case O5} extends Case O3 for cost-to-go. 
% If a successful cast results in the cheapest cost-to-go to the cast's source node, other paths that pass through the target corner's $T$-tree $\mnvy$ node will be modified based on Case O4.
% Unlike Case O3, Case O1 will not be called by Case O5.

% In Case 5, a successful cast finds the cheapest cost-to-go path at the source node.
% Paths with a more expensive cost-to-go at the source node are identified, and the paths would have to have a $T$-tree node with cumulative visibility at the source node's corner.
% The case could therefore be treated as Case 3 occurring at each of the expensive path's corresponding $T$-tree node. 
% Similar nodes lying the source direction of the initial $T$-tree node are recursively marked as expensive.

Conditions \textbf{O6} and \textbf{O7} are satisfied when an expensive $\mney$ link is identified, and which describes a path that lies closer to the obstacle than the cheapest path that passes through the anchored point. 
Condition \textit{O6} is triggered for $S$-tree $\mney$ links, and condition \textit{O7} is triggered for $T$-tree $\mney$ links, and the paths described by the links are discarded.
Fig. \ref{r2p:fig:overlap6and7a} illustrates condition \textit{O6}, and Fig. \ref{r2p:fig:overlap6and7b} illustrates condition \textit{O7}. 

For both cases, suppose the path segments $(\mx_e, \mx_a)$ and $(\mx_c, \mx_a)$ are on different paths, and the point at $\mx_a$ has cumulative visibility via both paths.
Suppose the path passing through $\mx_c$ is the cheapest cost known so far to reach $\mx_a$, and the path via $\mx_e$ is described by the more expensive $\mney$ link.
The condition for safely discarding the expensive path is
\begin{equation}
    \mtdir\mside(\mv_e \times \mv_c) < 0, \label{r2p:eq:newex}
\end{equation}
where $\mv_e = \mx_a - \mx_e$, $\mv_c = \mx_a - \mx_c$. 
$\mtdir=S$ or $\mtdir=T$ if the $\mney$ link is in the $S$-tree or $T$-tree respectively.
$\mside$ refers to the side of the turning point at $\mx_a$.
Theorem \ref{r2p:thm:newex} provides a proof for discarding the paths.

\input{chap_r2p/fig_overlap6and7}

% \textbf{Cases O6 and O7} discard more expensive paths at a corner if the local path segment lies closer to the obstacle than the cheapest path (see Fig. \ref{r2p:fig:overlap6and7}). Case O6 examines cost-to-come, while Case O7 examines cost-to-go.
% % In Case 6, suppose two unobstructed paths with different cost-to-come passes through a convex corner at $\mx_a$.
% % In Case 7, the two unobstructed paths have different cost-to-go.
% % For each case, if the more expensive path lies between the cheaper path and the closer obstacle edge adjacent to $\mx_a$, the expensive path can be discarded.

\begin{theorem}
Suppose a path $P_a = (\mx_0, \cdots, \mx_c, \mx_a, \cdots)$ has the shortest path known so far at $\mx_a$.
Consider another path $P_b = (\mx_0, \cdots, \mx_e, \mx_a, \cdots)$, which has a longer path to $\mx_a$ than $P_a$.
For both paths, there is cumulative visibility from $\mx_a$ to $\mx_0$.
Let $\mtdir=S$ or $\mtdir=T$ if the start point or goal point is at $\mx_0$ respectively.
Let $\mside$ be the side of the turning point at $\mx_a$.
The longer path $P_b$ can be discarded and \rtwop{} remains complete if 
\begin{equation*}
    \mtdir\mside(\mv_e \times \mv_c) < 0,
\end{equation*}
where $\mv_e = \mx_a - \mx_e$ and $\mv_c = \mx_a - \mx_c$.
\label{r2p:thm:newex}
\end{theorem}
\begin{proof}
% For a $\mside$-sided turning point to be placed, a $\mside$-sided trace has to occur. 
% Preceding the trace is a $-\mside$-sided trace or a collided ray, which allow queries to the $-\mside$-side of the node to be ignored by the proof, as \rtwop would still be complete.
% The proof would thus focus
% The body of the proof is to show that an expensive path that through $\mx_a$ can be discarded.

% If the link describing the path segment $(\mx_k, \mx_j)$ and anchored at $\mx_j$ is in the $S$-tree, the start point is at $\mx_0$ and the cost considered is cost-to-come; if the link is in the $T$-tree, the goal point is at $\mx_0$ and the cost considered is the cost-to-go.
% A query that passes through $\mx_a$ with a longer path is an \textit{expensive query}, and any subsequent queries are termed as such. 
% For all cases examined by the proof, the expensive query is allowed to continue from $\mx_a$. As such, the path examined by the expensive query has to initially intersect the cheaper path to $\mx_a$.

% By forming triangles  using between points and proofs of contradiction, it can be shown that the expensive query that passes through $\mx_a$ will always be more expensive if it passes through the cheapest known path.

Let the notation $c_{j\mid k}$ describe the length of the unobstructed path $(\mx_0, \cdots, \mx_k, \mx_j)$ at $\mx_j$ from $\mx_0$, that crosses $\mx_k$ immediately before reaching $\mx_j$.

\input{chap_r2p/fig_thm1case1}
In \textbf{Case 1.1}, the expensive query that passes through $\mx_e$ continues past $\mx_a$, causing the point at $\mx_a$ to be pruned, and the resulting path to intersect the cheaper path segment $(\mx_c, \mx_a)$. 
From a proof of contradiction, the resulting path will be more expensive if it intersects the cheaper path segment $(\mx_c, \mx_a)$.
Let the point of intersection be $\mx_i$.
The segment $(\mx_e, \mx_i)$ is assumed to be unobstructed, as this is the shortest possible distance from $\mx_e$ to $\mx_i$ on $(\mx_c, \mx_a)$.
If $c_{i \mid c} \ge c_{i \mid e}$, then $c_{a \mid i} + c_{i \mid c} \ge c_{a \mid i} + c_{i \mid e}$, which is a contradiction as $c_{a \mid c} < c_{a \mid e}$, $c_{a \mid c} = c_{a \mid i} + c_{i \mid e}$, and $c_{a \mid e} < c_{a \mid i} + c_{i \mid e}$.
As such, $c_{i \mid c} < c_{i \mid e}$, and a longer, unobstructed path found at $\mx_a$ that intersects the cheaper path segment $(\mx_c, \mx_a)$ has to be expensive.
Case 1.1 is illustrated in Fig. \ref{r2p:fig:thm1case1a}.
% A more expensive query will always remain expensive if it passes through $\mx_a$ with a larger cost, and if it finds a path that intersects the cheaper path segment $(\mx_c, \mx_a)$.

In \textbf{Case 1.2}, the longer path intersects the shorter path at the other segments beyond $\mx_c$. 
From Case 1.1, $c_{c \mid c} < c_{c \mid e}$ and it is costlier to reach $\mx_c$ from $\mx_e$. 
By applying the proof of contradiction recursively over the segments beyond $\mx_c$, any path from $\mx_e$ can be shown to be longer when it arrives at the intersection with the shorter path.
Case 1.2 is illustrated in Fig. \ref{r2p:fig:thm1case1b}.

\input{chap_r2p/fig_thm1case2}
Consider the cases where the point at $\mx_e$ is subsequently pruned, causing a point at $\mx_d$ to be exposed. 
% Let $\mx_f$ be the position where the prune occurs.
% $\mx_f$ lies at the opposite side of the cheaper path from $\mx_e$, and after a line colinear to ($\mx_d, \mx_e)$ is crossed.
For \textbf{Case 2.1}, let the intersection of the line colinear to $(\mx_d, \mx_e)$ with the cheaper path segment $(\mx_c, \mx_a)$ be at the point $\mx_j$; and the intersection of the path with the segment $(\mx_c, \mx_a)$ be at $\mx_i$.
From Case 1.1, $c_{j \mid c} < c_{j \mid e}$, and since $\mx_d$,  $\mx_e$, and $\mx_j$ are colinear, $c_{j \mid c} < c_{j \mid d}$.
By applying a proof of contradiction, $c_{i \mid c} < c_{i \mid d}$, and
any subsequent path from $\mx_d$ that crosses $(\mx_c, \mx_a)$ will be longer.
Case 2.1 is illustrated in Fig. \ref{r2p:fig:thm1case2a}.

Consider \textbf{Case 2.2}, where the longer path from $\mx_d$ intersects the shorter path beyond $\mx_c$. By applying proofs of contradictions from Cases 1.1, 1.2 and 2.1, any subsequent path from $\mx_d$ that crosses the shorter path will be longer.
Case 2.2 is illustrated in Fig. \ref{r2p:fig:thm1case2b}.

Consider \textbf{Case 3}, where more points are pruned from the longer path.
Repeating the proofs of Cases 2.1 and 2.2, any path originating from the pruned path will be longer at the intersection with the shorter path, provided that pruning stops at a point before the root point.
Case 3 is applicable for $S$-tree links as the pruning of the longer path will stop at a $(-\mside)$-sided point.
Case 3 is not applicable for $T$-tree links, but is admissible to discard a more expensive cost-to-go path at $\mx_a$ as \rtwop{} is complete.

\input{chap_r2p/fig_thm1case3A}
For Case 3, pruning will stop at a $(-\mside)$-sided turning point if the path is on the $S$-tree, where the root point is the start point. 
$\mnode_{-\mside}$ is first shown to exist.
% For pruning to occur, a trace examining the longer path has to be $\mside$-sided, so that the expensive node at $\mx_a$, and turning points at $\mx_d$ and $\mx_e$ can be pruned.
From a proof of contradiction, suppose that a $(-\mside)$-sided point does not exist and all turning points along the longer path are $\mside$-sided.
If all turning points are $\mside$-sided, the unobstructed longer path has to be a straight path, or bend monotonically to the $\mside$-side from the start point before reaching $\mx_a$ (see Fig. \ref{r2p:fig:thm1case3A}).
Since $(\mx_c, \mx_a)$ lies on the $\mside$-side of the longer path, the shorter path has to lie on the $\mside$ of the longer path when viewed from the start point.
However, it is impossible for a shorter unobstructed path to $\mx_a$ to exist on the $\mside$-side of a longer path that is straight or bends to the $(-\mside)$-side, and the longer path has to contain at least one $(-\mside)$-sided turning point. Let this $(-\mside)$-sided turning point be $\mnode_{-\mside}$.

\input{chap_r2p/fig_thm1case3B}
$\mnode_{-\mside}$ on the longer path cannot be pruned if it is in the $S$-tree.
For $\mnode_{-\mside}$ to be pruned, a trace has to be $(-\mside)$-sided.
Before the prune can occur, 
the $(-\mside)$-sided trace will have to cross the $(-\mside)$-sided sector-ray of $\mnode_{-\mside}$, which points to a previously pruned $(-\mside)$-sided turning point along the longer path (e.g. $\mx_d$).
The sector-ray is formed when a cast from $\mnode_{-\mside}$ had reached the pruned point.
The trace may be discarded, or be interrupted by a recursive angular sector trace, causing $\mnode_{-\mside}$ to be preserved (see Fig. \ref{r2p:fig:thm1case3Ba}).

For $T$-tree nodes in Case 3, a $(-\mside)$-sided turning point can be similarly shown to exist, but unlike the $S$-tree, the turning point can be pruned as sector-rays cannot be defined for nodes in the target direction.
Case 3 is not a problem for $T$-tree nodes, as \rtwop{} will be able to find the shortest path from the $(-\mside)$-sided turning point in another query even if it is pruned by the current query (see Fig. \ref{r2p:fig:thm1case3Bb}).

Consider \textbf{Case 4}, where a subsequent query reaches a point that causes the longer path to sweep past the root point and not intersect with the cheaper path. 
As such a path causes a loop, the longer path can be discarded if it does not fulfill Eq. \ref{r2p:eq:newex} at $\mx_a$.
\end{proof}

% The proof can be extended for paths with equal costs, provided that the paths are not allowed through checkerboard corners.
% A checkerboard corner occurs between four adjacent cells where the binary occupancy states resemble a checkerboard \cite{bib:r2}.
% A checkerboard corner lies at the same location as another, and are convex if paths can pass through them.
% A query from one corner can lead to the other, and an equal cost path that passes through one corner can be discarded by Eq. (\ref{r2p:eq:newex}) is kept at the other corner, and subsequently cause both paths to be discarded.

% The proof is generally applicable if the cost between both paths are equal, provided that paths are not allowed through checkerboard corners.
% A checkerboard corner occurs between four adjacent cells where the binary occupancy states resemble a checkerboard \cite{bib:r2}.
% If paths can pass through checkerboard corners, the corners are convex.
% As a checkerboard corner is at the same location as another, two equal cost paths from different directions can cause Eq. \ref{r2p:eq:newex} to be satisfied for a different path when evaluated at a different corner.
% As a subsequent query from one corner can lead to the other corner, a path that is discarded by the earlier corner 
% As \rtwop{} allow paths through checkerboard corners, the proof cannot be applied

%%%%%%%%%%%%%%%%%%%%%%%%%%%%%%%% ALGORITHM %%%%%%%%%%%%%%%%%%%%%%%%%%%%%%%%%
\section{\rtwop{} Algorithm}

The pseudocode in this section shows only the noteworthy steps in the algorithm.
A more detailed version is available in Appendix \ref{chap:suppr2p}, which describes how the tree is managed to avoid data races and limit the number of link connections for each link.
In the pseudocode, ``source" and ``target'' are abbreviated to ``src" and ``tgt" respectively.
Rays are merged only if the resulting angular sector shrinks.

\rtwop{} is run from Alg. \ref{r2p:alg:run}.
Alg. \ref{r2p:alg:caster} handles casts, while Alg. \ref{r2p:alg:tracer} handles traces. 
Alg. \ref{r2p:alg:casterreached} and Alg. \ref{r2p:alg:castercollided} are helper functions that manages a successful cast and collided cast respectively, and Alg. \ref{r2p:alg:tracerproc} is a helper function that manages nodes and links in the source or target direction of the trace.

\begin{algorithm}[!ht]
\begin{algorithmic}[1]
\caption{Main \rtwop{} algorithm.}
\label{r2p:alg:run}
\Function{Run}{$\mx_\mathrm{start}, \mx_\mathrm{goal}$}
    \State $\mlink \gets$ link from $\mnode_\mathrm{start}$ to $\mnode_\mathrm{goal}$.
    \State Queue $(\mqcast, \mlink)$.
    \While {open-list is not empty}
        \State Poll query ($\mqtype, \mlink$).
        \If {$\mqtype = \mqcast$}   \Comment{Casting query polled.}
            \IfThen {\Call{Caster}{$\mlink$}} {\Return path}
        \Else   \Comment{Tracing query polled.}
            \State Trace from target point of $\mlink$.
            % \State $\mlinks_T \gets $ tgt links of $\mlink$.
            % \State $\mside \gets $ side of tgt node of $\mlink$. 
            % \State \Call{Tracer}{$\mside, \{\mlink\}, \mlinks_T$}
        \EndIf
        \State Do actions for any point where overlap condition \textit{O1} is triggered.
    \EndWhile
    \State \Return $\{\}$ 
    \Comment{No path.}
\EndFunction
\end{algorithmic}
\end{algorithm}
\begin{algorithm}[!ht]
\begin{algorithmic}[1]
\caption{Handles casting queries.}
\label{r2p:alg:caster}
\Function{Caster}{$\mlink$}
    \If {cast from source point of $\mlink$ to target point of $\mlink$ succeeds}
        \IfThen{\Call{CastReached}{$\mlink$}}{\Return $\mtrue$}
    \Else
        \State \Call{CastCollided}{$\mlink$}
    \EndIf
    \State \Return $\mfalse$ 
\EndFunction
\end{algorithmic}
\end{algorithm}

\begin{algorithm}[!ht]
\begin{algorithmic}[1]
\caption{Handles tracing queries.}
\label{r2p:alg:tracer}
\Function{Tracer}{$\tau$}
    \DoWhile
        \Comment{$\tau$ encapsulates a tracing query.}
        \If {traced to source point}
            \State \Break
        \ElsIf {\underline{progression rule} finds no source progression}
            \IfThen {queued cast to phantom pt}{\Break}
        \ElsIf{\Call{TracerProc}{$T, \tau$} discards trace}
            \State \Break
        \ElsIf {\Call{TracerProc}{$S, \tau$} discards trace}
            \State \Break
        \ElsIf {\underline{interrupt rule} queues a tracing query}
            \State \Break
        \ElsIf {\underline{placement rule} has cast to all target nodes}
            \State \Break
        \EndIf
        \State Trace to next corner.
    \EndDoWhile{trace not out of map}
\EndFunction
\end{algorithmic}
\end{algorithm}
\begin{algorithm}[!ht]
\begin{algorithmic}[1]
\caption{Handles successful casting queries.}
\label{r2p:alg:casterreached}
\Function{CastReached}{$\mlink$}
    \If {source and target links of $\mlink$ have cumulative visibility}
        \Comment{$\mnvy$ type.}
        \State Generate path and \Return $\mtrue$.
    \ElsIf {$\mlink_T$ should not be reached}
        \Comment{$\mnun$ type.}
        \State \Return $\mfalse$ 
    \ElsIf {$\mlink_T$ is part of interrupted trace}
        \Comment{$\mntm$ type.}
        \State \parbox[t]{13cm}{%
        Try to place a turning point at target point, and change link type based on cumulative visibility and cost of source link.\strut}
        \State Trace from target point if target links of $\mlink$ are not castable.
    \EndIf
    \If {source and target links of $\mlink$ have no cumulative visibility}
        \Comment{$\mnvu$ type.}
        \State Test overlap condition \textit{O1} at target point.
        \State If no overlap, queue $(\mqcast,\mlink_T)$ for every target link $\mlink_T$ of $\mlink$.
    \ElsIf{source link has cumulative visibility}
        \Comment{$\mnvy$ type.}
        \State Merge sector-ray describing cast to angular sector in $\mlink$.
        \For{each target link $\mlink_T$ of $\mlink$}
            \State Merge sector-ray describing cast to angular sector in $\mlink_T$.
            \State Queue $(\mqcast, \mlink_T)$.
        \EndFor
        \State Test overlap conditions \textit{O2}, \textit{O3}, and \textit{O6} at target point.
    \ElsIf{target link has cumulative visibility}
        \Comment{$\mnvy$ type.}
        \State Queue $(\mqcast,\mlink_S)$ for source link $\mlink_S$ of $\mlink$.
        \State Test overlap conditions \textit{O4}, \textit{O5}, and \textit{O7} at source point.
    \EndIf
    \State \Return $\mfalse$
\EndFunction
\end{algorithmic}
\end{algorithm}
\begin{algorithm}[!ht]
\begin{algorithmic}[1]
\caption{Handles casting queries that collide.}
\label{r2p:alg:castercollided}
\Function{CastCollided}{$\mlink$}
    % \If {}
    % \State $\mnode_S \gets $ src node of $\mlink$.
    % \State $\mnode_T \gets $ tgt node of $\mlink$.
    % \State $\mside_\mmjr \gets $ side of $\mnode_S$.
    \State $p_S \gets $ source point of $\mlink$.
    \State merge sector-ray describing cast into angular sector of $\mlink$.
    \If {source link of $\mlink$ is not expensive}
        \Comment{not $\mney$ type.}
        \State Do \underline{minor trace} from collision point, which has different side from $p_S$.
        \If {target point of $\mlink$ is the goal point}
            \State Do \underline{third trace} from source point, which has the same side as $p_S$.
        \EndIf
    \EndIf
    \State Do \underline{major trace} from collision point, which has same side as $p_S$.
\EndFunction
\end{algorithmic}
\end{algorithm}
\begin{algorithm}[!ht]
\begin{algorithmic}[1]
\caption{Processes trace in one tree direction.}
\label{r2p:alg:tracerproc}
\Function{TracerProc}{$\mtdir, \tau$}
    \ForEach{$\mtdir$ link $\mlink_\mtdir$ of $\tau$}
        % \Comment{Trace has one src ($\mtdir=S$) link and $\ge 1$ tgt ($\mtdir=T$) links}
        % \If {$\mtdir=T$ and \underline{progression rule} finds no prog. w.r.t. $\mnode$}
        %     \State \Continue 
        \State $p_\mtdir \gets $ $\mtdir$ point of link.
        \If {\underline{progression rule} finds no progression for $p_\mtdir$ or if there is a cast}
            \State \Continue
        \ElsIf{$\mtdir=S$ and \underline{angular-sector rule} discards trace}
            \State \Continue
        \ElsIf{$p_\mtdir$ is start or goal point}
            \State \Continue
        \ElsIf{trace has same side as $p_\mtdir$ and \underline{pruning rule} prunes $\mlink$}
            \State \Continue
        \ElsIf{trace has different side from $p_\mtdir$ and \underline{occupied sector rule} generates trace from source point}
            \State \Continue
        \EndIf
    \EndFor
\EndFunction
\end{algorithmic}
\end{algorithm}

\newpage
\phantom{blabla}
\newpage
\section{Methodology of Comparing Algorithms}

The method used is the same as \cite{bib:r2}.
Algorithms are run on benchmarks, which are obtained from \cite{bib:bench}.
Each map in the benchmark contains between a few hundred to several thousand \textit{scenarios}, which are shortest path problems between two points.

As \rtwo{} and \rtwop{} do not pre-process the map and runs on binary occupancy grids, their results are compared with equivalent state-of-the-art algorithms Anya and \rsp{}.
As such, state-of-the-art algorithms that are not online or do not run on binary occupancy grids, such as Polyanya \cite{bib:polyanya} and Visibility Graphs \cite{bib:vg}, are not compared.

For \rsp{}, the skip, bypass, and block extensions are selected as it is the fastest online configuration.
\rsp{} requires a map to be scaled twice and the start and goal points to be shifted by one unit in both dimensions.
As such, the tested maps are scaled twice, and the chosen algorithms are run on the same scenarios as \rsp{}.
% In the results, the map names will be appended with the label ``\_scale2" for clarity.
% While the names are different as the method in \cite{bib:r2}, the maps and scenarios are the same.

Unlike \rtwo{}, \rtwop{} allows a path to pass through a checkerboard corner. 
A checkerboard corner is located at a vertex where the four diagonally adjacent cells have occupancy states resembling a checkerboard.
The passage through a checkerboard corner simplifies the algorithm by avoiding ambiguity when the starting point is located at a checkerboard corner.
To ensure that the returned paths are correct, the costs of \rtwo{} and \rtwop{} are verified against the visibility graph implementations and other algorithms, and the costs are found to agree.

To test the impact of overlap conditions \textit{O6} and \textit{O7} on search time, \rtwop{} is further re-run as the variant ``\rtwop{}N67" with the conditions disabled.

The tests are run on Ubuntu 20.04 in Windows Subsystem for Linux 2 (WSL2) and on a single core of an Intel i9-11900H (2.5 GHz), with Turbo-boost disabled. The machine and software is the same as \cite{bib:r2}.
\rtwo{} and \rtwop{} are available at \cite{bib:r2code}.

\section{Results}
In this section, a \textbf{speed-up} is the ratio of an algorithm's search time to \rtwop{}'s search time.
The speed-ups for selected maps are shown in Fig. \ref{r2p:fig:results}.
The average search times are  shown in Table \ref{r2p:tab:average},
and Table \ref{r2p:tab:points} show the average speed-ups with respect to the 3, 10 and 30 turning points.
As passage through checkerboard corners have negligible impact on search times (see below), the costs and number of turning points used for comparisons are based on paths that can pass through checkerboard corners.
Colinear turning points are removed from all results to avoid double counts in the comparisons.

The middle column of Fig. \ref{r2p:fig:results} shows the average speed-ups with respect to the number of turning points on the shortest path. 
The right column of Fig. \ref{r2p:fig:results} shows the benchmark characteristics by plotting the shortest paths' cost with  the number of turning points.
The correlation between the cost and number of turning points is indicated by $r$, and the ratio of the number of corners to the number of free cells is indicated in $\rho$.

As \rtwo{} and \rtwop{} are exponential in the worst case with respect to the number of collided casts, the algorithms are expected to perform poorly in benchmarks with high $r$.
A high $r$ indicates that paths are likely to turn around more obstacles as they get longer, implying that the maps have highly non-convex obstacles and many disjoint obstacles.
In such maps, collisions are highly likely to occur, and \rtwo{} and \rtwop{} are likely to be slow.

Unlike the other algorithms, \rtwop{} and \rtwop{}N67 allow passage through checkerboard corners.
As such, the shortest paths of the other algorithms are different from \rtwop{} for the maps ``random512-10-1" (25.25\% identical) and ``random512-20-2" (9.38\% identical). 
Coincidentally, \rtwop{} differs significantly from \rtwo{} in the search times for only the two maps (see Table \ref{r2p:tab:points}).
The difference in search time is due primarily to the overlap conditions \textit{O6} and \textit{O7}, as 
As \rtwop{}N67 performs similarly to \rtwo{} for the two maps (see Table \ref{r2p:tab:average}),
and \rtwop{}N67 is \rtwop{} without the conditions.
As such, allowing passage through checkerboard corners have negligible impact on the search times for the maps tested.

The overlap conditions \textit{O6} and \textit{O7} improves search time significantly in maps with many small disjoint obstacles like ``random512-10-1", instead of maps with highly non-convex obstacles  like ``maze512-8-0".
As queries are able to move around small obstacles faster than highly non-convex ones, path costs can be verified more quickly, and more overlapping paths satisfy the overlap conditions.
As such, \rtwop{} to perform significantly faster than \rtwo{} in maps with more disjoint obstacles.

While being simpler than \rtwo{}, \rtwop{} has similar performance to \rtwo{} in other maps.
As such, \rtwop{} preserves the speed advantage that \rtwo{} has over other algorithms when the shortest path is expected to turn around few obstacles, while 
significantly outperforming \rtwo{} in maps with many disjoint obstacles.

\begin{figure*}[!ht]
%%% TO LATEX PARSERS: IF THIS IMAGE NEEDS TO BE SHRUNK, PLS INFORM CORRESPONDING AUTHOR.
\centering
\includegraphics[width=\textwidth]{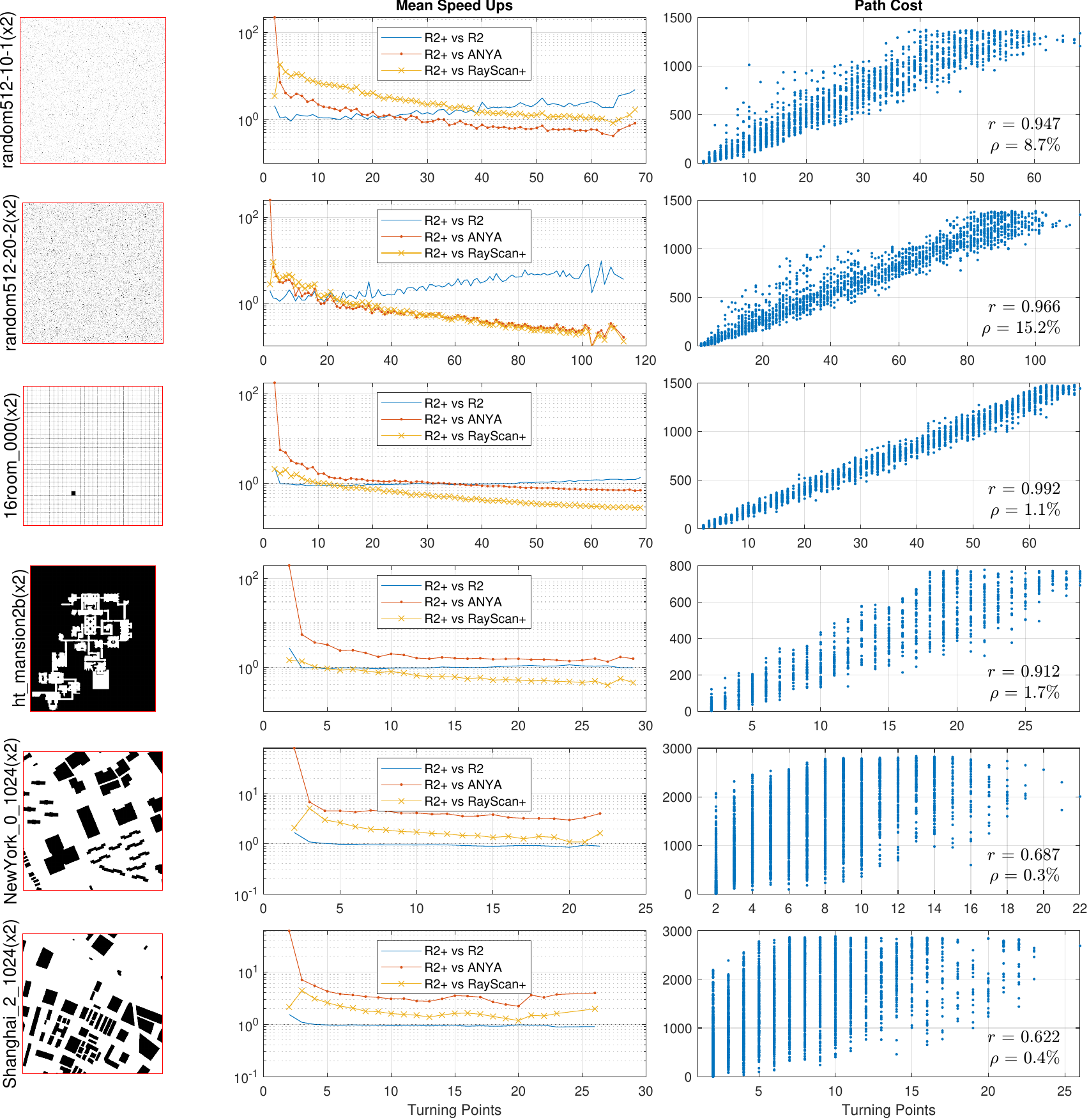}
\caption[Selected results for \rtwop{}]
{
Results for selected maps. All maps are scaled twice, and the start and goal points shifted by one unit to accommodate \rsp{}.
\rtwop{} and \rtwo{} performs well on maps with convex obstacles and few disjoint obstacles, \rtwop{} performs better than \rtwo{} on maps with many disjoint obstacles.
}
\label{r2p:fig:results}
\end{figure*}

\begin{sidewaystable}[!ht]
\centering
\caption{Benchmark characteristics and average search time.}
\setlength{\tabcolsep}{3pt}
\begin{tabular}{ c c c c c c c c c c }
\toprule
Map & $P$ & $G$ & $r$ & $\rho$ (\%) & \rtwop{} & \rtwop{}N7 & R2 & ANYA & RS+ \\
\hline
bg512/AR0709SR & 13 & 953.3 & 0.608 & 0.144 & 38.463 & 38.321 & 35.177 & 204.542 & 54.128 \\
\hline
bg512/AR0504SR & 22 & 1019.0 & 0.792 & 0.570 & 150.338 & 166.378 & 155.799 & 570.845 & 204.597 \\
\hline
bg512/AR0603SR & 42 & 2228.5 & 0.963 & 1.299 & 771.426 & 851.360 & 846.224 & 1040.011 & 335.797 \\
\hline
da2/ht\_mansion2b & 29 & 776.2 & 0.912 & 1.748 & 302.421 & 345.923 & 324.229 & 471.444 & 152.418 \\
\hline
da2/ht\_0\_hightown & 18 & 1061.9 & 0.908 & 0.876 & 251.374 & 301.591 & 288.901 & 1031.213 & 273.415 \\
\hline
dao/hrt201n & 31 & 905.8 & 0.942 & 2.751 & 427.294 & 477.768 & 442.440 & 634.030 & 193.464 \\
\hline
dao/arena & 5 & 100.5 & 0.428 & 1.315 & 4.289 & 4.215 & 4.403 & 98.121 & 2.750 \\
\hline
maze/maze512-32-0 & 56 & 4722.4 & 0.987 & 0.037 & 287.388 & 287.549 & 292.550 & 904.761 & 132.083 \\
\hline
maze/maze512-16-0 & 145 & 6935.0 & 0.994 & 0.143 & 2398.408 & 2395.561 & 2413.348 & 1708.210 & 369.548 \\
\hline
maze/maze512-8-0 & 205 & 4792.2 & 0.992 & 0.511 & 9942.258 & 9923.029 & 10443.452 & 2836.640 & 895.688 \\
\hline
random/random512-10-1 & 68 & 1372.1 & 0.947 & 8.667 & 13894.789 & 28495.995 & 28860.096 & 9175.201 & 19476.182 \\
\hline
random/random512-20-2 & 113 & 1386.8 & 0.966 & 15.219 & 113042.158 & 481993.141 & 480362.351 & 32668.658 & 29605.503 \\
\hline
room/32room\_000 & 41 & 1579.2 & 0.989 & 0.272 & 1003.133 & 1220.866 & 1021.478 & 1656.847 & 558.525 \\
\hline
room/16room\_000 & 69 & 1477.7 & 0.992 & 1.065 & 4671.431 & 6415.172 & 5411.462 & 3713.668 & 1641.957 \\
\hline
street/Denver\_2\_1024 & 16 & 2835.8 & 0.770 & 0.028 & 96.485 & 102.252 & 91.920 & 910.048 & 416.744 \\
\hline
street/NewYork\_0\_1024 & 22 & 2834.8 & 0.687 & 0.310 & 316.994 & 324.427 & 299.855 & 1273.206 & 511.943 \\
\hline
street/Shanghai\_2\_1024 & 26 & 2885.7 & 0.622 & 0.404 & 508.290 & 541.504 & 491.929 & 1520.395 & 750.569 \\
\hline
street/Shanghai\_0\_1024 & 22 & 2816.5 & 0.511 & 0.258 & 266.808 & 267.314 & 256.614 & 973.980 & 265.440 \\
\hline
street/Sydney\_1\_1024 & 24 & 2844.5 & 0.698 & 0.128 & 159.619 & 164.025 & 152.603 & 958.804 & 368.688 \\
\bottomrule
\multicolumn{10}{p{19cm}}{
\footnotesize
All maps are scaled twice and start and goal coordinates shifted by one unit to accommodate \rsp{} (RS+).
$r$ is the correlation coefficient between the number of turning points and the shortest path cost for all scenarios in each map. 
$\rho$ is the ratio of the number of corners to the number of free cells on the map. 
$P$ is the largest number of turning points and $G$ is the largest path cost among all scenarios.
}
\end{tabular}
\label{r2p:tab:average}
\end{sidewaystable}
\begin{sidewaystable}[!ht]
\centering
\caption{Average speed-ups for 3, 10, and 30 turning points.}
\setlength{\tabcolsep}{3pt}
\begin{tabular}{ c c c c c | c c c c | c c c c }
\hline
\multirow{2}{*}{Map} & \multicolumn{4}{c | }{3 Turning Pts.} & \multicolumn{4}{c | }{10 Turning Pts.} & \multicolumn{4}{c}{30 Turning Pts.} \\
\cline{2-13}
& $g_{3}$ & R2 & ANYA & RS+ & $g_{10}$ & R2 & ANYA & RS+ & $g_{30}$ & R2 & ANYA & RS+ \\
\hline
bg512/AR0709SR & 418.2 & 0.961 & 8.49 & 2.43 & 663.2 & 0.905 & 3.65 & 1.11 & -- & -- & -- & -- \\
\hline
bg512/AR0504SR & 242.3 & 1.11 & 7.12 & 3.66 & 747.1 & 1.04 & 4.07 & 1.53 & -- & -- & -- & -- \\
\hline
% bg512/AR0014SR & 231.0 & 1.05 & 5.49 & 2.89 & 584.2 & 0.916 & 2.35 & 1.23 & -- & -- & -- & -- \\
% \hline
% bg512/AR0304SR & 283.5 & 1.06 & 5.87 & 2.44 & 744.7 & 0.899 & 2.97 & 0.955 & -- & -- & -- & -- \\
% \hline
% bg512/AR0702SR & 220.5 & 0.977 & 6.26 & 2.08 & 725.5 & 0.902 & 3.65 & 1.09 & -- & -- & -- & -- \\
% \hline
% bg512/AR0205SR & 155.1 & 1.06 & 5.25 & 2.38 & 517.0 & 0.952 & 2.79 & 1.27 & 1246.3 & 1.24 & 2.28 & 0.613 \\
% \hline
% bg512/AR0602SR & 162.3 & 1.08 & 5.93 & 2.26 & 456.6 & 0.967 & 2.13 & 0.943 & 1388.2 & 1.06 & 1.2 & 0.383 \\
% \hline
bg512/AR0603SR & 212.1 & 1.09 & 4.8 & 2.75 & 595.5 & 0.987 & 2.05 & 0.925 & 1645.0 & 1.09 & 1.43 & 0.431 \\
\hline
da2/ht\_mansion2b & 59.4 & 0.954 & 5.46 & 1.35 & 249.1 & 0.944 & 1.99 & 0.79 & -- & -- & -- & -- \\
\hline
da2/ht\_0\_hightown & 134.4 & 0.981 & 5.95 & 2.19 & 584.1 & 1.11 & 4.82 & 1.39 & -- & -- & -- & -- \\
\hline
dao/hrt201n & 81.7 & 1.07 & 5.22 & 1.59 & 285.3 & 1.02 & 1.98 & 0.794 & 848.7 & 0.92 & 1.73 & 0.399 \\
\hline
dao/arena & 69.5 & 0.917 & 9.12 & 0.627 & -- & -- & -- & -- & -- & -- & -- & -- \\
\hline
maze/maze512-32-0 & 206.3 & 0.975 & 6.42 & 0.991 & 807.4 & 1.01 & 4.53 & 0.76 & 2413.4 & 1.02 & 3.68 & 0.536 \\
\hline
maze/maze512-16-0 & 122.2 & 0.964 & 6.5 & 0.952 & 442.5 & 1.01 & 2.57 & 0.829 & 1385.9 & 1.01 & 1.62 & 0.372 \\
\hline
maze/maze512-8-0 & 55.1 & 1.03 & 4.81 & 0.963 & 226.8 & 0.951 & 1.83 & 0.57 & 748.0 & 0.991 & 0.965 & 0.336 \\
\hline
random/random512-10-1 & 42.2 & 1.09 & 7.21 & 18.4 & 264.6 & 1.12 & 2.22 & 7.01 & 791.9 & 1.44 & 0.882 & 2.22 \\
\hline
random/random512-20-2 & 16.6 & 1.3 & 7.05 & 9.03 & 143.8 & 1.11 & 1.52 & 2.5 & 476.6 & 1.87 & 0.817 & 0.932 \\
\hline
room/32room\_000 & 79.0 & 0.979 & 6.23 & 1.4 & 344.8 & 0.911 & 2.31 & 1.03 & 1137.8 & 1.01 & 1.71 & 0.603 \\
\hline
room/16room\_000 & 42.9 & 0.995 & 5.62 & 1.69 & 184.7 & 0.901 & 1.65 & 1.01 & 614.8 & 0.98 & 1.06 & 0.562 \\
\hline
street/Denver\_2\_1024 & 774.3 & 0.962 & 10.8 & 5.85 & 2329.0 & 0.956 & 9 & 4.27 & -- & -- & -- & -- \\
\hline
street/NewYork\_0\_1024 & 865.3 & 1.09 & 6.84 & 5.13 & 1846.0 & 0.95 & 4.14 & 1.72 & -- & -- & -- & -- \\
\hline
street/Shanghai\_2\_1024 & 1025.3 & 1.09 & 7.05 & 4.38 & 1891.7 & 0.936 & 3.06 & 1.6 & -- & -- & -- & -- \\
\hline
street/Shanghai\_0\_1024 & 1371.7 & 1.13 & 7.34 & 2.71 & 1558.5 & 0.961 & 3.21 & 0.95 & -- & -- & -- & -- \\
\hline
street/Sydney\_1\_1024 & 878.0 & 1.06 & 8.98 & 4.32 & 1995.3 & 0.93 & 5.17 & 2.15 & -- & -- & -- & -- \\
\hline
\multicolumn{13}{p{19cm}}{
\footnotesize
Maps are scaled twice and points shifted by one unit to accommodate \rsp{}.
$g_i$ refers to the average path cost for the shortest paths with $i$ turning points.
``R2", ``ANYA" and ``RS+" (\rsp{}) are the speedups of \rtwop{} with respect to the algorithms.
The higher the ratio, the faster \rtwop{} is compared to an algorithm.
}
\end{tabular}
\label{r2p:tab:points}
\end{sidewaystable}
\clearpage

% For all maps, \rtwop{} tends to perform slightly worse than \rtwo{}. The difference may have been due to implementation differences and the addition of a few constant time checks to simplify \rtwo{}'s algorithm to \rtwop{}.

\section{Conclusion}
In this work, \rtwo{}, a vector-based any-angle path planner, is evolved into \rtwop{}.
Novel mechanisms are introduced in \rtwop{} to simplify the algorithm, and allow \rtwop{} to perform faster than \rtwo{} in maps with many disjoint obstacles while preserving the performance of \rtwo{} in other maps.
% \rtwop{} prevents chases from occurring by superseding ad hoc points in \rtwo{} with a short, recursive angular-sector trace from target nodes.

\rtwo{} and \rtwop{} are able to outperform state-of-the-art algorithms like Anya and \rsp{} when paths are expected to have few turning points.
\rtwo{} and \rtwop{} are fast due to delayed line-of-sight checks to expand the most promising turning points, which are points that deviate the least from the straight line between the start and goal points.

While fast when the shortest paths are expected to have few turning points, \rtwo{} and \rtwop{} are exponential in the worst case with respect to collided line-of-sight checks in the worst case. 
To improve average search time, \rtwo{} discards paths that have expensive nodes that cannot be pruned.
\rtwop{} improves upon \rtwo{} by discarding paths that intersect cheaper paths, allowing \rtwop{} to outperform \rtwo{} in maps with many disjoint obstacles.

% \rtwop{} is a superior algorithm to \rtwo{}, and supersedes \rtwo{}.
% \rtwop{} is terminable, and simpler to implement than \rtwo{}.
% \rtwop{} outperforms \rtwo{} in maps with many disjoint obstacles, while preserving the performance of \rtwo{} in other maps.
Future works may investigate ways to improve the speed of \rtwop{} in maps with highly non-convex obstacles, and improve the algorithm's complexity with respect to collided casts. 
In addition, the interminability  of \rtwop{} and \rtwo{} when no path can be found should be addressed.

\chapter{Future Works and Conclusion}
\label{chap:conc}
This chapter describes possible future extensions and concludes the thesis.

\section{Future Work}
\label{conc:sec:future}
This section describes future work directions that can be undertaken, mainly focusing on a novel extension of the two-dimensional angular sector to the three-dimensional one.
There are other ideas, but as the ideas are not well tested, they will not be described in detail.
The undescribed ideas include
\begin{enumerate}
    \item Merging searches along links that cannot be pruned. For example, if a collided cast has occurred before, the current search can simply connect the source link of a new minor trace to an existing link arising from a prior minor trace.
    In such an algorithm, a trace can have multiple source nodes and only one target node, which can potentially improve the search time to polynomial with respect to the number of collided casts.
    
    \item A maximum cost measurement based on the length of prior traces, allowing a search with a larger minimum cost than another search with smaller maximum cost to be discarded.
    
    \item Combining the angular sectors to a multiple cost grid based on refractive indices. While likely to be slow, it can potentially be a basis for the first any-angle algorithm that can work on a multi-cost occupancy grid.
\end{enumerate}
\subsection{Angular Sectors in Three Dimensions}
In three dimensions, the number of possible turning points is uncountably infinite.
A shortest path is a taut path,  a taut path has to bend around convex edges of obstacles, and a turning point of a taut path can lie anywhere along the edge.
As the shortest path is evaluated based on the convex Euclidean distance, a quadratic program can be implemented to find the shortest path, after the convex edges are identified.

As a quadratic program is slow and complicated, a reasonable approximation of the shortest path is acceptable. 
By assuming that the shortest paths has to pass through the vertices of the grid along each obstacle edge, the number of possible paths becomes countable and finite.
The assumption is used by Theta* and its derived algorithms to find paths in a three-dimensional occupancy grid.
However, a planner relying on such an assumption is still slow, as the search space grows exponentially with respect to the number of dimensions (curse of dimensionality).

To reduce the search space in three dimensions, angular sectors in two dimensions can be extended to three dimensions to discard paths that are not taut.
In two dimensions, angular sectors are conical areas originating from a turning point.
A sector is bounded in at most two sides by a ray parallel to an adjacent obstacle edge, and a ray that points from the turning point's parent node to the turning point.
Each ray can be extruded into three dimensions, forming a plane that intersects the convex obstacle edge where the turning point is located.

The extruded sector is unbounded along the axis of the convex edge, and can be bounded by examining the cost function.
The Euclidean distance between two points on different convex edges $\medge_i$ and $\medge_{i-1}$ can be represented by
\begin{equation}
    J_i = \lVert (\mx_{i-1} + k_{i-1} \mv_{i-1}) - (\mx_i + k_i \mv_i) \rVert
    \label{conc:eq:cost}
\end{equation}
where, $\mx$ is a point at an end of an edge, $\mv$ is the vector parallel to the edge, and $k$ is a scalar.
For a taut path that passes through $n$ edges, the total cost is
\begin{equation}
    J = \lVert \mx_S - (\mx_1 + k_1 \mv_1) \rVert 
        + \sum_{i=2}^{n} J_i
        +\lVert (\mx_n + k_n \mv_n) - \mx_T \rVert,
    \label{conc:eq:totalcost}
\end{equation}
where $\mx_S$ and $\mx_G$ is the start and goal points of the planner respectively.
$J$ is a three-dimensional convex function where a non-unique minimum can be found.

Suppose that every edge is an finite long line, and suppose that the shortest path passes through edges in the order $(\medge_1, \medge_2, \cdots, \medge_n)$. 
Let $K^*=\{k_1^*, k_2^*, \cdots \}$  be the values of $k$ that yield the shortest paths.
Consider a path segment lying between the edges $\medge_{i-1}$, $\medge_{i}$ and $\medge_{i+1}$. If the segment is part of a shortest path, perturbing $k_i$ from $k^*_i$ while keeping $k_{i-1}$ and $k_{i+1}$ constant will always yield a longer path; otherwise, $k^*_i$ will not be the solution.
As such, at $k^*_i$, the derivative of the cost $\partial J / \partial k_i$ around the edge $\medge_i$ has to be zero.
The partial derivative around $k_i$ is
\begin{equation}
    \frac{\partial J}{\partial k_i} = 
        \frac{[ (\mx_i + k_i \mv_i) - (\mx_{i-1} + k_{i-1} \mv_{i-1}) ]^\top \mv_i}
             {\lVert (\mx_i + k_i \mv_i) - (\mx_{i-1} + k_{i-1} \mv_{i-1}) \rVert}
        + \frac{[ (\mx_i + k_i \mv_i) - (\mx_{i+1} + k_{i+1} \mv_{i+1}) ]^\top \mv_i}
             {\lVert (\mx_i + k_i \mv_i) - (\mx_{i+1} + k_{i+1} \mv_{i+1}) \rVert}.
         \label{conc:eq:angsec}
\end{equation}
Let the turning point at $\medge_{i-1}$, $\medge_{i}$, and $\medge_{i+1}$ be $\mx_{i-1}^*$, $\mx_i^*$, and $\mx_{i+1}^*$ respectively.
Let $\mv_{i-1}^* = \mx_i^* - \mx_{i-1}^*$ and $\mv_{i+1}^* = \mx_i^* - \mx_{i+1}^*$.
Rearranging Eq. (\ref{conc:eq:angsec}) and equating to zero,
\begin{equation}
    \begin{aligned}
    \left. \frac{\partial J}{\partial k_i} \right|_{K=K^*} 
        &= \frac{v_{i-1}^{*\top} \mv_i }{\lVert v_{i-1}^* \rVert} + \frac{v_{i+1}^{*\top} \mv_i }{\lVert v_{i+1}^* \rVert}
        \\
        &= \lVert v_i \rVert \left(\frac{v_{i-1}^{*\top} \mv_i }{\lVert v_{i-1}^* \rVert \lVert v_i \rVert} + \frac{v_{i+1}^{*\top} \mv_i }{\lVert v_{i+1}^* \rVert \lVert v_i \rVert} \right)
        \\
        &= \lVert v_i \rVert \left( \cos(\theta_{i-1}) + \cos(\theta_{i+1}) \right)
        \\
        &= 0,
    \end{aligned}
\end{equation}
where $\theta_{i-1}$ and $\theta_{i+1}$ are the angles the edge $\medge_i$ make with $\mv_{i-1}*$, and $\mv_{i+1}^*$ respectively.
As $\lVert \mv_i \rVert \ne 0$ for any edge, the following geometric property can be found for any taut segment of the path:
\begin{equation}
    \cos(\theta_{i-1}) + \cos(\theta_{i+1}) = 0.
    \label{conc:eq:angsec2}
\end{equation}
The relation is akin to drawing a straight line on a piece of paper, and folding the paper in half. The straight line is the taut path, and the crease at which the paper is folded is the edge.

% Consider a point at $\mx_i^*$. An adjacent vertex along the edge has to lie at $\mx_i^* + k_\mathrm{adj}\mv_i$ where $\lVert k_\mathrm{adj}\mv_i \rVert = 1$.
Eq. (\ref{conc:eq:angsec2}) can be extended to constrain the three-dimensional angular sector for a path planner that only finds turning points on vertices. 
Let $\mx_h^* = \mx_i^* + k_h\mv_i$ be a point that lies halfway between $\mx_i^*$ and an adjacent vertex, such that $\lVert k_h\mv_i \rVert = 0.5$.
Let $\mv_{h,i-1}^* = \mx_h^* - \mx^*_{i-1}$ and  $\mv_{h,i+1}^* = \mx_h^* - \mx^*_{i+1}$.
A taut path that passes through $\mx_h^*$ has to obey Eq. (\ref{conc:eq:angsec2}), implying that the acceptable path at $\mx^*_i$ cannot pass through the conical surface
\begin{equation}
    \begin{aligned}
    0 &= \frac{v_{h,i-1}^{*\top} \mv_i }{\lVert v_{h,i-1}^* \rVert \lVert v_i \rVert} + \frac{v_{h,i+1}^{*\top} \mv_{i+1} }{\lVert v_{h,i+1}^* \rVert \lVert v_i \rVert} 
        \\
    &= \cos(\theta_{h,i-1}) + \cos(\theta_{h,i+1}),
    \end{aligned},
    \label{conc:eq:angsec3}
\end{equation}
where $\theta_{h,i-1}$ and $\theta_{h,i+1}$ are the angles the edge $\medge_i$ respectively makes with $\mv_{h,i-1}^*$ and $\mv_{h,i+1}^*$ at $\mx_h^*$.
Two conical surfaces described by Eq. (\ref{conc:eq:angsec3}) appear on both sides of a vertex along an edge, and forms the final boundaries for the three-dimensional angular sector, on top of the planes extruded from a two-dimensional angular sector.

\section{Conclusion}
\label{conc:sec:conc}
In this work, several methods to navigate non-convex obstacles are introduced for vector-based algorithms that delay \gls{los} checks.
Such methods include the source-pledge and target-pledge algorithms, and the source progression and target progression methods. 
Several lower-level mechanisms are introduced, such as the contour assumption, the phantom points, the best hull, and a versatile multi-dimensional ray tracer for collision detection in a binary occupancy grid.
By combining the mechanisms and methods, a novel vector-based path planner that delay \gls{los} checks, \rtwo{}, is introduced. \rtwo{} is subsequently evolved to \rtwop{}. 
\rtwop{} is simpler, resolves interminable searches when no paths exist, and is faster in maps with many disjoint obstacles.

Vector-based algorithms belong to a novel class of any-angle path planners that accelerates searches by rapidly moving towards a goal point.
A vector-based algorithm that delays \gls{los} checks can find the shortest path much faster than state-of-the-art when paths are expected to have few turning points.
By moving rapidly towards the goal and delaying line-of-sight checks, the time complexity is largely invariant to the distance between the turning points and the free space between the turning points, and largely dependent on the number of collided line-of-sight checks within a search.
As such, vector-based planners that delay \gls{los} checks are extremely fast if the path is expected to have few turning points regardless of the length.
Such a path is likely to occur in a map which is sparse and large, have few disjoint obstacles, and few highly non-convex obstacles with intersecting convex hulls.

While vector-based algorithms that delay \gls{los} checks are fast, they have exponential time complexity in the worst case.
As \gls{los} checks are delayed, the costs between turning points cannot be verified, and paths cannot be discarded by comparing costs like A*.
Several novel methods are introduced in the overlap rules of \rtwo{} and \rtwop{} to improve the average search complexity, by discarding searches that are guaranteed to yield longer paths.

Future works may improve upon the algorithms by introducing additional methods, and extend from the ideas stated in Sec. \ref{conc:sec:future}.

    \printbibliography[title=References]

@preamble{ " \newcommand{\noop}[1]{} " }

@article{bib:phistar,
author = {Nash, Alex and Koenig, Sven and Likhachev, Maxim},
year = {2009},
month = {07},
pages = {},
title = {Incremental Phi*: Incremental Any-Angle Path Planning on Grids},
journal = {Lab Papers (GRASP)}
}

@article{bib:3ddda,
  author={Fujimoto, Akira and Tanaka, Takayuki and Iwata, Kansei},
  journal={IEEE Computer Graphics and Applications}, 
  title={ARTS: Accelerated Ray-Tracing System}, 
  year={1986},
  volume={6},
  number={4},
  pages={16-26},
  doi={10.1109/MCG.1986.276715}
}

@article{bib:octree2017,
  author       = {Maxim Tatarchenko and
                  Alexey Dosovitskiy and
                  Thomas Brox},
  title        = {Octree Generating Networks: Efficient Convolutional Architectures
                  for High-resolution 3D Outputs},
  journal      = {CoRR},
  volume       = {abs/1703.09438},
  year         = {2017},
  url          = {http://arxiv.org/abs/1703.09438},
  eprinttype    = {arXiv},
  eprint       = {1703.09438},
  bibsource    = {dblp computer science bibliography, https://dblp.org}
}

@article{bib:rayslater,
  title={Tracing a ray through uniformly subdivided n-dimensional space},
  author={Slater, Mel},
  journal={The Visual Computer},
  volume={9},
  pages={39--46},
  year={1992},
  publisher={Springer}
}

@article{bib:raycleary,
  title={Analysis of an algorithm for fast ray tracing using uniform space subdivision},
  author={Cleary, John G and Wyvill, Geoff},
  journal={The Visual Computer},
  volume={4},
  pages={65--83},
  year={1988},
  publisher={Springer}
}

@inproceedings{bib:rayama,
  title={A fast voxel traversal algorithm for ray tracing.},
  author={Amanatides, John and Woo, Andrew and others},
  booktitle={Eurographics},
  volume={87},
  number={3},
  pages={3--10},
  year={1987},
  organization={Citeseer}
}

@article{bib:r2p,
title = {Evolving R2 to R2+: Optimal, Delayed Line-of-sight Vector-based Path Planning},
journal = {Robotics and Autonomous Systems},
year = {\noop{2024}(Submitted)},
author = {Yan Kai Lai and Prahlad Vadakkepat and Cheng Xiang},
}

@article{bib:r2,
title = {R2: Optimal vector-based and any-angle 2D path planning with non-convex obstacles},
journal = {Robotics and Autonomous Systems},
volume = {172},
pages = {104606},
year = {2024},
issn = {0921-8890},
doi = {https://doi.org/10.1016/j.robot.2023.104606},
url = {https://www.sciencedirect.com/science/article/pii/S0921889023002452},
author = {Yan Kai Lai and Prahlad Vadakkepat and Cheng Xiang}
}

@inproceedings{bib:dps,
  title={Dual Euclidean Shortest Path Search},
  author={Hechenberger, Ryan and Stuckey, Peter J and Le Bodic, Pierre and Harabor, Daniel D},
  booktitle={Proceedings of the International Symposium on Combinatorial Search},
  volume={15},
  number={1},
  pages={285--287},
  year={2022}
}

@phdthesis{bib:rpf,
  author        = "Paul Oprea",
  title         = "A Novel Online Any-Angle Path Planning Algorithm",
  school        = "Univ. of Kent",
  address       = "Kent",
  year          = "2017",
  url           = "https://kar.kent.ac.uk/71757/"
}

@article{bib:rayscan, 
  title={Online Computation of Euclidean Shortest Paths in Two Dimensions},
  volume={30}, 
  url={https://ojs.aaai.org/index.php/ICAPS/article/view/6654}, 
  number={1}, 
  journal={Proc. of the Int. Conf. on Automated Planning and Scheduling}, 
  author={Hechenberger, Ryan and Stuckey, Peter J and Harabor, Daniel and Le Bodic, Pierre and Cheema, Muhammad Aamir}, 
  year={2020},
  month={6}, 
  pages={134-142}
}

@article{bib:rayscanp,
      title={Multi-Target Search in Euclidean Space with Ray Shooting (Full Version)}, 
      author={Ryan Hechenberger and Daniel Harabor and Muhammad Aamir Cheema and Peter J Stuckey and Pierre Le Bodic},
      year={2022},
      eprint={2207.02436},
      archivePrefix={arXiv},
      primaryClass={cs.CG}
}

@misc{bib:my2dcode,
    author = {Lai, Yan Kai},
    title = {{2D Path Planners for Binary Occupancy Grids}},
    howpublished = "\url{https://github.com/LaiYanKai/2D}",
    year = 2023,
    note = "[Online; accessed 29-Oct-2023]"
}

@misc{bib:my2dvis,
    author = {Lai, Yan Kai},
    title = {{2D Planners}},
    howpublished = "\url{https://laiyankai.github.io/PlannersJS}",
    year = 2023,
    note = "[Online; accessed 15-Jun-2024]"
}

@article{bib:astar,
  author = {Hart, P. E. and Nilsson, N. J. and Raphael, B.},  
  doi = {10.1109/TSSC.1968.300136},  
  issn = {0536-1567},
  journal = {IEEE Trans. on Syst. Sci. and Cybern.},
  month = {7},
  number = 2,
  pages = {100-107},
  title = {A Formal Basis for the Heuristic Determination of Minimum Cost Paths},
  volume = 4,
  year = 1968
}

@book{bib:turtle,
  address = {Cambridge, MA},
  author = {Abelson, H. and DiSessa, A.},
  publisher = {MIT Press},
  title = {Turtle Geometry: The Computer as a Medium for Exploring Mathematics},
  year = 1981
}

@article{bib:anya,
  author    = {Daniel Damir Harabor and
               Alban Grastien and
               Dindar {\"{O}}z and
               Vural Aksakalli},
  title     = {Optimal Any-Angle Pathfinding In Practice},
  journal   = {J. Artif. Intell. Res.},
  volume    = {56},
  pages     = {89--118},
  year      = {2016},
  url       = {https://doi.org/10.1613/jair.5007},
  doi       = {10.1613/jair.5007},
  bibsource = {dblp computer science bibliography, https://dblp.org}
}

@inproceedings{bib:uras,
  title={An empirical comparison of any-angle path-planning algorithms},
  author={Uras, Tansel and Koenig, Sven},
  booktitle={Eighth Annu. Symp. on Combinatorial Search},
  year={2015},
  url={http://idm-lab.org/anyangle}
}

@article{bib:thetastar,
  title={Theta*: Any-angle path planning on grids},
  author={Daniel, Kenny and Nash, Alex and Koenig, Sven and Felner, Ariel},
  journal={J. of Artif. Intell. Res.},
  volume={39},
  pages={533--579},
  year={2010}
}

@incollection{bib:fielddstar,
  title={Field D*: An interpolation-based path planner and replanner},
  author={Ferguson, Dave and Stentz, Anthony},
  booktitle={Robot. research},
  pages={239--253},
  year={2007},
  publisher={Springer}
}

@inproceedings{bib:subgoal,
  title={Subgoal graphs for optimal pathfinding in eight-neighbor grids},
  author={Uras, Tansel and Koenig, Sven and Hern{\'a}ndez, Carlos},
  booktitle={Proc. of the Int. Conf. on Automated Planning and Scheduling},
  volume={23},
  number={1},
  year={2013}
}

@inproceedings{bib:blockastar,
  title={Block A*: Database-driven search with applications in any-angle path-planning},
  author={Yap, Peter and Burch, Neil and Holte, Robert and Schaeffer, Jonathan},
  booktitle={Proc. of the AAAI Conf. on Artif. Intell.},
  volume={25},
  number={1},
  year={2011}
}

@inproceedings{bib:lazythetastar,
  title={Lazy Theta*: Any-angle path planning and path length analysis in 3D},
  author={Nash, Alex and Koenig, Sven and Tovey, Craig},
  booktitle={Proc. of the AAAI Conf. on Artif. Intell.},
  volume={24},
  number={1},
  year={2010}
}

@article{bib:bug,
  title={Path-planning strategies for a point mobile automaton moving amidst unknown obstacles of arbitrary shape},
  author={Lumelsky, Vladimir J and Stepanov, Alexander A},
  journal={Algorithmica},
  volume={2},
  number={1},
  pages={403--430},
  year={1987},
  publisher={Springer}
}

@article{bib:tangentbug,
  title={Tangentbug: A range-sensor-based navigation algorithm},
  author={Kamon, Ishay and Rimon, Elon and Rivlin, Ehud},
  journal={The Int. J. of Robot. Res.},
  volume={17},
  number={9},
  pages={934--953},
  year={1998},
  publisher={Sage Publications Sage CA: Thousand Oaks, CA}
}

@inproceedings{bib:polyanya,
  author    = {Michael Cui and Daniel D. Harabor and Alban Grastien},
  title     = {Compromise-free Pathfinding on a Navigation Mesh},
  booktitle = {Proc. of the Twenty-Sixth Int. Joint Conf. on Artif. Intell., {IJCAI-17}},
  pages     = {496--502},
  year      = {2017},
  doi       = {10.24963/ijcai.2017/70},
  url       = {https://doi.org/10.24963/ijcai.2017/70},
}

@inproceedings{bib:ohleong,
  title={Edge n-level sparse visibility graphs: Fast optimal any-angle pathfinding using hierarchical taut paths},
  author={Oh, Shunhao and Leong, Hon Wai},
  booktitle={Proceedings of the International Symposium on Combinatorial Search},
  volume={8},
  number={1},
  pages={64--72},
  year={2017}
}

@article{bib:bench,
  title={Benchmarks for Grid-Based Pathfinding},
  author={Sturtevant, N.},
  journal={Trans. on Comput. Intell. and AI in Games},
  volume={4},
  number={2},
  pages={144 -- 148},
  year={2012},
  url = {http://web.cs.du.edu/~sturtevant/papers/benchmarks.pdf},
}

@phdthesis{bib:nashthesis,
  author        = {Alex Nash},
  title         = {Any-angle Path Planning},
  school        = {Univ. of Southern California},
  address       = {Los Angeles},
  month         = {8},
  year          = {2012},
  url           = {http://idm-lab.org/bib/abstracts/papers/dissertation-nash.pdf}
}

@article{bib:dijkstra,
  title={A note on two problems in connexion with graphs},
  author={Dijkstra, Edsger W and others},
  journal={Numerische mathematik},
  volume={1},
  number={1},
  pages={269--271},
  year={1959}
}

@article{bib:bresenham,
  author={J. E. {Bresenham}},
  journal={IBM Syst. J.}, 
  title={Algorithm for computer control of a digital plotter}, 
  year={1965},
  volume={4},
  number={1},
  pages={25-30},
  doi={10.1147/sj.41.0025}
}

@article{bib:oneshot,
  author={Kulvicius, Tomas and Herzog, Sebastian and Lüddecke, Timo and Tamosiunaite, Minija and Wörgötter, Florentin},    
  title={One-Shot Multi-Path Planning Using Fully Convolutional Networks in a Comparison to Other Algorithms},      
	journal={Frontiers in Neurorobotics},      
	volume={14},      
  pages={115},     
	year={2021},      
	url={https://www.frontiersin.org/article/10.3389/fnbot.2020.600984},       
	doi={10.3389/fnbot.2020.600984},      
	issn={1662-5218}
}

@article{bib:probrob,
  title={Probalistic robotics},
  author={Thrun, Sebastian and Burgard, Wolfram and Fox, Dieter},
  journal={Kybernetes},
  year={2006},
  publisher={Emerald Group Publishing Limited}
}

@inproceedings{bib:bfs,
  title={The shortest path through a maze},
  author={Moore, Edward F},
  booktitle={Proc. Int. Symp. Switching Theory, 1959},
  pages={285--292},
  year={1959}
}

@book{bib:dfs,
  title={Graph algorithms},
  author={Even, Shimon},
  year={2011}, 
  publisher={Cambridge University Press}
}

@inproceedings{bib:multires1,
  title={Local multiresolution path planning},
  author={Behnke, Sven},
  booktitle={Robot Soccer World Cup},
  pages={332--343},
  year={2003},
  organization={Springer}
}

@article{bib:multires0,
  title={Multiresolution path planning for mobile robots},
  author={Kambhampati, Subbarao and Davis, Larry},
  journal={IEEE Journal on Robotics and Automation},
  volume={2},
  number={3},
  pages={135--145},
  year={1986},
  publisher={IEEE}
}

@article{bib:hpastar,
  title={Near optimal hierarchical path-finding},
  author={Botea, Adi and M{\"u}ller, Martin and Schaeffer, Jonathan},
  journal={J. Game Dev.},
  volume={1},
  number={1},
  pages={1--30},
  year={2004},
  publisher={Citeseer}
}

@inproceedings{bib:fielddstarmars,
  title={Global path planning on board the mars exploration rovers},
  author={Carsten, Joseph and Rankin, Arturo and Ferguson, Dave and Stentz, Anthony},
  booktitle={2007 IEEE Aerospace Conference},
  pages={1--11},
  year={2007},
  organization={IEEE}
}

@article{bib:vg,
author = {Lozano-P\'{e}rez, Tom\'{a}s and Wesley, Michael A.},
title = {An Algorithm for Planning Collision-Free Paths among Polyhedral Obstacles},
year = {1979},
issue_date = {Oct. 1979},
publisher = {Association for Computing Machinery},
address = {New York, NY, USA},
volume = {22},
number = {10},
issn = {0001-0782},
url = {https://doi.org/10.1145/359156.359164},
doi = {10.1145/359156.359164},
journal = {Commun. ACM},
month = oct,
pages = {560–570},
numpages = {11},
}

@inproceedings{bib:voronoi0,  
  author={Bhattacharya, Priyadarshi and Gavrilova, Marina L.},  
  booktitle={4th International Symposium on Voronoi Diagrams in Science and Engineering (ISVD 2007)},   
  title={Voronoi diagram in optimal path planning},   
  year={2007},  
  volume={},  
  number={},  
  pages={38-47},  
  doi={10.1109/ISVD.2007.43}
}

@inproceedings{bib:vg3d0,  
  author={Jiang, K. and Seneviratne, L.S. and Earles, S.W.E.},  
  booktitle={Proceedings of 1993 IEEE/RSJ International Conference on Intelligent Robots and Systems (IROS '93)},   
  title={Finding the 3D shortest path with visibility graph and minimum potential energy},   
  year={1993},  
  volume={1},  
  number={},  
  pages={679-684 vol.1},  
  doi={10.1109/IROS.1993.583190}
}

@inproceedings{bib:vg3d1,  
  author={Huang, Sunan and Teo, Rodney Swee Huat},  
  booktitle={2019 International Conference on Unmanned Aircraft Systems (ICUAS)},   
  title={Computationally Efficient Visibility Graph-Based Generation Of 3D Shortest Collision-Free Path Among Polyhedral Obstacles For Unmanned Aerial Vehicles},   
  year={2019},  
  volume={},  
  number={},  
  pages={1218-1223},  
  doi={10.1109/ICUAS.2019.8798322}
}

@inproceedings{bib:vg3d2,
  author = {You, Yangwei and Cai, Caixia and Wu, Yan},
  title = {3D Visibility Graph Based Motion Planning and Control},
  year = {2019},
  isbn = {9781450372350},
  publisher = {Association for Computing Machinery},
  address = {New York, NY, USA},
  url = {https://doi.org/10.1145/3373724.3373735},
  doi = {10.1145/3373724.3373735},
  booktitle = {Proceedings of the 2019 5th International Conference on Robotics and Artificial Intelligence},
  pages = {48–53},
  numpages = {6},
  location = {Singapore, Singapore},
  series = {ICRAI '19}
}

@inproceedings{bib:2dpolyphard,
  author={Reif, John H.},
  booktitle={20th Annual Symposium on Foundations of Computer Science (sfcs 1979)}, 
  title={Complexity of the mover's problem and generalizations}, 
  year={1979},
  volume={},
  number={},
  pages={421-427},
  doi={10.1109/SFCS.1979.10}
}

@inproceedings{bib:rrtconnect,
  author={Kuffner, J.J. and LaValle, S.M.},
  booktitle={Proceedings 2000 ICRA. Millennium Conference. IEEE International Conference on Robotics and Automation. Symposia Proceedings (Cat. No.00CH37065)}, 
  title={RRT-connect: An efficient approach to single-query path planning}, 
  year={2000},
  volume={2},
  number={},
  pages={995-1001 vol.2},
  doi={10.1109/ROBOT.2000.844730}
}

@ARTICLE{bib:3dpolynphard,
  author={Gewali, L.P. and Ntafos, S. and Tollis, I.G.},
  journal={IEEE Transactions on Robotics and Automation}, 
  title={Path planning in the presence of vertical obstacles}, 
  year={1990},
  volume={6},
  number={3},
  pages={331-341},
  doi={10.1109/70.56665}
}

@inproceedings{bib:3dpolynphard2,
author = {Canny, John and Reif, John},
title = {New Lower Bound Techniques for Robot Motion Planning Problems},
year = {1987},
isbn = {0818608072},
publisher = {IEEE Computer Society},
address = {USA},
url = {https://doi.org/10.1109/SFCS.1987.42},
doi = {10.1109/SFCS.1987.42},
pages = {49–60},
numpages = {12},
series = {SFCS '87}
}

@article{bib:karman&frazzoli,
  author = {Sertac Karaman and Emilio Frazzoli},
  title ={Sampling-based algorithms for optimal motion planning},
  journal = {The International Journal of Robotics Research},
  volume = {30},
  number = {7},
  pages = {846-894},
  year = {2011},
  doi = {10.1177/0278364911406761},
}

@article{bib:prm,
  author={Kavraki, L.E. and Svestka, P. and Latombe, J.-C. and Overmars, M.H.},
  journal={IEEE Transactions on Robotics and Automation}, 
  title={Probabilistic roadmaps for path planning in high-dimensional configuration spaces}, 
  year={1996},
  volume={12},
  number={4},
  pages={566-580},
  doi={10.1109/70.508439}
}

@inproceedings{bib:rrtstar,
  author={Karaman, Sertac and Walter, Matthew R. and Perez, Alejandro and Frazzoli, Emilio and Teller, Seth},
  booktitle={2011 IEEE International Conference on Robotics and Automation}, 
  title={Anytime Motion Planning using the RRT*}, 
  year={2011},
  volume={},
  number={},
  pages={1478-1483},
  doi={10.1109/ICRA.2011.5980479}
}

@article{bib:rrt,
    author = {Steven M. Lavalle},
    title = {Rapidly-Exploring Random Trees: A New Tool for Path Planning},
    institution = {},
    year = {1998}
}

@inproceedings{bib:bitstar,
  author={Gammell, Jonathan D. and Srinivasa, Siddhartha S. and Barfoot, Timothy D.},
  booktitle={2015 IEEE International Conference on Robotics and Automation (ICRA)}, 
  title={Batch Informed Trees (BIT*): Sampling-based optimal planning via the heuristically guided search of implicit random geometric graphs}, 
  year={2015},
  volume={},
  number={},
  pages={3067-3074},
  doi={10.1109/ICRA.2015.7139620}
}

@inproceedings{bib:apf,
  author={Khatib, O.},
  booktitle={Proceedings. 1985 IEEE International Conference on Robotics and Automation}, 
  title={Real-time obstacle avoidance for manipulators and mobile robots}, 
  year={1985},
  volume={2},
  number={},
  pages={500-505},
  doi={10.1109/ROBOT.1985.1087247}
}

@article{bib:apfplan,
  author={Hwang, Y.K. and Ahuja, N.},
  journal={IEEE Transactions on Robotics and Automation}, 
  title={A potential field approach to path planning}, 
  year={1992},
  volume={8},
  number={1},
  pages={23-32},
  doi={10.1109/70.127236}
}

@inproceedings{bib:apfplan1,
  author={Warren, C.W.},
  booktitle={Proceedings, 1989 International Conference on Robotics and Automation}, 
  title={Global path planning using artificial potential fields}, 
  year={1989},
  volume={},
  number={},
  pages={316-321 vol.1},
  doi={10.1109/ROBOT.1989.100007}
}

@article{bib:drl0,  
  author={Yu, Jinglun and Su, Yuancheng and Liao, Yifan},   
  title={The Path Planning of Mobile Robot by Neural Networks and Hierarchical Reinforcement Learning},      
  journal={Frontiers in Neurorobotics},      
  volume={14},      
  pages={63},     
  year={2020},      
  doi={10.3389/fnbot.2020.00063},      
  issn={1662-5218},     
}

@article{bib:drl1,
  author = {Gao, Junli and Ye, Weijie and Guo, Jing and Li, Zhongjuan},
  title = {Deep Reinforcement Learning for Indoor Mobile Robot Path Planning},
  journal = {Sensors},
  volume = {20},
  year = {2020},
  number = {19},
  url = {https://www.mdpi.com/1424-8220/20/19/5493},
  issn = {1424-8220},
  doi = {10.3390/s20195493}
}

@article{bib:drl2,
  author={Zhu, Kai and Zhang, Tao},
  journal={Tsinghua Science and Technology}, 
  title={Deep reinforcement learning based mobile robot navigation: A review}, 
  year={2021},
  volume={26},
  number={5},
  pages={674-691},
  doi={10.26599/TST.2021.9010012}
}

@article{bib:drl3,
author = {Yu, Xiaoqiang and Wang, Ping and Zhang, Zexu},
title = {Learning-Based End-to-End Path Planning for Lunar Rovers with Safety Constraints},
journal = {Sensors},
volume = {21},
year = {2021},
number = {3},
url = {https://www.mdpi.com/1424-8220/21/3/796},
issn = {1424-8220},
doi = {10.3390/s21030796}
}

@article{bib:drl4,
  title={Dynamic path planning of unknown environment based on deep reinforcement learning},
  author={Lei, Xiaoyun and Zhang, Zhian and Dong, Peifang},
  journal={Journal of Robotics},
  volume={2018},
  year={2018},
  publisher={Hindawi}
}

@article{bib:next,
  title={Learning to plan in high dimensions via neural exploration-exploitation trees},
  author={Chen, Binghong and Dai, Bo and Lin, Qinjie and Ye, Guo and Liu, Han and Song, Le},
  journal={arXiv preprint arXiv:1903.00070},
  year={2019}
}

@article{bib:drl5,
  title = {Grid Path Planning with Deep Reinforcement Learning: Preliminary Results},
  journal = {Procedia Computer Science},
  volume = {123},
  pages = {347-353},
  year = {2018},
  note = {8th Annual International Conference on Biologically Inspired Cognitive Architectures, BICA 2017 (Eighth Annual Meeting of the BICA Society), held August 1-6, 2017 in Moscow, Russia},
  issn = {1877-0509},
  doi = {https://doi.org/10.1016/j.procs.2018.01.054},
  url = {https://www.sciencedirect.com/science/article/pii/S1877050918300553},
  author = {Aleksandr I. Panov and Konstantin S. Yakovlev and Roman Suvorov}
}

@article{bib:drl6,
  title={Path planning for multi-arm manipulators using deep reinforcement learning: Soft actor--critic with hindsight experience replay},
  author={Prianto, Evan and Kim, MyeongSeop and Park, Jae-Han and Bae, Ji-Hun and Kim, Jung-Su},
  journal={Sensors},
  volume={20},
  number={20},
  pages={5911},
  year={2020},
  publisher={Multidisciplinary Digital Publishing Institute}
}

@book{bib:lavalle,
  title={Planning algorithms},
  author={LaValle, Steven M},
  year={2006},
  publisher={Cambridge university press}
}

@article{bib:petrovic,
  title={Motion planning in high-dimensional spaces},
  author={Petrovi{\'c}, Luka},
  journal={arXiv preprint arXiv:1806.07457},
  year={2018}
}

@inproceedings{bib:rescomprrt,
  title={Resolution complete rapidly-exploring random trees},
  author={Cheng, Peng and LaValle, Steven M},
  booktitle={Proceedings 2002 IEEE international conference on robotics and automation (cat. no. 02CH37292)},
  volume={1},
  pages={267--272},
  year={2002},
  organization={IEEE}
}

@inproceedings{bib:me,
  title={Development and Analysis of an Improved Prototype within a Class of Bug-based Heuristic Path Planners},
  author={Lai, Yan Kai and Vadakkepat, Prahlad and Al Mamun, Abdullah and Xiang, Cheng and Lee, Tong Heng},
  booktitle={2021 IEEE 30th International Symposium on Industrial Electronics (ISIE)},
  pages={1--6},
  year={2021},
  organization={IEEE}
}

@misc{bib:r2code,
  author = {Lai, Yan Kai},
  title = {R2 Github Repository},
  year = {2022},
  url = {\url{https://github.com/LaiYanKai/R2}}
}

@article{bib:dl1,
title = {Safe deep learning-based global path planning using a fast collision-free path generator},
journal = {Robotics and Autonomous Systems},
volume = {163},
pages = {104384},
year = {2023},
issn = {0921-8890},
doi = {https://doi.org/10.1016/j.robot.2023.104384},
url = {https://www.sciencedirect.com/science/article/pii/S0921889023000234},
author = {Shirin Chehelgami and Erfan Ashtari and Mohammad Amin Basiri and Mehdi {Tale Masouleh} and Ahmad Kalhor}
}

@inproceedings{bib:drl7,
  author={Tai, Lei and Paolo, Giuseppe and Liu, Ming},
  booktitle={2017 IEEE/RSJ International Conference on Intelligent Robots and Systems (IROS)}, 
  title={Virtual-to-real deep reinforcement learning: Continuous control of mobile robots for mapless navigation}, 
  year={2017},
  volume={},
  number={},
  pages={31-36},
  doi={10.1109/IROS.2017.8202134}
}
    \appendix
    \fancyhead[RE]{Appendix~\thechapter}
    \chapter{Terms and Conventions in the Thesis}
\label{chap:suppterm}

This section attempts to explain commonly used terms in the thesis for vector-based planning. 
Terms that are commonly used in path planning, such as \textit{tautness}, \textit{admissibility}, \textit{completeness}, and \textit{optimality} are not explained.

\section{Tree Directions and Path}
Consider a path
\begin{equation}
    P = (\mxstart, \cdots, \mxss, \mxs, \mx, \mxt, \mxtt, \cdots ,\mxgoal), \label{suppterm:eq:path}
\end{equation}
illustrated in Fig. \ref{suppterm:fig:path}.
A \textbf{start point} located at $\mxstart$ denotes the point where the algorithm begins searching, and an optimal path has to be found to the \textbf{goal point} at $\mxgoal$.

The \textbf{source direction} and \textbf{target direction} indicate the direction along a searched path.
The source direction leads to the start point, while the target direction leads to the goal point.
In the path $P$, $\mxs$ lies in the source direction of $\mx$, $\mxt$, $\mxtt$, etc.; and in the path $P$, $\mxt$ lies in the target direction of $\mx$, $\mxs$, $\mxss$, etc.
Both directions are referred to as the \textbf{tree directions}.

\def\clipper {\clip (-.5*\ul, -.5*\ul) rectangle ++(20*\ul, 7*\ul)}

\begin{figure}[!ht]
\centering
\begin{tikzpicture}[]
    \clipper;

    \path 
        (1*\ul, 1*\ul) coordinate (xstart)
        (3*\ul, 4*\ul) coordinate (xss)
        (6*\ul, 4*\ul) coordinate (xs)
        (9*\ul, 1*\ul) coordinate (x)
        (12.5*\ul, 2.5*\ul) coordinate (xt)
        (14*\ul, 5*\ul) coordinate (xtt)
        (18*\ul, 6*\ul) coordinate (xgoal)
        ;

    \draw [obs]
        ($(xss) + (-45:{sqrt(2)*\u})$) -- ($(xs) + (-135:{sqrt(2)*\u})$);
    \fill [swatch_obs]
        (xt) rectangle ++(-1*\ul, 1*\ul);
    \fill [swatch_obs]
        (xtt) rectangle ++(1*\ul, -1*\ul);
        
    \node (nstart) [svy pt={shift={(1mm, -3.5mm)}}{center:$\mxstart$}] at (xstart) {};
    \node (nss) [svy pt={shift={(1mm, 3.5mm)}}{center:$\mxss$}] at (xss) {};
    \node (ns) [svy pt={shift={(1mm, 3.5mm)}}{center:$\mxs$}] at (xs) {};
    \node (nx) [black pt={shift={(1mm, -3.5mm)}}{center:$\mx$}] at (x) {};
    \node (nt) [tvy pt={shift={(1mm, -3.5mm)}}{center:$\mxt$}] at (xt) {};
    \node (ntt) [tvy pt={shift={(1mm, 3.5mm)}}{center:$\mxtt$}] at (xtt) {};
    \node (ngoal) [tvy pt={shift={(1mm, -3.5mm)}}{center:$\mxgoal$}] at (xgoal) {};

    \draw [dashed] (nstart) -- (nss);
    \draw [dotted, thick] (nss) -- (ns) -- (nx) -- (nt) -- (ntt);
    \draw [dashed] (ntt) -- (ngoal);

\end{tikzpicture}

\caption[Path and tree directions]{
    A path of points is illustrated.
    The points at $\mxstart$, $\mxss$, and $\mxs$ lie in the source direction of $\mx$, and the points at $\mxt$, $\mxtt$, and $\mxgoal$ lie in the target direction of $\mx$.
    The point at $\mxs$ is the source point of $\mx$, and the point at $\mxt$ is the target point of $\mx$.
}
\label{suppterm:fig:path}
\end{figure}

A \textbf{source point} and \textbf{target point} refers to a point that is \textit{adjacent} to a point along the path. In Eq. (\ref{suppterm:eq:path}), $\mxs$ is a source point of $\mx$, but $\mxss$ is not; $\mxss$ is a source point of $\mxs$. 
Likewise, $\mxt$ is a target point of $\mx$, but $\mxtt$ is not.
The definition can be extended to the fundamental search units \textit{nodes} and \textit{links}.

\section{Search Trees} \label{suppterm:sec:tree}
The \textbf{search tree} are the collection of paths investigated by a path planner.
A path is a \textit{branch} on the search tree, and the \textbf{root point} of a search tree in \rtwo{}, A*, Anya, Theta* etc. is the starting point.
The tree will branch into a different path as the planner searches. For example, a new path $P_2 = (\mxstart, \cdots, \mxss, \mxs, \mx, \mx_{T,2}, \mx_{TT,2}, \allowbreak \cdots, \mxgoal)$ may branch from the path $P$ at $\mx$ for the planner to investigate the new route from $\mx$ that passes through $\mx_{T,2}$.

\input{chap_supp/fig_suppterm_tree}

In \rtwop{}, two search trees are introduced, which are \textbf{Source Tree} ($S$-tree) and \textbf{Target Tree} ($T$-tree). The $S$-tree is rooted at the start point, and the $T$-tree is rooted at the goal point.
Both trees are connected at their leaf points. At a connected leaf point, an enqueued query (Sec. \ref{suppterm:sec:openlist}) can be found. 
The trees are illustrated in Fig. \ref{suppterm:fig:r2ptree}.

\input{chap_supp/fig_suppterm_r2ptree}

\section{Line-of-sight Checks and Visibility}

If two points have \textbf{line-of-sight} or \textbf{visibility}, a straight line can be drawn between the two points, and the line will not intersect any obstacle. 
If the point at $\mx$ in Eq. (\ref{suppterm:eq:path}) has \textbf{cumulative visibility} to $\mxss$, the path segment ($\mxss, \mxs, \mx$) can be drawn without obstruction.

\input{chap_supp/fig_suppterm_vis}

In \rtwop{}, the definition of \textit{cumulative visibility} is overloaded, as only the cumulative visibilities to the root points of the $S$-tree and $T$-tree (start and goal points, respectively) are considered.
If \textit{cumulative visibility} is used in the context of an $S$-tree, the term refers to the cumulative visibility of a currently investigated point $\mx$ to the start point at $\mxstart$. 
If the term is used in the context of a $T$-tree, it refers to the cumulative visibility of a currently investigated point $\mx$ to the goal point at $\mxgoal$.
The concepts of visibility are illustrated in Fig. \ref{suppterm:fig:vis}.

\section{Cast, Projection, and Traces}
A \textbf{cast} refers to a line-of-sight check originating from a point called the \textbf{cast point}, and a \textbf{trace} refers to a search that traces an obstacle contour \cite{bib:rpf}. 
When a cast has reached its destination, the line-of-sight can be \textbf{projected} in the same direction as the cast from the destination \cite{bib:rayscan}.

\input {chap_supp/fig_suppterm_castandtrace}

When a cast \textit{collides}, the line-of-sight check has been obstructed by an obstacle.
In general, two traces will occur on the \textbf{left} and \textbf{right} side of the collision point.
While tracing a contour, the obstacle will lie on the right side of a \textit{left}-sided trace. For a \textit{right}-sided trace, the obstacle will lie on the left side of the trace direction.
In \rtwo{} and \rtwop{}, a \textbf{third trace} will continue from the casting point if the destination is the goal point. 
The cast and traces are illustrated in Fig. \ref{suppterm:fig:castandtrace}.

A similar way to describe traces is to use \textbf{angular directions}.
A trace will travel in the clockwise or anti-clockwise angular direction with respect to a point, or in some cases, the angular direction will not change.
Typically, if travelling across a convex corner causes the angular direction of a trace to reverse, a potential turning point is found \cite{bib:anya, bib:rayscan}.
As the angular direction becomes unreliable while tracing within the convex hull of non-convex obstacles, the \textit{left} or \textit{right} side is  used to describe traces instead.

\input {chap_supp/fig_suppterm_angle}

A trace's \textbf{angular deviation} is used to describe the angle displaced by the trace when viewed from a point. The angle is measured from the start of the trace (at the collision point) to the current point of the trace.
In \rtwo{} and \rtwop{}, the \textbf{maximum angular deviation} of the trace is used by the algorithms' progression rule to infer if the trace is in the convex hull of a non-convex obstacle.
The angular directions and deviation are illustrated in Fig. \ref{suppterm:fig:angle}

\section{Turning Points and Phantom Points}

A \textbf{turning point} refers to a point located at a convex corner that causes the path $P$ to change direction around the point. All points along a taut path are turning points, except for the start and goal points.

A \textbf{phantom point} is an imaginary turning point introduced by the thesis.
A phantom point is located at a non-convex corner that mimics a turning point in the target direction. The path would have to turn around the point in order for the path to be admissible, and for the path length to increase monotonically as the search progresses.
A phantom point will be pruned before a taut path is found, and will not form part of the optimal path solution.
The illustration for a phantom point can be found in Fig. \ref{suppterm:fig:points}.

\input {chap_supp/fig_suppterm_points}

\section{Fundamental Search Units} \label{suppterm:sec:fundamental}
A fundamental search unit refers to the programming object that is used extensively by an algorithm to encapsulate a location being searched. 
For A*, Theta*, the fundamental unit is the \textbf{node}, which is located at a grid cell or grid vertex.  
For \rs{}, \rsp{}, and \rtwo{}, the unit is the \textbf{node}, and is located at a convex corner.
For Anya, the units are the \textbf{cone node} and \textbf{flat node}, and are located at convex corners.

In \rtwop{}, the unit is the \textbf{link}, which is a path segment between two points.
A link forms part of the $S$-tree or $T$-tree (Sec. \ref{suppterm:sec:tree}), and is \textbf{anchored} at a point closer to the leaves of the tree the link is in.
The illustration for links can be found in Fig. \ref{suppterm:fig:r2ptree}.

\section{Rays and Sectors}
A \textbf{ray} describes a directional vector that originates from one point. 
A cast or a projection can be described as a ray.
An \textbf{angular sector} of a turning point is the circular sector where another taut turning point (successor, \cite{bib:rayscan}) can be found, and is bounded by two rays originating from the former turning point.
An \textbf{occupied sector} of a \textit{corner} is the circular sector where the adjacent obstacle can be found.

\input{chap_supp/fig_suppterm_sectors}

\section{Expansion, Query, and Open List}
\label{suppterm:sec:openlist}
The \textbf{open-list} is a common concept in algorithms derived from A* such as any-angle planners.
It is a priority queue that sorts paths being searched by the sum ($f$-cost) of their cost-to-come ($g$-cost) and cost-to-go ($h$-cost) \cite{bib:astar}.
When a search unit (Sec. \ref{suppterm:sec:fundamental}) is \textbf{queued} or \textbf{enqueued} into the open-list, the path that passes through the search unit is being sorted into the open-list.
If a search unit is \textbf{polled} from the open-list, its path is removed from the open-list.
After a poll, the search unit will be \textbf{expanded}, where the algorithm continues searching from the search unit.
The search that is conducted for a trace and a cast is called a \textbf{query}.
When a query is said to be \textit{queued}, the search unit being expanded by the query is queued.

% \section{Coordinate Frames and Directional Indices}
% This section describes geometry that is used in the algorithms.

% The term \textbf{grid coordinates} is used to refer to the coordinates within the \textbf{grid frame} of the occupancy grid.
% For an occupancy grid, the origin of the grid frame is at the lower left corner, and the positive $x$-axis points to the right, and every integer $x$ coordinate corresponds to a vertical grid line. 
% The positive $y$-axis points upwards, and every integer $y$ coordinate corresponds to a horizontal grid line.
% \textbf{Cell coordinates} refer to the \textit{integer} coordinates used for accessing the cost within each grid cell.
% The cell coordinates are in the \textbf{cell frame}, and its origin is translated from the grid frame by 0.5 in the positive $x$ and positive $y$ direction.

% A directional index is one of the eight directions on a two-dimensional occupancy grid.
% 0 corresponds to north, which is in the positive $y$-direction, and increments anti-clockwise until 7.

% * Visibility
% ** Line-of-sight check 
% ** Cumulative Visibility

% *Vector Based
% ** Cast
% ** Collided Cast, Reached Cast
% ** Projection
% ** Trace
% ** Sides of Trace: Left Trace, Right Trace
% ** Angular Deviation of Trace

% * Rays
% ** Progression Ray
% ** Progressed if angular deviation increases at current corner.
% ** Sector Ray

% * Sectors
% ** angular sector,
% ** occupied sector,

% * Fundamental Search Units
% ** Node
% ** Links

% * Open List
% ** Enqueuing
% ** Polling
% ** Expanding
\chapter{Implementation for \rtwo{}}
\label{chap:suppr2}
\setcounter{algorithm}{0}
\renewcommand{\thealgorithm}{\thechapter.\arabic{algorithm}}\setcounter{equation}{0}
\setcounter{figure}{0}
\setcounter{table}{0}
\makeatletter

\section{Detailed Pseudocode for \rtwo{}} 
The detailed implementation of \rtwo{} are in Alg. \ref{suppalgr2:run} to Algorithms \ref{suppalgr2:placenode} below.

The node types for \rtwo{} are described below, and are different as the ones described in \rtwop{}.
A node with $\mnsy$ type is a source node that has cumulative visibility to the start node, and its cost-to-come is known.
An $\mnsu$ node is a source node where cumulative visibility and cost-to-come is unknown.
A source node stores only cost-to-come and a target node stores only cost-to-go.

A $\mnty$ node is a target node has cumulative visibility to the goal node, and its cost-to-go is known.
A $\mntu$ node is a target node where cumulative visibility and cost-to-go is not known.
A $\mntm$ target node is created when traces are interrupted and queued.
A $\mnph$ target node is an ad hoc point or a phantom point. Queries that reach a $\mnph$ node (ad-hoc point) can be discarded. 

An $\mney$ source node is an $\mnsy$ node with more expensive cost-to-come than other nodes at the same location. 
An $\mneu$ node is the same as an $\mnsu$ node, but has a source $\mney$ node along the tree in the source direction. 
A trace that is castable from an $\mneu$ node does not cast, and moves to the most recent $\mney$ ancestor node for casting.

All source nodes point to one source node, and may point to at least one target nodes.
All target nodes may point to at least one source node and at least one target node, except for $\mnty$ node which points to only one target node.
The number of pointers must be maintained to ensure that the path can be found.

The reader may observe that a tracing query continues from a $\mntm$ node which points to one source node and at least one target node. 
A casting query checks line-of-sight between two pairs of nodes, where the source node may point to multiple targets and the target node may point to multiple sources.

\begin{algorithm}
\begin{algorithmic}[1]
\Function{Run}{$\mnode_\mstart$, $\mnode_\mgoal$}
    \State path $\, \gets \varnothing$
    \State \Call{Caster}{$\mnode_\mstart$, $\mnode_\mgoal$}
    \While {open-list $\ne \varnothing$ \An no path found}
        \State Poll search $(\mnode_S, \mnode_T)$ from open-list.
        \If {$\mnode_T$ is $\mntm$} \Comment{Continue interrupted trace}
            \State \Call{Tracer}{$\mside_T$, $\mxt$, $\{ \mnode_S \}$, $\mnodes_{TT}$}, \Comment {$\mnodes_{TT}$ is set of descendants of $\mnode_T$}
        \Else   
            \State \Call{Caster}{$\mnode_S$, $\mnode_T$}
        \EndIf
    \EndWhile
    \State \Return path
\EndFunction
\end{algorithmic}
\caption{Main method for \rtwo{}.}
\label{suppalgr2:run}
\end{algorithm}

\begin{algorithm}
\begin{algorithmic}[1]
\Function{Caster}{$\mnode_S=(\mntype_S, \mside_S, \mxs, \cdots )$, $\mnode_T=(\mntype_T, \mside_T, \mxt, \cdots)$} 
    \If {$\mray = \{\mrtype, \mxs, \mxt, \mxcol \}$ has line-of-sight}
        \State \Call{CasterReached}{$\mray$, $\mnode_S$, $\mnode_T$}
    \Else
        \State \Call{CasterCollided}{$\mray$, $\mnode_S$, $\mnode_T$}
    \EndIf
\EndFunction
\end{algorithmic}
\caption{Caster: handles casts.}
\label{suppalgr2:caster}
\end{algorithm}

\begin{algorithm}
\begin{algorithmic}[1]
\Function{CasterReached}{$\mray$, $\mnode_S=(\mntype_S, \mside_S, \mxs, \cdots )$, $\mnode_T=(\mntype_T, \mside_T, \mxt, \cdots)$} 
    \label{suppalgr2:castertaut}
    \State For each $ \mnode_{TT} \in \mnodes_{TT}$, remove $\mnode_{TT}$ from $\mnodes_{TT}$ if $(\mnode_S, \mnode_T, \mnode_{TT})$ is not taut. 
    \If{$\mnode_T$ is $\mnph$}
        \State \Return \Comment{Reached a phantom point.}
    \ElsIf{($\mnode_S$ is $\mneu$ or $\mney$) \An ($\mside_S \ne \mside_T$ \Or $\mnode_T$ is $\mnty$)}
        \State \Return \Comment{Reject expensive searches from unprunable tgts.}
    \ElsIf {$\mnode_T$ is $\mnty$ \An $\mnode_S$ is $\mnsy$}
        \State \Return shortest path    \Comment{Shortest path}
    \ElsIf {$\mnode_T$ is $\mnty$ \An $\mnode_S$ is $\mnsu$}    
        \State Set $\mnode_S$ to $\mnty$
        \State Queue $(\mnode_{SS}, \mnode_S)$
        \State \Return  \Comment{Move down tree}
    \EndIf
    \State \Comment{Check cost-to-come and overlaps}
    \State $n \gets $ $\mnode_T$'s type
    \If {$\mnode_S$ is $\mnsy$}
        \If{$\mnode_T$ is not the cheapest cost-to-come node at $\mxt$}
            \State Set $\mnode_T$ to $\mney$.
        \Else   \Comment{$\mnode_T$ is cheapest node at $\mxt$}
            \State Set $\mnode_T$ to $\mnsy$. 
            \ForEach{$\mnode$ at $\mxt$ except $\mnode_T$}
                \If {$\mnode$ is $\mnsy$ and costlier than $\mnode_T$}
                    \State Convert $\mnode$ and all $\mnsy$ nodes in target direction of $\mnode$ (descendants) to $\mney$.
                    \State Discard all searches from all descendants.
                    \State Discard any $\mney$ descendant node if it has different side from $\mnode$.
                    \State Queue all pairs of descendant nodes $(\mnode_{p}, \mnode_{q})$ where $\mnode_{p}$ is $\mney$ and $\mnode_{q}$ is not $\mney$.
                \ElsIf{$\mnode$ is $\mnsu$} \Comment{Every node in source direction of $\mnode$ points to one source node.}
                    \State Convert all $\mnsu$ nodes in source direction of $\mnode$ (ancestors) to $\mntu$.
                    \State Remove any searches from all descendants of $\mnode$.
                    \State Queue the ancestor pair $(\mnode_{S,p}, \mnode_{S,q})$ where $\mnode_{S,p}$ is $\mnsy$ and $\mnode_{S,q}$ is not $\mnsy$.
                \EndIf
            \EndFor
        \EndIf
        \State \textsc{MergeRay}($\neg\,\mside_T$, $\mnode_S$, $\mray$)
        \State \textsc{MergeRay}($\mside_T$, $\mnode_T$, $\mray$)
    \EndIf

    \If {$n$ is $\mntm$}
        \If {cast points into obs. at $\mxt$} \Comment{$\mnode_T$ is no longer a valid turning point}
            \State \Call{Tracer} {$\mside_T$, $\mxt$, $\{\mnode_S\}$, $\mnodes_{TT}$}
        \Else
            \State $\mx \gets $ subsequent corner on $\mside_T$ edge of $\mxt$
            \State \Call{Tracer} {$\mside_T$, $\mx$, $\{\mnode_T\}$, $\mnodes_{TT}$}
        \EndIf
    \ElsIf {$n$ is $\mntu$}
        \State Queue $(\mnode_T, \mnode_{TT})$
    \EndIf
\EndFunction
\end{algorithmic}
\caption{CasterReached: handles casts that have line-of-sight.}
\label{suppalgr2:caster-reacned}
\end{algorithm}

\begin{algorithm}
\begin{algorithmic}[1]
\Function{CasterCollided}{$\mray$, $\mnode_S=(\mntype_S, \mside_S, \mxs, \cdots )$, $\mnode_T=(\mntype_T, \mside_T, \mxt, \cdots)$}

    \If {$\mnode_S$ is $\mnsy$ or $\mnsu$}
        \State $\mnode_{S,\mmnr} \gets \mnode_S$  
        \State \textsc{MergeRay}($\mside_S$, $\mnode_{S,\mmnr}$, $\mray$)
        \State \Call{Tracer}{$\neg\,\mside_S$, $\mxcol$, $\{ \mnode_{S,\mmnr} \}$, $\{\mnode_T\}$} \Comment{$\neg\,\mside_S$ trace}
        
        \If {$\mnode_T$ is goal \An $\mnode_S$ is not start} 
            \State $\mnode_{S,\mthird} \gets \mnode_S$ \Comment{Third-trace}
            \State \textsc{MergeRay}($\mside_S$, $\mnode_{S,\mthird}$, $\mray$)
            \State Place ad-hoc point $\mnode_{ad,a}$ with side $\mside_S$ at $\mx_p$
            \State $\mx \gets $ next corner along $\mside_S$ edge of $\mx_p$
            \State \Call{Tracer}{$\mside_S$, $\mx$, $\{ \mnode_{S,\mthird} \}$, $\{\mnode_{ad,a}\}$} \
        \EndIf
    \EndIf
    \State $\mnode_{S,\mmjr} \gets \mnode_S$ 
    \State \textsc{MergeRay}($\neg\,\mside_S$, $\mnode_{S,\mmjr}$, $\mray$)
    \State \Call{Tracer}{$\mside_S$, $\mxcol$, $\{ \mnode_{S,\mmjr} \}$, $\{\mnode_T\}$} \Comment{$\mside_S$ trace}
\EndFunction
\end{algorithmic}
\caption{CasterCollided: handles a collided cast.}
\label{suppalgr2:caster-collided}
\end{algorithm}

\begin{algorithm}
\begin{algorithmic}[1]
\Function{Tracer}{$\msidetrace$, $\mx$, $\mnodes_S$, $\mnodes_T$}
    \State $\mtrace \gets (\msidetrace, \mx, \mnodes_S, \mnodes_T, \cdots)$
    \While{$\mx$ not out-of-map}
        \State \Call{Process}{$\mtrace$, $\mnodes_S$}, return if $\mnodes_S = \varnothing$
        \State \Call{Process}{$\mtrace$, $\mnodes_T$}, return if $\mnodes_T = \varnothing$
        \State \Call{PlaceNode}{$\mtrace$}, return if $\mnodes_T = \varnothing$
        \State Go to next corner and update $\mtrace$.
    \EndWhile
\EndFunction
\Function{Process}{$\mtrace$, $\mnodes$}
    \ForEach{$\mnode_\mtdir \in \mnodes$}
        \If {$\mx_r = \mx$} \Comment {Trace refound node}
            \State Remove $\mnode_\mtdir$ from $\mnodes$ 
            \State \Continue
        \ElsIf {$\mnodes = \mnodes_S$}
            \State Do angular-sector rule, generating recursive traces if required. Replace $\mnode_\mtdir$ with its source if pruned by sector rule, else remove $\mnode_\mtdir$ from $\mnodes$ and return.
            \State Do occupied-sector rule, return if recursive trace called.
        \ElsIf {$\mnodes = \mnodes_T$}
            \State Place ad-hoc points $\mnode_{\mathrm{ad}, b}$ or $\mnode_{\mathrm{ad}, c}$ if necessary.
        \EndIf
        \If {$\mnode_\mtdir$ is prunable}
            \State Remove $\mnode_\mtdir$ from $\mnodes$.
            \State Push all nodes $\mnode_{\mtdir\mtdir}$ of $\mnode_\mtdir$ to back of $\mnodes$.
            \Comment{$\mnode_{\mtdir\mtdir}$ is the source node of $\mnode_\mtdir$ if $\mtdir=S$, or target node if $\mtdir=T$.}
        \EndIf
    \EndFor
\EndFunction
\end{algorithmic}
\caption{Tracer: handles a trace.}
\label{suppalgr2:tracer}
\end{algorithm}

\begin{algorithm}
\begin{algorithmic}[1]
\Function{PlaceNode}{$\mtrace$}
    \If {new turning pt. $\mnode_{S}'$ placed at $\mx$} \Comment{a new turning point replaces the original source point in $\mnodes_S$}
        \If {$\mnode_S$ is $\mney$ or $\mneu$}
            \State Set $\mnode_{S}'$ to $\mneu$.
            \If {$\mnode_S$ is castable to at least one $\mnode_T \in \mnodes_T$}
                \State Search from $\mnode_S$ in source direction and queue source pair $(\mnode_{S,p}, \mnode_{S,q})$ where $\mnode_{S,p}$ is $\mney$ and $\mnode_{S,q}$ is $\mneu$.
                \State $\mnodes_T \gets \varnothing$ 
                \State \Return
            \EndIf
        \Else \Comment{$\mnode_S$ is $\mnsu$ or $\mnsy$}
            \State Set $\mnode_{S}'$ to $\mnsu$.
            \If {multiple nodes exist at $\mx$} \Comment{Overlap rule}
                \ForEach{$\mnode$ at $\mx$}
                    \If {$\mnode$ is $\mnsu$}
                    \State Convert all $\mnsu$ ancestors of $\mnode$ to $\mntu$.
                    \State Remove any searches from descendants of $\mnode$.
                    \State Queue the ancestor pair $(\mnode_{S,p}, \mnode_{S,q})$ where $\mnode_{S,p}$ is $\mnsy$ and $\mnode_{S,q}$ is not $\mnsy$.
                    \EndIf
                \EndFor
                \State $\mnodes_T \gets \varnothing$
                \State \Return
            \EndIf
            
            \ForEach{progressed and castable $\mnode_T \in \mnodes_T$}
                \State Queue $(\mnode_{S}', \mnode_T )$.
                \State Remove $\mnode_T$ from $\mnodes_T$.
            \EndFor 
        \EndIf
    \Else
        \State Try to place phantom point node at $\mx$. \Comment{A new phantom point target node is created in $\mnodes_T$ for all target nodes where the angular progression reverses at $\mx$. The target nodes are removed from $\mnodes_T$ and become the new target nodes of the phantom point.}
    \EndIf
    \If {$>m$ nodes placed \An trace prog. w.r.t. $\mnode_S$ and all $\mnode_T \in \mnodes_T$}
        \State Place $\mntm$ target node $\mnode_T'$ at $\mx$
        \State Queue $(\mnode_S, \mnode_T')$
        \State $\mnodes_T \gets \varnothing$ \Comment{Queue interrupted trace}
    \EndIf
\EndFunction
\end{algorithmic}
\caption{PlaceNode: places turning points and phantom points, and queues a trace.}
\label{suppalgr2:placenode}
\end{algorithm}

\newcommand{\mleaf}{\mathrm{leaf}}
\renewcommand{\mlink}{l}
\renewcommand{\mlinks}{\mathbb{L}}
\newcommand{\mleu}{\textc{Eu}}
\newcommand{\mley}{\textc{Ey}}
\newcommand{\mloc}{\textc{Oc}}
\newcommand{\mltm}{\textc{Tm}}
\renewcommand{\mltype}{y_\mlink}
\newcommand{\mlun}{\textc{Un}}
\newcommand{\mlvu}{\textc{Vu}}
\newcommand{\mlvy}{\textc{Vy}}
\newcommand{\mpoint}{p}
\renewcommand{\mqcast}{\textc{Cast}}
\newcommand{\mqnode}{q}
\renewcommand{\mqtrace}{\textc{Trace}}
\renewcommand{\mqtype}{y_\mqnode}
\renewcommand{\mray}{r}
\newcommand{\mroot}{\mathrm{root}}
\newcommand{\msbest}{\mbest_S}
\newcommand{\mtbest}{\mbest_T}
\renewcommand{\mtdir}{\kappa}
\renewcommand{\mtrace}{\tau}
\renewcommand{\mvprog}{\mv_\mathrm{prg}}
\renewcommand{\mvprogs}{\mv_{\mathrm{prg},S}}
\renewcommand{\mvprogt}{\mv_{\mathrm{prg},T}}
\newcommand{\mvdiff}{\mv_\mathrm{dif}}
\newcommand{\mvdiffs}{\mv_{\mathrm{dif},S}}
\newcommand{\mvdifft}{\mv_{\mathrm{dif},T}}
\renewcommand{\mtrue}{\textc{True}}
\renewcommand{\mfalse}{\textc{False}}
\renewcommand{\mfront}{\textc{Front}}
\renewcommand{\mback}{\textc{Back}}
\newcommand{\mcalc}{\textc{Calc}}
\newcommand{\msort}{\textc{Sort}}
\newcommand{\mpointtrace}{\mpoint_\mtrace}
\newcommand{\mcast}{\mathrm{cast}}
\newcommand{\mlinkcast}{\mlink_\mcast}
\renewcommand{\mnext}{\mathrm{nxt}}
\renewcommand{\mprev}{\mathrm{prv}}
\renewcommand{\mmin}{\mathrm{min}}

\chapter{Implementation for \rtwop{}}
\label{chap:suppr2p}
\renewcommand{\thealgorithm}{\thesubsection}
\renewcommand{\algorithmiccomment}[1]{{\hfill $\triangleright$\ \underline{#1}}}

This supplementary material attempts to describe \rtwop{} in more detail, with a focus on managing the data, particularly pointers. A more brief version of the pseudocode can be found in the underlined comments annotating the pseudocode.

The bracket operator $[\cdot]$ will be used extensively in the pseudocode. $a[b]$ means accessing the property $b$ of the object $a$.

A visualization of \rtwop{} is available at \cite{bib:my2dvis}, and the code, as of writing, will be available at \cite{bib:my2dcode}. This material will be superseded by the version appended with the published paper of \rtwop{}.

\section{Enums}
This section describes enums that are used by \rtwop{}. More details on the underlying concepts can be found in Appendix \ref{chap:suppterm}.

\subsection{Side ($\mside$)}
A \textit{side} is represented by $\mside \in \{L,R\}$. $L=-1$ represents the left side and $R=1$ represents the right side. 
While tracing a contour, the obstacle will be on the right side of an $L$-trace and on the left side of an $R$-trace.
An $L$-trace will place $L$-sided points, and an $R$-trace will place $R$-sided points.

\subsection{Tree-Direction ($\mtdir$)}
A \textit{tree-direction} $\mtdir \in \{S, T\}$ is used to represent the direction of an object with respect to another along a path. 
$S=-1$ represents the \textit{source direction} and $T=1$ represents the \textit{target direction}.
For example, if a point at $\mx_S$ leads to the start point along a path from $\mx$, the point at $\mx_S$ lies in the source direction from the point at $\mx$; if another point at $\mx_T$ leads to the goal point, the point at $\mx_T$ lies in the target direction from the point at $\mx$.

\subsection{Link Type ($\mltype$)}
The \textit{link type}, $\mltype \in \{\mlvu, \mlvy, \mleu, \mley, \mltm, \mlun, \mloc\}$,  determines the actions taken during a cast or a trace.
The link types are described in Table \ref{suppr2p:tab:ltype}.

\begin{table}[!ht]
\centering
\caption{Description of link types.}
\setlength{\tabcolsep}{3pt}
\renewcommand{\arraystretch}{1.5}
\begin{tabular}{ c p{13cm} }
\hline
\textbf{Type}  & \textbf{Description} \\
\hline
$\mlvu$ & Let $\mlink_\mroot$ be the root link of the search tree (start link or goal link), which the link belongs to. The cumulative visibility to $\mlink_\mroot$ from a $\mlvu$ link is not verified.\\
\hline
$\mlvy$ & The link has cumulative visibility to $\mlink_\mroot$. \\
\hline
$\mleu$ & A temporary $S$-tree link that is placed during a trace that has an ancestor $\mley$ link in the source direction. Will be converted to a $T$-tree $\mlvu$ link when the trace can be cast. \\
\hline 
$\mley$ & A $\mlvy$ link that forms a costlier path to its anchored point. \\
\hline
$\mltm$ & A temporary link placed during trace interrupts or recursive traces. Indicates an incomplete trace. \\
\hline
$\mloc$ & A $T$-tree link that is placed after a trace enters the occupied sector of its root point. \\
\hline
$\mlun$ & A $T$-tree link that cannot be reached. A query that reaches the anchored point of the link will be discarded. \\
\hline
\end{tabular}
\label{suppr2p:tab:ltype}
\end{table}

\rtwop{} indirectly constrains the number of links a link can point to because of the way  the link types are handled. 
All $S$-tree links will point to one source link, and $\mley$ and $\mlvy$ links will point to only one root link. $\mloc$ and $\mlun$ links will point to only one target link.

\subsection{Query Type ($\mqtype$)}
The query type, $\mqtype \in \{\mqcast, \mqtrace\}$, determines the query type of a queued link after is polled from the open-list.

\section{Data Structures}
This section describes objects and their suggested properties.

\subsection{Point ($\mpoint$)}
The \textit{Point} object $\mpoint$ is used to encapsulate a physical point or corner, and owns pointers to links. Its properties are described in Table \ref{suppr2p:tab:point}.

\begin{table}[!ht]
\centering
\caption{Suggested properties for a \textit{Point} object ($\mpoint$).}
\setlength{\tabcolsep}{3pt}
\renewcommand{\arraystretch}{1.5}
\begin{tabular}{ c c p{11cm} }
\hline
\textbf{Symbol} & \textbf{Name} & \textbf{Description} \\
\hline
$\mx$ & \verb|coord| & Coordinates of the point. \\
\hline
$\mside$ & \verb|side| & Side of the corner at the point, where $\mside \in \{L, R\}$ \\
\hline
$\mv$ & \verb|diff| & Gives the directional vector bisecting the corner at the point. \\
\hline
$n$ & \verb|convex| & Gives the convexity of the corner at the point. \\
\hline
$\tilde{\mlinks}_S$ & \verb|slinks| & An ordered array of $S$-tree links anchored at the point. \\
\hline
$\tilde{\mlinks}_T$ & \verb|tlinks| & An ordered array of $T$-tree links anchored at the point. \\
\hline
$\msbest$ & \verb|sbest| & Stores the minimum cost-to-come known so far and the corresponding directional vector that is used to reach the point. \\
\hline
$\mtbest$ & \verb|tbest| & Stores the minimum cost-to-go known so far and the corresponding directional vector that is used to reach the point. \\
\hline
\end{tabular}
\label{suppr2p:tab:point}
\end{table}

\clearpage
\subsection{Best ($\mbest$)}
The \textit{Best} object $\mbest$ stores the minimum cost-to-go or cost-to-come so far to reach a corner, and the directional vector of the link responsible for the minimum cost.
Its properties are described in Table \ref{suppr2p:tab:best}.

\begin{table}[!ht]
\centering
\caption{Suggested properties for a \textit{Best} object ($\mbest$).}
\setlength{\tabcolsep}{3pt}
\renewcommand{\arraystretch}{1.5}
\begin{tabular}{ c c p{11cm} }
\hline
\textbf{Symbol} & \textbf{Name} & \textbf{Description} \\
\hline
$c_\mathrm{min}$ & \verb|cost| & The minimum cost-to-come or cost-to-go to reach the current point described by the best object. 
As the point has only one side, the point at the other side and at the same coordinates has to be considered, if the latter point exists. 
The minimum cost refers to the minimum cost to reach both points. \\
\hline
$\mv_\mathrm{best}$ & \verb|diff| & The directional vector of the link responsible for reaching the point with the minimum cost. 
$\mv_\mathrm{best}$ does not consider links at the other point, only the links at the current point. \\
\hline
\end{tabular}
\label{suppr2p:tab:best}
\end{table}

\clearpage
\subsection{Link ($\mlink$)}
The \textit{Link} object $\mlink$ is the fundamental search unit for \rtwop{}.
Its properties are described in Table \ref{suppr2p:tab:link}.

\begin{table}[!ht]
\centering
\caption{Suggested properties for a \textit{Link} object ($\mlink$).}
\setlength{\tabcolsep}{3pt}
\renewcommand{\arraystretch}{1.5}
\begin{tabular}{ c c p{11cm} }
\hline
\textbf{Symbol} & \textbf{Name}  & \textbf{Description} \\
\hline
$\mpoint$ & \verb|point| & The point anchored by the link. \\
\hline
$\mltype$ & \verb|type| & The type of the link, where $\mltype \in \{ \mlvu, \mlvy, \mleu, \mley, \mltm, \allowbreak \mlun, \mloc\}$. \\
\hline
$\mtdir$ & \verb|tdir| & The tree which the link belongs to, where $\mtdir \in \{S, T\}$. \\
\hline
$c$ & \verb|cost| & The cost of the link. Cost-to-come if $\mtdir = S$, or cost-to-go if $\mtdir = T$. \\
\hline
$\mray_L$ & \verb|left_ray| & The left sector-ray from the link's source point, if any. \\
\hline
$\mray_R$ & \verb|right_ray| & The right sector-ray from the link's source point, if any. \\
\hline
$\mvprog$ & \verb|prog_diff| & The directional vector of the progression ray from the link's root point, if any. \\
\hline
$\mlinks_S$ & \verb|src_links| & An array containing the source links of the link. \\
\hline
$\mlinks_T$ & \verb|tgt_links| & An array containing the target links of the link. \\
\hline
$\mqnode$ & \verb|qnode| & The queue node pointing to this link, if the link is queued. \\
\hline
-- & \verb|is_prog| & A boolean indicating if the link is progressed at a traced corner during a trace. \\
\hline
\end{tabular}
\label{suppr2p:tab:link}
\end{table}

A \textit{root link} of a link is the $\mtdir$-link of the link, 
and a \textit{leaf link} of a link is a $(-\mtdir)$-link of a link. 
For example, if the link $\mlink$ is an $S$-tree link, a root link is a source link of $\mlink$, and a leaf link is a target link of $\mlink$.

A \textit{root point} of a link is the $\mtdir$-point of the link, and the \textit{leaf point} or \textit{anchored point} of a link is the $(-\mtdir)$-point of the link.

\clearpage
\subsection{Ray ($\mray$)}
The \textit{Ray} object $\mray$ encapsulates a sector-ray for angular sectors.
Its properties are described in Table \ref{suppr2p:tab:ray}.

\begin{table}[!ht]
\centering
\caption{Suggested properties for a \textit{Ray} object ($\mray$).}
\setlength{\tabcolsep}{3pt}
\renewcommand{\arraystretch}{1.5}
\begin{tabular}{ c c p{11cm} }
\hline
\textbf{Symbol} & \textbf{Name} & \textbf{Description} \\
\hline
$\mv$ & \verb|diff| & The directional vector of the ray. \\
\hline
-- & \verb|closed| & A boolean indicating if the ray cannot be crossed. \\
\hline
\end{tabular}
\label{suppr2p:tab:ray}
\end{table}

\subsection{Trace ($\mtrace$)}
The \textit{Trace} object $\mtrace$ encapsulates a trace.
Its properties are described in Table \ref{suppr2p:tab:trace}.

\begin{table}[!ht]
\centering
\caption{Suggested properties for a \textit{Trace} object ($\mtrace$).}
\setlength{\tabcolsep}{3pt}
\renewcommand{\arraystretch}{1.5}
\begin{tabular}{ c c p{11cm} }
\hline
\textbf{Symbol} & \textbf{Name} & \textbf{Description} \\
\hline
$\mpoint$ & \verb|point| & The traced point. \\
\hline
$m$ & \verb|num_crns| & The number of corners traced. \\
\hline
-- & \verb|refound_src| & A boolean indicating if the trace has exited because it has traced to the root point of the source link. \\
\hline
-- & \verb|has_overlap| & A boolean indicating if the placement rule has encountered overlapping links. \\
\hline
\end{tabular}
\label{suppr2p:tab:trace}
\end{table}

\subsection{Queue Node ($\mqnode$)}
The \textit {Queue Node} object $\mqnode$ encapsulates a queued query by storing the type of the query, total cost, and the link to expand.
Its properties are described in Table \ref{suppr2p:tab:qnode}.

\begin{table}[!ht]
\centering
\caption{Suggested properties for a \textit{Queue node} object ($\mqnode$).}
\setlength{\tabcolsep}{3pt}
\renewcommand{\arraystretch}{1.5}
\begin{tabular}{ c c p{11cm} }
\hline
\textbf{Symbol} & \textbf{Name} & \textbf{Description} \\
\hline
$\mqtype$ & \verb|type| & The query type, where $\mqtype \in \{\mqcast, \mqtrace\}$ \\
\hline
$\mlink$ & \verb|link| & The queued link. \\
\hline
$c_f$ & \verb|f_cost| & The sum of cost-to-go and cost-to-come of the queued link. \\
\hline
\end{tabular}
\label{suppr2p:tab:qnode}
\end{table}

%%%%%%%%%%%%%%%%%%%%%%%%%%%%%%%%%%%%% UTILITY %%%%%%%%%%%%%%%%%%%%%%%%%%%%%%%%%%%%%%%%%%%
\section{Utility Functions}
This section contains utility functions that are used extensively by \rtwop{}.

% \subsection{PushOverlap}
% The function pushes a point $\mpoint$ into the overlap-buffer for examination by the overlap rule at the end of an iteration.
% \input {chap_suppr2p/alg_suppr2p_pushoverlap}

\subsection{Trace}
The function traces a contour from the coordinates at $\mx$ on the $\mside$-side, stopping at the first corner it encounters or at the map boundary. A \textit{Point} object at the stopped position is returned.

\begin{algorithm}[!ht]
\begin{algorithmic}[1]
\caption{\textproc{Trace}: Traces an obstacle's contour.}
\label{suppr2p:alg:trace}
\Function{Trace}{$\mx, \mside$}
    \State Do a $\mside$-sided trace from $\mx$ to map boundary or corner at $\mx_\mnext$.
    \State \Return $\Call{GetPoint}{(\mx_\mnext, \mside, \cdots), \mside}$
\EndFunction
\end{algorithmic}
\end{algorithm}

\subsection{Cast}
The function attempts a cast from a point at the coordinates at $\mx_S$ to a point at $\mx_T$.
If a collision occurs, a point at the first corner (or point at the map boundary) on each side of the collision point is returned.
If the cast reaches $\mx_T$, nothing is returned.

\begin{algorithm}[!ht]
\begin{algorithmic}[1]
\caption{\textproc{Cast}: Performs a line-of-sight check.}
\label{suppr2p:alg:cast}
\Function{Cast}{$\mxs, \mxt$}
    \If {cast from $\mxs$ to $\mxt$ collided at $\mxcol$}
        \State $\mpoint_L \gets \Call{Trace}{\mxcol, L}$
        \State $\mpoint_R \gets \Call{Trace}{\mxcol, R}$
        \State \Return $(\mpoint_L, \mpoint_R)$
    \ElsIf {cast from $\mxs$ can reach $\mxt$}
        \State \Return $\varnothing$
    \EndIf
\EndFunction
\end{algorithmic}
\end{algorithm}

\clearpage 
\subsection{Project}
The function attempts a projection from the point $\mx$ in the direction $\mv_\mathrm{ray}$.
The projection always collides in \rtwop{}, and a point at the first corner (or point at the map boundary) on each side of the collision point is returned.

\begin{algorithm}[!ht]
\begin{algorithmic}[1]
\caption{\textproc{Project}: Extrapolates a line-of-sight check.}
\label{suppr2p:alg:project}
\Function{Project}{$\mx, \mv_\mathrm{ray}$}
    \State $\mxcol \gets $ collision point for a line-of-sight check from $\mx$ in the direction $\mv_\mathrm{ray}$.
    \State $\mpoint_L \gets \Call{Trace}{\mxcol, L}$
    \State $\mpoint_R \gets \Call{Trace}{\mxcol, R}$
    \State \Return $(\mpoint_L, \mpoint_R)$
\EndFunction
\end{algorithmic}
\end{algorithm}

\subsection{Queue}
The function enqueues a link to the open-list with a query type and the path cost at the link.

\begin{algorithm}[!ht]
\begin{algorithmic}[1]
\caption{\textproc{Queue}: Queues a link into the open-list.}
\label{suppr2p:alg:queue}
\Function{Queue}{$\mqtype, \mlink, c_f$}
    \State $\mqnode \gets (\mqtype, \mlink, c_f)$
    \State $\mlink[\mqnode] \gets \mqnode$
    \State Sort $\mqnode$ into open-list by $c_f$.
\EndFunction
\end{algorithmic}
\end{algorithm}

\subsection{Unqueue}
The function unqueues a link from the open-list.

\begin{algorithm}[!ht]
\begin{algorithmic}[1]
\caption{\textproc{Unqueue}: Removes a link from the open list.}
\label{suppr2p:alg:unqueue}
\Function{Unqueue}{$\mlink$}
    \If {link $\mlink$ is queued s.t. $\mlink[\mqnode] \ne \varnothing$}
        \State Remove $\mlink[\mqnode]$ from open-list.
        \State $\mlink[\mqnode] \gets \varnothing$
    \EndIf
\EndFunction
\end{algorithmic}
\end{algorithm}

\clearpage
\subsection{Poll}
The function removes the cheapest link from the open-list and returns the link and the query type.

\begin{algorithm}[!ht]
\begin{algorithmic}[1]
\caption{\textproc{Poll}: Removes and returns the cheapest link from the open-list.}
\label{suppr2p:alg:poll}
\Function{Poll}{\null}
    \State $\mqnode_\mathrm{min} \gets $the queue node with smallest $c_f$ in open-list.
    \State $\Call{Unqueue}{\mqnode_\mathrm{min}[\mlink]}$ \Comment{Unqueue the cheapest link}
    \State \Return $(\mqnode_\mathrm{min}[\mqtype], \mqnode_\mathrm{min}[\mlink])$
\EndFunction
\end{algorithmic}
\end{algorithm}

\subsection{Disconnect}
The function disconnects two links $\mlink$ and $\mlink_\mtdir$.
$\mlink_\mtdir$ must be a $\mtdir$ link of $\mlink$, and $\mlink$ must be a $(-\mtdir)$ link of $\mlink_\mtdir$.

\begin{algorithm}[!ht]
\begin{algorithmic}[1]
\caption{\textproc{Disconnect}: Disconnects two links.}
\label{suppr2p:alg:disconnect}
\Function{Disconnect}{$\mtdir, \mlink, \mlink_\mtdir$}
    \State Remove pointer to $\mlink_\mtdir$ from $\mlink[\mlinks_\mtdir]$.
    \State Remove pointer to $\mlink$ from $\mlink_\mtdir[\mlinks_{-\mtdir}]$.
\EndFunction
\end{algorithmic}
\end{algorithm}

In Alg. \ref{suppr2p:alg:disconnect}, $\mlink[\mlinks_\mtdir]$ refers to the array of $\mtdir$ link pointers in $\mlink$, and $\mlink_\mtdir[\mlinks_{-\mtdir}]$ refers to the array of $(-\mtdir)$ link pointers in $\mlink_\mtdir$. The arrays are either $\mlinks_S$ or $\mlinks_T$ in Table \ref{suppr2p:tab:link}.

\subsection{Connect}
The function connects two links $\mlink$ and $\mlink_\mtdir$, such that $\mlink_\mtdir$ becomes a $\mtdir$ link of $\mlink$, and $\mlink$ becomes a $(-\mtdir)$ link of $\mlink_\mtdir$.

\begin{algorithm}[!ht]
\begin{algorithmic}[1]
\caption{\textproc{Connect}: Connects two links.}
\label{suppr2p:alg:connect}
\Function{Connect}{$\mtdir, \mlink, \mlink_\mtdir$}
    \State Add pointer to $\mlink_\mtdir$ to $\mlink[\mlinks_\mtdir]$.
    \State Add pointer to $\mlink$ to $\mlink_\mtdir[\mlinks_{-\mtdir}]$.
\EndFunction
\end{algorithmic}
\end{algorithm}

In Alg. \ref{suppr2p:alg:connect}, $\mlink[\mlinks_\mtdir]$ refers to the array of $\mtdir$ link pointers in $\mlink$, and $\mlink_\mtdir[\mlinks_{-\mtdir}]$ refers to the array of $(-\mtdir)$ link pointers in $\mlink_\mtdir$. The arrays are either $\mlinks_S$ or $\mlinks_T$ in Table \ref{suppr2p:tab:link}.

\clearpage
\subsection{Isolate}
The function first checks that the link $\mlink$ is connected to one $\mtdir$ link that is $\mlink_\mtdir$. 
If $\mlink_\mtdir = \varnothing$, $\mlink$ cannot be connected to any $\mtdir$ link.
If the condition is satisfied, $\mlink$ is returned.
If the condition is not satisfied, such that $\mlink$ is connected to other $\mtdir$ links, $\mlink$ is duplicated to a new link that fulfills the condition.

\begin{algorithm}[!ht]
\begin{algorithmic}[1]
\caption{\textproc{Isolate}: Isolates a link connection.}
\label{suppr2p:alg:isolate}
\Function{Isolate}{$\mtdir, \mlink, \mlink_\mtdir$}
    \If {$\mlink_\mtdir = \varnothing$} \Comment{$\mlink$ cannot be connected to $\mtdir$ links.}
        \If {$\mlink_\mtdir$ has no $\mtdir$ links}
            \State \Return $\mlink$
        \Else 
            \State new $\mlinks_S \gets \{\} $ \textbf{if} $\mtdir=S$ \textbf{else} $\mlink[\mlinks_S]$
            \State new $\mlinks_T \gets \{\} $ \textbf{if} $\mtdir=T$ \textbf{else} $\mlink[\mlinks_T]$
        \EndIf
    \Else \Comment{$\mlink$ must only be connected to $\mtdir$ link $\mlink_\mtdir$}
        \If {$\mlink$ is only connected to one $\mtdir$ link that is $\mlink_\mtdir$]}
            \State \Return $\mlink$
        \Else
            \State new $\mlinks_S \gets \{\mlink_\mtdir\} $ \textbf{if} $\mtdir=S$ \textbf{else} $\mlink[\mlinks_S]$
            \State new $\mlinks_T \gets \{\mlink_\mtdir\} $ \textbf{if} $\mtdir=T$ \textbf{else} $\mlink[\mlinks_T]$
        \EndIf
    \EndIf
    \State new $\mlink \gets $ create new link.
    \State $\Call{Change}{\text{new } \mlink, \mlink[\mpoint], \mlink[\mltype], \allowbreak
        \; \mlink[\mray_L], \mlink[\mray_R], \mlink[\mvprog], \allowbreak 
        \; \text{new } \mlinks_S, \text{new } \mlinks_T, \mcalc, \msort}$
    \State \Return new $\mlink$
\EndFunction
\end{algorithmic}
\end{algorithm}

\clearpage
\subsection{Change}
The function modifies the link $\mlink$ and ensures that pointers are correctly handled.

\begin{algorithm}[!ht]
\begin{algorithmic}[1]
\caption{\textproc{Change}: Modifies a link.}
\label{suppr2p:alg:change}
\Function{Change}{$\mlink, \mpoint, \mltype, \mtdir, \mray_L, \mray_R, \mvprog, \mlinks_S, \mlinks_T, cost, sort\_to$}
    \If {$\mlink$ is anchored to a point s.t. $\mlink[\mpoint] \ne \varnothing$}
        \State old $\mtdir \gets \mlink[\mtdir]$
        \State Un-anchor $\mlink$ by removing $\mlink$ from $\mlink[\mpoint][\mlinks_{(\text{old } \mtdir)}]$.
    \EndIf
    \State Assign $\mtdir, \mltype, \mpoint, \mray_L, \mray_R, \mvprog$ to $\mlink$.
    \State Anchor $\mlink$ to $\mpoint[\mlinks_\mtdir]$ with $sort\_to$.
    \State Disconnect all source links of $\mlink$ in $\mlink[\mlinks_S]$.
    \State Connect $\mlink$ to all new source links in $\mlinks_S$.
    \State Disconnect all target links of $\mlink$ in $\mlink[\mlinks_T]$.
    \State Connect $\mlink$ to all new target links in $\mlinks_T$.
    \If {$cost = \mcalc$}
        \State $\mlink[c] \gets $ length of $\mlink$ $+$ cost of cheapest $\mtdir$ link.
    \Else
        \State $\mlink[c] \gets cost$
    \EndIf
\EndFunction
\end{algorithmic}
\end{algorithm}

$sort\_to \in \{\msort, \mfront, \mback\}$ refers to the position in which the link $\mlink$ appears in the anchored point's ordered array of pointers $\tilde{\mlinks}_S$ or $\tilde{\mlinks}_T$ (see Table \ref{suppr2p:tab:point}). 
If $\mtdir = S$, $\tilde{\mlinks}_S$ is selected; if $\mtdir=T$, $\tilde{\mlinks}_T$ is selected. 
$sort\_to = \mfront$ places the link at the front of the selected array, while $sort\_to = \mback$ places the link at the back. 
The front and back positioning are \textit{required} by the algorithm to determine which link has been processed during a trace.
$sort\_to = \msort$ is optional to implement, and sorts the link based on the link type $\mltype$ for faster lookups by the overlap rule.

\clearpage
\subsection{GetPoint}
The function retrieves an existing point at the trace point's $\mpoint_\mtrace$ corner that has the same side, creating a new point if no point exists.
The returned point is permanent, unlike the moving trace point $\mpointtrace$ that is created for every trace.

\begin{algorithm}[!ht]
\begin{algorithmic}[1]
\caption{\textproc{GetPoint}: Retrieves or create a new point from a trace point.}
\label{suppr2p:alg:getpoint}
\Function{GetPoint}{$\mpoint_i, \mside$}
    \State $\mpoint \gets $ globally accessible point that is at the same corner as $\mpoint_i$ and that has side $\mside$.
    \If {$\mpoint$ does not exist}
        \State $\mpoint \gets $ new point that is at the same corner as $\mpoint_i$ and that has side $\mside$.
    \EndIf
    \State \Return $\mpoint$
\EndFunction
\end{algorithmic}
\end{algorithm}

\subsection{Erase}
The function deletes the link $\mlink$ and removes pointers to itself from other objects.

\begin{algorithm}[!ht]
\begin{algorithmic}[1]
\caption{\textproc{Erase}: Deletes a links.}
\label{suppr2p:alg:erase}
\Function{Erase}{$\mlink$}
    \State \Call{Unqueue}{\mlink}
    \If {$\mlink$ is anchored s.t. $\mlink[\mpoint] \ne \varnothing$}
        \State old $\mtdir \gets \mlink[\mtdir]$
        \State Un-anchor $\mlink$ by removing $\mlink$ from $\mlink[\mpoint][\mlinks_{(\text{old } \mtdir)}]$.
    \EndIf
    \State Disconnect all source links of $\mlink$.
    \State Disconnect all target links of $\mlink$.
\EndFunction
\end{algorithmic}
\end{algorithm}

\clearpage

\subsection{EraseTree}
The function recursively deletes a link $\mlink$ and all connected $\mtdir$-links if the links do not have any more $(-\mtdir)$ links.

\begin{algorithm}[!ht]
\begin{algorithmic}[1]
\caption{\textproc{EraseTree}: Deletes a links.}
\label{suppr2p:alg:erasetree}
\Function{EraseTree}{$\mtdir, \mlink$}
    \If {$\mlink$ has no more $-\mtdir$ links s.t. $\mlink[\mlinks_{-\mtdir}]$ is empty}
        \State $\mlinks_\mtdir \gets $ copy of $\mlink[\mlinks_{\mtdir}]$
        \State $\Call{Erase}{\mlink}$
        \For {$\mlink_\mtdir \in \mlinks_\mtdir$}
            \State $\Call{EraseTree}{\mtdir, \mlink_\mtdir}$
        \EndFor
    \EndIf
\EndFunction
\end{algorithmic}
\end{algorithm}

\subsection{MergeRay}
Replaces the $\mside$-side ray of a link $\mlink$ if the new ray $\mray$ shrinks the angular sector at the source point of the link.

\begin{algorithm}[!ht]
\begin{algorithmic}[1]
\caption{\textproc{MergeRay}: Modifies a link.}
\label{suppr2p:alg:mergeray}
\Function{MergeRay}{$\mlink, \mside, \mray$}
    \If {link has no $\mside$-side ray s.t. $\mlink[\mray_\mside] = \varnothing$ }
        \State $\mlink[\mray_\mside] \gets \mray$
    \Else
        \State $\mv_\mathrm{old} \gets \mlink[\mray_\mside][\mv]$
        \State $\mv_\mathrm{new} \gets \mray[\mv]$
        \If {$\mside$-side ray of $\mlink$ lies to $\mside$ side of new ray s.t. $\mside(\mv_\mathrm{old} \times \mv_\mathrm{new}) > 0$}
            \State $\mlink[\mray_\mside] \gets \mray$
        \EndIf
    \EndIf
\EndFunction
\end{algorithmic}
\end{algorithm}

\clearpage
\subsection{CrossedRay}
The function determines if a ray has been crossed during a trace by checking against the contour assumption. It is used by the pruning rule and angular-sector rule.

\begin{algorithm}[!ht]
\begin{algorithmic}[1]
\caption{\textproc{CrossedRay}: Determines if a ray has been crossed.}
\label{suppr2p:alg:crossedray}
\Function{CrossedRay}{$\mside, \mlink, \mv_\mathrm{ray}$}
    \State $\mvdiff \gets \mlink's$ anchor coordinates $-$ $\mlink$'s root coordinates.
    \State $\mpoint \gets \mlink$'s anchor point.
    \State $d \gets \mv_\mathrm{ray} \times \mvdiff$
    \If {$d = 0$}
        \State $\mvdiff \gets $ directional vector that bisects the corner at $\mpoint$.
        \State $d \gets \mv_\mathrm{ray} \times \mvdiff$
    \EndIf
    \State $\mtdir \gets \mlink[\mtdir]$
    \State \Return $\mtdir \mside d > 0$
\EndFunction
\end{algorithmic}
\end{algorithm}

\section{Main Function and Initial Cast}

\subsection{Run}
The function is used to run the \rtwop{} algorithm.

\begin{algorithm}[!ht]
\begin{algorithmic}[1]
\caption{\textproc{Run}: Main function for \rtwop{}}
\label{suppr2p:alg:run}
\Function{Run}{$\mxstart, \mxgoal$}
    \State $\Call{Initial}{\mxstart, \mxgoal}$

    \While (open-list is not empty)
        \State $(\mqtype, \mlink) \gets \Call{Poll}{}$
        \If {$\mqtype = \mqcast$}
            \If {$\Call{Caster}{\mlink}$}
                \State \Break
            \EndIf
        \Else
            \State $\Call{SetupTracerFromLink}{\mlink}$
        \EndIf

        \State $\Call{OverlapRule}{}$
    \EndWhile

    \State \Return path
\EndFunction
\end{algorithmic}
\end{algorithm}

\subsection{Initial}
The function is called to initialize \rtwop{} and attempt the first cast.

\begin{algorithm}[!ht]
\begin{algorithmic}[1]
\caption{\textproc{Initial}: Initializes \rtwop{} and conducts the first cast.}
\label{suppr2p:alg:initial}
\Function{Initial}{$\mxstart, \mxgoal$}

    \Comment{No path if the start point is trapped.}
    \If {$\mxstart$ is surrounded by occupied cells}
        \State path $\gets ( )$
        \State \Return
    \EndIf

    \Comment{Return path immediately if start and goal points have line-of-sight.}
    \State result $\gets \Call{Cast}{\mxstart, \mxgoal}$
    \If {result $ = \varnothing$}
        \State path $\gets (\mxgoal, \mxstart)$
        \State \Return
    \EndIf

    \Comment {If collision has occurred, prepare traces.}
    \State $(\mpoint_L, \mpoint_R) \gets $ result
    \State $\mvdiff \gets \mxgoal - \mxstart$

    \State $\mpoint_T \gets \Call{GetPoint}{(\mxgoal, L, \cdots), L}$
    \Comment{Create start and goal points (can be $L$ or $R$).}
    \State $\mpoint_S \gets \Call{GetPoint}{(\mxstart, L, \cdots), L}$

    \State $\mlink_{TT} \gets $ a new link.
    \Comment{Create goal link.}
    \State $\Call{Change}{\mlink_T, \mpoint_T, \mlvy, T, \allowbreak
        \varnothing, \varnothing, \varnothing, \allowbreak
        \{\}, \{\}, 0, \msort
    }$
    \State $\mtrace_L \gets \varnothing$
    \Comment{Prepare links for left trace.}
    \If {$\mpoint_L$ is not at map boundary}
        \State $\mpointtrace \gets $ copy of $\mpoint_L$ with corner and side information only.
        
        \State $\Call{Change}{\text{new link}, \mpointtrace, \mltm, T, \allowbreak
            \varnothing, \varnothing, -\mvdiff, \allowbreak
            \{\}, \{\mlink_TT\}, \infty, \mback
        }$

        \State $\mlink_S \gets $ a new link.
        \State $\Call{Change}{\mlink_S, \mpointtrace, \mltm, S, \allowbreak
            (-\mvdiff, \mfalse), (\mvdiff, \mtrue), \mvdiff, \allowbreak
            \{\}, \{\}, \infty, \mback
        }$

        \State $\mlink_{SS} \gets $ a new link
        \State $\Call{Change}{\mlink_{SS}, \mpoint_S, \mlvy, S, \allowbreak
            (\mvdiff, \mtrue), (-\mvdiff, \mtrue), \varnothing, \allowbreak
            \{\}, \{\mlink_S\}, 0, \msort
        }$

        \State $\Call{Change}{\text{new link}, \mpoint_S, \mlvy, S, \allowbreak
            \varnothing, \varnothing, \varnothing, \allowbreak,
            \{\}, \{\mlink_{SS}\}, 0, \msort
        }$

        \State $\mtrace_L \gets (\mpointtrace, \cdots)$
    \EndIf
    \State $\mtrace_R \gets \varnothing$
    \Comment{Prepare links for right trace.}
    \If {$\mpoint_R$ is not at map boundary}
        \State $\mpointtrace \gets $ copy of $\mpoint_R$ with corner and side information only.
        
        \State $\Call{Change}{\text{new link}, \mpointtrace, \mltm, T, \allowbreak
            \varnothing, \varnothing, -\mvdiff, \allowbreak
            \{\}, \{\mlink_TT\}, \infty, \mback
        }$

        \State $\mlink_S \gets $ a new link.
        \State $\Call{Change}{\mlink_S, \mpointtrace, \mltm, S, \allowbreak
            (\mvdiff, \mtrue), (-\mvdiff, \mfalse), \mvdiff, \allowbreak
            \{\}, \{\}, \infty, \mback
        }$

        \State $\mlink_{SS} \gets $ a new link
        \State $\Call{Change}{\mlink_{SS}, \mpoint_S, \mlvy, S, \allowbreak
            (-\mvdiff, \mtrue), (\mvdiff, \mtrue), \varnothing, \allowbreak
            \{\}, \{\mlink_S\}, 0, \msort
        }$

        \State $\Call{Change}{\text{new link}, \mpoint_S, \mlvy, S, \allowbreak
            \varnothing, \varnothing, \varnothing, \allowbreak,
            \{\}, \{\mlink_{SS}\}, 0, \msort
        }$

        \State $\mtrace_R \gets (\mpointtrace, \cdots)$
    \EndIf

    \If {$\mtrace_L \ne \varnothing$}
    \Comment{Begin left trace if left point is not at map boundary.}
        \State $\Call{Tracer}{\mtrace_L}$
    \EndIf
    
    \If {$\mtrace_R \ne \varnothing$}
    \Comment{Begin right trace if right point is not at map boundary.}
        \State $\Call{Tracer}{\mtrace_R}$
    \EndIf
\EndFunction
\end{algorithmic}
\end{algorithm}

\clearpage
\section{Functions for Casting}
This section describes functions for cast queries.

\subsection{Caster}
The caster function handles a cast query.

\begin{algorithm}[!ht]
\begin{algorithmic}[1]
\caption{\textproc{Caster}: Handles cast queries.}
\label{suppr2p:alg:caster}
\Function{Caster}{$\mlinkcast$}
    \State $\mpoint_S \gets $ source point of $\mlinkcast$.
    \State $\mpoint_T \gets $ target point of $\mlinkcast$.
    \State result $\gets \Call{Cast}{\mpoint_S[\mx], \mpoint_T[\mx]}$
    \If {result $ = \varnothing$}
        \State $\Call{CasterReached}{\mlinkcast}$
    \Else
        \State $(\mpoint_L, \mpoint_R) \gets $ result
        \State $\Call{CasterCollided}{\mlinkcast, \mpoint_L, \mpoint_R}$
    \EndIf
\EndFunction
\end{algorithmic}
\end{algorithm}

\subsection{CasterReached}
The function handles the case when a cast is successful.

\begin{algorithm}[!ht]
\begin{algorithmic}[1]
\caption{\textproc{CasterReached}: Handles a successful cast.}
\label{suppr2p:alg:casterreached}
\Function{CasterReached}{$\mlinkcast$}
    \State $\mlink_S \gets $ any source link of $\mlinkcast$, from $\mlinkcast[\mlinks_S]$.
    \State $\mlink_T \gets $ any target link in $\mlinkcast$, from $\mlinkcast[\mlinks_T]$.

    \Comment{Found path if source link and target link are $\mlvy$ type.}
    \If {$\mlink_S$ is $\mlvy$ type \An $\mlink_T$ is $\mlvy$ type}
        \State $\Call{CasterReachedFoundPath}{\mlinkcast}$

    \Comment{Discard if target link is $\mlun$ type.}
    \ElsIf {$\mlink_T$ is $\mlun$ type}
        \State $\mlinks_T \gets $ copy of $\mlinkcast[\mlinks_T]$.
        \For {$\mlink_T \in \mlinks_T$}
            \State $\Call{Disconnect}{T, \mlinkcast, \mlink_T}$
            \State $\Call{EraseTree}{T, \mlink_T}$
        \EndFor
        \State $\Call{EraseTree}{S, \mlinkcast}$

    \Comment{Try to continue interrupted trace if target link is $\mltm$ type.}
    \ElsIf {$\mlink_T$ is $\mltm$ type}
        \State $\Call{CasterReachedTm}{\mlinkcast}$

    \Comment{$\mlinkcast$ is connected to a link with cumulative visibility.}
    \ElsIf {$\mlink_S$ is $\mlvy$ or $\mley$ type \Or $\mlink_S$ is $\mlvy$ or $\mley$ type}
        \State $\Call{CasterReachedWithCmlVis}{\mlinkcast}$

    \Comment{$\mlinkcast$ is not connected to a link with cumulative visibility.}
    \Else
        \State $\Call{CasterReachedWithoutCmlVis}{\mlinkcast}$
    \EndIf
\EndFunction
\end{algorithmic}
\end{algorithm}

\subsection{CasterReachedFoundPath}
The function is called to generate the optimal path when a cast is successful for $\mlinkcast$, and its source link and target link have cumulative visibility.

\begin{algorithm}[!ht]
\begin{algorithmic}[1]
\caption{\textproc{CasterReachedFoundPath}: Generates the optimal path.}
\label{suppr2p:alg:casterreachedfoundpath}
\Function{CasterReachedFoundPath}{$\mlinkcast$}

    \Comment{Iterate through $\mlvy$ source links.}
    \State $\mlink \gets \mlinkcast$'s source link.
    \State path $\gets \{\mlink\text{'s anchor coordinates}\}$
    \While {$\mlink$ is not anchored at start point}
        \State $\mlink \gets \mlink$'s source link.
        \State Insert $\mlink$'s anchor coordinates to back of path.
    \EndWhile

    \Comment{Iterate through $\mlvy$ target links.}
    \State $\mlink \gets \mlinkcast$'s target link.
    \State Insert $\mlink$'s anchor coordinates to front of path.
    \While {$\mlink$ is not anchored at goal point}
        \State $\mlink \gets \mlink$'s target link.
        \State Insert $\mlink$'s anchor coordinates to front of path.
    \EndWhile
\EndFunction
\end{algorithmic}
\end{algorithm}

If the function is called, the cast link $\mlinkcast$ will have exactly \textit{one} source link and exactly \textit{one} target link. 
The variable \textit{path} is globally accessible and can be read by Alg. \ref{suppr2p:alg:run}.

\subsection{CasterReachedTm}
The function is called when a cast reaches an interrupted trace at the target point of $\mlinkcast$.
If a turning point can be placed at the target point, the function attempts to queue casts for the $\mlinkcast$'s target links.
If no turning point can be placed, or if there are target links that cannot be cast, the trace from the target point continues.

\begin{algorithm}[!ht]
\begin{algorithmic}[1]
\caption{\textproc{CasterReachedTm}: Handles a cast that reached an interrupted trace.}
\label{suppr2p:alg:casterreachedtm}
\Function{CasterReachedTm}{$\mlinkcast$}
    \State $\mpoint_T \gets $ target point of $\mlinkcast$.
    \State $\mside \gets $ side of $\mpointtrace$, which is $\mpointtrace[\mside]$.
    \State $\mv_\mnext \gets $ directional vector pointing away from $\mpoint_T$ and along $\mside$-side edge of $\mpoint_T$.
    \State $\mvdiff \gets $ $\mlinkcast$'s target coordinates $-$ $\mlinkcast$'s source coordinates.

    \Comment{If a turning point can be placed at $\mpoint_T$, identify target links that can cast.}
    \If {$\mpoint_T$ is convex \An $\mside (\mvdiff, \mv_\mnext) > 0$}
    
        \State $\mlinks_T \gets \{\}$
        \For ($\mlink_T \in \{\mlinkcast[\mlinks_T]\}$)
            \State $\mvdifft \gets \mlink_T$'s anchor coordinates $-$ $\mlink_T$'s root coordinates. 
            \If {$\mside (\mv_\mnext \times \mvdifft) \ge 0$ }
                \State Push $\mlink_T$ into $\mlinks_T$.
            \EndIf
        \EndFor

        \Comment{If target links can be cast, isolate $\mlinkcast$ and handle according to cml. vis.}
        \If {$\mlinks_T$ is not empty}
            \State $\text{new } \mlinkcast \gets $ a new link.
            \State $\Call{Change}{\text{new } \mlinkcast, \mpoint_T, \mlvu, T, \allowbreak
                \mlinkcast[\mray_L], \mlinkcast[\mray_R], \varnothing, 
                \mlinkcast[\mlinks_S], \mlinks_T, \infty, \msort
            }$

            \For {$\mlink_T \in \mlinks_T$}
                \State $\Call{Disconnect}{T, \mlinkcast, \mlink_T}$
            \EndFor

            \If {source link of $\mlinkcast$ is $\mlvy$ or $\mley$ type}
                \State $\Call{CasterReachedWithCmlVis}{\text{new }\mlinkcast}$
            \Else
                \State $\Call{CasterReachedWithoutCmlVis}{\text{new }\mlinkcast}$
            \EndIf
        \EndIf

        \Comment {Discard the query if all target links can be cast.}
        \If {$\mlinkcast$ has no more target links s.t. $\mlinkcast[\mlinks_T]$ is empty}
            \State $\Call{EraseTree}{S, \mlinkcast}$
            \State \Return
        \EndIf
    \EndIf

    \Comment{Continue trace if no point can be placed, or if there are non-castable target links.}
    \State $\Call{MergeRay}{-\mside, \mlinkcast, (\mvdiff, \mtrue)}$
    \State $\Call{SetupTraceFromLink}{\mlinkcast}$
    
\EndFunction
\end{algorithmic}
\end{algorithm}

\subsection{CasterReachedWithCmlVis}
The function is called when a cast is successful and the source link or target link of the cast link has cumulative visibility.
Links will be checked by the overlap rule, and a cast is prepared for each adjacent link that has no verified cumulative visibility.

\begin{algorithm}[!ht]
\begin{algorithmic}[1]
\caption{\textproc{CasterReachedWithCmlVis}: Handles a successful cast when a source or target link has cumulative visibility.}
\label{suppr2p:alg:casterreachedwithcmlvis}
\Function{CasterReachedWithCmlVis}{$\mlinkcast$}
    \Comment{Retrieve information about the cast.}
    \State $\mlink_S \gets $ source link of $\mlinkcast$, in $\mlinkcast[\mlinks_S]$.
    \State $\mlink_T \gets $ any target link of $\mlinkcast$, in $\mlinkcast[\mlinks_T]$.
    \State $\mtdir \gets T $ \textbf{if} $\mlink_S$ is $\mley$ or $\mlvy$ \textbf{else} $S$
    \State $\mlink \gets $ $\mlink_S$ \textbf{if} $\mtdir = S$ \textbf{else} $\mlink_T$
    \State $\mlink' \gets $ $\mlink_T$ \textbf{if} $\mtdir = S$ \textbf{else} $\mlink_S$
    \State $\mpoint \gets $ anchor point of $\mlink$.
    \State $\mpoint' \gets $ anchor point of $\mlink'$.
    \State $\mside \gets \mpoint[\mside]$
    \State $\mside' \gets \mpoint'[\mside]$
    \State $\mvdiff \gets \mpoint$'s coordinates $-$ $\mpoint'$'s coordinates.

    \Comment{Retrieve best information at the anchor point (both sides) of next link to cast.}
    \State $\mbest \gets \mpoint[\mbest_T]$ \textbf{if} $\mtdir = S$ \textbf{else} $\mpoint[\mbest_S]$
    \State $\mpoint_o \gets \Call{GetPoint}{\mpoint, -\mpoint[\mside]}$.
    \State $\mbest_o \gets \mpoint_o[\mbest_T]$ \textbf{if} $\mtdir = S$ \textbf{else} $\mpoint_{o}[\mbest_S]$

    \Comment{Erase $\mlinkcast$ and connected links if condition O2 of overlap rule is satisfied.}
    \If {$\mlink'$ is $\mley$ type \An $\mside \ne \mside'$}
        \State $\Call{Disconnect}{S, \mlinkcast, \mlink_S}$
        \State $\Call{EraseTree}{T, \mlinkcast}$
        \State $\Call{EraseTree}{S, \mlink_S}$
        \State \Return
    \EndIf  

    \State $c \gets $ (cost of cheapest $-\mtdir$ link of $\mlinkcast$) $+$ (length of $\mlinkcast$). 

    \Comment{Update best cost and best ray if cheapest so far to reach next point.}
    \State $\Call{CasterReachedWithCmlVisUpdateBest}{\mtdir, \mside, \mvdiff, c, \mbest, \mbest_o}$

    \Comment{Prepare the cast link for subsequent casts.}
    \If {$\Call{CasterReachedWithCmlVisChangeLink}{\mtdir \mside, \mvdiff, c, \mlinkcast, \mpoint, \mbest, \mbest_o}$}

        \Comment{If not discarded, merge sector rays and queue subsequent links.}
        \State $\Call{CasterReachedWithCmlVisQueue}{\mtdir, \mside, \mvdiff, \mlinkcast}$
    \EndIf

\EndFunction
\end{algorithmic}
\end{algorithm}

The point $\mpoint$ anchors links, connected to $\mlinkcast$, which the algorithm will cast next. 
The point $\mpoint'$ anchors a link, connected to $\mlinkcast$, that has cumulative visibility. 
The point $\mpoint_o$ is at the same corner as $\mpoint$, but has a different side from $\mpoint$.

\subsection{CasterReachedWithCmlVisUpdateBest}
If the cast in (Alg. \ref{suppr2p:alg:casterreachedwithcmlvis}) is the cheapest to reach $\mpoint$ with cumulative visibility, the function tries to update the best ray pointing to the point at $\mpoint$ and the best cost for doing so.

\begin{algorithm}[!ht]
\begin{algorithmic}[1]
\caption{\textproc{CasterReachedWithCmlVisUpdateBest}: Updates the best ray and best cost at the next point if cast link is cheapest so far.}
\label{suppr2p:alg:casterreachedwithcmlvisupdatebest}
\Function{CasterReachedWithCmlVisUpdateBest}{$\mtdir, \mside, \mvdiff, c, \mbest, \mbest_o$}
    \Comment{Update best cost for points on both sides.}
    \State $\mvdiff' \gets \mbest[\mvdiff]$
    \If {$c \le \mbest[c_\mmin]$}
        \State $\mbest[c_\mmin] \gets c$

        \Comment{Update best ray if links are likelier to satisfy condition O6 or O7 of the overlap rule.}
        \If {$\mvdiff' = \varnothing$ \Or $\mtdir \mside (\mvdiff' \times \mvdiff) > 0$} 
        
            \State $\mbest[\mvdiff] \gets \mvdiff$ 
        \EndIf
    \EndIf
    \If {$c < \mbest_o[c_\mmin]$}
        \State $\mbest_o[c_\mmin] \gets c$
    \EndIf
\EndFunction
\end{algorithmic}
\end{algorithm}

The next point $\mpoint$ has a complimentary point $\mpoint_o$ with a different side.
While the function will update the best cost at \textit{both} points, which are $\mbest[c_\mmin]$ and $\mbest_o[c_\mmin]$,
the function will \textit{only} update the best ray for $\mpoint$, which is $\mbest[\mvdiff]$.

\subsection{CasterReachedWithCmlVisChangeLink}
The function anchors the visible cast link $\mlinkcast$ at the next point $\mpoint$ (Alg. \ref{suppr2p:alg:casterreachedwithcmlvis}) and modifies the link based on the cost of reaching $\mpoint$.
The function will discard the query if reaching $\mpoint$ is expensive and if conditions O6 and O7 of the overlap rule are satisfied.

\begin{algorithm}[!ht]
\begin{algorithmic}[1]
\caption{\textproc{CasterReachedWithCmlVisChangeLink}: Changes the cast link based on the its cost at the next point.}
\label{suppr2p:alg:casterreachedwithcmlvischangelink}
\Function{CasterReachedWithCmlVisChangeLink}{$\mtdir \mside, \mvdiff, c, \mlinkcast, \mpoint, \mbest, \allowbreak \mbest_o$}
    \If {$c > \mbest[c_\mmin]$ \Or $c > \mbest_o[c_\mmin]$ }
        \State $\mvdiff' \gets \mbest[\mvdiff]$
        
        \Comment{Discard cast link if condition O6 or O7 of the overlap rule is satisfied.}
        \If {$\mvdiff' \ne \varnothing$ \An  $\mtdir \mside (\mvdiff \times \mvdiff') > 0$}
            \State $\Call{Disconnect}{S, \mlinkcast, \mlink_S}$
            \State $\Call{EraseTree}{T, \mlinkcast}$
            \State $\Call{EraseTree}{S, \mlink_S}$
            \State \Return $\mfalse$

        \Comment{Convert cast link to $\mley$ if expensive and not in expensive sector.}
        \Else
            \State $\Call{Change}{\mlinkcast, \mpoint, \mley, -\mtdir, \allowbreak
                \mlinkcast[\mray_L], \mlinkcast[\mray_R], \varnothing, \allowbreak
                \mlinkcast[\mlinks_S], \mlinkcast[\mlinks_T], c, \msort}$
        \EndIf

    \Comment{Convert cast link to $\mlvy$ otherwise, and call overlap rule for other links.}
    \Else
        \State $\Call{OverlapRuleConvToEy}{-\mtdir, \mpoint}$
        \State $\Call{Change}{\mlinkcast, \mpoint, \mlvy, -\mtdir, \allowbreak
            \mlinkcast[\mray_L], \mlinkcast[\mray_R], \varnothing, \allowbreak
            \mlinkcast[\mlinks_S], \mlinkcast[\mlinks_T], c, \msort}$
    \EndIf
    \State \Return $\mtrue$
\EndFunction
\end{algorithmic}
\end{algorithm}

\subsection{CasterReachedWithCmlVisQueue}
The function queues the links of the visible cast link $\mlinkcast$ that has no known cumulative visibility.
If $\mlinkcast$ has cumulative visibility to the start point, sector rays are merged into the the link and the queued links.

\begin{algorithm}[!ht]
\begin{algorithmic}[1]
\caption{\textproc{CasterReachedWithCmlVisQueue}: Merge sector rays and queue the subsequent links.}
\label{suppr2p:alg:casterreachedwithcmlvisqueue}
\Function{CasterReachedWithCmlVisQueue}{$\mtdir, \mside, \mvdiff, \mlinkcast$}
    \If {$\mtdir = T$}
        \State $\Call{MergeRay}{-\mside, \mlinkcast, (\mvdiff, \mtrue)}$

        \State $\mlinks_T \gets $ copy of $\mlinkcast[\mlinks_T]$
        \For {$\mlink_T \in \mlinks_T$}
            \State $\text{new } \mlink_T \gets \Call{Isolate}{S, \mlink_T, \mlinkcast}$
            \State $\Call{MergeRay}{\mside, \text{new } \mlink_T, (\mvdiff, \mfalse)}$
            \State $\Call{Queue}{\mqcast, \text{new } \mlink_T, \mlinkcast[c] + (\text{new } \mlink_T)[c]}$
        \EndFor
    \Else
        \State $\mlink_S \gets $ source link of $\mlinkcast$
        \State $\text{new } \mlink_S \gets \Call{Isolate}{T, \mlink_S, \mlinkcast}$
        \State $\Call{Queue}{\mqcast, \text{new } \mlink_S, \mlinkcast[c] + (\text{new } \mlink_T)[c]}$
    \EndIf
\EndFunction
\end{algorithmic}
\end{algorithm}

\subsection{CasterReachedWithoutCmlVis}
The function is called when a cast is successful for the link $\mlinkcast$, and when $\mlinkcast$ is not connected to links which have cumulative visibility. 

\begin{algorithm}[!ht]
\begin{algorithmic}[1]
\caption{\textproc{CasterReachedWithoutCmlVis}: Handles a successful cast on a link that has no cumulative visibility.}
\label{suppr2p:alg:casterreachedwithoutcmlvis}
\Function{CasterReachedWithoutCmlVis}{$\mlinkcast$}
    \State $\mpoint_T \gets \mlinkcast$'s target point.
    \State $\mvdiff \gets \mlinkcast$'s target coordinates $-$ $\mlinkcast$'s source coordinates.
    \State $\mside \gets \mpoint_T[\mside]$ (side of target point).

    \Comment{Anchor cast link to target point and merge ray.}
    \State $\Call{Change}{\mlinkcast, \mpoint_T, \mlvy, S, \allowbreak
        \mlinkcast[\mray_L], \mlinkcast[\mray_R], \varnothing, 
        \mlinkcast[\mlinks_S], \mlinkcast[\mlinks_T], \mcalc, \msort, 
    }$
    \State $\Call{MergeRay}{-\mside, \mlinkcast, (\mvdiff, \mtrue)}$

    \Comment{Isolate target links and merge ray.}
    \State $\mlinks_T \gets $ copy of $\mlinkcast[\mlinks_T]$.
    \For {$\mlink_T \in \mlinks_T$}
        \State new $\mlink_T \gets \Call{Isolate}{S, \mlink_T, \mlinkcast}$
        \State $\Call{MergeRay}{-\mside, \mlinkcast, (\mvdiff, \mfalse)}$
    \EndFor

    \Comment{If other links are anchored at the target point, mark for check by overlap rule.}
    \State $\mpoint_o \gets \Call{GetPoint}{\mpoint_T, -\mside}$
    \If {$\mpoint$ and $\mpoint_o$ anchors links other than $\mlinkcast$ and its target links}
        \State Add $\mpoint$ to overlap-buffer.

    \Comment{Otherwise, queue the target links.}
    \Else
        \For {$\mlink_T \in \mlinkcast[\mlinks_T]$}
            \State $\Call{Queue}{\mqcast, \mlink_T, \mlinkcast[c] + \mlink_T[c]}$
        \EndFor
    \EndIf
\EndFunction
\end{algorithmic}
\end{algorithm}

\subsection{CasterCollided}
The function creates traces after a cast collides. A \textit{major} trace has the same side as the cast link $\mlinkcast$'s source point. A \textit{minor} trace has the opposite side. A \textit{third} trace occurs if $\mlinkcast$'s target point is the goal point.

\begin{algorithm}[!ht]
\begin{algorithmic}[1]
\caption{\textproc{CasterCollided}: Handles collided cast queries.}
\label{suppr2p:alg:castercollided}
\Function{CasterCollided}{$\mlinkcast, \mpoint_L, \mpoint_R$}
    \State $\mside_\mmjr \gets $ side of source point of $\mlinkcast$.
    \State $\mside_\mmnr \gets -\mside_\mmjr$
    \State $\mpoint_\mmjr \gets \mpoint_L$ \textbf{if} $\mside_\mmjr = L$ \textbf{else} $\mpoint_R$
    \State $\mpoint_\mmnr \gets \mpoint_L$ \textbf{if} $\mside_\mmnr = L$ \textbf{else} $\mpoint_R$
    \State $\mlink_S \gets $ source link of $\mlinkcast$.

    \Comment{Initialize third and minor traces if source link is not $S$-tree $\mley$ link.}
    \State $\mtrace_\mthird \gets \varnothing$
    \State $\mtrace_\mmnr \gets \varnothing$
    \If {$\mlink_S$ is not $\mley$ type}
        \State $\mtrace_\mthird \gets \Call{CasterCollidedThirdTrace}{\mlinkcast}$
        \State $\mtrace_\mmnr \gets \Call{CasterCollidedMjrMnrTrace}{\mlinkcast, \mpoint_\mmnr}$
    \EndIf

    \Comment{Initialize major trace.}
    \State $\mtrace_\mmjr \gets \Call{CasterCollidedMjrMnrTrace}{\mlinkcast, \mpoint_\mmjr}$

    \Comment{Erase cast link and try to run traces.}
    \State $\Call{Erase}{\mlinkcast}$
    \If {$\mtrace_\mmnr \ne \varnothing$}
        \State $\Call{Tracer}{\mtrace_\mmnr}$
    \EndIf
    \If {$\mtrace_\mthird \ne \varnothing$}
        \State $\Call{Tracer}{\mtrace_\mthird}$
    \EndIf
    \If {$\mtrace_\mmjr \ne \varnothing$}
        \State $\Call{Tracer}{\mtrace_\mmjr}$
    \EndIf
\EndFunction
\end{algorithmic}
\end{algorithm}

The third trace can be discarded if the minor trace traces back to the source point (not shown in Alg. \ref{suppr2p:alg:castercollided}).
This can be done by examining $\mtrace_\mmnr[refound\_src]$.

\subsection{CasterCollidedThirdTrace}
If the target point of a collided cast is the goal point, the function tries to create a third trace from the source point of the collided cast.

\begin{algorithm}[!ht]
\begin{algorithmic}[1]
\caption{\textproc{CasterCollidedThirdTrace}: Generates a third trace.}
\label{suppr2p:alg:castercollidedthirdtrace}
\Function{CasterCollidedThirdTrace}{$\mlink_\mathrm{cast}$}

    \Comment{Return nothing if target point is not the goal point}
    \If {target point of $\mlink_\mathrm{cast}$ is not the goal point}
        \State \Return $\varnothing$
    \EndIf
    \State $\mpoint_\mathrm{un} \gets $ source point of $\mlink_\mathrm{cast}$.
    \State $\mside \gets \mpoint_\mathrm{un}[\mside]$
    
    \Comment{Return nothing if the corner before the source point is at the map boundary.}
    \State $\mpoint_\mathrm{oc} \gets \Call{Trace}{\mpoint_\mathrm{un}[\mx], -\mside}$
    \State $\mpoint_\mathrm{oc} \gets \Call{GetPoint}{\mpoint_\mathrm{oc}, \mside}$
    \If {$\mpoint_\mathrm{oc}$ is at map boundary}
        \State \Return $\varnothing$
    \EndIf

    \Comment{Return nothing if the corner after the source point is at the map boundary.}
    \State $\mpointtrace \gets \Call{Trace}{\mpoint_\mathrm{un}[\mx], \mside}$
    \State $\mpointtrace \gets $ copy of $\mpointtrace$ with side and corner information only.
    \If {$\mpointtrace$ is at map boundary}
        \State \Return $\varnothing$
    \EndIf

    \Comment{Create an $\mlun$ link and $\mloc$ link to guide the trace around the obstacle.}
    \State $\mlink_\mathrm{un} \gets $ a new link.
    \State $\Call{Change}{\mlink_\mathrm{un}, \mpoint_\mathrm{un}, \mlun, T, \allowbreak
        \varnothing, \varnothing, \varnothing, \allowbreak
        \{\}, \mlink_\mathrm{cast}[\mlinks_T], \mcalc, \msort
    }$
    % \State $\mlink_\mathrm{oc} \gets $ a new link.
    % \State $\Call{Change}{\mlink_\mathrm{oc}, \mpoint_\mathrm{oc}, \mloc, T, \allowbreak
    %     \varnothing, \varnothing, \varnothing, \allowbreak
    %     \{\}, \{\mlink_\mathrm{un}\}, \mcalc, \msort
    % }$

    \Comment {Create a target link for the trace.}
    \State $\mvprogt \gets \mpointtrace$'s coordinates $-$ $\mpoint_\mathrm{un}$'s coordinates.
    \State $\Call{Change}{\text{new link}, \mpointtrace, \mltm, T, \allowbreak
        \varnothing, \varnothing, \mvprogt, \allowbreak
        \{\}, \{\mlink_\mathrm{un}\}, \infty, \mback
    }$

    \Comment{Create a source link for the trace and merge the cast ray.}
    \State $\mv_\mathrm{cast} \gets \mlink_\mathrm{cast}$'s target coordinates $-$ $\mlink_\mathrm{cast}$'s source coordinates.
    \State $\mray_\mathrm{cast} \gets (\mv_\mathrm{cast}, \mtrue)$
    \State $\mvprogs \gets \mpointtrace$'s coordinates $-$ $\mpoint_\mathrm{un}$'s coordinates.
    \State $\mlink_S \gets $ a new link.
    \State $\Call{Change}{\mlink_S, \mpointtrace, \mltm, S, \allowbreak
        \mlink_\mathrm{cast}[\mray_L], \mlink_\mathrm{cast}[\mray_R], \mvprogs, \allowbreak
        \mlink_\mathrm{cast}[\mlinks_S], \{\}, \infty, \mback
    }$
    \State $\Call{MergeRay}{\mside, \mlink_S, \mray_\mathrm{cast}}$

    \Comment{Return Trace object.}
    \State $\mtrace \gets (\mpointtrace, \cdots)$
    \State \Return $\mtrace$
\EndFunction
\end{algorithmic}
\end{algorithm}

\subsection{CasterCollidedMjrMnrTrace}
The function tries to create and return a trace from the collision point.
The side of the created trace is obtained from the side of $\mpoint$.

\begin{algorithm}[!ht]
\begin{algorithmic}[1]
\caption{\textproc{CasterCollidedMjrMnrTrace}: Generates traces after a cast collides.}
\label{suppr2p:alg:castercollidedmjrmnrtrace}
\Function{CasterCollidedMjrMnrTrace}{$\mlink_\mathrm{cast}, \mpoint$}

    \Comment{Return nothing if point is at map boundary; otherwise, create Trace object.}
    \If {$\mpoint$ at map boundary}
        \State \Return $\varnothing$
    \EndIf

    \Comment{Create a closed ray based on the cast.}
    \State $\mside \gets \mpoint[\mside]$
    \State $\mpointtrace \gets $ copy of $\mpoint$ with side and corner information only.
    \State $\mv_\mathrm{cast} \gets \mlink_\mathrm{cast}$'s target coordinates $-$ $\mlink_\mathrm{cast}$'s source coordinates.
    \State $\mray_\mathrm{cast} \gets (\mv_\mathrm{cast}, \mtrue)$

    \Comment{Create new source link for trace and merge the cast ray into it.}
    \State new $\mlink_S \gets $ a new link.
    \State $\Call{Change}{\text{new } \mlink_S, \mpointtrace, \mltm, S, \allowbreak
        \mlink_\mathrm{cast}[\mray_L], \mlink_\mathrm{cast}[\mray_R], \mv_\mathrm{cast}, 
        \mlink_\mathrm{cast}[\mlinks_S], \{\}, \infty, \mback
    }$
    \State $\Call{MergeRay}{-\mside, \text{new } \mlink_S, \mray_\mathrm{cast}}$

    \Comment{Create new target link for trace.}
    \State $\Call{Change}{\text{new link}, \mpointtrace, \mltm, T, \allowbreak
        \varnothing, \varnothing, -\mv_\mathrm{cast}, 
        \{\}, \mlink_\mathrm{cast}[\mlinks_T], \infty, \mback
    }$

    \Comment{Return Trace object.}
    \State $\mtrace \gets (\mpointtrace, \cdots)$
    \State \Return $\mtrace$
\EndFunction
\end{algorithmic}
\end{algorithm}

\clearpage
\section{Functions for Tracing}
This section describes functions for trace queries.

\subsection{SetupTracerFromLink}
The function initializes a trace query from a $S$-tree $\mltm$ link $\mlink$.
The function is called when a trace query is polled from the open-list, or when a cast reaches an interrupted trace.

\begin{algorithm}[!ht]
\begin{algorithmic}[1]
\caption{\textproc{SetupTracerFromLink}: Initializes a trace from a link.}
\label{suppr2p:alg:setuptracerfromlink}
\Function{SetupTracerFromLink}{$\mlink$}
    
    \Comment{Prepare trace point $\mpoint_\mtrace$}
    \State $\mpoint_\mathrm{tm} \gets $ anchored point of $\mlink$. 
    \State $\mpoint_\mtrace \gets$ copy of $p$ with corner and side information only.
    \State $\mlinks_T \gets $ copy of $\mlink$'s target links

    \Comment {Re-anchor the target links of $\mlink$ to $\mpoint_\mtrace$.}
    \For {$\mlink_T \in \mlinks_T$} 
        \State $\mv_{\mathrm{prog},T} \gets \mlink_T$'s anchor coordinates $-$ $\mlink_T$'s root coordinates. 
        \State new $\mlink_T \gets \Call{Isolate}{S, \mlink_T, \mlink}$
        \State $\Call{Change}{\text{new } \mlink_T, \mpoint_\mtrace, \mltm, T, \allowbreak
            \varnothing, \varnothing, \mv_{\mathrm{prog},T}, \allowbreak
            \{\}, \mlink_T[\mlinks_T], \infty, \mback
        }$
    \EndFor
    
    \Comment {Re-anchor $\mlink$ to $\mpoint_\mtrace$.}
    \State $\mv_{\mathrm{prog},S} \gets \mlink$'s anchor coordinates $-$ $\mlink$'s root coordinates. 
    \State $\Call{Change}{\mlink, \mpoint_\mtrace, \mltm, S, \allowbreak,
        \mlink[\mray_L], \mlink[\mray_R], \mv_{\mathrm{prog},S}, \allowbreak,
        \mlink[\mlinks_S], \{\}, \infty, \mback
        }$

    \Comment{Begin trace query.}
    \State $\mtrace \gets (\mpoint_\mtrace, \cdots)$ 
    \State $\Call{Tracer}{\mtrace}$
\EndFunction
\end{algorithmic}
\end{algorithm}

\clearpage
\subsection{Tracer}
The main function that handles a trace query.

\begin{algorithm}[!ht]
\begin{algorithmic}[1]
\caption{\textproc{Tracer}: Handles a trace query.}
\label{suppr2p:alg:tracer}
\Function{Tracer}{$\mtrace$}

    \Comment{Mark all links anchored at trace point as progressed.}
    \For {each link $\mlink$ anchored at $\mtrace_\mpoint$}
        \State $\mlink[is\_prog] \gets \mtrue$
    \EndFor

    \Comment{Apply rules to all links for each corner traced.}
    \While {$\mtrue$}
        \If {$\Call{TracerRefoundSrc}{\mtrace}$}
            \State \Break
        \ElsIf {$\Call{TracerProcess}{\mtrace, S}$}
            \State \Break
        \ElsIf {$\Call{TracerProcess}{\mtrace, T}$}
            \State \Break
        \ElsIf {$\Call{TracerInterruptRule}{\mtrace}$}
            \State \Break
        \ElsIf {$\Call{TracerPlaceRule}{\mtrace}$}
            \State \Break
        \EndIf

        \Comment {Go to next corner, or stop if at map boundary.}
        \State $\mpoint_\mathrm{next} \gets \Call{Trace}{\mtrace[\mpoint][\mx], \mtrace[\mpoint][\mside]}$ 
        \If{$\mpoint_\mathrm{next}$ is at map boundary}
            \State \Break
        \Else
            \State $\mtrace[\mpoint] \gets $ copy of $\mpoint_\mnext$ with side and corner information only.
            \State $\mtrace[m] \gets \mtrace[m] + 1$
        \EndIf
    \EndWhile
    
    \Comment{Discard branch of links still anchored at trace point.}
    \For {every anchored link $\mlink$ of $\mtrace[\mpoint]$} 
        \State $\Call{EraseTree}{\mlink[\mtdir], \mlink}$
    \EndFor
\EndFunction
\end{algorithmic}
\end{algorithm}

\clearpage
\subsection{TracerRefoundSrc}
The function returns $\mtrue$ if the trace query has traced back to the source point.

\begin{algorithm}[!ht]
\begin{algorithmic}[1]
\caption{\textproc{TracerRefoundSrc}: Indicates if a trace has traced back to the source point.}
\label{suppr2p:alg:tracerrefoundsrc}
\Function{TracerRefoundSrc}{$\mtrace$}
    \State $\mlink_S \gets $ $S$-tree link anchored at $\mtrace[\mpoint]$.
    \State $\mtrace[refound\_src] \gets \mlink_S$'s anchored coordinates $=$ $\mlink_S$'s root coordinates.
    \State \Return $\mtrace[refound\_src]$
\EndFunction
\end{algorithmic}
\end{algorithm}

There is only one $S$-tree link (source link of the trace) anchored at the trace point $\mtrace[\mpoint]$ at all times during a trace.

\clearpage
\subsection{TracerProcess}
The function processes $\mtdir$-tree links of the trace by subjecting each link to the trace rules. $\mtrue$ is returned if the trace has no more $\mtdir$-tree links, $\mfalse$ otherwise.

\begin{algorithm}[!ht]
\begin{algorithmic}[1]
\caption{\textproc{TracerProcess}: Examines a link during a trace.}
\label{suppr2p:alg:tracerprocess}
\Function{TracerProcess}{$\mtrace, \mtdir$}
    \State $i \gets 0$
    \While {$i < $ number of $\mtdir$-tree links anchored at $\mtrace[\mpoint]$}
        \State $\mlink \gets i\textsuperscript{th}$ $\mtdir$-tree link anchored at $\mtrace[\mpoint]$.
        \If {$\Call{TracerProgRule}{\mtrace, \mlink}$}
            \State $i \gets i + 1$
        \ElsIf {$\Call{TracerAngSecRule}{\mtrace, \mlink}$}
            \State \Continue
        \ElsIf {root point of $\mlink$ is start point or goal point}
            \State $i \gets i + 1$
        \ElsIf {$\Call{TracerOcSecRule}{\mtrace, \mlink}$}
            \State $i \gets i + 1$
        \ElsIf {$\Call{TracerPruneRule}{\mtrace, \mlink}$}
            \State \Continue
        \EndIf
    \EndWhile

    \State \Return $\mtrue$ \textbf{if} no more $\mtdir$-tree links at trace point $\mtrace[\mpoint]$ \textbf{else} $\mfalse$.
\EndFunction
\end{algorithmic}
\end{algorithm}

\clearpage
\subsection{TracerProgRule}
The function implements the progression rule, and updates progression ray of the link $\mlink$ if the trace's angular deviation (progression) increases when viewed from $\mlink$'s root point.
Additionally, if the angular deviation for the source link (source progression) decreases by more than $180^\circ$, a cast from the source point is queued.

\begin{algorithm}[!ht]
\begin{algorithmic}[1]
\caption{\textproc{TracerProgRule}: Implements the progression rule.}
\label{suppr2p:alg:tracerprogrule}
\Function{TracerProgRule}{$\mtrace, \mlink$}
    \State $\mvdiff \gets \mlink$'s anchored coordinates $-$ root coordinates.
    \If {$\mvdiff$ is zero}
        \State $\mvdiff \gets $ directional vector bisecting corner at $\mtrace[\mpoint]$.
    \EndIf
    \State $(\mtdir, \mside, \mvprog') \gets (\mlink[\mtdir], \mtrace[\mside], \mlink[\mvprog])$ 
        
    \Comment{Return $\mtrue$ if there is no source or target progression.}
    \If {$\mtdir\mside(\mvdiff \times \mvprog') > 0$}
        \State $\mlink[is\_prog] \gets \mfalse$
        \State \Return $\mtrue$ 
    \Else
        
        \Comment{Queue a cast and return $\mtrue$ if source progression has decreased by $> 180^\circ$.}
        \If{$\mlink$ is $S$-tree link 
            \An $\mlink[is\_prog] = \mfalse$ 
            \An $\Call{TracerProgRuleCast}{\mtrace, \mlink}$}
                \State \Return $\mtrue$ 
        \EndIf

        \Comment{Return $\mfalse$ and update progression ray if there is source or target progression.}
        \State $\mlink[is\_prog] \gets \mtrue$
        \State $\mlink[\mvprog] \gets \mvdiff$
        \State \Return $\mfalse$
    \EndIf
\EndFunction
\end{algorithmic}
\end{algorithm}

\clearpage
\subsection{TracerProgRuleCast}
If the source progression has decreased by more than $180^\circ$, the function queues a cast query from the source point to the phantom point where the source progression was the largest. 

\begin{algorithm}[!ht]
\begin{algorithmic}[1]
\caption{\textproc{TracerProgRuleCast}: Queues a cast when the source progression decreases by more than $180^\circ$.}
\label{suppr2p:alg:tracerprogrulecast}
\Function{TracerProgRuleCast}{$\mtrace, \mlink$}
    \State $\mv_\mathrm{prev} \gets $ directional vector of trace before reaching the current corner.
    \State $\mvdiff \gets $ $\mlink$'s anchor coordinates $-$ $\mlink$'s root coordinates.
    \State $\mvprog' \gets \mlink[\mvprog]$
    
    \Comment{Queue a cast if source progression has decreased by $>180^\circ$}
    \If {$(\mvdiff \times \mv_\mathrm{prev}) (\mv_\mathrm{prev} \times \mvprog') > 0$}
        \State $\mlink_T \gets $ $T$-tree link anchored at $\mtrace[\mpoint]$. 
        \State $\mlink_S \gets $ source link of $\mlink$.
        \State $\mpoint_S \gets $ anchored point of $\mlink_S$.
        \State $\Call{Change}{\mlink_T, \mpoint_S, \mlvu, T, \allowbreak
            \varnothing, \varnothing, \varnothing, \allowbreak
            \{\mlink_S\}, \mlink_T[\mlinks_T], \mcalc, \msort
            }$
        \State $\Call{Erase}{\mlink}$
        \State $\Call{Queue}{\mqcast, \mlink_T, \mlink_S[c] + \mlink_T[c]}$
    \EndIf 
\EndFunction
\end{algorithmic}
\end{algorithm}

The cast is a \textit{necessary} step to guarantee source and target progression when all trace queries begin, but is not a \textit{sufficient} one.

As the maximum source progression can only occur at a phantom point,
$\mlink_T$ in Alg. \ref{suppr2p:alg:tracerprogrulecast} is the only $T$-tree link that is anchored at the moving trace point $\mtrace[\mpoint]$. 
The link's root point is the phantom point, and the link is connected to at least one $\mlun$ target link.

\clearpage
\subsection{TracerAngSecRule}
The function implements the angular sector rule.

\begin{algorithm}[!ht]
\begin{algorithmic}[1]
\caption{\textproc{TracerAngSecRule}: Implements the angular-sector rule.}
\label{suppr2p:alg:tracerangsecrule}
\Function{TracerAngSecRule}{$\mtrace, \mlink$}

    \Comment{Return if $\mlink$ is not $S$-tree link.}
    \If {$\mlink$ is $T$-tree link}
        \State \Return $\mfalse$
    \EndIf

    \Comment{Return if there is no sector-ray on same side as trace.}
    \State $\mpoint_\mtrace \gets \text{trace point } \mtrace[\mpoint]$
    \State $\mside \gets \text{trace side } \mpoint_\mtrace[\mside])$
    \State $\mray \gets \mside$-side ray of $\mlink$.
    \If {$\mray$ does not exist s.t. $\mray = \varnothing$}
        \State \Return $\mfalse$
    \EndIf

    \Comment{If sector-ray $\mray$ is crossed...}
    \State $\mv_\mathrm{ray} \gets \mray[\mv]$
    \If {$\Call{CrossedRay}{\mside, \mlink, \mv_\mathrm{ray}}$}
        \State $ray\_was\_closed \gets \mray[$closed$]$
        \State $\mray[closed] \gets \mtrue$

        \Comment{Generate recursive ang. sec. trace if projected ray collides at different obstacle edge.}
        \State $\Call{TraceAngSecRuleRecur}{\mtrace, \mlink, \mv_\mathrm{ray}}$

        \Comment{Prune $\mlink$ from trace if ray is not closed.}
        \If {\Not $ray\_was\_closed$}
            \State $\mlink_S \gets $ source link of $\mlink$.
            \State $\mlink_\mnew \gets \Call{Isolate}{T, \mlink_S, \mlink}$
            \State $\mvprog \gets \mpoint_\mtrace$'s coordinates - $\mlink_S$'s root coordinates.
            \State $\Call{Change}{\mlink_\mnew, \mpoint_\mtrace, \mltm, S, \allowbreak
                \mlink_\mnew[\mray_L], \mlink_\mnew[\mray_R], \mvprog,
                \mlink_\mnew[\mlinks_S], \{\}, \infty, \mback
                }$
        \EndIf

        \Comment{Erase $\mlink$ if no more target links}
        \State $\Call{EraseTree}{S, \mlink}$
        \State \Return $\mtrue$
    \EndIf
    \State \Return $\mfalse$
\EndFunction
\end{algorithmic}
\end{algorithm}

\clearpage
\subsection{TracerAngSecRuleRecur}
The function calls a recursive trace if the projection of the crossed sector-ray collides with a different obstacle edge as the trace.

\begin{algorithm}[!ht]
\begin{algorithmic}[1]
\caption{\textproc{TracerAngSecRuleRecur}: Implements the angular-sector rule.}
\label{suppr2p:alg:tracerangsecrulerecur}
\Function{TracerAngSecRuleRecur}{$\mtrace, \mlink, \mv_\mathrm{ray}$}

    \Comment{Project the ray.}
    \State $\mpoint_\mtrace \gets $ trace point $\mtrace[\mpoint]$
    \State $\mside \gets $ trace side $ \mpoint_\mtrace[\mside]$
    \State $(\mpoint_L, \mpoint_R) \gets \Call{Project}{\mlink\text{'s root coordinates}, \mv_\mathrm{ray}}$
    \State $\mpoint_\mside \gets \mpoint_L$ \textbf{if} $\mside = L$ \textbf{else} $\mpoint_R$ 
    \State $\mpoint_{-\mside} \gets \mpoint_R$ \textbf{if} $\mside = L$ \textbf{else} $\mpoint_L$ 
    
    \Comment{Generate recursive ang-sec trace if projection hits a different obstacle edge.}
    \If {trace point $\mtrace[\mpoint]$ is not at same corner as $\mpoint_\mside$}
    
        \Comment{Copy all target links of trace.}
        \State $\mlink_\mathrm{un} \gets $ a new link.
        \State $\mpoint_\mathrm{tm} \gets \Call{GetPoint}{\mpoint_\mtrace, \mside}$ 
        \For{ each $T$-tree link $\mlink_T$ anchored at $\mpoint_\mtrace$}
            \State $\mlinks_{TT} \gets $ $\mlink_T$'s target links $\mlink_T[\mlinks_T]$.
            \State $\Call{Change}{\text{new link}, \mpoint_\mathrm{tm}, \mltm, T, \allowbreak
                \varnothing, \varnothing, \varnothing, \allowbreak
                \{\mlink_\mathrm{un}\}, \mlinks_{TT}, \mcalc, \msort
                }$
        \EndFor 

        \Comment{Create unreachable target link.}
        \State $\mpoint_\mathrm{un} \gets \Call{GetPoint}{\mpoint_\mside, -\mside}$
        \State $\Call{Change}{\mlink_\mathrm{un}, \mpoint_\mathrm{un}, \mlun, T, \allowbreak
            \varnothing, \varnothing, \varnothing, 
            \{\}, \mlink_\mathrm{un}[\mlinks_T], \mcalc, \msort
        }$

        \Comment{Create target link of recur. trace.}
        \State $\mlink_T \gets $ a new link.
        \State new $\mpoint_\mtrace \gets $  copy of $\mpoint_{-\mside}$ with side and corner information only.
        \State $\mv_{\mathrm{prog},T} \gets $ new $ \mpoint_\mtrace$'s coordinates $-$ $\mpoint_\mathrm{un}$'s coordinates.
        \State $\Call{Change}{\mlink_T, \text{new } \mpoint_\mtrace, \mltm, T, \allowbreak
            \varnothing, \varnothing, \mv_{\mathrm{prog},T}, \allowbreak
            \{\}, \{\mlink_\mathrm{un}\}, \infty, \mback
            }$

        \Comment{Create source link of recur. trace by copying from $\mlink$}
        \State $\mlink_S \gets $ a new link.
        \State $\mv_{\mathrm{prog},S} \gets $ new $ \mpoint_\mtrace$'s coordinates $-$ $\mlink$'s root coordinates.
        \State $\Call{Change}{\mlink_S, \text{new } \mpoint_\mtrace, \mltm, S, \allowbreak
            \mlink[\mray_L], \mlink[\mray_R], \mv_{\mathrm{prog},S}, \allowbreak
            \mlink[\mlinks_S], \{\}, \infty, \mback
            }$

        \Comment{Begin recur. trace.}
        \State new $\mtrace \gets (\text{new } \mpoint_\mtrace, \cdots)$
        \State $\Call{Tracer}{\text{new } \mtrace}$
        
    \EndIf
\EndFunction
\end{algorithmic}
\end{algorithm}

\clearpage
\subsection{TracerOcSecRule}
The function implements the occupied sector rule.

\begin{algorithm}[!ht]
\begin{algorithmic}[1]
\caption{\textproc{TracerOcSecRule}: Implements the occupied-sector rule.}
\label{suppr2p:alg:tracerocsecrule}
\Function{TracerOcSecRule}{$\mtrace, \mlink$}
    \State $\mpoint_\mtrace \gets $ trace point $\mtrace[\mpoint]$
    \State $\mpoint_\mtdir \gets \mlink$'s root point.
    \If {side of $\mpoint_\mtrace \ne$ side of $\mpoint_\mtdir$}
        \State \Return $\mfalse$
    \EndIf
    
    \Comment{If target point anchors an $\mloc$ link...}
    \State $\mlink_\mtdir \gets $ any of $\mlink$'s root links.
    \State $\mvdiff \gets \mlink's$ anchor coordinates $-$ $\mlink$'s root coordinate.
    \If {$\mlink_\mtdir$ is $\mloc$ or $\mlun$ type}
        \State $\mside \gets $ trace side $\mpoint_\mtrace[\mside]$
        \State $\mv_{TT} \gets \mlink_\mtdir's$ anchor coordinates $-$ $\mlink_\mtdir$'s root coordinate.
        
        \Comment{Discard trace if moved $180^\circ$ around target point's oc. sec.}
        \If {$\mside (\mv_{TT} \times \mvdiff) > 0$}
            \State $\Call{EraseTree}{T, \mlink}$
            \State \Return $\mtrue$

        \Comment{Continue trace if not moved $180^\circ$ around target point's oc. sec.}
        \Else
            \State \Return $\mfalse$
        \EndIf
    \EndIf

    \Comment{Otherwise, check if trace has entered oc. sec. of root (source/target) point.}
    \State $\mtdir \gets \mlink[\mtdir]$
    \State $\mside_\mtdir \gets $ side of root point $ \mpoint_\mtdir[\mside]$
    \State $\mside_\mathrm{edge} \gets -\mtdir\mside_\mtdir$, which is side of edge at root point that is nearer to $\mlink$.
    \State $\mv_\mathrm{edge} \gets $ directional vector pointing away from $\mpoint_\mtdir$ and parallel to $\mside_\mathrm{edge}$ edge.
    \If {$\mside_\mathrm{edge} (\mv_\mathrm{edge} \times \mvdiff) > 0$}
        \State $\mpoint_\mathrm{edge} \gets \Call{Trace}{\mpoint_\mtdir[\mx], \mside_\mathrm{edge}}$
        
        \Comment{Generate recur. trace. if in oc. sec. of source point.}    
        \If {$\mtdir = S$}
            \State $\Call{TracerOcSecRuleRecur}{\mtrace, \mlink, \mpoint_\mathrm{edge}}$
        
        \Comment {Place $\mloc$ link if in oc. sec. of target point.}
        \Else
            \State $\mpoint_\mathrm{oc} \gets \Call{GetPoint}{\mpoint_\mathrm{edge}, \mside_\mtdir}$
            \State $\Call{Change}{\mlink, \mpoint_\mathrm{oc}, \mloc, T, \allowbreak
                \varnothing, \varnothing, \varnothing, \allowbreak
                \mlink[\mlinks_S], \mlink[\mlinks_T], \mcalc, \msort
            }$

            \State $\mvprog \gets \mpoint_\mtrace$'s coordinates $-$ $\mpoint_\mathrm{oc}$'s coordinates.
            \State new $\mlink \gets $ a new link.
            \State $\Call{Change}{\text{new } \mlink, \mpoint_\mtrace, \mltm, T, \allowbreak
                \varnothing, \varnothing, \mvprog, \allowbreak
                \{\}, \{\mlink\}, \infty, \mfront
                }$
        \EndIf
    \EndIf

    \State \Return $\mtrue$
    
\EndFunction
\end{algorithmic}
\end{algorithm}

\clearpage
\subsection{TracerOcSecRuleRecur}
The function calls a recursive trace from the source point of the trace.

\begin{algorithm}[!ht]
\begin{algorithmic}[1]
\caption{\textproc{TracerOcSecRuleRecur}: Generates the recursive occupied sector trace.}
\label{suppr2p:alg:tracerocsecrulerecur}
\Function{TracerOcSecRuleRecur}{$\mtrace, \mlink, \mpoint_\mathrm{edge}$}

    \Comment{Re-anchor link $\mlink$ to new trace point of oc. sec. trace.}
    \State new $\mpointtrace \gets $ copy of $\mpoint_\mathrm{edge}$ with side and corner information only.
    \State $\mpoint_\mtdir \gets $ root point of $\mlink$.
    \State $\mvprogs \gets \mpointtrace$ 's coordinates $-$ $\mpoint_\mtdir$'s coordinates.
    \State $\Call{Change}{\mlink, \text{new } \mpointtrace, \mltm, S, \allowbreak
        \mlink[\mray_L], \mlink[\mray_R], \mvprogs, \allowbreak
        \mlink[\mlinks_S], \mlink[\mlinks_T], \infty, \mback
    }$

    \Comment{Re-anchor all target links of current trace.}
    \State new $\mlink_T \gets $ a new link.
    \State $\mpoint_\mathrm{tm} \gets \Call{GetPoint}{\mtrace[\mpoint], \mtrace[\mpoint][\mside]}$
    \For {each target link $\mlink_T$ anchored at trace point $\mtrace[\mpoint]$}
        \State $\Call{Change}{\mlink_T, \mpoint_\mathrm{tm}, \mltm, T, \allowbreak
            \varnothing, \varnothing,   \varnothing, \allowbreak
            \{\text{new } \mlink_T\}, \mlink_T{\mlinks_T}, \mcalc, \msort
        }$
    \EndFor

    \Comment{Create new target link for oc. sec. trace.}
    \State $\mvprogt \gets $ new $\mpointtrace$'s coordinates $-$ $\mpoint_\mathrm{tm}$'s coordinates.
    \State $\Call{Change}{\text{new } \mlink_T, \text{new } \mpointtrace, \mltm, T, \allowbreak
        \varnothing, \varnothing, \mvprogt, \allowbreak
        \{\}, (\text{new } \mlink_T)[\mlinks_T], \infty, \mback
        }$
    
    \Comment{Begin recursive oc. sec. trace.}
    \State $\text{new } \mtrace \gets (\text{new } \mpointtrace, \cdots)$ 
    \State $\Call{Tracer}{\text{new }\mtrace}$
\EndFunction
\end{algorithmic}
\end{algorithm}

\clearpage
\subsection{TracerPruneRule}
The function implements the pruning rule. 
The function returns $\mtrue$ if the link $\mlink$ is fully pruned and erased, or $\mfalse$ otherwise.

\begin{algorithm}[!ht]
\begin{algorithmic}[1]
\caption{\textproc{TracerPruneRule}: Implements the Pruning Rule.}
\label{suppr2p:alg:tracerprunerule}
\Function{TracerPruneRule}{$\mtrace, \mlink$}
    \State $\mpointtrace \gets $ trace point $\mtrace[\mpoint]$.
    % \State $\mvdiff \gets \mlink$'s anchor coordinates $-$ $\mlink$'s root coordinates.
    \State $\mtdir \gets \mlink[\mtdir]$ 
    \State $\mside_\mtdir \gets $ side of $\mlink$'s root point.

    \Comment{Try to prune link w.r.t. all of its root links.}
    \State $\mlinks_\mtdir \gets $ copy of $\mtdir$ links of $\mlink$. 
    \For {$\mlink_\mtdir \in \mlinks_\mtdir$}
        \State $\mv_{\mathrm{diff},\mtdir} \gets \mlink_\mtdir$'s anchor coordinates $-$ $\mlink_\mtdir$'s root coordinates. 
        \If {$\Call{CrossedRay}{\mside_\mtdir, \mlink_\mtdir, \mv_{\mathrm{diff}, \mtdir}}$}
            \State $\Call{Disconnect}{\mtdir, \mlink, \mlink_\mtdir}$
            \State $\text{new } \mlink \gets \Call{Isolate}{-\mtdir, \mlink_\mtdir, \varnothing}$

            \State $\mray_L \gets (\text{new } \mlink)[\mray_L] $ \textbf{if} $\mtdir = S$ \textbf{else} $\varnothing$
            \State $\mray_R \gets (\text{new } \mlink)[\mray_R] $ \textbf{if} $\mtdir = S$ \textbf{else} $\varnothing$
            \State $\mvprog \gets \mpointtrace$'s coordinates $-$ new $\mlink$'s root coordinates.
            \State $\Call{Change}{\text{new } \mlink, \mpointtrace, \mltm, \mtdir, \allowbreak
                \mray_L, \mray_R, \mvprog, \allowbreak
                (\text{new } \mlink)[\mlinks_S], (\text{new } \mlink)[\mlinks_T], \infty, \mback
            }$
        \EndIf
    \EndFor

    \Comment{Return $\mtrue$ if link is fully pruned; otherwise, return $\mfalse$.}
    \If {$\mlink$ has no more $\mtdir$ links}
        \State $\Call{Erase}{\mlink}$
        \State \Return $\mtrue$
    \Else
        \State \Return $\mfalse$
    \EndIf
\EndFunction
\end{algorithmic}
\end{algorithm}

\clearpage
\subsection{TracerInterruptRule}
The function interrupts the trace if $M$ corners have been traced, and if the trace has progression with respect to all links. The default value of $M$ is ten.

\begin{algorithm}[!ht]
\begin{algorithmic}[1]
\caption{\textproc{TracerInterruptRule}: Implements the interrupt rule.}
\label{suppr2p:alg:tracerinterruptrule}
\Function{TracerInterruptRule}{$\mtrace$}
    \State $\mpointtrace \gets $ trace point $\mtrace[\mpoint]$

    \Comment{Interrupt and return $\mtrue$ if $\ge M$ corners are traced and all links have progression.}
    \If {$\mtrace[m] \ge M $ \An $\mlink[is\_prog] $ for all $\mlink$ anchored at $\mpointtrace$}

        \Comment{Re-anchor source link.}
        \State $\mpoint_\mathrm{tm} \gets \Call{GetPoint}{\mpointtrace, \mpointtrace[\mside]}$
        \State $\mlink_S \gets $ source link of trace, in $\mpointtrace[\mlinks_S]$.
        \State $\Call{Change}{\mlink_S, \mpoint_\mathrm{tm}, \mltm, S, \allowbreak
            \mlink_S[\mray_L], \mlink_S[\mray_R], \varnothing, \allowbreak
            \mlink_S[\mlinks_S], \mlink_S[\mlinks_T], \mcalc, \msort
        }$

        \Comment{Re-anchor target link(s).}
        \State $\mlinks_T \gets $ copy of $ \mpointtrace[\mlinks_T]$
        \For {$\mlink_T \in \mlinks_T$}
            \State $\Call{Change}{\mlink_T, \mpoint_\mathrm{tm}, \mltm, T, \allowbreak
            \varnothing, \varnothing, \varnothing, \allowbreak
            \{\mlink_S\}, \mlink_T[\mlinks_T], \mcalc, \msort}$
        \EndFor

        \Comment{Queue a trace query if there have not been any overlapping links.}
        \If {$\mtrace[has\_overlap]$}
            \State Push $\mpoint_\mathrm{tm}$ to overlap-buffer.
        \Else
            \State $c_f \gets \mlink_S[c]$ $+$ cost of cheapest target link in $\mlink_S[\mlinks_T]$.
            \State $\Call{Queue}{\mqtrace, \mlink_S, c_f}$
        \EndIf
        \State \Return $\mtrue$
        
    \Comment{Return $\mfalse$ if trace cannot be interrupted.}
    \Else 
        \State \Return $\mfalse$
    \EndIf
\EndFunction
\end{algorithmic}
\end{algorithm}

\clearpage
\subsection{TracerPlaceRule}
The function implements the placement rule.

\begin{algorithm}[!ht]
\begin{algorithmic}[1]
\caption{\textproc{TracerPlaceRule}: Implements the placement rule.}
\label{suppr2p:alg:tracerplacerule}
\Function{TracerPlaceRule}{$\mtrace$}
    \State $\mpointtrace \gets $ trace point $\mtrace[\mpoint]$.
    \If {$\mpointtrace$ is convex}
        \State \Return $\Call{TracerPlaceRuleConvex}{\mtrace}$
    \Else
        \State $\Call{TracerPlaceRuleNonconvex}{\mtrace}$
        \State \Return $\mfalse$
    \EndIf
\EndFunction
\end{algorithmic}
\end{algorithm}

\subsection{TracerPlaceRuleNonconvex}
The function tries to place a phantom point at a non-convex corner.

\begin{algorithm}[!ht]
\begin{algorithmic}[1]
\caption{\textproc{TracerPlaceRuleNonconvex}: Tries to place a phantom point.}
\label{suppr2p:alg:tracerplacerulenonconvex}
\Function{TracerPlaceRuleNonconvex}{$\mtrace$}
    \State $\mpointtrace \gets $ trace point of trace $\mtrace[\mpoint]$.
    \State $\mside \gets $ trace side $\mpointtrace[\mside]$.
    \State $\mv_\mnext \gets $ directional vector of next trace from $\mpointtrace$.

    \Comment{Find target links for which a phantom point is placeable at the trace point.}  
    \State $\mlinks_\mathrm{un} \gets \{\}$
    \For {$\mlink_T \in \mpointtrace[\mlinks_T]$}
        \State $\mvdifft \gets \mlink_T$'s anchor coordinates - $\mlink_T$'s root coordinates.
        \If {$\mlink_T[is\_prog]$ \An $\mside (\mv_\mnext \times \mvdifft) \ge 0$}
            \State Push $\mlink_T$ into $\mlinks_\mathrm{un}$
        \EndIf
    \EndFor

    \Comment{Place a phantom point for the target links.}
    \If {$\mlinks_\mathrm{un}$ has links}
        \State $\mpoint_\mathrm{un} \gets \Call{GetPoint}{\mpointtrace, \mside}$
        \State $\text{new } \mlink_T \gets$ a new link.
        \For {$\mlink_\mathrm{un} \in \mlinks_\mathrm{un}$}
            \State $\Call{Change}{\mlink_\mathrm{un}, \mpoint_\mathrm{un}, \mlun, T, \allowbreak
                \varnothing, \varnothing, \varnothing, \allowbreak
                \{\text{new } \mlink_T\}, \mlink_\mathrm{un}[\mlinks_T], \mcalc, \msort
            }$
        \EndFor
        \State $\Call{Change}{\text{new } \mlink_T, \mpointtrace, \mltm, T, \allowbreak
            \varnothing, \varnothing, \mv_\mnext, \allowbreak
            \{\}, (\text{new } \mlink_T)[\mlinks_T], \mback
        }$
    \EndIf
\EndFunction
\end{algorithmic}
\end{algorithm}

\clearpage
\subsection{TracerPlaceRuleConvex}
The function tries to place a turning point at a convex corner.
If a turning point is placed, the function attempts to queue a cast query for each target link.

\begin{algorithm}[!ht]
\begin{algorithmic}[1]
\caption{\textproc{TracerPlaceRuleConvex}: Tries to place a turning point and cast.}
\label{suppr2p:alg:tracerplaceruleconvex}
\Function{TracerPlaceRuleConvex}{$\mtrace$}

    \Comment{Return $\mfalse$ if no source progression or cannot place a turning point.}
    \State $\mpointtrace \gets $ trace point of trace $\mtrace[\mpoint]$.
    \State $\mside \gets $ trace side $\mpointtrace[\mside]$.
    \State $\mv_\mnext \gets $ directional vector of next trace from $\mpointtrace$.
    \State $\mlink_S \gets $ source link of trace in $\mpointtrace[\mlinks_S]$.
    \State $\mvdiffs \gets \mlink_S$'s anchor coordinates $-$ $\mlink_S$'s root coordinates. 
    \If {\Not $\mlink_S[is\_prog]$ \Or $\mside (\mvdiff \times \mv_\mnext) > 0$}
        \State \Return $\mfalse$
    \EndIf

    \Comment{Re-anchor source link to turning point and retype the source link.}
    \State $\mpoint_\mathrm{turn} \gets \Call{GetPoint}{\mpointtrace, \mside}$
    \State $\mlink_{SS} \gets $ source link of $\mlink_S$.
    \State $\mltype \gets \mlvu$ \textbf{if } $\mlink_{SS}$ is $\mlvu$ or $\mlvy$ type \textbf{else } $\mleu$.
    \State $\Call{Change}{\mlink_S, \mpoint_\mathrm{turn}, \mltype, S, \allowbreak
        \mlink_S[\mray_L], \mlink_S[\mray_R], \varnothing, \allowbreak
        \mlink_S[\mlinks_S], \mlink_S[\mlinks_T], \mcalc, \msort
    }$

    \Comment{Mark for overlap rule if other links are encountered.}
    \State $\mpoint_o \gets \Call{GetPoint}{\mpointtrace, -\mside}$
    \If {$\mpoint_\mathrm{turn}$ and $\mpoint_o$ anchors links other than $\mlink_S$}
        \State $\mtrace[has\_overlap] \gets \mtrue$
    \EndIf

    \Comment{If a target link of the trace is castable...}
    \State $\mlinks_T \gets $ copy of target links of trace $\mpointtrace[\mlinks_T]$
    \For {$\mlink_T \in \mlinks_T$}
        \State $\mvdifft \gets \mlink_T$'s anchor coordinates - $\mlink_T$'s root coordinates.
        \If {$\mlink_T[is\_prog]$ \An $\mside(\mvnext \times \mvdifft) >= 0$}
            \State $\Call{Change}{\mlink_T, \mpoint_\mathrm{turn}, \mlvu, T, \allowbreak
                \varnothing, \varnothing, \varnothing, \allowbreak
                \{\mlink_S\}, \mlink_T[\mlinks_T], \mcalc, \msort
            }$

            \Comment{...push to overlap-buffer if there are overlaps and source link is $\mleu$ type, or...}
            \If {$\mtrace[has\_overlap]$ \Or $\mltype = \mleu$}
                \State Push $\mpoint_\mathrm{turn}$ into overlap-buffer.
                
            \Comment{...queue a cast otherwise.}
            \Else 
                \State $\Call{Queue}{\mqcast, \mlink_T, \mlink_S[c] + \mlink_T[c]}$
            \EndIf
        \EndIf
    \EndFor

    \Comment{Stop trace if no more target links...}
    \If {$\mpointtrace[\mlinks_T]$ is empty}
        \State \Return $\mtrue$
    
    \Comment{...or continue trace and create new source link otherwise.}
    \Else
        \State $\Call{Change}{\text{new link}, \mpointtrace, \mltm, S, \allowbreak
            \varnothing, \varnothing, \mv_\mnext,
            \{\mlink_S\}, \{\}, \infty, \mback
        }$
        \State \Return $\mfalse$
    \EndIf
\EndFunction
\end{algorithmic}
\end{algorithm}

\clearpage
\section{Functions for Overlap Rule}
This section describes functions that implement the overlap rule.

\subsection{OverlapRule}
Processes branches of overlapping links which have triggered condition O1 of the overlap rule. 
The function shrinks the $S$-tree and moves forward all affected queries, in order to verify line-of-sight and cost-to-come for the affected links.

\begin{algorithm}[!ht]
\begin{algorithmic}[1]
\caption{\textproc{OverlapRule}: Applies the overlap rules for overlapping links.}
\label{suppr2p:alg:overlaprule}
\Function{OverlapRule}{\null}
    \For {$\mpoint \in $ overlap-buffer}
        \State $\Call{OverlapRuleGotoSrcVyEyFromPoint}{\mpoint}$
        \State $\mpoint_o \gets $ other point that has same coordinates as $\mpoint$ but different side.
        \If {$\mpoint_o$ exists}
            \State $\Call{OverlapRuleGotoSrcVyEyFromPoint}{\mpoint_o}$
        \EndIf
        \State Empty the overlap-buffer.
    \EndFor
\EndFunction
\end{algorithmic}
\end{algorithm}

\clearpage
\subsection{OverlapRuleConvToEy}
Converts branches of $\mlvy$ links to expensive $\mley$ links if conditions O2, O3, O4, and O5 of the overlap rule are satisfied, and deletes links if conditions O6 and O7 of the overlap rule are satisfied.

\begin{algorithm}[!ht]
\begin{algorithmic}[1]
\caption{\textproc{OverlapRuleConvToEy}: Converts all affected branches $\mlvy$ links at a point to $\mley$ links.}
\label{suppr2p:alg:overlapruleconvtoey}
\Function{OverlapRuleConvToEy}{$\mtdir, \mpoint_i$}
    \State $\mpoint_o \gets $ point with the same coordinates as $\mpoint_i$ but different side.
    \For {$\mpoint \in \{\mpoint_i, \mpoint_o\}$}
        \If {$\mpoint$ does not exist \Or $\mpoint[\mbest_\mtdir]$ does not exist \Or $\mpoint[\mbest_\mtdir][\mv_\mathrm{best}]$ does not exist}
            \State \Continue
        \EndIf
        \State $\mv_\mathrm{best} \gets \mpoint[\mbest_\mtdir][\mv_\mathrm{best}]$

        \While  {$\mpoint$ anchors a $\mtdir$-tree $\mlvy$ link}
            \State $\mlink \gets $ $\mtdir$-tree $\mlvy$ link anchored at $\mpoint$.
            \State $\mvdiff \gets \mlink$'s anchored point coordinates - $\mlink$'s root point coordinates.
            \State $\mlink_\mtdir \gets $ root link of $\mlink$ \Comment{A $\mlvy$ link has only one root link.}
            \State $\mside \gets \mpoint[\mside]$
            \If {$\mtdir \mside (\mv_\mathrm{best} \times \mvdiff)$} \Comment{Conditions O6 and O7 of overlap rule.}
                \State $\Call{Disconnect}{\mtdir, \mlink, \mlink_\mtdir}$
                \State $\Call{EraseTree}{-\mtdir, \mlink}$
                \State $\Call{EraseTree}{\mtdir, \mlink_\mtdir}$
            \ElsIf {$\Call{OverlapRuleConvToEyForVyLink}{\mtdir, \mlink, \mlink_\mtdir}$}
                \State $\Call{EraseTree}{\mtdir, \mlink_\mtdir}$
            \EndIf
        \EndWhile
    \EndFor

\EndFunction
\end{algorithmic}
\end{algorithm}

\clearpage
\subsection{OverlapRuleConvToEyForVyLink}
An auxiliary recursive function for \textproc{OverlapRuleConvToEy} (Alg. \ref{suppr2p:alg:overlapruleconvtoey}). 
Converts a branch of $\mlvy$ links to expensive $\mley$ links if conditions O2 and O4 of the overlap rule are satisfied.

\begin{algorithm}[!ht]
\begin{algorithmic}[1]
\caption{\textproc{OverlapRuleConvToEyForVyLink}: Converts a branch of $\mlvy$ links to $\mley$ links.}
\label{suppr2p:alg:overlapruleconvtoeyforvylink}
\Function{OverlapRuleConvToEyForVyLink}{$\mtdir, \mlink, \mlink_\mtdir$}
    \State $\mside_\mathrm{root} \gets $ side of root point pf $\mlink$.
    \State $\mside_\mathrm{anchor} \gets $ side of anchored point of $\mlink$.
    \If {$\mside_\mathrm{root} \ne \mside_\mathrm{anchor}$ \An $\mlink$ is $\mlvy$ or $\mley$ type} 
        \State $\Call{Disconnect}{\mtdir, \mlink, \mlink_\mtdir}$
        \Comment{Discard branch if $\mlink$ connects two points with different sides}
        \State $\Call{EraseTree}{-\mtdir, \mlink}$
        \State \Return $\mtrue$
    \EndIf

    \If {$\mlink$ is not $\mlvy$ type}
        \If {$\mlink$ is $S$-tree $\mlvu$ link}
            \Comment {Move the query forward if $S$-tree $\mlvu$ link encountered}
            \State $\Call{OverlapRuleConvToTgtTree}{\mlink}$
            \State $\Call{Queue}{\mqcast, \mlink, \mlink[c] + \mlink_\mtdir[c]}$
        \EndIf
        \State \Return $\mfalse$
    \EndIf

    \While {$\mlink$ has $(-\mtdir) \; \mlvy$ link}
        \Comment {Recursively inspect the branch of leaf links.}
        \State $\mlink_{-\mtdir} \gets (-\mtdir) \; \mlvy$ link of $\mlink$.
        \State $\Call{OverlapRuleConvToEyForVyLink}{\mtdir, \mlink_{-\mtdir}, \mlink}$
    \EndWhile

    \If {$\mlink$ has no more $-\mtdir$ links}
        \Comment{Delete link if leaf branch is deleted.}
        \State $\Call{Disconnect}{\mtdir, \mlink, \mlink_\mtdir}$
        \State $\Call{Erase}{\mlink}$
        \State \Return $\mtrue$
    \Else 
        \Comment{Convert to $\mtdir$ $\mley$ link if leaf branch exists.}
        \State $\Call{Change}{\mlink, \mlink[\mpoint], \mley, \mtdir, \allowbreak
            \mlink[\mray_L], \mlink[\mray_R], \varnothing, \allowbreak
            \mlink[\mlinks_S], \mlink[\mlinks_T], \mlink[c], \msort
        }$
        \State \Return $\mfalse$
    \EndIf
    
\EndFunction
\end{algorithmic}
\end{algorithm}

\clearpage
\subsection{OverlapRuleGotoSrcVyEyFromPoint}
Shrinks the $S$-tree to verify line-of-sight by bringing forward queries in overlapping branches. 
Executed when condition O1 of the overlap rule is satisfied at the points at $\mpoint$'s coordinates.

\begin{algorithm}[!ht]
\begin{algorithmic}[1]
\caption{\textproc{OverlapRuleGotoSrcVyEyFromPoint}: Identifies the most recent ancestor $S$-tree $\mlvy$ or $\mley$ links for all $S$-tree links anchored at the point.}
\label{suppr2p:alg:overlaprulegotosrcvyeyfrompoint}
\Function{OverlapRuleGotoSrcVyEyFromPoint}{$\mpoint$}
    \While {$\mpoint$ has anchored $S$-tree $\mlvu$, $\mleu$, or $\mltm$ links}
        \State $\mlink \gets $ anchored $S$-tree $\mlvu$, $\mleu$, or $\mltm$ link.
        \State $\mlink_S \gets \varnothing$
        \While {$\mtrue$}
            \State $\mlink_S \gets $ source link of $\mlink$. \Comment{An $S$-tree link has only one source link.}
            \If {$\mlink_S$ is $\mlvy$ or $\mley$ type}
                \State \Break
            \EndIf
            \State $\mlink \gets \mlink_S$
        \EndWhile

        \State $\Call{OverlapRuleConvToTgtTree}{\mlink}$

        \State $\Call{Queue}{\mqcast, \mlink,  \mlink[c] + \mlink_S[c]}$
    \EndWhile
\EndFunction
\end{algorithmic}
\end{algorithm}

\subsection{OverlapRuleConvToTgtTree}
Converts a branch of $S$-tree $\mlvu$ and $\mleu$ links to $T$-tree $\mlvu$ links when conditions O1, O2, O3, O6, and O7 of the overlap rule are satisfied.

\begin{algorithm}[!ht]
\begin{algorithmic}[1]
\caption{\textproc{OverlapRuleConvToTgtTree}: Converts a branch $S$-tree links to $T$-tree links.}
\label{suppr2p:alg:overlapruleconvtotgttree}
\Function{OverlapRuleConvToTgtTree}{$\mlink$}
    \State $\Call{Unqueue}{\mlink}$
    \If {$\mlink$ is $T$-tree link}
        \State \Return
    \EndIf
    \For {each target link $\mlink_T$ of $\mlink$}
        \State $\Call{OverlapRuleConvToTgtTree}{\mlink_T}$
    \EndFor
    \State $\mpoint_S \gets $ source point of $\mlink$.
    \Comment{Convert $\mlink$ to $T$-tree $\mlvu$ link.}
    \State $\Call{Change}{\mlink, \mpoint_S, \mlvu, T, \allowbreak
        \mlink[\mray_L], \mlink[\mray_R], \varnothing, \allowbreak
        \mlink[\mlinks_S], \mlink[\mlinks_T], \mcalc, \msort
        }$
\EndFunction
\end{algorithmic}
\end{algorithm}

    \clearpage
    \chapter*{Publications}
    \section*{Conference}
    
    \begin{enumerate}[label={[\arabic*]}]
        \item  \fullcite{bib:me}
    \end{enumerate}
    
    \section*{Journal}
    \begin{enumerate}[label={[\arabic*]}, resume]
        \item  \fullcite{bib:r2}
        \item \fullcite{bib:r2p}
    \end{enumerate}
\end{document}